\documentclass{article}
\usepackage{iclr2019_conference,times}
 \iclrfinalcopy
\usepackage[colorlinks,
linkcolor=blue,
anchorcolor=blue,
citecolor=blue
]{hyperref}
\usepackage{smile}
\usepackage{float}
\usepackage{threeparttable}
\usepackage[inline]{enumitem}
\usepackage{bbm}
\usepackage{tikz}
\usetikzlibrary{positioning} 
\tikzstyle{line}=[draw] 
\tikzstyle{arrow}=[draw, -latex]
\usetikzlibrary{shapes,shadows,arrows,calc}
\tikzset{
    >=stealth',
    punkt/.style={
    fill=white,
           rectangle,
           rounded corners,
           text width=5em,
           minimum height=2.5em,
           text centered,
           node distance=4em},
    pil/.style={
           ->,
           thick,
           shorten <=0pt,
           shorten >=0pt,}
}
\tikzstyle{epibox} = [
  draw=epigraphbordercolour,
  shade,
  top color=epigraphfillcolour!40,
  bottom color=epigraphfillcolour!5,
  drop shadow=dropshadowcolour,
  very thick,
  rectangle,
  rounded corners,
  inner sep=10pt,
  inner ysep=15pt
]
\newcolumntype{P}[1]{>{\centering\arraybackslash}p{#1}}
\newcolumntype{L}[1]{>{\centering\arraybackslash}m{#1}}

\newcommand{\Fr}{\text{F}}

\newcommand{\removed}[1]{}
\newcommand{\sgn}{{\rm sign}}
\newcommand{\HT}{{\c{T}}}
\newcommand{\gradvec}{\b{g}_i}
\newcommand{\gradmat}{\b{g}}
\newcommand{\supp}{{\rm supp}}

\newcommand{\siri}[1]{{\color{black}{#1}}}
\usepackage[font=footnotesize,labelfont=bf]{caption}

\usepackage[lined,commentsnumbered,ruled]{algorithm2e}
\SetKwInput{KwInput}{Input}
\SetKwInput{KwInit}{Initialize}
\SetKwInput{KwOutput}{Output}
\SetKwInput{KwReturn}{Return}
\SetKwInput{KwSet}{Set}

\def\[#1\]{{$#1$}}
\def\b#1{{\mathbf{#1}}}
\def\c#1{{\mathcal{#1}}}

\def\##1\#{\begin{align}#1\end{align}}
\def\$#1\${\begin{align*}#1\end{align*}}

\newcommand{\remove}[1]{}
\usepackage{cancel,soul}

\usepackage{setspace}
\begin{document}

\title{NOODL: Provable Online Dictionary Learning \\ and Sparse Coding \vspace{-12pt}}

\author{\normalsize Sirisha Rambhatla\[^\dagger\], Xingguo Li\[^\ddagger\], and Jarvis Haupt\[^\dagger\]\\ \[^\dagger\]Dept. of Electrical and Computer Engineering, University of Minnesota -- Twin Cities, USA\\ \[^\ddagger\]Dept. of Computer Science, Princeton University, Princeton, NJ, USA \\{\small ~~Email: \texttt{rambh002@umn.edu},  \texttt{xingguol@cs.princeton.edu}, and \texttt{jdhaupt@umn.edu}}\vspace{-14pt}}

\date{}

\maketitle

\begin{abstract}
	\vspace{-5pt}
We consider the dictionary learning problem, where the aim is to model the given data as a linear combination of a few columns of a matrix known as a \textit{dictionary}, where the sparse weights forming the linear combination are known as \textit{coefficients}. Since the dictionary and coefficients, parameterizing the linear model are unknown, the corresponding optimization is inherently non-convex. This was a major challenge until recently, when provable algorithms for dictionary learning were proposed. Yet, these provide guarantees only on the recovery of the dictionary, without explicit recovery guarantees on the coefficients. Moreover, any estimation error in the dictionary adversely impacts the ability to successfully localize and estimate the coefficients. This potentially limits the utility of existing provable dictionary learning methods in applications where coefficient recovery is of interest. To this end, we develop NOODL: a simple Neurally plausible alternating Optimization-based Online Dictionary Learning algorithm, which recovers \textit{both} the dictionary and coefficients \textit{exactly} at a geometric rate, when initialized appropriately. Our algorithm, NOODL, is also scalable and amenable for large scale distributed implementations in neural architectures, by which we mean that it only involves simple linear and non-linear operations. Finally, we corroborate these theoretical results via experimental evaluation of the proposed algorithm with the current state-of-the-art techniques.

\end{abstract}

\vspace{-12pt}
\section{Introduction}
\vspace{-4pt}

Sparse models avoid overfitting by favoring simple yet highly expressive representations. Since signals of interest may not be inherently sparse, expressing them as a sparse linear combination of a few columns of a dictionary is used to exploit the sparsity properties. Of specific interest are overcomplete dictionaries, since they provide a flexible way of capturing the richness of a dataset, while yielding sparse representations that are robust to noise; see \cite{Mallat1993,Chen1998,Donoho2006}. 
In practice however, these dictionaries may not be known, warranting a need to {learn} such representations -- known as \textit{dictionary learning} (DL) or \textit{sparse coding} \citep{Olshausen97}. 
Formally, this entails learning an \textit{a priori} unknown dictionary \[\b{A} \in \RR^{n \times m}\] and sparse coefficients \[\b{x}_{(j)}^* \in \RR^{m}\] from data samples \[\b{y}_{(j)} \in \RR^{n}\] generated as
\vspace{-2pt}
\begin{align}\label{eq:model}
\b{y}_{(j)} = \b{A}^*\b{x}_{(j)}^*,~\|\b{x}_{(j)}^*\|_0 \leq k~~\text{for all}~~j = 1, 2, \dots \vspace{-2pt}%+ \b{w}_{(i)},
\end{align}
%where \[\b{w}_{(i)} \in \RR^{n}\] denotes the i.i.d Gaussian noise vector.
This particular model can also be viewed as an extension of the low-rank model \citep{Pearson1901}. Here, instead of sharing a low-dimensional structure, each data vector can now reside in a separate low-dimensional subspace. Therefore, together the data matrix admits a \textit{union-of-subspace} model. As a result of this additional flexibility, DL finds applications in a wide range of signal processing and machine learning tasks, such as denoising \citep{Elad2006}, image inpainting \citep{Mairal09}, clustering and classification \citep{Ramirez2010, Rambhatla2013sbmca,Rambhatla2016DRPCA, Rambhatla17DRPCAApp, Rambhatla18LrApp,Rambhatla18LrTheo}, and analysis of deep learning primitives \citep{Ranzato2007,Gregor2010}; see also \cite{Elad2010}, and references therein.

Notwithstanding the non-convexity of the associated optimization problems (since both factors are unknown),  alternating minimization-based dictionary learning techniques have enjoyed significant success in practice. Popular heuristics include regularized least squares-based \citep{Olshausen97,Lee2007,Mairal09, Lewicki2000, Kreutz2003}, and greedy approaches such as the method of optimal directions (MOD) \citep{Engan1999} and k-SVD \citep{ Aharon2006}. However, dictionary learning, and matrix factorization models in general, are difficult to analyze in theory; see also \cite{li2016symmetry}.

%Since both the dictionary and the coefficients are unknown \textit{a priori}, the associated learning problems are inherently non-convex. Nevertheless,

To this end, motivated from a string of recent theoretical works \citep{Remi2010, Jenatton2012, Geng2014}, provable algorithms for DL have been proposed recently to explain the success of aforementioned alternating minimization-based algorithms \citep{Agarwal14, Arora14, Arora15}. However, these works exclusively focus on guarantees for dictionary recovery. %, which either require very good initializations \citep{Arora14,Agarwal14} or incur bias in dictionary estimation \citep{Arora15}. 
On the other hand, for applications of DL in tasks such as classification and clustering -- which rely on coefficient recovery -- it is crucial to have guarantees on coefficients recovery as well. 

Contrary to conventional prescription, a sparse approximation step after recovery of the dictionary does not help; since any error in the dictionary -- which leads to an error-in-variables (EIV) \citep{Fuller2009} model for the dictionary -- degrades our ability to even recover the support of the coefficients \citep{Wainwright2009}. Further, when this error is non-negligible, the existing results guarantee recovery of the sparse coefficients only in $\ell_2$-norm sense \citep{Donoho2006}. As a result, there is a need for scalable dictionary learning techniques with guaranteed recovery of both factors.

\vspace{-3pt}
\subsection{Summary of Our Contributions}
\vspace{-3pt}

%Our analysis, to the best of our knowledge, is a first result on exact recovery guarantees for coefficients using the IHT, which may be of independent interest when the dictionary is known. As an added benefit, the guarantees on coefficient recovery allows us to get rid of the bias incurred by the prior art, as substantiated by our analysis and empirical evaluations.

In this work,  we present a simple online DL algorithm motivated from the following regularized least squares-based problem, where \[S(\cdot)\] is a nonlinear function that promotes sparsity.
%In particular, we examine the problem of recovering an unknown dictionary \[\b{A}^*\] and the corresponding sparse coefficients \[\{\b{x}_{(i)}^*\}_{i=1}^p\] by drawing motivations from the properties of the following non-convex regularized least squares objective,
\vspace*{-2pt}
\begin{align}\tag{P1}\label{eq:objective}
\textstyle \min\limits_{\b{A}, \cbr{\b{x}_{(j)}}_{j=1}^p } \sum\limits_{j = 1}^{p}\|\b{y}_{(j)} - \b{A}\b{x}_{(j)}\|^2_2 + \sum\limits_{j=1}^{p}S(\b{x}_{(j)}).\vspace*{-2pt}
\end{align}
 Although our algorithm does not optimize this objective, it leverages the fact that the problem \eqref{eq:objective} is convex w.r.t \[\b{A}\], given the sparse coefficients \[\{\b{x}_{(j)}\}\]. Following this, we recover the dictionary by choosing an appropriate gradient descent-based strategy \citep{Arora15, Engan1999}. %Of course, since we don't know the coefficients exactly, this results in an approximate gradient descent strategy for dictionary update.
To recover the coefficients, we develop an iterative hard thresholding (IHT)-based update step \citep{Haupt2006, Blumensath09}, and show that -- given an appropriate initial estimate of the dictionary and a mini-batch of $p$ data samples at each iteration \[t\] of the online algorithm -- alternating between this IHT-based update for coefficients, and a gradient descent-based step for the dictionary leads to geometric convergence to the true factors, i.e., \[\hspace{-1pt}\b{x}_{(j)}\hspace{-2pt}{\tiny\rightarrow}\b{x}_{(j)}^*\hspace{-2pt}\] and \[\b{A}^{(t)}_i{\hspace{-2pt}\tiny\rightarrow}\b{A}^*_i\hspace{-2pt}\] as \[\hspace{-1pt}t{\tiny\rightarrow}\hspace{-0pt}\infty\].

%\xl{(what are we actually solving? mention that we do not need the exact form of $S$. Also, we consider the online setting, and $p$ correspond to one mini-batch of the data.)} 

% Leo To estimate both the dictionary and the coefficients with convergence guarantees to the true model parameters and the computational efficiency, we propose the Neurally plausible alternating Optimization-based Online Dictionary Learning algorithm (NOODL). The algorithm is based on alternating gradient descent type of updates for both factors, combined with an iterative hard thresholding (IHT) type \citep{Haupt2006, Blumensath09} update for sparse coefficient recovery. Our major contributions are summarized as follows:
%Specifically, we leverage the fact that this objective is strongly convex w.r.t \[\b{A}\], given the sparse coefficients. As a result, if the sparse coefficients are known, the dictionary can be recovered using a gradient descent strategy \citep{Arora15, Engan1999}. Now the question is -- how can we find the coefficients? To answer this, we develop an iterative hard thresholding (IHT)-based \citep{Haupt2006, Blumensath09} update step for coefficient recovery. And show that given an appropriate initial estimate of the dictionary, alternating between this IHT-based update for coefficients and a gradient descent-based update for dictionary, recovers the true factors exactly. 

In addition to achieving exact recovery of both factors, our algorithm -- Neurally plausible alternating Optimization-based Online Dictionary Learning (NOODL) -- has linear convergence properties. Furthermore, it is scalable, and involves simple operations, making it an attractive choice for practical DL applications. Our major contributions are summarized as follows:
\vspace{-3pt}
\begin{itemize}[leftmargin=*]
\setlength{\itemsep}{-0.1pt}
\item \textbf{Provable coefficient recovery:} To the best of our knowledge, this is the first result on \textit{exact} recovery of the sparse coefficients \[\{\b{x}^*_{(j)}\}\], including their support recovery, for the  DL problem. The proposed IHT-based strategy to update coefficient under the EIV model, is of independent interest for recovery of the sparse coefficients via IHT, which is challenging even when the dictionary is known; see also \cite{Yuan2016} and \cite{li2016stochastic}.
%\item \textbf{Provable coefficient recovery:} To the best of our knowledge, this is the first result on \textit{exact} recovery of the sparse coefficients \[\{\b{x}^*_{(j)}\}\] for the DL problem, including support recovery, using an IHT-based strategy to update coefficient under the EIV model. This is challenging even when the dictionary is known; see also \cite{Yuan2016}. % for a related work on Hard Threshold Pursuit (HTP).%, and show that it depends on maximum per-column error in the dictionary.  
%As a result, our estimation error in the coefficients also drops geometrically. In addition,
\item \textbf{Unbiased estimation of factors and linear convergence:} The recovery guarantees on the coefficients also helps us to get rid of the bias incurred by the prior-art in dictionary estimation. Furthermore, our technique geometrically converges to the true factors.
%The exact recovery result for both the factors implies that NOODL facilitates unbiased estimation of both the dictionary $\Ab^*$ and coefficients $\cbr{\b{x}^*_{(i)}}$. 
%\item \textbf{Online Linear convergence and Online nature:} 
%In addition, through our convergence analysis and experimental evaluations, we establish that when \[\b{A}\] is initialized appropriately, NOODL converges geometrically to the true dictionary \[\b{A}^*\] and coefficients \[\cbr{\b{x}^*_{(i)}}\]. 

%We design and analyze NOODL: a simple neurally plausible alternating optimization-based \textit{online} DL algorithm for recovering the dictionary and the coefficients. 
%We establish that when $\b{A}$ is initialized appropriately %\[\mathcal{O}^*(1/\log(n))\]-close, i.e., \[\nbr{\Ab_{i} - \Ab_{i}^*} \leq \mathcal{O}^*(1/\log(n))\], 
%NOODL converges geometrically to the true dictionary \[\b{A}^*\] and coefficients \[\cbr{\b{x}^*_{(i)}}\]. %; see Definition~\ref{def:del_kappa}.
\item \textbf{Online nature and neural implementation:}  The online nature of algorithm, makes it suitable for machine learning applications with streaming data. 
%The fresh data samples at each time step are processed in a ``Predict'' (coefficient estimation) and ``Learn'' (dictionary update) fashion, making NOODL suitable for on-the-fly machine learning applications. %In addition, our linear convergence result ensures that the NOODL keeps improving as it sees more data, while becoming better at predicting the coefficient. %This makes it an attractive choice for a number of machine learning and signal processing tasks.
%\item \textbf{Neural Implementation:}  
In addition, the separability of the coefficient update allows for distributed implementations in neural architectures (only involves simple linear and non-linear operations) to solve large-scale problems. To showcase this, we also present a prototype neural implementation of NOODL.
%\item \textbf{Experimentally validate the theoretical results:}  
\end{itemize}
\vspace{-2pt}
In addition, we also verify these theoretical properties of NOODL through experimental evaluations on synthetic data, and compare its performance with state-of-the-art provable DL techniques.

\vspace{-3pt}
\subsection{Related Works}
\vspace{-3pt}
%The model studied in this paper is related to a number of inference problems. For instance, when \[{\b{x}_{(i)}^{*} }\] are known, we can recover \[\b{A}^*\] via linear least squares, given sufficient number of samples, \[p\].
%For the case when \[\b{A}\] is known and \[\cbr{\b{x}_{(i)}^{*} }\] are unknown, the problem transforms into the ones studied in the area of sparse recovery and compressive sensing. 

%While the explorations of representing a signal of interest in a known dictionary -- both orthogonal (the dictionary elements are orthogonal to each other) \citep{Mallat1993, Chen1998} and non-orthogonal \citep{Donoho2003, Fuchs2004, Tropp2004}-- were underway, the idea of learning these representations from data, i.e. dictionary learning was born \citep{Olshausen97}. 
%Despite the non-convexity of the learning problem, algorithms based on alternating update of the estimates of the dictionary and the coefficients, such as MOD \citep{Engan1999}, k-SVD \citep{Aharon2006}, and others \citep{Lewicki2000, Kreutz2003, Lee2007, Mairal09}, have been widely successful in practice. However, the DL problem still remained difficult to analyze. 

With the success of the alternating minimization-based techniques in practice, a push to study the DL problem began when \cite{Remi2010} showed that for \[m = n\], the solution pair \[\rbr{\b{A}^*, \b{X}^*}\] lies at a local minima of the following non-convex optimization program, where \[\b{X}=[\b{x}_{(1)}, \b{x}_{(2)}, \dots,\b{x}_{(p)} ]\] and \[\b{Y}=[\b{y}_{(1)}, \b{y}_{(2)}, \dots,\b{y}_{(p)} ]\], with high probability over the randomness of the coefficients,
\begin{align}\label{eq:alt_min_obj}
\underset{\b{A}, \b{X}}{\rm min}~\|\b{X}\|_1 ~~~\text{s.t.}~ \b{Y} = \b{AX}, ~~~\|\b{A}_i\| = 1, \forall ~i \in [m].
\end{align}
Following this, \cite{Geng2014} and \cite{Jenatton2012} extended these results to the overcomplete case (\[n<m\]), and the noisy case, respectively. Concurrently, \cite{Jung2014, Jung2016} studied the nature of the DL problem for \[S(\cdot) = \|\cdot\|_1\] (in \eqref{eq:objective}), and derived a lower-bound on the minimax risk of the DL problem. However, these works do not provide any algorithms for DL.
%Although these works did not prescribe any algorithm for DL, these inspired an entire line of work which explored ``good'' initialization followed by alternating optimization to recover the true factors. %\[\b{A}\] and \[\b{X}\]. 

\begin{table}[!t]
\vspace{-5pt}
\caption{Comparison of provable algorithms for dictionary learning. }\label{tab:compare}
\vspace{-5pt}
\renewcommand{\arraystretch}{1.2}
\resizebox{\columnwidth}{!}{
\begin{threeparttable}[h]
\begin{tabular}{c|c|c|c|c|c}
\Xhline{1pt}
\multirow{3}{*}{\textbf{Method}}&\multicolumn{3}{c|}{\textbf{Conditions}}& \multicolumn{2}{c}{\textbf{Recovery Guarantees} }\\ \cline{2-6}
& {\small\textbf{Initial Gap of} }&  {\small\textbf{Maximum} }& {\small\textbf{Sample} }& \multirow{2}{*}{\small\textbf{Dictionary}} & \multirow{2}{*}{\small\textbf{Coefficients}} \vspace{-1pt}\\ 
& {\small \bf Dictionary} & {\small\bf Sparsity} & {\small\bf Complexity} & & \\ \hline
NOODL (this work) & \multirow{3}{*}{\[\mathcal{O}^*\rbr{\tfrac{1}{\log(n)}} \]}  & %\multirow{3}{*}{\[\mathcal{O}^*\rbr{\tfrac{\sqrt{n}}{\mu \log(n)}} \]}& \[\tilde{\Omega}\rbr{mk^2}\] & No bias & \[\mathcal{O}\rbr{\sqrt{k(1 - \omega)^t\epsilon_0}} \]\tnote{$\ast$} \\ \cline{1-1}\cline{4-6}
\multirow{3}{*}{\[\mathcal{O}^*\rbr{\tfrac{\sqrt{n}}{\mu \log(n)}} \]}& \[\tilde{\Omega}\rbr{mk^2}\] & No bias & No bias\\ \cline{1-1}\cline{4-6}
\texttt{Arora15(``biased'')}\tnote{$\dagger$}&   &   & \[\tilde{\Omega}\rbr{mk}\] &  \[\mathcal{O}(\sqrt{{k}/{n}})\] & N/A\\ \cline{1-1}\cline{4-6}
\texttt{Arora15(``unbiased'')}\tnote{$\dagger$}&  &   & \[{\rm poly}(m)\] &  \textit{Negligible} bias \[^\S\]& N/A \\ \hline
\cite{Barak2015}\[^\P\] &N/A & \[\c{O}(m^{(1-\delta)})\] for \[\delta>0\]  &  \[n^{O{(d)}}/{\rm poly}{(k/m)}\]& \[\epsilon\]&N/A  \\\hline
\cite{Agarwal14}{$^\ddagger$} & \[\mathcal{O}^*\rbr{{1}/{\rm poly(m)}}\] &  \[\mathcal{O}\rbr{{\sqrt[6]{n}}/{\mu}} \] & \[\Omega(m^2)\]&  No bias & N/A \\\hline
\cite{Spielman2012} \[(\text{for}~n\leq m)\]& N/A&  \[\c{O}(\sqrt{n})\] &\[\tilde{\Omega}(n^2)\] & No bias&N/A   \\
\Xhline{1pt}
\end{tabular}
\begin{tablenotes}
%\item[$\ast$] \xl{There should be no bias for coefficient, o.w., we need to write $\epsilon_0$ for dict.} At the iterate \[t\] of the online procedure and for \[0<\omega<1/2\]. For an initial estimate \[\b{A}^{(0)}\]. \[\|\b{A}_i^{(0)} -\b{A}_i^*\| \leq \epsilon_0 \].
\item[] Dictionary recovery reported in terms of column-wise error. $\dagger$ See Section~\ref{sec:exp} for description. $\ddagger$ This procedure is not \textit{online}. \[^\S\] The bias is not explicitly quantified. The authors claim it will be \textit{negligible}. \[^\P\] Here, \[d = \Omega(\tfrac{1}{\epsilon} \log(m/n))\] for column-wise error of \[\epsilon\].% and employs a quadratic convex program for the coefficients.  
\end{tablenotes}
\end{threeparttable}}%
\vspace{-5pt}
\end{table}
Motivated from these theoretical advances, \cite{Spielman2012} proposed an algorithm for the under-complete case \[n\geq m\] that works up-to a sparsity of \[k = O(\sqrt{n})\]. Later, \cite{Agarwal14} and \cite{Arora14} proposed clustering-based provable algorithms for the overcomplete setting, motivated from MOD  \citep{Engan1999} and k-SVD \citep{Aharon2006}, respectively. % based on restricted isometry property (RIP) and incoherence assumptions on the dictionary, respectively.
Here, in addition to requiring stringent conditions on dictionary initialization, \cite{Agarwal14} alternates between solving a quadratic program for coefficients and an MOD-like \citep{Engan1999} update for the dictionary, which is too expensive in practice. Recently, a DL algorithm that works for almost linear sparsity was proposed by \cite{Barak2015}; however, as shown in Table~\ref{tab:compare}, this algorithm may result in exponential running time. Finally, \cite{Arora15} proposed a provable online DL algorithm, which provided improvements on initialization, sparsity, and sample complexity, and is closely related to our work. A follow-up work by \cite{Chatterji2017} extends this to random initializations while recovering the dictionary exactly, however the effect described therein kicks-in only in very high dimensions. We summarize the relevant provable DL techniques in Table~\ref{tab:compare}.

The algorithms discussed above implicitly assume that the coefficients can be recovered, after dictionary recovery, via some sparse approximation technique.  However, as alluded to earlier, the guarantees for coefficient recovery -- when the dictionary is known approximately -- may be limited to some \[\ell_2\] norm bounds \citep{Donoho2006}. This means that, the resulting coefficient estimates may not even be sparse.  Therefore, for practical applications, there is a need for efficient online algorithms with guarantees, which serves as the primary motivation for our work.

{\let\thefootnote\relax\footnote{\noindent\textbf{Notation.} Given an integer $n$, we denote $[n] = \cbr{1, 2,\ldots,n}$. The bold upper-case and lower-case letters are used to denote matrices \[\b{M}\] and vectors \[\b{v}\], respectively. \[\b{M}_i\], \[\b{M}_{(i,:)}\], \[\b{M}_{ij}\], and \[\b{v}_i\] (and \[\b{v} (i)\]) denote the \[i\]-th column, \[i\]-th row, \[(i,j)\] element of a matrix, and \[i\]-th element of a vector, respectively.  The superscript \[(\cdot)^{(n)}\] denotes the \[n\]-th iterate, while the subscript \[(\cdot)_{(n)}\] is reserved for the \[n\]-th data sample. Given a matrix $\Mb$, we use \[\|\b{M}\|\] and \[\|\b{M}\|_{\Fr}\] as the spectral norm and Frobenius norm. Given a vector $\vb$, we use \[\|\b{v}\|\], \[\|\b{v}\|_0\], and \[\|\b{v}\|_1\] to denote the \[\ell_2\] norm, \[\ell_0\] (number of non-zero entries), and \[\ell_1\] norm, respectively. We also use standard notations \[\cO(\cdot), \Omega(\cdot)\] (\[\tilde{\cO}(\cdot), \tilde{\Omega}(\cdot)\]) to indicate the asymptotic behavior (ignoring logarithmic factors). Further, we use \[g(n) = \mathcal{O}^*(f(n))\] to indicate that \[g(n) \leq L f(n)\] for a small enough constant \[L\], which is independent of \[n\].  We use \[c(\cdot)\] for constants parameterized by the quantities in \[(\cdot)\]. 
\[\HT_{\tau}(z) := z\cdot\mathbbm{1}_{|z|\geq \tau}\] denotes the hard-thresholding operator, where ``\[\mathbbm{1}\]'' is the indicator function. We use \[\supp(\cdot)\] for the support (the set of non-zero elements) and \[\sgn(\cdot)\] for the element-wise sign.}}
\vspace*{-20pt}
\section{Algorithm}\label{sec:algorithm}
\vspace{-5pt} 
We now detail the specifics of our algorithm -- NOODL, outlined in Algorithm~\ref{alg:main_alg}. NOODL recovers both the dictionary and the coefficients exactly given an appropriate initial estimate \[\b{A}^{(0)}\] of the dictionary. Specifically, it requires \[\b{A}^{(0)}\] to be \[(\epsilon_0, 2)\]-close to \[\b{A}^*\] for \[\epsilon_0 = \mathcal{O}^*(1/\log(n))\], where \[(\epsilon, \kappa)\]-closeness is defined as follows.  This implies that, the initial dictionary estimate needs to be column-wise, and in spectral norm sense, close to \[\b{A}^*\], which can be achieved via certain initialization algorithms, such as those presented in \cite{Arora15}. 
\vspace{-1pt} 
\begin{definition}[\[(\epsilon, \kappa)\]-closeness]\label{def:del_kappa}
\textit{A dictionary \[\b{A}\] is \[(\epsilon, \kappa)\]-close to \[\b{A}^*\] if \[\|\b{A} - \b{A}^*\| \leq \kappa\|\b{A}^*\|\], and if there is a permutation \[\pi: [m] \rightarrow [m]\] and a collection of signs \[\sigma: [m] \rightarrow \{\pm 1\}\] such that \[\|\sigma(i)\b{A}_{\pi(i)} - \b{A}^*_i\|\leq \epsilon, ~\forall~ i\in [m]\].}\vspace{-3pt}
\end{definition}
\setlength{\textfloatsep}{11pt}
\begin{algorithm}[h]
\setstretch{1}
\caption{\small NOODL: Neurally plausible alternating Optimization-based Online Dictionary Learning.}\label{alg:main_alg}
\SetAlgoLined
\KwInput{Fresh data samples \[\b{y}_{(j)} \in \mathbb{R}^{n} \] for \[j \in [p]\] at each iteration \[t\] generated as per \eqref {eq:model}, where \[|\b{x}^*_i|\geq C\] for \[i \in \supp(\b{x}^*)\]. Parameters \[\eta_A\], \[\eta_x^{(r)}\] and \[\tau^{(r)}\] chosen as per \ref{assumption:step dict} and \ref{assumption:step coeff}. No. of iterations \[T = \Omega(\log(1/\epsilon_{T}))\] and \[R = \Omega(\log(1/\delta_{R}))\], for target tolerances \[\epsilon_{T}\] and \[\delta_{R}\].}
\KwOutput{The dictionary \[\b{A}^{(t)}\] and coefficient estimates \[\hat{\b{x}}_{(j)}^{(t)} \] for \[j \in [p]\] at each iterate \[t\]. }
\KwInit{Estimate \[\b{A}^{(0)}\], which is \[(\epsilon_0, 2)\]-near to \[\b{A}^*\] for \[\epsilon_0 = \mathcal{O}^*(1/\log(n))\]}
\For{\[t=0\] \textbf{to} \[T-1\]}{
\textbf{Predict: (Estimate Coefficients)}\\
\For{\[j = 1\] \textbf{to} \[p\]}{\vspace*{-7pt}
\begin{flalign}\label{alg:coeff_init}
&\hspace{-0.06in}\textbf{Initialize:}~~\b{x}^{(0)}_{(j)} = \HT_{C/2}(\b{A}^{(t)^\top}\b{y}_{(j)})& 
\end{flalign}\\[-6pt]
	\For{\[r = 0\]  \textbf{to} \[R-1\]}{\vspace{-5pt}
	\begin{flalign}\label{alg:coeff_iht}
	&\textbf{Update:}~~\b{x}^{(r+1)}_{(j)} = \HT_{\tau^{(r)}}(\b{x}^{(r)}_{(j)} - \eta_{x}^{(r)}~\b{A}^{(t)^\top}(\b{A}^{(t)}\b{x}^{(r)}_{(j)} - \b{y}_{(j)}))&
	\end{flalign}\vspace{-14pt}}}
\vspace{-1pt}
%	Update: \[ \b{x}^{(r+1)}_{(j)} = \HT_{\tau^{(r)}}(\b{x}^{(r)}_{(j)} - \eta_{x}^{(r)}~\b{A}^{(t)^\top}(\b{A}^{(t)}\b{x}^{(r)}_{(j)} - \b{y}_{(j)}))\]
\[\hat{\b{x}}_{(j)}^{(t)}  :=\b{x}^{(R)}_{(j)} \] for \[j \in [p]\]\\
\textbf{Learn: (Update Dictionary)}\vspace{-8pt}
\begin{flalign}\label{eq:emp_grad_est}
&\hspace{0.14in}\text{Form empirical gradient estimate:}~~\hat{\b{g}}^{(t)} = \tfrac{1}{p}\textstyle\sum_{j = 1}^{p}(\b{A}^{(t)}\hat{\b{x}}_{(j)}^{(t)}   - \b{y}_{(j)})\sgn(\hat{\b{x}}_{(j)}^{(t)}  )^\top&
\end{flalign}\vspace{-20pt}
%Form empirical gradient estimate : \[\hat{\b{g}}^{(t)} = \tfrac{1}{p}\sum_{j = 1}^{p}(\b{A}^{(t)}\hat{\b{x}}_{(j)}  - \b{y}_{(j)})\sgn(\hat{\b{x}}_{(j)} )^\top\] \\ %gr = (Y_m - M.*(A*XS))*(XS)';
\begin{flalign}\label{eq:apx_grad}
&\hspace{0.14in}\text{Take a gradient descent step:}~~\b{A}^{(t+1)} = \b{A}^{(t)} - \eta_A~\hat{\b{g}}^{(t)}&
\end{flalign}\vspace{-20pt}
%Take a gradient descent step: \[\b{A}^{(t+1)} = \b{A}^{(t)} - \eta_A~\hat{\b{g}}^{(t)}\]\\
\begin{flalign*}
&\hspace{0.14in}\text{Normalize:}~~\b{A}^{(t+1)}_i = \b{A}^{(t+1)}_i/\|\b{A}^{(t+1)}_i\| ~\forall~i \in [m]&
\end{flalign*}\vspace{-18pt}
%Normalize: \[\b{A}^{(t+1)}_i = \b{A}^{(t+1)}_i/\|\b{A}^{(t+1)}_i\|\] for \[i \in [m]\]. % took out "optional" this step is not optional
}
\end{algorithm}

Due to the streaming nature of the incoming data, NOODL takes a mini-batch of \[p\] data samples at the \[t\]-th iteration of the algorithm, as shown in Algorithm~\ref{alg:main_alg}. It then proceeds by alternating between two update stages: coefficient estimation (``Predict'') and dictionary update (``Learn'') as follows. 

\noindent\textbf{Predict Stage}: 
For a general data sample \[\b{y} = \b{A}^*\b{x}^*\], the algorithm begins by forming an initial coefficient estimate \[\b{x}^{(0)}\] based on a hard thresholding (HT) step as shown in \eqref{alg:coeff_init}, where \[\HT_{\tau}(z) := z\cdot\mathbbm{1}_{|z|\geq \tau}\] for a vector \[\b{z}\].
%, for each data sample generated as \[\b{y} = \b{A}^*\b{x}^*\], where \[|\b{x}^*_i|\geq C\] for \[i \in \supp(\b{x}^*)\]. %the algorithm begins by forming an initial estimate of the coefficient vector \[\b{x}^{(0)}\] of \[\b{x}^*\] as shown in \eqref{alg:coeff_init}.
%\begin{align}\label{alg:coeff_init}
%\b{x}^{(0)} = \HT_{C/2}(\b{A}^{(t)^\top}\b{y}),
%\end{align}%
% and \[\HT_{\tau}(z) := z\cdot\mathbbm{1}_{|z|\geq \tau}\] denotes the hard-thresholding operator. 
Given this initial estimate $\b{x}^{(0)}$, the algorithm iterates over  \[R= \Omega(\log(1/\delta_{R}))\] IHT-based steps \eqref{alg:coeff_iht} to achieve a target tolerance of \[\delta_{R}\], such that \[(1 - \eta_{x})^{R} \leq \delta_{R}\]. %In our analysis, the effect of the error incurred by the initial coefficient estimate \[\b{x}^{(0)}\] contracts by a factor of \[(1 - \eta_{x}^{(r)})\]  at each iteration \[r \in [R]\], which results in this choice of \[R\].
%\begin{align}\label{alg:coeff_iht}
%\b{x}^{(r+1)} = \HT_{\tau^{(r)}} \rbr{\b{x}^{(r)} - \eta_{x}^{(r)}~\b{A}^{(t)^\top}(\b{A}^{(t)}\b{x}^{(r)} - \b{y})}.
%\end{align}
Here, \[\eta_{x}^{(r)}\] is the learning rate, and \[\tau^{(r)}\] is the threshold at the \[r\]-th iterate of the IHT. In practice, these can be fixed to some constants for all iterations; see \ref{assumption:step coeff} for details. %The dependence on iterations \[r\], just allows for additional flexibility in practice. %Note that all \[\eta_{x}^{(r)}\] and \[\tau^{(r)}\] can be fixed to some constant \[\eta_{x}\] and \[\tau\], respectively. %Further, the number of iterations \[R\] are determined by a tolerance \[\delta_{R}\] such that \[(1 - \eta_{x})^{R} \leq \delta_{R}\]. Therefore, the inner loop requires \[{R} = \Omega(\log(1/\delta_{R}))\] iterations. 
%We show that the resulting coefficient estimate still has the correct signed-support, under some conditions on \[\eta_{x}^{(r)}\] and \[\tau^{(r)}\]. 
Finally at the end of this stage, we have estimate \[\hat{\b{x}}^{(t)} :=\b{x}^{(R)}\] of \[\b{x}^*\].

%\[(1 - \eta_{x})^{R} \leq (1 - \eta_{x} + \eta_x\tfrac{\mu_t}{\sqrt{n}})^{R} :=\delta_{R}\]
%\vspace{0.02in}
\noindent\textbf{Learn Stage:} Using this estimate of the coefficients, we update the dictionary at \[t\]-th iteration \[\b{A}^{(t)}\] by an approximate gradient descent step \eqref{eq:apx_grad}, using the empirical gradient estimate \eqref{eq:emp_grad_est} and the learning rate \[\eta_A = \Theta(m/k)\]; see also \ref{assumption:step dict}.
%\begin{align*}
%\b{A}^{(t+1)} = \b{A}^{(t)} - \eta_A~\hat{\b{g}}^{(t)},
%\end{align*}
%where we use the following empirical gradient estimate shown in \eqref{eq:emp_grad_est}.
%\begin{align*}
%\hat{\b{g}}^{(t)} = \tfrac{1}{p}\sum_{j = 1}^{p}(\b{A}^{(t)}\hat{\b{x}}_{(j)} - \b{y}_{(j)})\sgn(\hat{\b{x}}_{(j)} )^\top.
%\end{align*}
Finally, we normalize the columns of the dictionary and continue to the next batch.  The running time of each step \[t\] of NOODL is therefore \[\mathcal{O}(mnp \log(1/\delta_{R}))\]. For a target tolerance of \[\epsilon_{T} \] and \[\delta_T\], such that \[ \|\b{A}_i^{(T)} -\b{A}_i^*\| \leq \epsilon_{T} , \forall i\in [m]\] and \[|\hat{\b{x}}_{i}^{(T)} - \b{x}_{i}^*| \leq \delta_T\] we choose \[T=\max(\Omega(\log(1/\epsilon_{T})), \Omega(\log(\sqrt{k}/\delta_{T})))\].

%Further, to reach a column-wise error of \[\epsilon_T\], the online algorithm takes \[T = \Omega(\log(1/\epsilon_{T}))\] iterations.

NOODL uses an initial HT step and an approximate gradient descent-based strategy as in \cite{Arora15}. Following which, our IHT-based coefficient update step yields an estimate of the coefficients at each iteration of the online algorithm. Coupled with the guaranteed progress made on the dictionary, this also removes the bias in dictionary estimation. Further, the simultaneous recovery of both factors also avoids an often expensive post-processing step for recovery of the coefficients.

\vspace*{-3pt}
\section{Main Result} \label{sec:analysis}
\vspace*{-3pt}

We start by introducing a few important definitions. First, as discussed in the previous section we require that the initial estimate \[\b{A}^{(0)}\] of the dictionary is \[(\epsilon_0, 2)\]-close to \[\b{A}^*\]. In fact, we require this closeness property to hold at each subsequent iteration \[t\], which is a key ingredient in our analysis. This initialization achieves two goals. First, the \[\|\sigma(i)\b{A}_{\pi(i)} - \b{A}^*_i\|\leq \epsilon_0\] condition ensures that the signed-support of the coefficients are recovered correctly (with high probability) by the hard thresholding-based coefficient initialization step, where signed-support is defined as follows.
\begin{definition}\label{def:signed-support}
\textit{The signed-support of a vector \[\b{x}\] is defined as \[\sgn(\b{x})\cdot \supp(\b{x})\].}
\end{definition}
\vspace{-1pt}
Next, the \[\|\b{A} - \b{A}^*\| \leq 2\|\b{A}^*\|\] condition keeps the dictionary estimates close to \[\b{A}^*\] and is used in our analysis to ensure that the gradient direction \eqref{eq:emp_grad_est} makes progress. Further, in our analysis, we ensure \[\epsilon_t\] (defined as \[\|\b{A}_i^{(t)} - \b{A}^*_i\| \leq \epsilon_t\]) contracts at every iteration, and assume \[\epsilon_0, \epsilon_t = \mathcal{O}^*(1/\log(n))\]. %Here \[\|\b{A}_i^{(t)} - \b{A}^*_i\| \leq \epsilon_t\] assuming the dictionaries are properly permuted.
%We start with our assumptions on the dictionary and the coefficients. 
Also, we assume that the dictionary \[\b{A}\] is fixed (deterministic) and \[\mu\]-incoherent, defined as follows.
\begin{definition} \label{def:mu}
	\textit{ A matrix \[\b{A} \in \mathbb{R}^{n \times m}\] with unit-norm columns is \[\mu\]-incoherent if for all \[{i \neq j}\] the inner-product between the columns of the matrix follow \[ |\langle \b{A}_i, \b{A}_j\rangle| \leq \mu/\sqrt{n}\].}
\end{definition}%	
\vspace{-1pt}  
The incoherence parameter measures the degree of closeness of the dictionary elements. Smaller values (i.e., close to \[0\]) of \[\mu\] are preferred, since they indicate that the dictionary elements do not resemble each other. This helps us to effectively tell dictionary elements apart \citep{Donoho2001, Candes2007}. We assume that \[\mu = \c{O}(\log(n))\] \citep{Donoho2001}. 
%Smaller values (i.e., close to \[0\]) of \[\mu\] are preferred since they indicate higher incoherence and lead to sparser solutions. Further, the incoherence assumption ensures that its columns are ``spread-out'', which helps us to effectively tell dictionary elements apart \citep{Donoho2001, Candes2007} and hence recover them. For our discussion, we assume that \[\mu = \c{O}(\log(n))\] \citep{Donoho2001}. 
%In addition, we also assume that \[\|\b{A}^*\| = \mathcal{O}(\sqrt{m/n})\] and \[m = \mathcal{O}(n)\], which can be relaxed by reducing sparsity \[k\], where we assume that \[k = \mathcal{O}({\sqrt{n}}/{\mu~\log(n)})\]. 
Next, we assume that the coefficients are drawn from a distribution class \[\c{D}\] defined as follows.
\begin{definition}[Distribution class \[\c{D}\]]\label{dist_x} %Made imp change j to i in expectation
	\textit{The coefficient vector \[\b{x}^*\] belongs to an unknown distribution \[\c{D}\], where the support \[S = \supp(\b{x}^*)\] is at most of size \[k\], \[\b{Pr}[i \in S] = \Theta(k/m)\] and \[\b{Pr}[i, j \in S] = \Theta(k^2/m^2)\]. Moreover, the distribution is normalized such that \[\b{E}[\b{x}_i^*|i \in S] = 0\] and \[\b{E}[\b{x}_i^{*^2}|i \in S] = 1\], and when \[i \in S\],  \[|\b{x}^*_i| \geq C\] for some constant \[C \leq 1\]. In addition, the non-zero entries are sub-Gaussian and pairwise independent conditioned on the support.}
\end{definition}
\vspace{-1pt}  
%\begin{definition}[Distribution class \[\c{D}\]]\label{dist_x} %Made imp change j to i in expectation
%	The coefficient vector \[\b{x}^*\] belongs to an unknown distribution \[\c{D}\] if: the support \[S = \supp(\b{x}^*)\] is at-most of size \[k\], \[\b{Pr}[i \in S] = \Theta(k/m)\] and \[\b{Pr}[i, j \in S] = \Theta(k^2/m^2)\]; the distribution is normalized such that \[\b{E}[\b{x}_i^*|\b{x}_j^* \neq 0] = 0\] and \[\b{E}[\b{x}_i^{*^2}|\b{x}_i^* \neq 0] = 1\], and when \[\b{x}^*_i \neq 0\],  \[|\b{x}^*_i| \geq C\] for some constant \[C \leq 1\]; the non-zero entries are pairwise independent and sub-Gaussian, conditioned on the support.
%\end{definition}

The randomness of the coefficient is necessary for our finite sample analysis of the convergence. Here, there are two sources of randomness. The first is the randomness of the support, where the non-zero elements are assumed to pair-wise independent. The second is the value an element in the support takes, which is assumed to be zero mean with variance one, and bounded in magnitude. Similar conditions are also required for support recovery of sparse coefficients, even when the dictionary is known \citep{Wainwright2009, Yuan2016}.
%The randomness of the coefficients allows us to ensure that coefficient vectors drawn from \[\c{D}\] are independent, which is necessary for the concentration results in our finite sample analysis of the convergence. There are two sources of randomness here -- 1) the locations of non-zero elements in the coefficient vector (i.e., the support), and 2) the randomness associated with the value each element in the support takes. Definition~\ref{dist_x} formalizes our assumption on both of these. Specifically, we require that the supports are randomly chosen, and that the locations of non-zeros are pairwise independent. Further, to make the analysis tractable, we assume the expected value of a coefficient element in the support is zero and has variance one, whose magnitudes are bounded away from some \[C\leq 1\].
 %In addition, the independence of any pair of non-zero elements of a coefficient vector also helps us with this end. Further, the normalization of the distribution makes the analysis tractable. 
Note that, although we only consider the case \[|\b{x}^*_i | \geq C\] for ease of discussion, analogous results may hold more generally for \[\b{x}_i^*\]’s drawn from a distribution with sufficiently (exponentially) small probability of taking values in \[[-C, C]\].

%Next, turning to the gradient estimate for the dictionary update step. 
%The primary motivation of our approach is the observation that \eqref{eq:objective} is (strongly) convex on \[\b{A}\] given the coefficients. As a result 
Recall that, given the coefficients, we recover the dictionary by making progress on the least squares objective \eqref{eq:objective} (ignoring the term penalizing \[S(\cdot)\]). Note that, our algorithm is based on finding an appropriate direction to ensure descent than the geometry of the objective. To this end, we adopt a gradient descent-based strategy for dictionary update. However, since the coefficients are not exactly known, this results in an approximate gradient descent-based approach, where the empirical
gradient estimate is formed as \eqref{eq:emp_grad_est}. In our analysis, we establish the conditions under which both the empirical gradient vector (corresponding to each dictionary element) and the gradient matrix concentrate around their means. To ensure progress at each iterate \[t\], we show that the expected gradient vector is \[(\Omega(k/m), \Omega(m/k),0 )\]-correlated with the descent direction, defined as follows.
%since we use a gradient descent-based strategy for dictionary updates, it is natural to assume smoothness and convexity  As alluded to earlier, we form our gradient estimate as shown in \eqref{eq:emp_grad_est}.
%\begin{align}
%\label{eq:emp_grad_est}
%\hat{\b{g}}^{(t)} = \tfrac{1}{p}\sum_{j = 1}^{p}(\b{A}^{(t)}\hat{\b{x}}_{(j)} - \b{y}_{(j)})\sgn(\hat{\b{x}}_{(j)} )^\top.
%\end{align}
%Since, we do not have access to the true coefficient vectors\[\cbr{{\b{x}}_{(j)} }\], and use \[\sgn(\hat{\b{x}}_{(j)} )^\top\] in gradient evaluation instead of \[\hat{\b{x}}_{(j)}^\top\] for each sample \[j\], our algorithm falls into the regime of approximate gradient descent algorithm. As a result, we can only hope that our gradient estimate in expectation, i.e., \[\b{g}^{(t)}\], is correlated with the descent direction. This is formalized by the following definition. 
\begin{definition}\label{def:grad_alpha_beta}
	\textit{A vector \[\b{g}^{(t)}\] is  \[(\rho_{-}, \rho_{_+}, \zeta_t)\]-correlated with a vector \[\b{z}^*\] if 
	\vspace{-3pt}
	\begin{align*}
	\langle \b{g}^{(t)}, \b{z}^{(t)} - \b{z}^{*}\rangle \geq \rho_{-}\|\b{z}^{(t)} - \b{z}^{*}\|^2 + \rho_{+} \|\b{g}^{(t)}\|^2 - \zeta_{t}.
	\end{align*}}
	\vspace{-18pt}
\end{definition}
This can be viewed as a local descent condition which leads to the true dictionary columns; see also \cite{Candes2015}, \cite{Chen2015} and \cite{Arora15}. In convex optimization literature, this condition is implied by the \[2\rho_{-}\]-strong convexity, and \[1/2\rho_{_+}\]-smoothness of the objective.  We show that for NOODL, \[\zeta_{t}=0\], which facilitates linear convergence to \[\b{A}^*\] without incurring any bias.
Overall our specific model assumptions for the analysis can be formalized as:% formalizing the specifics of the model, shown in \eqref{eq:model}. 
\vspace{-4pt}
\begin{enumerate}[label=\textbf{A.\arabic*}, ref=\textbf{A.\arabic*}, leftmargin=*]
\setlength{\itemsep}{-1pt}
\item \label{assumption:mu} \[\b{A}^*\] is \[\mu\]-incoherent (Def.~\ref{def:mu}), where \[\mu = \mathcal{O}(\log(n))\], \[\|\b{A}^*\| = \mathcal{O}(\sqrt{m/n})\] and \[m = \mathcal{O}(n)\];
\item \label{assumption:dist}The coefficients are drawn from the distribution class \[\c{D}\], as per Def.~\ref{dist_x};
\item \label{assumption:k}The sparsity \[k\] satisfies \[k = \mathcal{O}({\sqrt{n}}/{\mu~\log(n)})\];
\item \label{assumption:close} \[\b{A}^{(0)}\] is \[(\epsilon_0, 2)\]-close to \[\b{A}^*\]  as per Def.~\ref{def:del_kappa}, and \[\epsilon_0 = \mathcal{O}^*(1/\log(n))\];
\item \label{assumption:step dict}The step-size for dictionary update satisfies \[\eta_A = \Theta(m/k)\];
\item \label{assumption:step coeff}The step-size and threshold for coefficient estimation satisfies \[\eta_x^{(r)}< c_1(\epsilon_t, \mu, n, k) = \tilde{\Omega}({k}/{\sqrt{n}})<1\] and \[\tau^{(r)} = c_2(\epsilon_t, \mu, k, n) = \tilde{\Omega}({k^2}/{n})\] for small constants \[c_1\] and \[c_2\].
\vspace{-3pt}
\end{enumerate}

We are now ready to state our main result. A summary of the notation followed by a details of the analysis is provided in Appendix~\ref{app:summary_notation} and Appendix~\ref{app:pf_main}, respectively.

\begin{theorem}[Main Result]\label{main_result}\textit{Suppose that assumptions \ref{assumption:mu}-\ref{assumption:step coeff}  hold, and Algorithm~\ref{alg:main_alg} is provided with  \[ p = \tilde{\Omega}(mk^2)\] new samples generated according to model \eqref{eq:model} at each iteration \[t\]. %then for \[\eta_A = \Theta(m/k)\] 
Then,  with probability at least \[(1 - \delta_{\text{alg}}^{(t)})\] for some small constant \[\delta_{\text{alg}}^{(t)}\], given \[R = \Omega({\rm log}(n))\], the coefficient estimate \[\hat{\b{x}}_{i}^{(t)}\] at \[t\]-th iteration has the correct signed-support and satisfies
\vspace{-2pt}
\begin{align*}
(\hat{\b{x}}_{i}^{(t)} - \b{x}_{i}^*)^2 %C_{i_1}^{(R)} 
&= \mathcal{O}(k(1 - \omega)^{t/2}\|\b{A}_i^{(0)} - \b{A}_i^*\|), ~\text{for all}~i \in \supp({\b{x}^*}).\vspace{-2pt}
%(c_x k + 1)(\tfrac{\epsilon_t^2}{2}|\b{x}_{\max}^*|+ t_\beta) + 2\delta_{R} k\eta_x\tfrac{\mu_t}{\sqrt{n}},\\
%& \leq c + (c_x k + 1)\tfrac{\epsilon_t^2}{2}|\b{x}_{\max}^*|
\end{align*} 
Furthermore, for some \[0 < \omega < 1/2\], the estimate \[\b{A}^{(t)}\] at \[(t)\]-th iteration satisfies 
\vspace{-2pt}
\begin{align*}
\|\b{A}_i^{(t)} - \b{A}_i^*\|^2 \leq (1 - \omega)^t\|\b{A}_i^{(0)} - \b{A}_i^*\|^2,~\text{for all}~t = 1,2,\ldots.\vspace{-2pt} .
\end{align*}
}%where \[\delta_{\text{alg}}^{(t)} =  \delta_{\HT}^{(t)}  + \delta_{\beta}^{(t)} + \delta_{\rm HW} +\delta_{\gradvec}^{(t)} + \delta_{\gradmat}^{(t)}\], \[\delta_{\HT}^{(t)} = 2m~{\exp}({-{C^2}/{\mathcal{O}^*(\epsilon_t^2)}})\], \[\delta_\beta^{(t)} = 2k~{\exp}(-{1}/{\mathcal{O}(\epsilon_t)})\], \[\delta_{\rm HW}^{(t)} = \exp(-{1}/{\mathcal{O}(\epsilon_t)})\], \[\delta_{\gradvec}^{(t)} =  \exp(-\Omega(k))\], and \[\delta_{\gradmat}^{(t)} = (n+m)\exp(-\Omega(m\sqrt{\log(n)})\].} %\[{T} = \Omega(\log(1/\epsilon_t))\] \[\|\b{A}^*\| = \mathcal{O}(\sqrt{m/n})\]
\end{theorem}
\vspace{-3pt}
Our main result establishes that when the model satisfies \ref{assumption:mu}\[\sim\]\ref{assumption:k}, the errors corresponding to the dictionary and coefficients geometrically decrease to the true model parameters, given appropriate dictionary initialization and learning parameters (step sizes and threshold); see \ref{assumption:close}\[\sim\]\ref{assumption:step coeff}. In other words, to attain a target tolerance of  \[\epsilon_{T}\] and \[\delta_T\], where \[\|\b{A}_i^{(T)} -\b{A}_i^*\| \leq \epsilon_{T}\], \[|\hat{\b{x}}_{i}^{(T)} - \b{x}_{i}^*| \leq \delta_T\], we require \[T=\max(\Omega(\log(1/\epsilon_{T})), \Omega(\log(\sqrt{k}/\delta_{T})))\] outer iterations and \[R= \Omega(\log(1/\delta_{R}))\] IHT steps per outer iteration. Here, \[\delta_{R}\geq (1 - \eta_{x})^{R} \] is the target decay tolerance for the IHT steps. An appropriate number of IHT steps, \[R\], remove the dependence of final coefficient error (per outer iteration) on the initial \[\b{x}^{(0)}\]. In \cite{Arora15}, this dependence in fact results in an irreducible error, which is the source of bias in dictionary estimation. As a result, since  (for NOODL) the error in the coefficients only depends on the error in the dictionary, it can be made arbitrarily small, at a geometric rate, by the choice of \[ \epsilon_{T}\], \[\delta_T\], and \[\delta_R\].
%drops geometrically at every iteration \[t\]. 
%Moreover, given  \[R= \Omega(\log(1/\delta_{R}))\] IHT iterations (for a target decay tolerance parameter \[\delta_{R}\] which removes the effect of the initial \[\b{x}^{(0)}\] and is defined as \[\delta_{R}\geq (1 - \eta_{x})^{R} \]), the error in the coefficients only depends on the error in the dictionary and drops geometrically at every iteration \[t\]. 
%\sr{Therefore, by the choice of \[ \epsilon_{T}\], \[\delta_T\], and \[\delta_R\], the error in the coefficients can be made arbitrarily small at a geometric rate.} 
Also, note that, NOODL can tolerate i.i.d. noise, as long as the noise variance is controlled to enable the concentration results to hold; we consider the noiseless case here for ease of discussion, which is already highly involved.
%The introduction of IHT step adds complexity in the analysis of both the dictionary and coefficients. 

Intuitively, Theorem~\ref{main_result} highlights the symbiotic relationship between the two factors. It shows that, to make progress on one, it is imperative to make progress on the other. The primary condition that allows us to make progress on both factors is the signed-support recovery (Def.~\ref{def:signed-support}). However, the introduction of IHT step adds complexity in the analysis of both the dictionary and coefficients. To analyze the coefficients, in addition to deriving conditions on the parameters to preserve the correct signed-support, we analyze the recursive IHT update step, and decompose the noise term into a component that depends on the error in the dictionary, and the other that depends on the initial coefficient estimate. For the dictionary update, we analyze the interactions between elements of the coefficient vector (introduces by the IHT-based update step) and show that the gradient vector for the dictionary update is \[(\Omega(k/m), \Omega(m/k),0 )\]-correlated with the descent direction. In the end, this leads to exact recovery of the coefficients and removal of bias in the dictionary estimation. Note that our analysis pipeline is standard for the convergence analysis for iterative algorithms. However, the introduction of the IHT-based strategy for coefficient update makes the analysis highly involved as compared to existing results, e.g., the simple HT-based coefficient estimate in \cite{Arora15}.

NOODL has an overall running time of \[\mathcal{O}(mnp \log(1/\delta_{R})\max(\log(1/\epsilon_{T}), \log(\sqrt{k}/\delta_T))\] to achieve target tolerances \[\epsilon_{T} \]  and \[\delta_T\], with a total sample complexity of \[ p \cdot T  = \tilde{\Omega}(mk^2)\]. Thus to remove bias, the IHT-based coefficient update introduces a factor of \[\log(1/\delta_{R})\] in the computational complexity as compared to \cite{Arora15} (has a total sample complexity of \[ p\cdot T = \tilde{\Omega}(mk)\]), and also does not have the exponential running time and sample complexity as \cite{Barak2015}; see Table~\ref{tab:compare}.

\vspace*{-3pt}
\section{Neural implementation of NOODL}\label{app:neural}
\vspace{-3pt}
The neural plausibility of our algorithm implies that it can be implemented as a neural network. This is because, NOODL employs simple linear and non-linear operations (such as inner-product and hard-thresholding) and the coefficient updates are separable across data samples, as shown in \eqref{alg:coeff_iht} of Algorithm~\ref{alg:main_alg}. To this end, we present a neural implementation of our algorithm in Fig.~\ref{fig:neural_imple}, which showcases the applicability of NOODL in large-scale distributed learning tasks, motivated from the implementations described in \citep{Olshausen97} and \citep{Arora15}. 

The neural architecture shown in Fig.~\ref{fig:neural_imple}(a) has three layers -- input layer, weighted residual evaluation layer, and the output layer.  The input to the network is a data and step-size pair \[(\b{y}_{(j)},\eta_x)\] to each input node. Given an input, the second layer evaluates the weighted residuals as shown in Fig.~\ref{fig:neural_imple}. Finally, the output layer neurons evaluate the IHT iterates \[\b{x}_{(j)}^{(r+1)}\] \eqref{alg:coeff_iht}. We illustrate the operation of this architecture using the timing diagram in Fig.~\ref{fig:neural_imple}(b). The main stages of operation are as follows.
%As described previously, the separability of coefficient updates across samples allows for distributed implementations of NOODL, making our algorithm suitable for large scale implementations. Further, the neural plausibility of our algorithm can be attributed to the simplicity of the IHT update steps, which includes simple calculation of inner products and hard thresholding. 
%We use the timing diagram shown in Fig.~\ref{fig:neural_imple}(b) to illustrate the operation of the neural architecture shown in Fig.~\ref{fig:neural_imple}(a). The main steps are described below.
%The data samples \[\b{y}_{(j)}\], and the step-size \[\eta_x\] are provided to the input layer, which communicates these to the weighted residual evaluation stage. 

\begin{figure}[!t]
\vspace{-10pt}
\centering
\begin{minipage}{0.55\textwidth}
\centering
\includegraphics[width=\textwidth]{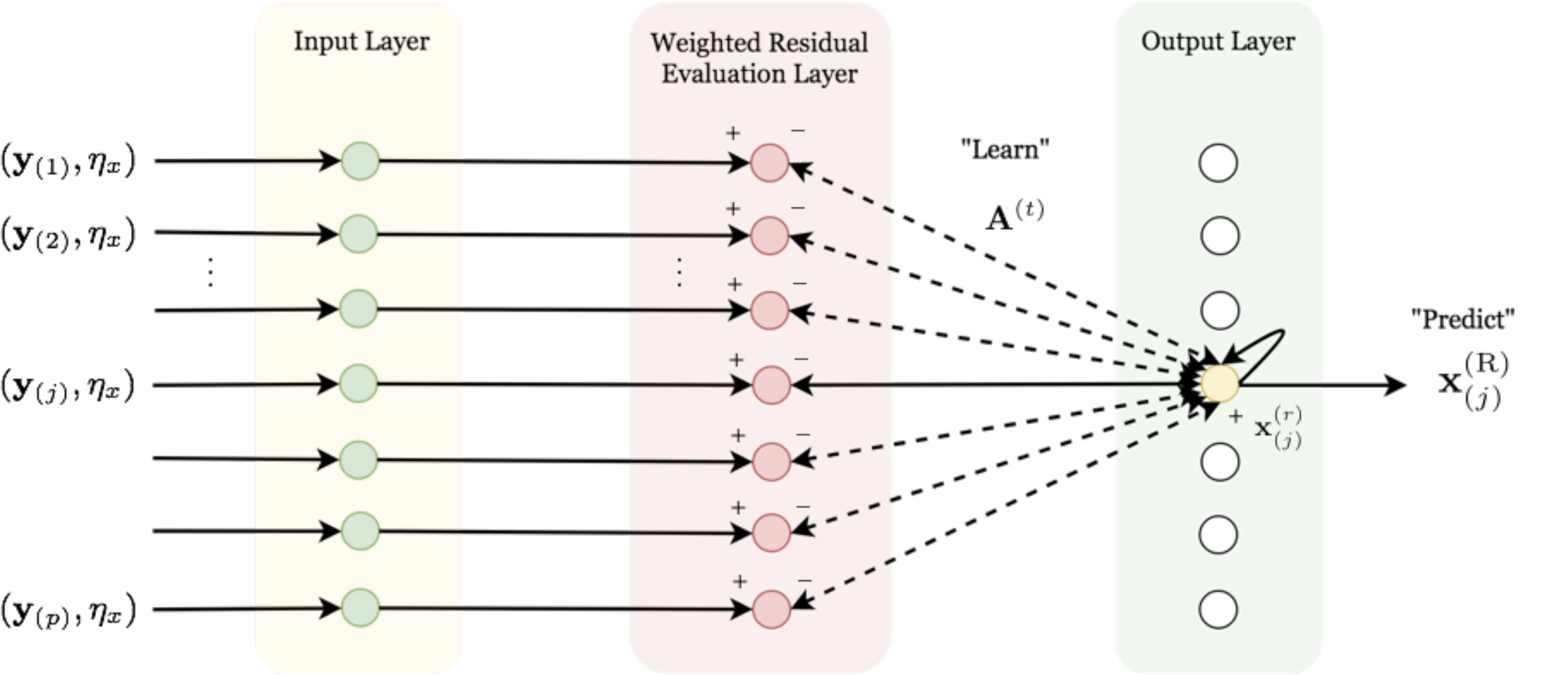}\\
{\footnotesize(a) Neural implementation of NOODL}
\end{minipage}\hfill\hspace{-8pt}
\begin{minipage}{0.42\textwidth}
%\vspace{-0.2in}
\caption{A neural implementation of NOODL. Panel (a) shows the neural architecture, which consists of three layers: an input layer, a weighted residual evaluation layer (evaluates \[\eta_x \big(\b{y}_{(j)} - \b{A}^{(t)}\b{x}_{(j)}^{(r)} \big) \]), and an output layer. Panel (b) shows the operation of the neural architecture in panel (a). The update of \[\b{x}_{(j)}^{(r+1)}\] is given by \eqref{alg:coeff_iht}. }
\label{fig:neural_imple}
\end{minipage}
\begin{threeparttable}
\colorbox[RGB]{239,240,241}{
\begin{minipage}{0.98\textwidth}
\resizebox{\linewidth}{!}{%
\centering
{\small \begin{tabular}{lccccccc|m{1.5cm}}
	\hspace{-0.15in}\[\ell = 0\]&\[\ell =1\] & \[\ell  = 2\] &\hspace{-0.1in} \[\ell  = 3\]&\[\ell  = 4\]&\[\ell  = 5\] &\[\dots\] & \[\ell = 2{R} + 1\]&
	\multicolumn{1}{m{1.3cm}}{\multirow{3}{1.5cm}{\footnotesize\textbf{Hebbian Learning}:   Residual sharing and dictionary update.}}\hspace{-0.1in}\\ 
	\hspace{-0.15in}&&&&&&&&\hspace{-0.1in}\\
	\hspace{-0.15in}Output: \[\b{x}\leftarrow \b{0}\] & \[\b{0}\]&\hspace{-0.1in}\[\b{x}_{(j)}^{(0)} = \HT_{\tau}({\b{A}^{(t)}}^\top\b{y}_{(j)})\]& \[\b{x}_{(j)}^{(0)} \]& \[\b{x}_{(j)}^{(1)}\]&\[\b{x}_{(j)}^{(1)}\]&\[\dots\]&\[\b{x}_{(j)}^{(R)}\]&\vspace{-10pt}\\ 
%	&&&&&&&&\hspace{-0.1in}\\
	\hspace{-0.15in}Residual: \[\b{0}\] & \[\b{y}_{(j)}\]&\[\b{y}_{(j)}\]&\hspace{-0.1in}\[\eta_x(\b{y}_{(j)} - \b{A}^{(t)}\b{x}_{(j)}^{(0)}) \]&\hspace{-0.1in}\[\eta_x(\b{y}_{(j)} - \b{A}^{(t)}\b{x}_{(j)}^{(0)}) \]&\hspace{-0.1in}\[\eta_x(\b{y}_{(j)} - \b{A}^{(t)}\b{x}_{(j)}^{(1)}) \]&\[\dots\]&\hspace{-0.1in}\[\eta_x(\b{y}_{(j)} - \b{A}^{(t)}\b{x}_{(j)}^{(R-1)})\]&\hspace{-0.1in}\\
%	&&&&&&&&\hspace{-0.1in}\\
	\hspace{-0.15in}Input: \[(\b{y}_{(j)}, 1)\]&\[.\]&\[(\b{y}_{(j)}, \eta_x)\]&\hspace{-0.1in}\[.\]&\[.\]&\[.\]&\[\dots\]&\[(\b{y}_{(j)}, 1)\]& \hspace{-0.1in}
\end{tabular}}}
\end{minipage}}
%	\begin{tablenotes}
%\item[**] {\footnotesize Here, the \[t\]-th iterate of coefficient update is given by $\footnotesize \b{x}_{(j)}^{(r+1)} = \HT_{\tau}(\b{x}_{(j)}^{(r)} +\eta_x{\b{A}^{(t)}}^\top(\b{y}_{(j)} - \b{A}^{(t)}\b{x}_{(j)}^{(r)}) )$}.
%	\end{tablenotes}
\end{threeparttable}\\
{\footnotesize (b) The timing sequence of the neural implementation.}
\vspace{-5pt}
\end{figure}

\noindent\textbf{Initial Hard Thresholding Phase}: The coefficients initialized to zero, and an input \[(\b{y}_{(j)}, 1)\] is provided to the input layer at a time instant \[\ell=0\], which communicates these to the second layer. Therefore, the residual at the output of the weighted residual evaluation layer evaluates to \[\b{y}_{(j)}\] at \[\ell=1\]. Next, at \[\ell=2\], this residual is communicated to the output layer, which results in evaluation of the initialization \[ \b{x}_{(j)}^{(0)}\] as per \eqref{alg:coeff_init}. This iterate is communicated to the second layer for the next residual evaluation. Also, at this time, the input layer is injected with  \[(\b{y}_{(j)}, \eta_x)\] to set the step size parameter \[\eta_x\] for the IHT phase, as shown in Fig.~\ref{fig:neural_imple}(b). 

\noindent\textbf{Iterative Hard Thresholding (IHT) Phase}: Beginning \[\ell =3\], the timing sequence enters the IHT phase. Here, the output layer neurons communicate the iterates \[\b{x}^{(r+1)}_{(j)}\] to the second layer for evaluation of subsequent iterates as shown in Fig.~\ref{fig:neural_imple}(b). The process then continues till the time instance \[\ell = 2R+1\], for \[{R} = \Omega(\log(1/\delta_{R}))\] to generate the final coefficient estimate \[ \widehat{\b{x}}_{(j)}^{(t)} := \b{x}_{(j)}^{(R)} \] for the current batch of data. At this time, the input layer is again injected with \[(\b{y}_{(j)}, 1)\] to prepare the network for residual sharing and gradient evaluation for dictionary update.

\noindent\textbf{Dictionary Update Phase:} The procedure now enters the dictionary update phase, denoted as ``Hebbian Learning'' in the timing sequence. In this phase, each output layer neuron communicates the final coefficient estimate \[ \widehat{\b{x}}_{(j)}^{(t)} = \b{x}_{(j)}^{(R)} \] to the second layer, which evaluates the residual for one last time (with \[\eta_x =1\]), and shares it across all second layer neurons (``Hebbian learning''). This allows each second layer neuron to evaluate the empirical gradient estimate \eqref{eq:emp_grad_est},
%\begin{align*}
%\hat{\b{g}}^{(t)} = \tfrac{1}{p}\sum_{j = 1}^{p}(\b{A}^{(t)}\hat{\b{x}}_{(j)} - \b{y}_{(j)})\sgn(\hat{\b{x}}_{(j)} )^\top, 
%\end{align*}
which is used to update the current dictionary estimate (stored as weights) via an approximate gradient descent step. This completes one outer iteration of Algorithm~\ref{alg:main_alg}, and the process continues for \[T\] iterations to achieve target tolerances \[\epsilon_T\] and \[\delta_T\], with each step receiving a new mini-batch of data. %and outputting the corresponding coefficient estimates. 

%\[(\b{y}_{(j)} - \b{A}\b{x}_{(j)}^{(R)})\sgn(\b{x}_{(j)}^{(R)})^\top\] to all second layer neurons (``Hebbian learning''), so that they can combine these and update the stored weights (dictionary) \[\b{A}^{(t)}\].

%\noindent\textbf{At \[\b{\ell=1}\]}:At a subsequent time-step \[\eta_x\] is set to a desired level. Next, the weighted residual evaluation layer calculates \[\eta_x(\b{y}_{(j)} - \b{A}\b{x}_{(j)}^{(\ell-1)})\], and communicates it to the third layer, which forms the coefficient estimate \[ \b{x}_{(j)}^{(r)+1}\], as shown in Fig.~\ref{fig:neural_imple}(b).  
%
%
%
%The updated coefficient estimates are then again communicated to the second layer, to aid the evaluation of the new residual. The algorithm then proceeds in this ``predict-learn'' fashion, until a fixed number of steps or until the residuals drop below certain level.
%
%
% At this time, each output layer unit communicates the local gradient estimate \[(\b{y}_{(j)} - \b{A}\b{x}_{(j)}^{(R)})\sgn(\b{x}_{(j)}^{(R)})^\top\] to all second layer neurons (``Hebbian learning''), so that they can combine these and update the stored weights (dictionary) \[\b{A}^{(t)}\].

\vspace{-4pt}
\section{Experiments}\label{sec:exp} 
\vspace*{-4pt}

We now analyze the convergence properties and sample complexity of NOODL via experimental evaluations \footnote{ The associated code is made available at \texttt{https://github.com/srambhatla/NOODL}.}. The experimental data generation set-up, additional results, including analysis of computational time, are shown in Appendix~\ref{app:additional_res}.

% compare the convergence properties of NOODL with the state-of-the-art provable techniques. We also analyze the sample complexity of our algorithm via phase transitions; see Appendix~\ref{app:additional_res} for additional results.

\vspace*{-5pt}
\subsection{Convergence Analysis} 
\vspace*{-5pt}

We compare the performance of our algorithm NOODL with the current state-of-the-art alternating optimization-based online algorithms presented in \cite{Arora15}, and the popular algorithm presented in \cite{Mairal09} (denoted as {\texttt{Mairal} `\texttt{09}}). First of these, \texttt{Arora15(``biased'')}, is a simple neurally plausible method which incurs a bias and has a sample complexity of \[\Omega(mk)\]. The other, referred to as \texttt{Arora15(``unbiased'')}, incurs no bias as per \cite{Arora15}, but the sample complexity results were not established.
%which works by projecting out the noise caused by other dictionary elements, is more complex, but incurs no bias as per \cite{Arora15}. However, the sample complexity of this algorithm has not been established.
{\footnotesize
\begin{figure}[!t]
\vspace{-0.1in}
\centering
\begin{tabular}{ccccc}
\hspace*{25pt}{\scriptsize $k = 10$, $\eta_A = 30$} 
&\hspace{-5pt}{\scriptsize $k = 20$, $\eta_A = 30$} 
&\hspace{-5pt}{\scriptsize $k = 50$, $\eta_A = 15$} 
&\hspace{-5pt}{\scriptsize $k = 100$, $\eta_A = 15$} \vspace{-0pt}
&\hspace{-5pt} {\scriptsize  Phase Transition} 
\vspace{-12pt}\\
\hspace{-10pt}
{\rotatebox{90}{\parbox{2.4cm}{\centering\scriptsize\textbf{Dictionary Recovery Across Techniques}}}}\hspace{4pt}
\includegraphics[width=0.18\textwidth]{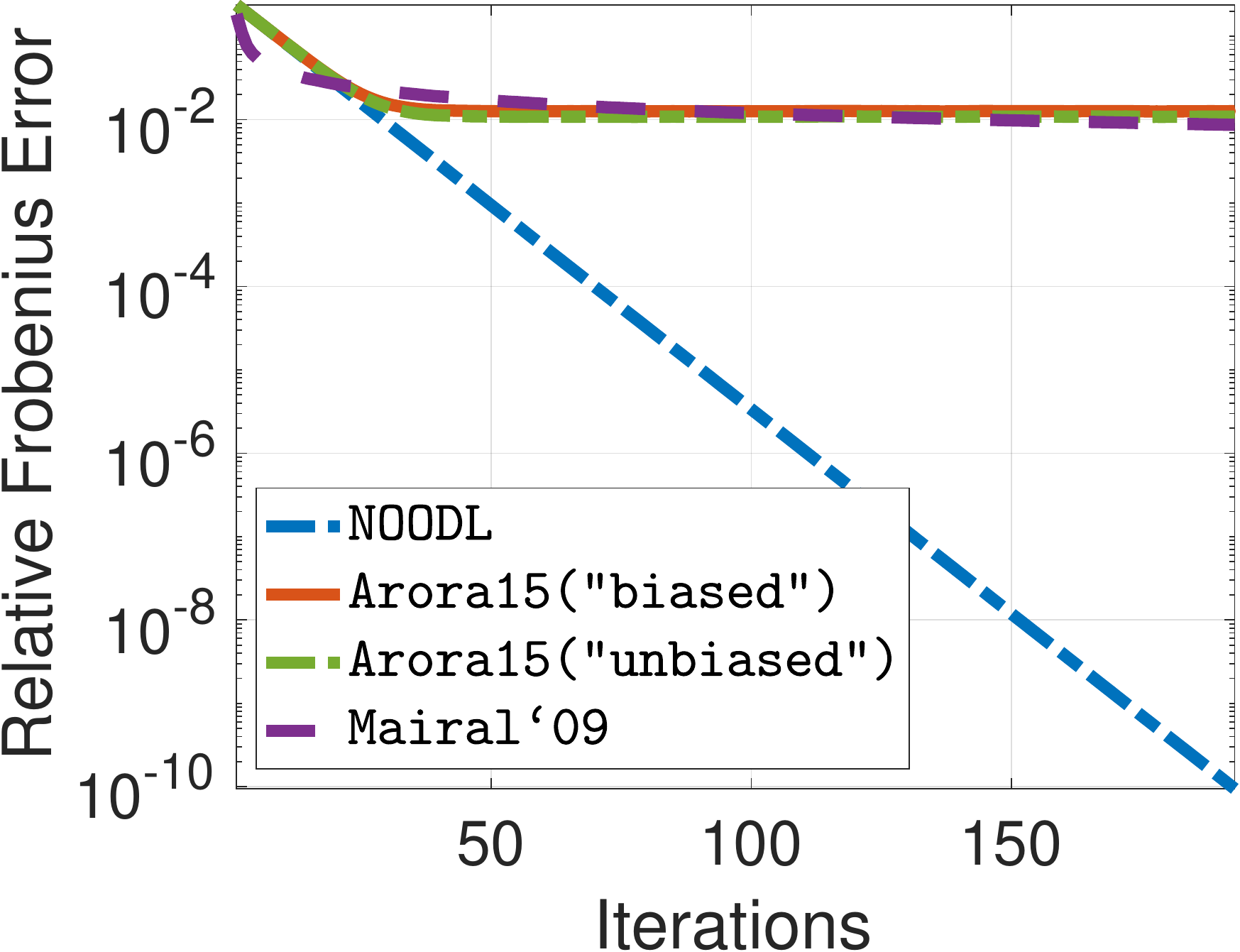}&\hspace*{-13pt}
\includegraphics[width=0.185\textwidth]{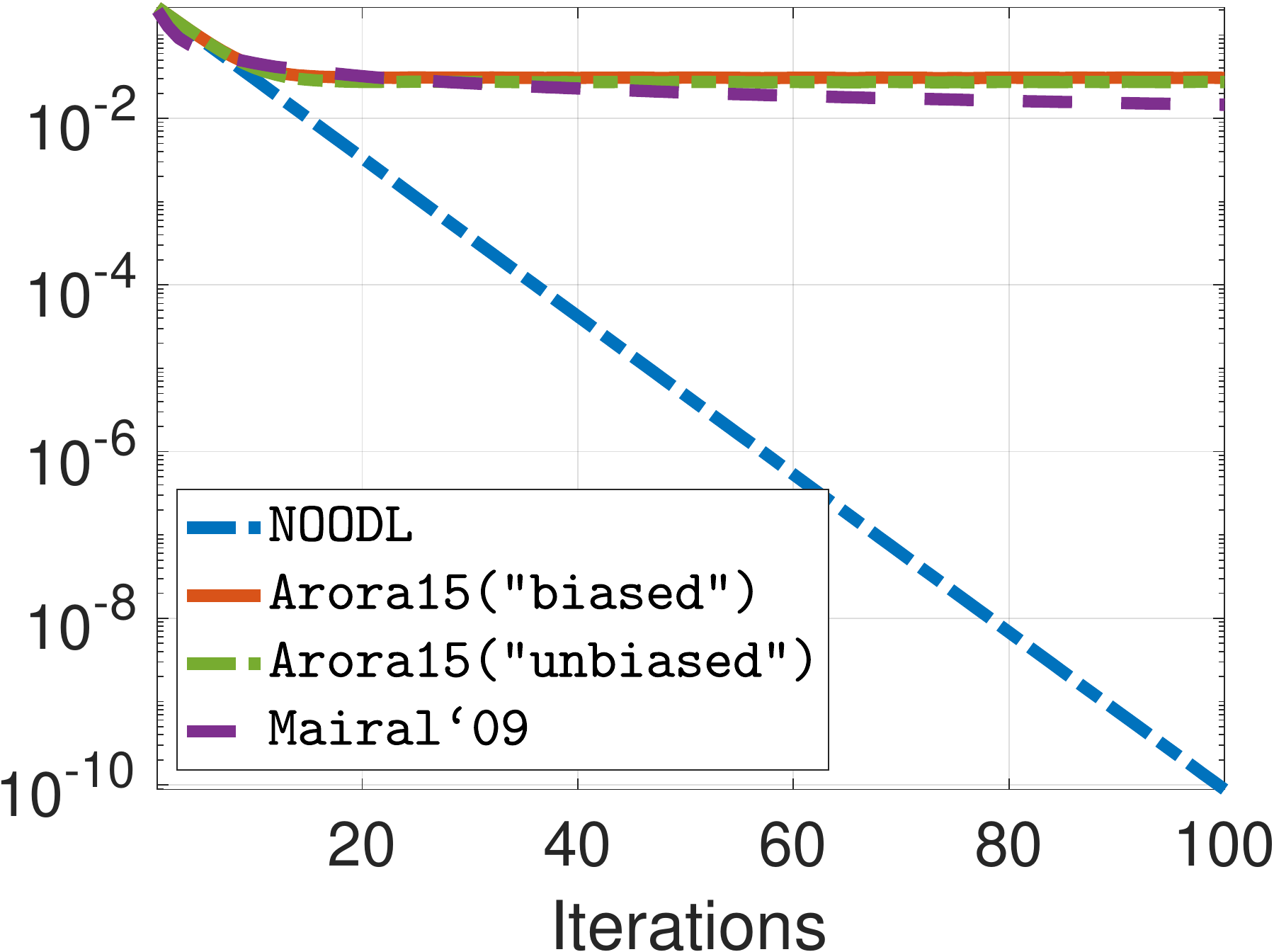}&\hspace*{-13pt}
\includegraphics[width=0.18\textwidth]{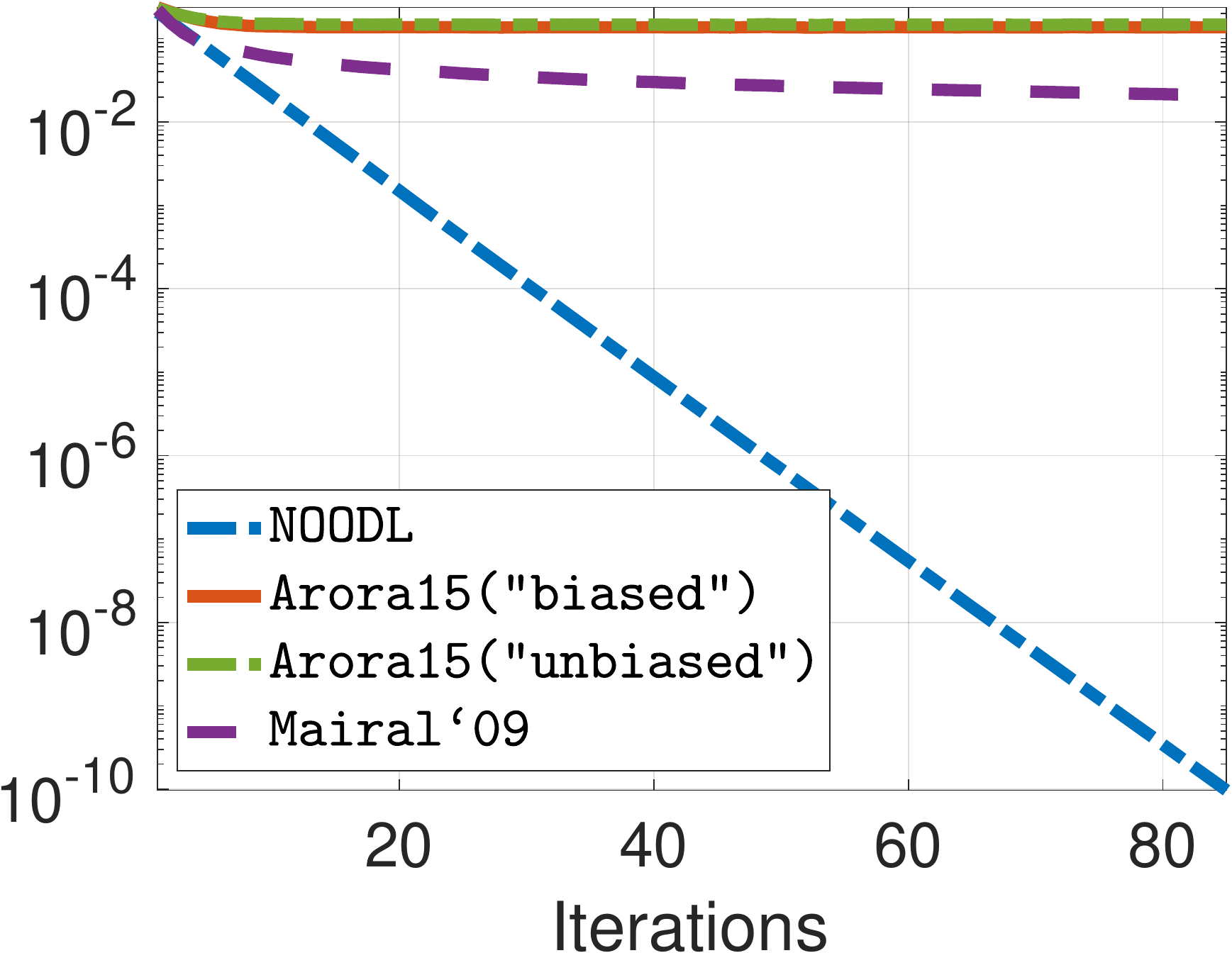}&\hspace*{-13pt}
\includegraphics[width=0.18\textwidth]{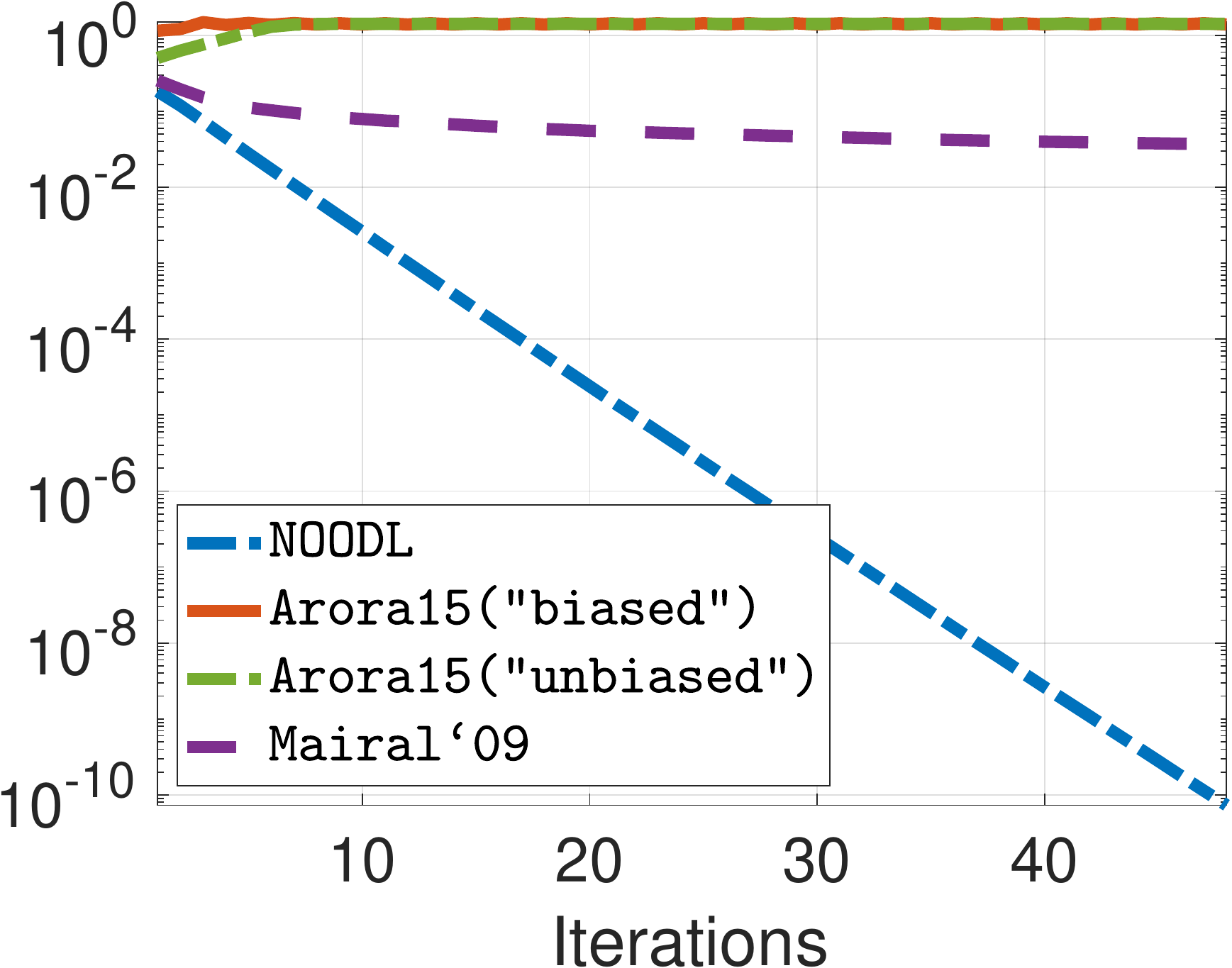}&\hspace*{-11pt}
\includegraphics[width=0.18\textwidth]{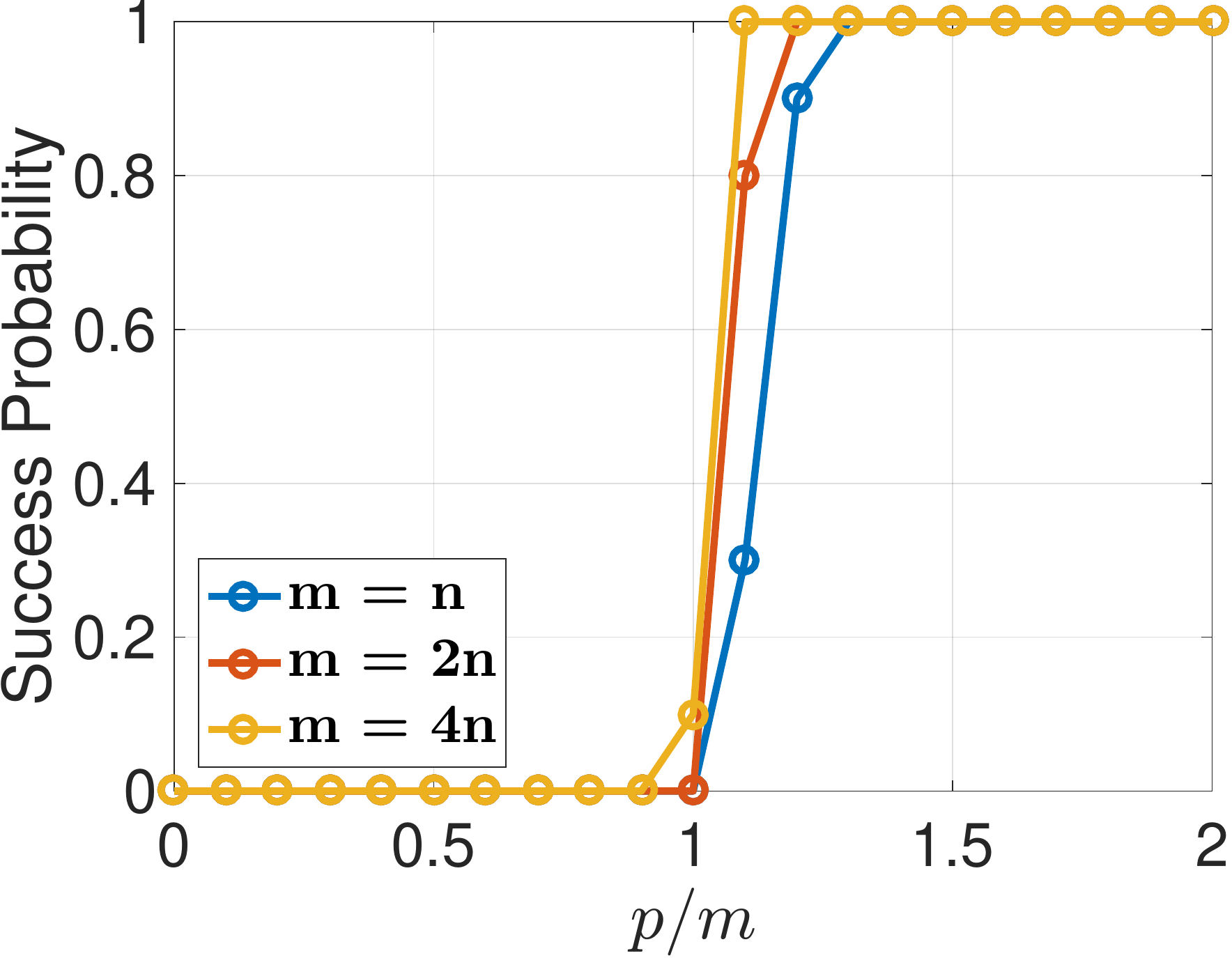}\vspace{-7pt}

\vspace{0.05in}

\\

~~~~~~~~{\footnotesize (a-i)} 
&{\footnotesize  (b-i)} 
&{\footnotesize  (c-i)} 
&{\footnotesize (d-i) }
&{\footnotesize (e-i) {\scriptsize Dictionary} } \vspace{-15pt}\\
\hspace{-10pt}
{\rotatebox{90}{ \parbox{2.4cm}{\centering\scriptsize\textbf{Performance of NOODL}}}}\hspace{5pt}
\includegraphics[width=0.18\textwidth]{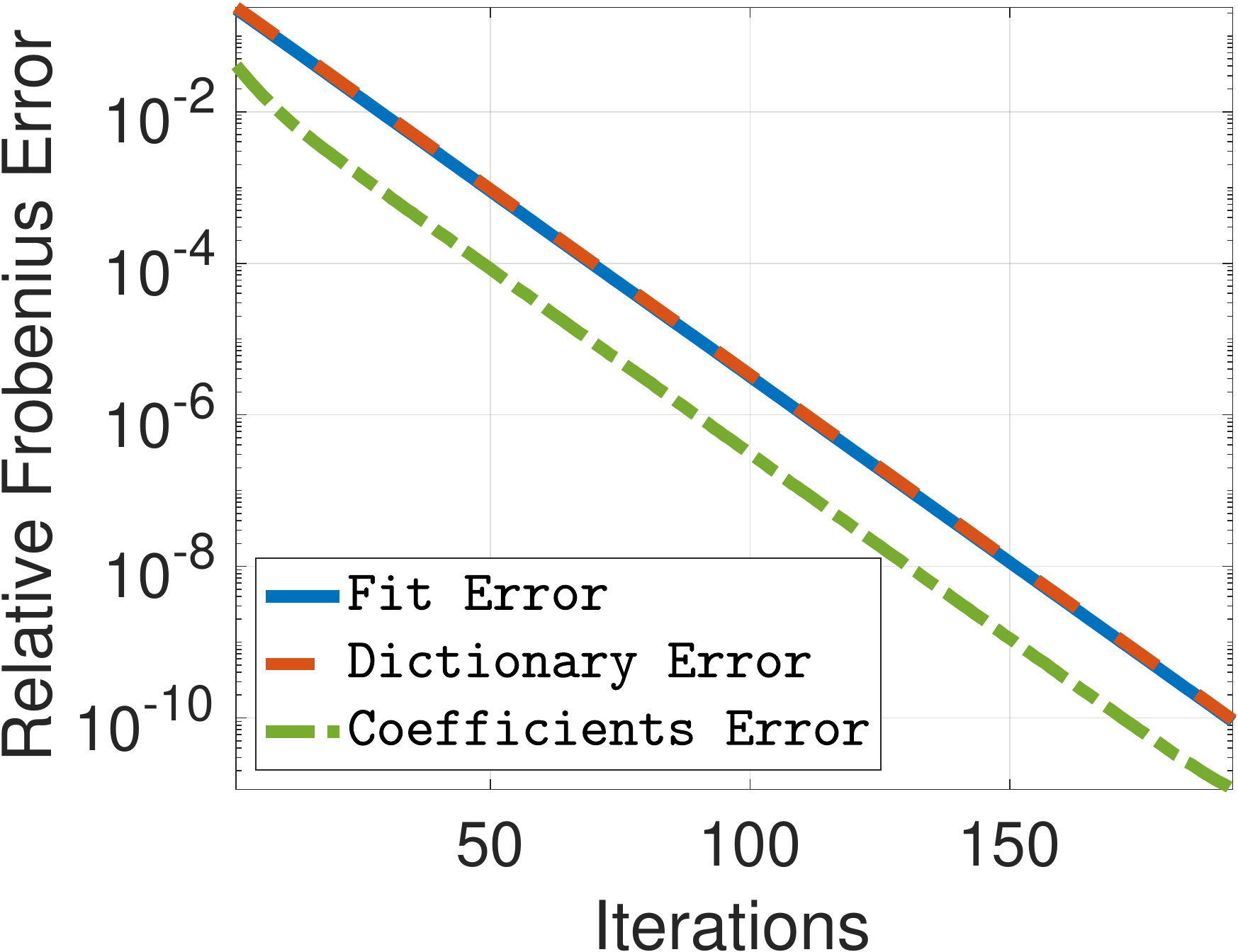}&\hspace{-13pt}
\includegraphics[width=0.175\textwidth]{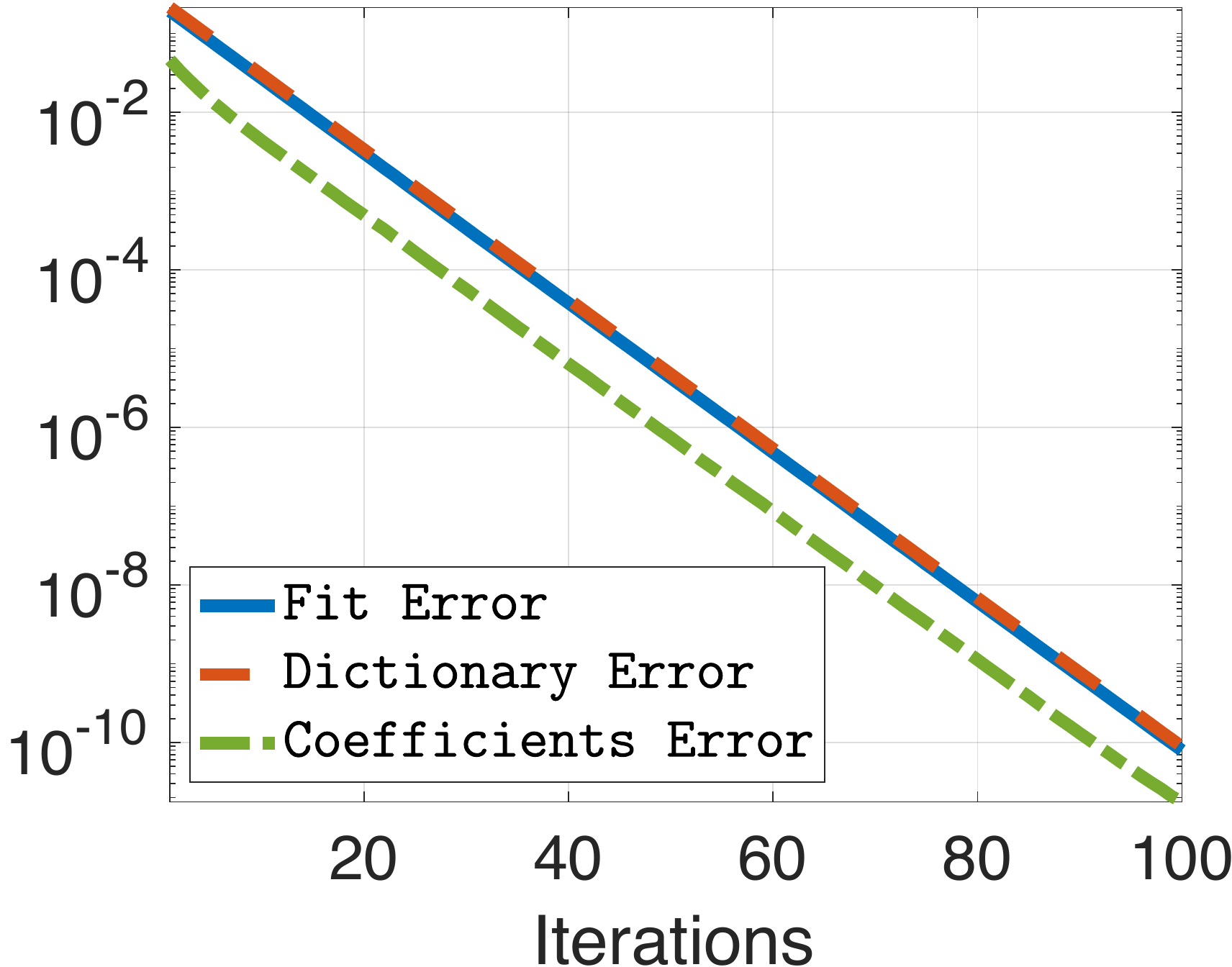}&\hspace{-13pt}
\includegraphics[width=0.17\textwidth]{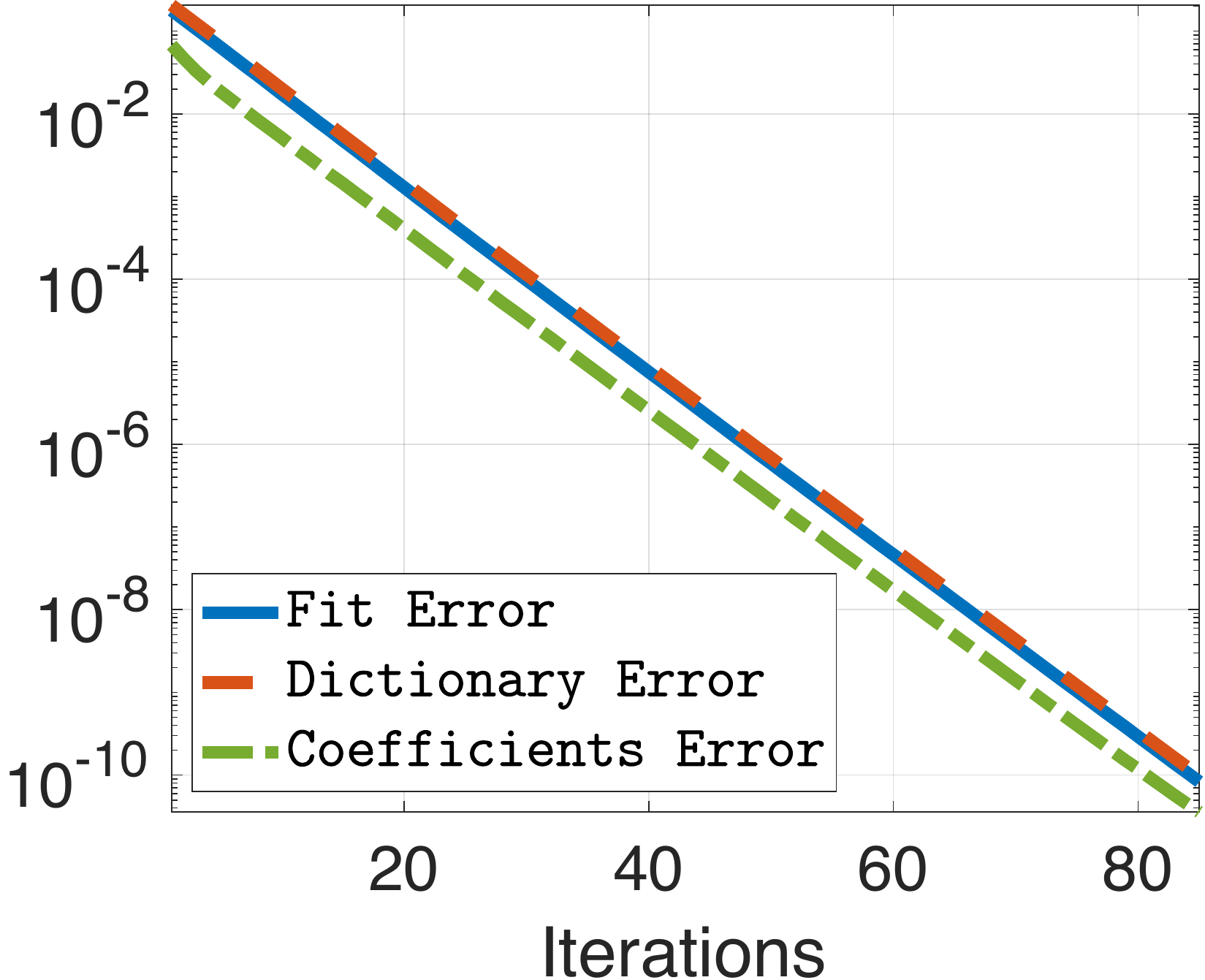}&\hspace{-13pt}
\includegraphics[width=0.17\textwidth]{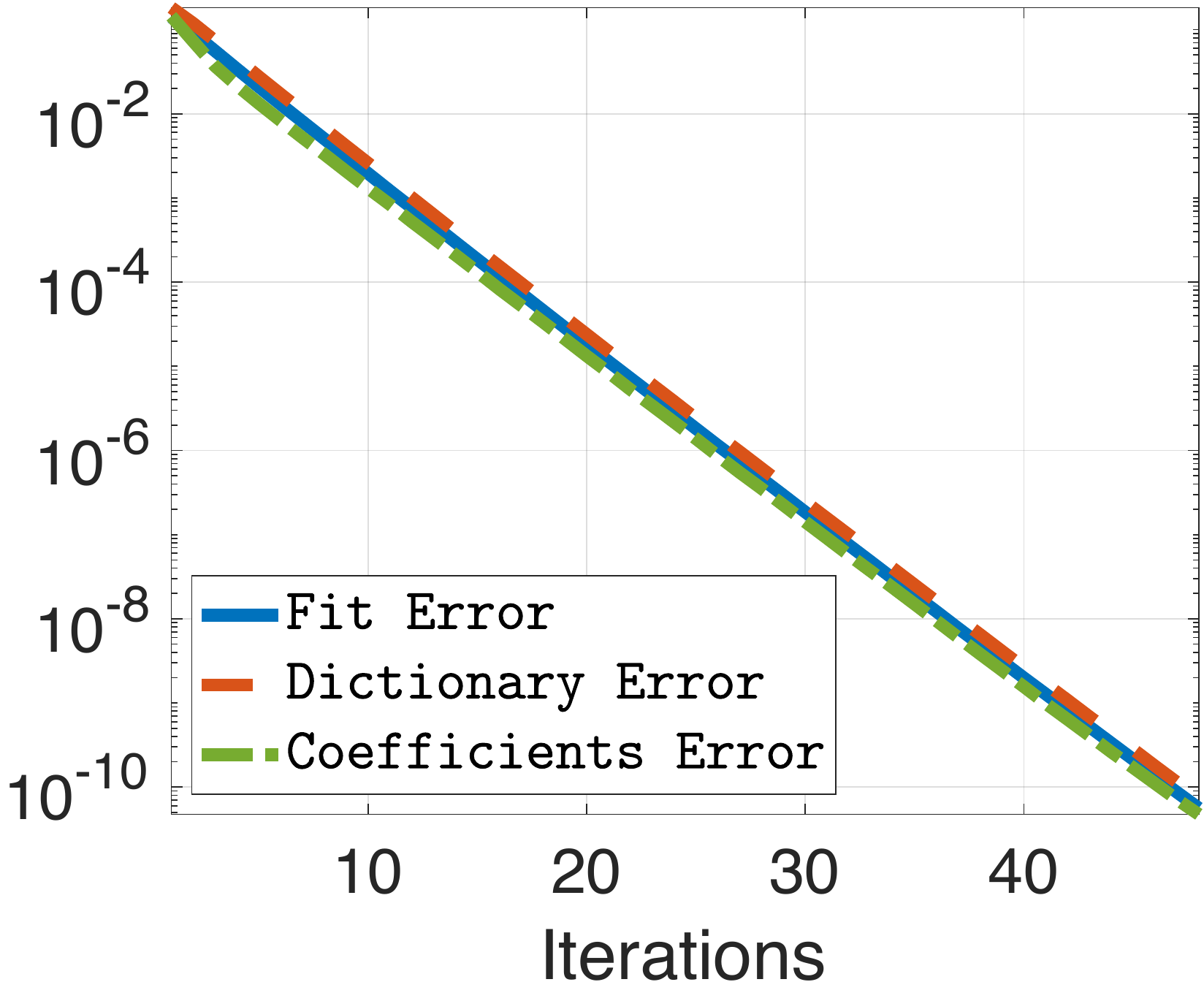}&\hspace{-11pt}
\includegraphics[width=0.18\textwidth]{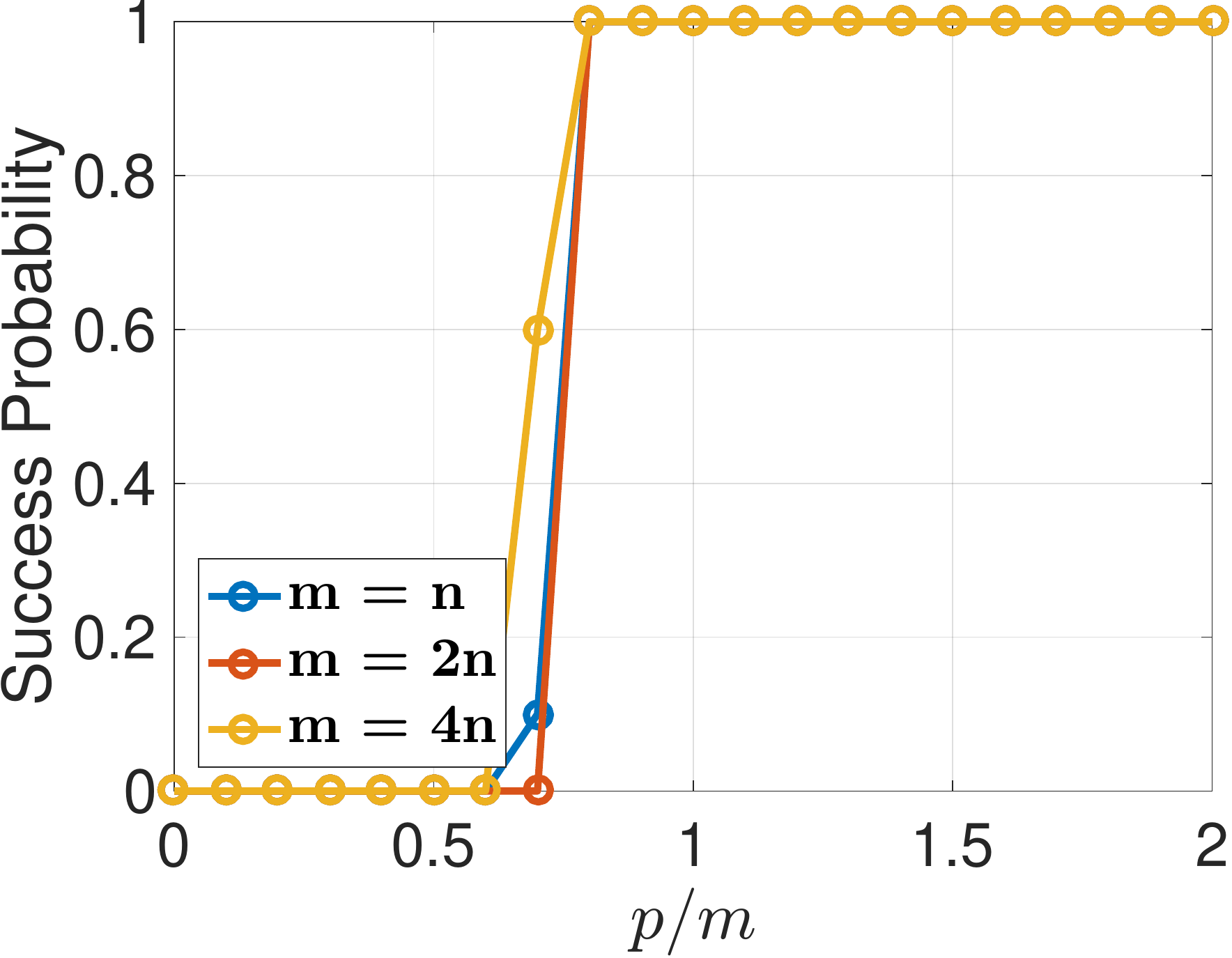}\vspace{-7pt}

\vspace{0.05in}

\\

~~~~~~~~{\footnotesize (a-ii)} 
&{\footnotesize  (b-ii)} 
&{\footnotesize  (c-ii)} 
&{\footnotesize (d-ii) }
&{\footnotesize (e-ii) {\scriptsize Coefficients}}
\vspace{-7pt}\\
\end{tabular}
\caption{Comparative analysis of convergence properties. Panels (a-i), (b-i), (c-i), and (d-i) show the convergence of NOODL,  \texttt{Arora15(``biased'')}, \texttt{Arora15(``unbiased'')} and {\texttt{Mairal} `\texttt{09}}, for different sparsity levels for $n=1000$, $m=1500$ and $p =5000$. Since NOODL also recovers the coefficients, we show the corresponding recovery of the dictionary, coefficients, and overall fit in panels (a-ii), (b-ii), (c-ii), and (d-ii), respectively. Further, panels (e-i)  and (e-ii) show the phase transition in samples \[p\] (per iteration) with the size of the dictionary \[m\] averaged across \[10\] Monte Carlo simulations for the two factors.  Here, \[n=100\], \[k =3\], \[\eta_x = 0.2\], \[\tau = 0.1\], \[\epsilon_0 =  2/\log(n)\], \[\eta_A\] is chosen as per \ref{assumption:step dict}. A trial is considered successful if the relative Frobenius error incurred by \[\hat{\b{A}}\] and \[\hat{\b{X}}\] is below \[5\times10^{-7}\] after \[50\] iterations.}
\label{fig:recovery}
\vspace{-7pt}
\end{figure}}

\noindent\textbf{Discussion:} Fig.~\ref{fig:recovery} panels (a-i), (b-i), (c-i), and (d-i) show the performance of the aforementioned methods for \[k = 10,~20,~50\],
and \[100\], respectively. Here, for all experiments we set \[\eta_x = 0.2\] and \[\tau = 0.1\]. We terminate NOODL when the error in dictionary is less than $10^{-10}$. Also, for coefficient update, we terminate when change in the iterates is below \[10^{-12}\]. %Further, for $k=10$ and \[20\] we set \[\eta_A = 30\], while for \[k=50\] and \[100\] we set \[\eta_A = 15\] in accordance with \ref{assumption:step dict}. 
For \[k=10,~20\] and \[k=50\], we note that \texttt{Arora15(``biased'')} and \texttt{Arora15(``unbiased'')} incur significant bias, while NOODL converges to \[\b{A}^*\] \textit{linearly}. NOODL also converges for significantly higher choices of sparsity $k$, i.e., for $k=100$ as shown in panel (d), beyond \[k=\c{O}(\sqrt{n})\], indicating a potential for improving this bound. Further, we observe that {\texttt{Mairal} `\texttt{09}} exhibits significantly slow convergence as compared to NOODL.
%For \[k=100\], \texttt{Arora15(``biased'')} and \texttt{Arora15(``unbiased'')} do not converge
%NOODL also converges for significantly higher choices of sparsity $k$, i.e., for $k=100$ as shown in panel (d), beyond \[\sqrt{n}\], while the other methods do not converge for this sparsity level.%We conclude experimentally that NOODL requires lower number of samples than \texttt{Arora15(``unbiased'')} \xl{(why?)}\unsure{See the x axis, the}. 
Also, in panels (a-ii), (b-ii), (c-ii) and (d-ii) we show the corresponding performance of NOODL in terms of the error in the overall fit (\[{\|\b{Y} - \b{AX} \|_{\Fr}}/{\|\b{Y}\|_{\Fr}}\]), and the error in the coefficients and the dictionary, in terms of relative Frobenius error metric discussed above.  We observe that the error in dictionary and coefficients drops linearly as indicated by our main result.

%In addition, we notice that the number of iterations needed for NOODL to converge drops with the increase in sparsity $k$. This can be attributed to the fact that, the number of non-zero columns in empirical gradients estimate increase with increasing \[k\], leading to better gradient estimates.

%
%
%{\footnotesize\begin{figure}[!t]
%\centering
%\begin{tabular}{cc}\hspace{-0.3cm}
%\includegraphics[width=0.4\textwidth]{../imgs/phase_tr_dict_m-crop.pdf}&\hspace{30pt}
%\includegraphics[width=0.4\textwidth]{../imgs/phase_tr_coeff_m-crop.pdf}\\
%\hspace{-0.3cm}{\footnotesize (a) Recovery of dictionary} & \hspace{30pt}{\footnotesize (b) Recovery of coefficients}
%\end{tabular}
%\caption{\footnotesize Phase transition analysis in samples per iteration.  Panels (a)  and (b) show the phase transition in samples \[p\] with the size of the dictionary \[m\] averaged across \[10\] Monte Carlo simulations, for recovery of the dictionary and coefficients, respectively. Here, \[n=100\], \[k =3\], \[\eta_x = 0.2\], \[\tau = 0.1\], \[\epsilon_0 =  2/\log(n)\], \[\eta_A\] is chosen proportional to \[m/k\] as per \ref{assumption:step dict}. A trial is considered successful if the relative Frobenius error incurred by \[\hat{\b{A}}\] and \[\hat{\b{X}}\] is below \[5\times10^{-7}\] after \[50\] iterations.}
%\label{fig:phase_tr}
%\end{figure}}

\vspace*{-6pt}
\subsection{Phase transitions }
\vspace*{-5pt}

Fig.~\ref{fig:recovery} panels (e-i) and (e-ii), shows the phase transition in number of samples with respect to the size of the dictionary \[m\].
%we present the phase transition analysis in sample complexity \[p\] and size of dictionary \[m\] for recovery of the dictionary and the coefficients, respectively, averaged across \[10\] Monte Carlo simulations. We generate the data similar to the convergence analysis experiments but with $n=100$, \[k =3\], \[\eta_x = 0.2\], \[\tau = 0.1\], \[\epsilon_0 =  2/\log(n)\], \[\eta_A\] is chosen proportional to \[m/k\] as per \ref{assumption:step dict}. In addition, we consider a trial successful if the relative Frobenius error incurred by \[\hat{\b{A}}\] and \[\hat{\b{X}}\] is below \[5\times10^{-7}\] after \[50\] iterations.
We observe a sharp phase transition at \[\tfrac{p}{m} = 1\] for the dictionary, and at \[\tfrac{p}{m} = 0.75\] for the coefficients. This phenomenon is similar one observed by \cite{Agarwal14} (however, theoretically they required \[p = \mathcal{O}(m^2)\]). Here, we confirm number of samples required by NOODL are linearly dependent on the dictionary elements \[m\].
%indeed f\[p = \mathcal{O}(mk^2)\], and provide  NOODL, an algorithm which achieves this bound both theoretically, and experimentally. 
\vspace*{-6pt}
\section{Future Work}
\vspace{-7pt}
We consider the online DL setting in this work. We note that, empirically NOODL works for the batch setting also. However, analysis for this case will require more sophisticated concentration results, which can address the resulting dependence between iterations of the algorithm. In addition, our experiments indicate that NOODL works beyond the sparsity ranges prescribed by our theoretical results. Arguably, the bounds on sparsity can potentially be improved by moving away from the incoherence-based analysis.  We also note that in our experiments, NOODL converges even when initialized outside the prescribed initialization region, albeit it achieves the linear rate once it satisfies the closeness condition \ref{assumption:close}. These potential directions may significantly impact the analysis and development of provable algorithms for other factorization problems as well. We leave these research directions, and a precise analysis under the noisy setting, for future explorations.
\vspace*{-6pt}
\section{Conclusions}
\vspace{-7pt}
%We present NOODL, an alternating optimization-based online dictionary learning (DL) algorithm which bridges the gap between successful algorithms which do not guarantee recovery of the factors, and the current state-of-the-art provable algorithms, which incur non-negligible bias in estimation. 
%This alternating optimization-based algorithm bridges the gap between successful algorithms which do not guarantee recovery of the factors, and the current state-of-the-art provable algorithms, which incur non-negligible bias in estimation. 
We present NOODL, to the best of our knowledge, the first neurally plausible provable online algorithm for exact recovery of both factors of the dictionary learning (DL) model. NOODL alternates between: (a) an iterative hard thresholding (IHT)-based step for coefficient recovery, and (b) a gradient descent-based update for the dictionary, resulting in a simple and scalable algorithm, suitable for large-scale distributed implementations. We show that once initialized appropriately, the sequence of estimates produced by NOODL converge \textit{linearly} to the true dictionary and coefficients without incurring any bias in the estimation.  Complementary to our theoretical and numerical results, we also design an implementation of NOODL in a neural architecture for use in practical applications. 
%We pose the DL problem as a factorization task, and demonstrate -- both theoretically and through experimental validation -- that both factors can be recovered successfully by employing an iterative hard thresholding (IHT)-based coefficient update strategy in an alternating optimization-based online DL technique. Complementary to our theoretical and numerical results, we design and present a prototype implementation of NOODL in a neural architecture. 
In essence, the analysis of this inherently non-convex problem impacts other matrix and tensor factorization tasks arising in signal processing, collaborative filtering, and machine learning. %error-in-variables (EIV) models, and other factor models.

\vspace*{-20pt}
\subsubsection*{Acknowledgment}
\vspace*{-4pt}
 The authors would like to graciously acknowledge support from DARPA Young Faculty Award, Grant No. N66001-14-1-4047.
\vspace*{-4pt}
\setlength{\bibsep}{6.5pt plus 10pt}
\bibliographystyle{ims}
\bibliography{referDL_arxiv_v2}

\newpage

\appendix

\allowdisplaybreaks

\appendix
\section{Summary of Notation}\label{app:summary_notation}
We summarizes the definitions of some frequently used symbols in our analysis in Table~\ref{tab:symbols}. In addition, we use \[\b{D}_{(\vb)}\] as a diagonal matrix with elements of a vector \[\b{v}\] on the diagonal. Given a matrix $\Mb$, we use \[\b{M}_{-i}\] to denote a resulting matrix without \[i\]-th column. Also note that, since we show that \[\|\b{A}^{(t)}_i - \b{A}^*_i\| \leq \epsilon_t\] contracts in every step, therefore we fix  \[\epsilon_t, \epsilon_0 = \mathcal{O}^*(1/\log(n))\] in our analysis.

 \begin{table}[!b]
 \centering
 	\caption{Frequently used symbols}
 	\label{tab:symbols}
 \begin{minipage}{\textwidth}
  \centering
 		\resizebox{0.79\columnwidth}{!}{
 		\begin{tabular}{c|l|p{5cm}}
 		%\multicolumn{2}{|l|}{\textbf{Claims}}\\ \hline
 		\multicolumn{3}{l}{\textbf{Dictionary Related}}\\ \hline
 		\textbf{Symbol} &\multicolumn{2}{l}{\textbf{Definition}}\\ 
 		\hline
 		\[\b{A}^{(t)}_i \] & \multicolumn{2}{l}{\[i\]-th column of the dictionary estimate at the \[t\]-th iterate.}\\ \hline
 		\[\epsilon_t\] & \[\|\b{A}^{(t)}_i - \b{A}^*_i\| \leq \epsilon_t = \mathcal{O}^*(\tfrac{1}{\log(n)})\] & Upper-bound on column-wise error at the \[t\]-th iterate.\\ 
 		\hline
 	   \[\mu_t\] & \[\tfrac{\mu_t}{\sqrt{n}} = \tfrac{\mu}{\sqrt{n}} + 2\epsilon_t\] & Incoherence between the columns of  \[\b{A}^{(t)}\]; See Claim~\ref{claim:incoherence of B}. \\ 
 	   \hline
 	    \[\lambda^{(t)}_j \] & \[\lambda^{(t)}_j := |\langle \b{A}^{(t)}_j - \b{A}^{*}_j, \b{A}^*_j\rangle| \leq \tfrac{\epsilon_t^2}{2}\]  & Inner-product between the error and the dictionary element.\\ 
 	    \hline
 	     \[\Lambda^{(t)}_S (i,j)\] & \[\Lambda^{(t)}_S (i,j) = \begin{cases} \lambda^{(t)}_{j}, &\text{for}~ j = i, i \in S\\ 0, &\text{otherwise}. 
 	     \end{cases}\] & A diagonal matrix of size \[|S| \times |S|\] with \[\lambda^{(t)}_{j}\] on the diagonal for \[j \in S\].\\ \hline
 	     \multicolumn{3}{l}{}\\ 
		\multicolumn{3}{l}{\textbf{Coefficient Related}}\\ \hline
 		 \textbf{Symbol} &\multicolumn{2}{l}{\textbf{Definition}} \\ \hline
 		 \[\b{x}^{(r)}_i \] & \multicolumn{2}{l}{\[i\]-th element the coefficient estimate at the \[r\]-th IHT iterate.}\\ \hline
 		\[C\] & \[|\b{x}^*_i|\geq C\] for \[i \in \supp(\b{x}^*)\] and \[C \leq 1\] & Lower-bound on \[\b{x}^*_i\]s. \\ \hline
 		 \[S\] & \[S : = \supp(\b{x}^*)\] where \[|S| \leq k\] & Support of \[\b{x}^*\]\\ \hline
 		\[\delta_{R}\] & \[ \delta_{R}:=(1 - \eta_{x} + \eta_x\tfrac{\mu_t}{\sqrt{n}})^{R} \geq (1 - \eta_{x})^{R}\] & Decay parameter for coefficients.\\ \hline
 		\[\delta_T\] & \[|\hat{\b{x}}_{i}^{(T)} - \b{x}_{i}^*| \leq \delta_T \forall i\in \supp(\b{x}^*)\] & Target coefficient element error tolerance. \\ \hline
 		\[C_{i}^{(\ell)}\] & \[C_{i}^{(\ell)} := |\b{x}_{i}^* - \b{x}_{i}^{(\ell)}|\] for \[i \in \supp(\b{x}^*)\] & Error in non-zero elements of the coefficient vector.\\ \hline
 		 \end{tabular}}
 		 \end{minipage}
 		 \begin{minipage}{\textwidth}
 		 \centering
 		  \resizebox{0.88\columnwidth}{!}{
 		  \begin{tabular}{c|l|c|l}
 		  \multicolumn{4}{l}{}\\ 
		\multicolumn{4}{l}{\textbf{Probabilities}}\\ \hline
 		 \textbf{Symbol} &\textbf{Definition}&\textbf{Symbol} &\textbf{Definition} \\ \hline
 		\[q_i \] & \[q_i = \b{Pr}[i \in S] = \Theta(\tfrac{k}{m})\] &
 		 \[q_{i,j} \] & \[q_{i,j} = \b{Pr}[i,j \in S] = \Theta(\tfrac{k^2}{m^2})\] \\ \hline
 		 \[p_i\] & \[p_i = \b{E}[\b{x}_i^*\sgn(\b{x}^*_i)| \b{x}_i^* \neq 0]\] &
 		 \[\delta_{\HT}^{(t)}\] & \[\delta_{\HT}^{(t)} = 2m~{\exp}({-{C^2}/{\mathcal{O}^*(\epsilon_t^2)}})\]\\ \hline
 		 \[	\delta_{\beta}^{(t)}\] & \[\delta_{\beta}^{(t)} = 2k~{\exp}(-{1}/{\mathcal{O}(\epsilon_t)})\]&
 		 \[	\delta_{\rm HW}^{(t)} \] &  \[\delta_{\rm HW}^{(t)} = \exp(-{1}/{\mathcal{O}(\epsilon_t)})\]\\ \hline
		 \[	\delta_{\gradvec}^{(t)}\] & \[\delta_{\gradvec}^{(t)} =  \exp(-\Omega(k))\] &
 		 \[\delta_{\gradmat}^{(t)} \] & \[\delta_{\gradmat}^{(t)} = (n+m)\exp(-\Omega(m\sqrt{\log(n)})\]\\ \hline
 		  	\end{tabular}	\vspace*{10pt}}
 		  	\end{minipage}
 		  \begin{minipage}{\textwidth}
 		   \centering
 		  \resizebox{0.88\columnwidth}{!}{
 		  \begin{threeparttable}[h]
 		  \begin{tabular}{c|l|l}
 		  	\multicolumn{3}{l}{}\\ 
 		\multicolumn{3}{l}{\textbf{Other terms}}\\ \hline
 		 \textbf{Symbol} &\multicolumn{2}{l}{\textbf{Definition}} \\ \hline
 	    \[\xi^{(r+1)}_j\] & \multicolumn{2}{l}{\[\xi^{(r+1)}_{j}  := \textstyle\sum\limits_{i \neq j}( \langle \b{A}^{(t)}_j - \b{A}^{*}_j, \b{A}^*_i\rangle +  \langle \b{A}^{*}_j, \b{A}^*_i\rangle) \b{x}^*_i -\textstyle\sum\limits_{i \neq j} \langle \b{A}^{(t)}_j, \b{A}^{(t)}_i\rangle \b{x}_i^{(r)}\]} \\\hline
 	    \[\beta^{(t)}_j\] & \multicolumn{2}{l}{\[\beta^{(t)}_j := \textstyle\sum\limits_{i \neq j} (\langle \b{A}^{*}_j,  \b{A}^*_i - \b{A}^{(t)}_i\rangle  + \langle \b{A}^*_j -\b{A}^{(t)}_j, \b{A}^{(t)}_i\rangle + \langle \b{A}^{(t)}_j - \b{A}^{*}_j, \b{A}^*_i\rangle) \b{x}_i^*\] }\\ \hline
 	    \[t_\beta \]& \multicolumn{2}{l}{\[t_\beta = \mathcal{O}(\sqrt{k\epsilon_t})\] is an upper-bound on \[\beta^{(t)}_j\]  with probability  at least \[(1 -	\delta_{\beta}^{(t)})\]} \\ \hline
 	    \[\tilde{\xi}^{(r+1)}_{j}\]  & \multicolumn{2}{l}{\[	\tilde{\xi}^{(r+1)}_{j}  := \beta^{(t)}_j + \textstyle\sum\limits_{i \neq j}|\langle \b{A}^{(t)}_j, \b{A}^{(t)}_i\rangle| ~|\b{x}_i^{*} - \b{x}_i^{(r)}|\]} \\ \hline
	    \[\Delta^{(t)}_j \] & \multicolumn{2}{l}{\[\Delta^{(t)}_j := \b{E}[\b{A}^{(t)}_S\vartheta^{(R)}_S\sgn(\b{x}^*_j)]\]} \\ \hline
 	    \[\vartheta^{(R)}_{i}\] & \multicolumn{2}{l}{\[\vartheta^{(R)}_{i} :=  \textstyle\sum\limits_{r = 1}^{R }\eta_{x}\xi^{(r)}_{i} (1- \eta_x)^{{R} - r} + \gamma^{(R)}_{i}\]} \\ \hline
 	    \[\gamma^{(R)}_{i}\] & \multicolumn{2}{l}{\[\gamma^{(R)}_{i} :=(1 - \eta_{x})^{R} (\b{x}_{i}^{(0)}  - \b{x}_{i}^* (1 - \lambda^{(t)}_{i}))\]} \\ \hline
 	    \[\gamma\] & \multicolumn{2}{l}{\[\gamma:=\mathbf{E}[(\b{A}^{(t)}\hat{\b{x}} - \b{y})\sgn(\b{x}^*_j)\mathbbm{1}_{\overline{\mathcal{F}}_{\b{x}^*}}]\]; See \[\dagger\] below.} \\ \hline
 	    \[\hat{\b{x}}_{i}\] & \multicolumn{2}{l}{\[\hat{\b{x}}_{i} :=	\b{x}_{i}^{(R)} =  \b{x}_{i}^* (1 - \lambda^{(t)}_{i}) + \vartheta^{(R)}_{i}\]}\\ \hline
 	\end{tabular}
 	\begin{tablenotes}
 	\item[$\dagger$] \[\mathbbm{1}_{\mathcal{F}_{\b{x}^*}}\] is the indicator function corresponding to the event that \[\sgn(\b{x}^*) = \sgn(\hat{\b{x}})\], denoted by \[\mathcal{F}_{\b{x}^*}\], and similarly for the complement \[\overline{\mathcal{F}_{\b{x}^*}}\]
 	\end{tablenotes}
 	\end{threeparttable}}
 	\end{minipage}
 \end{table}

\vspace*{-5pt}
\section{Proof of Theorem~\ref{main_result}}\label{app:pf_main}
We now prove our main result. The detailed proofs of intermediate lemmas and claims are organized in Appendix~\ref{app:proof of main lemmas} and Appendix~\ref{app:proof of intermediate results}, respectively. Furthermore, the standard concentration results are stated in Appendix~\ref{useful results} for completeness. Also, see Table~\ref{tab:res dependence} for a map of dependence between the results.
%\section{Appendix: Proof of Theorem~\ref{main_result}}\label{app:proof of main result}
%\addcontentsline{toc}{chapter}{\protect\numberline{}Section 1}%
\subsection*{Overview}
Given an \[(\epsilon_0, 2)\]-close estimate of the dictionary, the main property that allows us to make progress on the dictionary is the recovery of the correct sign and support of the coefficients. Therefore, we first show that the initial coefficient estimate \eqref{alg:coeff_init} recovers the correct signed-support in \hyperref[Step I.A]{Step I.A}. Now, the IHT-based coefficient update step also needs to preserve the correct signed-support. This is to ensure that the approximate gradient descent-based update for the dictionary makes progress. Therefore, in \hyperref[Step I.B]{Step I.B}, we derive the conditions under which the signed-support recovery condition is preserved by the IHT update.

To get a handle on the coefficients, in \hyperref[Step II.A]{Step II.A}, we derive an upper-bound on the error incurred by each non-zero element of the estimated coefficient vector, i.e., \[|\hat{\b{x}}_{i} - \b{x}_{i}^*|\] for \[i \in S\] for a general coefficient vector \[\b{x}^*\], and show that this error only depends on \[\epsilon_t\] (the column-wise error in the dictionary) given enough IHT iterations \[R\] as per the chosen decay parameter \[\delta_R\]. In addition, for analysis of the dictionary update, we develop an expression for the estimated coefficient vector in \hyperref[Step II.B]{Step II.B}.

We then use the coefficient estimate to show that the gradient vector satisfies the local descent condition (Def.~\ref{def:grad_alpha_beta}). This ensures that the gradient makes progress after taking the gradient descent-based step \eqref{eq:apx_grad}. To begin, we first develop an expression for the expected gradient vector (corresponding to each dictionary element) in \hyperref[Step III.A]{Step III.A}. Here, we use the closeness property Def~\ref{def:del_kappa} of the dictionary estimate. Further, since we use an empirical estimate, we show that the empirical gradient vector concentrates around its mean in \hyperref[Step III.B]{Step III.B}. Now using Lemma~\ref{arora:thm40}, we have that descent along this direction makes progress. 

Next in \hyperref[Step IV.A]{Step IV.A} and \hyperref[Step IV.B]{Step IV.B}, we show that the updated dictionary estimate maintains the closeness property Def~\ref{def:del_kappa}. This sets the stage for the next dictionary update iteration. As a result, our main result establishes the conditions under which any \[t\]-th iteration succeeds. 

Our main result is as follows.

\noindent\textbf{Theorem~\ref{main_result}} (Main Result) \textit{Suppose that assumptions \ref{assumption:mu}-\ref{assumption:step coeff}  hold, and Algorithm~\ref{alg:main_alg} is provided with  \[ p = \tilde{\Omega}(mk^2)\] new samples generated according to model \eqref{eq:model} at each iteration \[t\]. %then for \[\eta_A = \Theta(m/k)\] 
Then,  with probability at least \[(1 - \delta_{\text{alg}}^{(t)})\], given \[R = \Omega({\rm log}(n))\], the coefficient estimate \[\hat{\b{x}}_{i}^{(t)}\] at \[t\]-th iteration has the correct signed-support and satisfies
\vspace{-2pt}
\begin{align*}
(\hat{\b{x}}_{i}^{(t)} - \b{x}_{i}^*)^2 %C_{i_1}^{(R)} 
&= \mathcal{O}(k(1 - \omega)^{t/2}\|\b{A}_i^{(0)} - \b{A}_i^*\|), ~\text{for all}~i \in \supp({\b{x}^*}).\vspace{-2pt}
%(c_x k + 1)(\tfrac{\epsilon_t^2}{2}|\b{x}_{\max}^*|+ t_\beta) + 2\delta_{R} k\eta_x\tfrac{\mu_t}{\sqrt{n}},\\
%& \leq c + (c_x k + 1)\tfrac{\epsilon_t^2}{2}|\b{x}_{\max}^*|
\end{align*} 
Furthermore, for some \[0 < \omega < 1/2\], the estimate \[\b{A}^{(t)}\] at \[(t)\]-th iteration satisfies 
\vspace{-2pt}
\begin{align*}
\|\b{A}_i^{(t)} - \b{A}_i^*\|^2 \leq (1 - \omega)^t\|\b{A}_i^{(0)} - \b{A}_i^*\|^2,~\text{for all}~t = 1,2,\ldots.\vspace{-2pt} .
\end{align*}
Here, \[\delta_{\text{alg}}^{(t)}\] is some small constant, where \[\delta_{\text{alg}}^{(t)} =  \delta_{\HT}^{(t)}  + \delta_{\beta}^{(t)} + \delta_{\rm HW} +\delta_{\gradvec}^{(t)} + \delta_{\gradmat}^{(t)}\], \[\delta_{\HT}^{(t)} = 2m~{\exp}({-{C^2}/{\mathcal{O}^*(\epsilon_t^2)}})\], \[\delta_\beta^{(t)} = 2k~{\exp}(-{1}/{\mathcal{O}(\epsilon_t)})\], \[\delta_{\rm HW}^{(t)} = \exp(-{1}/{\mathcal{O}(\epsilon_t)})\], \[\delta_{\gradvec}^{(t)} =  \exp(-\Omega(k))\], \[\delta_{\gradmat}^{(t)} = (n+m)\exp(-\Omega(m\sqrt{\log(n)})\], and \[\|\b{A}_i^{(t)} -\b{A}_i^*\| \leq \epsilon_{t}\].} %\[{T} = \Omega(\log(1/\epsilon_t))\] \[\|\b{A}^*\| = \mathcal{O}(\sqrt{m/n})\]

\subsection*{Step I: Coefficients have the correct signed-support }
As a first step, we ensure that our coefficient estimate has the correct signed-support (Def.~\ref{def:signed-support}). To this end, we first show that the initialization has the correct signed-support, and then show that the iterative hard-thresholding (IHT)-based update step preserves the correct signed-support for a suitable choice of parameters. 
\begin{itemize}
\item\textbf{Step I.A: Showing that the initial coefficient estimate has the correct signed-support--} \label{Step I.A}
Given an \[(\epsilon_0, 2)\]-close estimate \[\b{A}^{(0)}\] of \[\b{A}^*\], we first show that for a general sample \[\b{y}\] the initialization step \eqref{alg:coeff_init}
%\vspace{-10pt}
%\begin{align}\label{eq_iht:hard_thr_step}
%\b{x}^{(0)} = \HT_{C/2}(\b{A}^{(t)^\top}\b{y}),
%\end{align}
%
recovers the correct signed-support 
%\begin{align*}
% \sgn(\HT_{C/2}(\b{A}^{(t)^\top}\b{y})) = \sgn(\b{x}^{*})
%\end{align*}
%%
with probability at least \[(1 - \delta_{\HT}^{(t)})\], where \[\delta_{\HT}^{(t)} = 2m~{\exp}({-\tfrac{C^2}{\mathcal{O}^*(\epsilon_t^2)}})\]. This is encapsulated by the following lemma. 

%\noindent\textbf{Lemma~\ref{lem:recover_sign}} \textbf{(Signed-support recovery by coefficient initialization step).}
%\begin{em}
%	Suppose \[\b{A}^{(t)}\] is \[\epsilon_t\]-close to \[\b{A}^*\]. Then, if \[\mu = \c{O}(\log(n))\], \[k = \sqrt{n}/\mu\log(n)\], and \[\epsilon_t = \mathcal{O}^*(1/\sqrt{\log(m)})\], with probability at least \[(1 - \delta_{\HT})\] for each random sample \[\b{y} = \b{A}^*\b{x}^*\]:
%	\begin{align*}
%	\sgn(\HT_{C/2}((\b{A}^{(t)})^\top \b{y}) = \sgn(\b{x}^*),
%	\end{align*}
%	%
%	where \[\delta_{\HT} = 2m~{\exp}({-\tfrac{C^2}{\mathcal{O}^*(\epsilon_t^2)}})\].
%\end{em}
\begin{lemma} [\textbf{Signed-support recovery by coefficient initialization step}]\label{lem:recover_sign}
	Suppose \[\b{A}^{(t)}\] is \[\epsilon_t\]-close to \[\b{A}^*\]. Then, if \[\mu = \c{O}(\log(n))\], \[k = \sqrt{n}/\mu\log(n)\], and \[\epsilon_t = \mathcal{O}^*(1/\sqrt{\log(m)})\], with probability at least \[(1 - \delta_{\HT}^{(t)})\] for each random sample \[\b{y} = \b{A}^*\b{x}^*\]:
	\begin{align*}
	\sgn(\HT_{C/2}((\b{A}^{(t)})^\top \b{y}) = \sgn(\b{x}^*),
	\end{align*}
	where \[\delta_{\HT}^{(t)} = 2m~{\exp}({-\tfrac{C^2}{\mathcal{O}^*(\epsilon_t^2)}})\].
\end{lemma}
%\begin{lemma} [\textbf{Signed-support recovery by coefficient initialization step}]\label{lem:recover_sign}
%	Suppose \[\b{A}^{(t)}\] is \[\epsilon_t\]-close to \[\b{A}^*\]. Then, if \[\tfrac{\mu}{\sqrt{n}} \leq \tfrac{1}{2k}\], \[k = \Omega^*(\log (m))\], and \[\epsilon_t = \mathcal{O}^*(1/\sqrt{\log(m)})\], with probability at least \[(1 - \delta_{\HT}^{(t)})\] for each random sample \[\b{y} = \b{A}^*\b{x}^*\]:
%	\begin{align*}
%	\sgn(\HT_{C/2}((\b{A}^{(t)})^\top \b{y}) = \sgn(\b{x}^*),
%	\end{align*}
%	%
%	where \[\delta_{\HT}^{(t)} = 2m~{\exp}({-\tfrac{C^2}{\mathcal{O}^*(\epsilon_t^2)}})\].
%\end{lemma}
Note that this result only requires the dictionary to be column-wise close to the true dictionary, and works for less stringent conditions on the initial dictionary estimate, i.e., requires \[\epsilon_t = \mathcal{O}^*(1/\sqrt{\log(m)})\] instead of \[\epsilon_t = \mathcal{O}^*(1/\log(m))\]; see also \citep{Arora15}.

 % Next we show the conditions under which the IHT-based updates preserve the correct signed-support.

\item\textbf{Step I.B: The iterative IHT-type updates preserve the correct signed support--}\label{Step I.B} Next, we show that the IHT-type coefficient update step \eqref{alg:coeff_iht} preserves the correct signed-support for an appropriate choice of step-size parameter \[\eta_x^{(r)}\] and threshold \[\tau^{(r)}\]. The choice of these parameters arises from the analysis of the  IHT-based update step. Specifically, we show that at each iterate \[r\], the step-size \[\eta_x^{(r)}\] should be chosen to ensure that the component corresponding to the true coefficient value is greater than the ``interference'' introduced by other non-zero coefficient elements. Then, if the threshold is chosen to reject this ``noise'', each iteration of the IHT-based update step preserves the correct signed-support.
%\begin{align}\label{supp:coeff_iht}
%\b{x}^{(r+1)} = \HT_{\tau^{(r)}}(\b{x}^{(r)} - \eta_{x}^{(r)}~\b{A}^{(t)^\top}(\b{A}^{(t)}\b{x}^{(r)} - \b{y})).
%\end{align}

\begin{lemma}[\textbf{IHT update step preserves the correct signed-support}]\label{our:signed_supp}
	Suppose \[\b{A}^{(t)}\] is \[\epsilon_t\]-close to \[\b{A}^*\],  \[\mu = \c{O}(\log(n))\], \[k = \sqrt{n}/\mu\log(n)\], and \[\epsilon_t = \mathcal{O}^*(1/\log(m))\] Then, with probability at least \[(1 - \delta_{\beta}^{(t)} - \delta_{\HT}^{(t)}  )\], each iterate of the IHT-based coefficient update step shown in \eqref{alg:coeff_iht} has the correct signed-support, if for a constant \[c^{(r)}_1(\epsilon_t, \mu, k, n) = \tilde{\Omega}({k^2}/{n})\], the step size is chosen as \[\eta_x^{(r)}\leq c^{(r)}_1\] ,
	%is chosen as
%	
%	\begin{align*}
%	\eta_x^{(r)} \leq c^{(r)} \tfrac{\eta_x^{(r-1)}\|\b{x}^{(r-2)} - \b{x}^*\|_1}{\|\b{x}^{(r-1)} - \b{x}^*\|_1 + \eta_x^{(r-1)}\|\b{x}^{(r-2)} - \b{x}^*\|_1},
%	\end{align*}
%	
	and the threshold \[\tau^{(r)}\] is chosen as
	\begin{align*}
	\tau^{(r)} = \eta_x^{(r)}(t_\beta + \tfrac{\mu_t}{\sqrt{n}} \|\b{x}^{(r-1)} - \b{x}^*\|_1) :=c_2^{(r)}(\epsilon_t, \mu, k, n) = \tilde{\Omega}({k^2}/{n}),
	\end{align*}
	for some constants \[c_1\] and \[c_2\]. Here, \[t_\beta = \mathcal{O}(\sqrt{k\epsilon_t})\], \[\delta_{\HT}^{(t)} = 2m~{\exp}({-\tfrac{C^2}{\mathcal{O}^*(\epsilon_t^2)}})\] ,and \[\delta_{\beta}^{(t)} = 2k~{\exp}(-\tfrac{1}{\mathcal{O}(\epsilon_t)})\].
\end{lemma}
Note that, although we have a dependence on the iterate \[r\] in choice of \[\eta_x^{(r)}\] and \[\tau^{(r)}\], these can be set to some constants independent of \[r\]. In practice, this dependence allows for greater flexibility in the choice of these parameters.
\end{itemize}
\subsection*{Step II: Analyzing the coefficient estimate}
We now derive an upper-bound on the error incurred by each non-zero coefficient element. Further, we derive an expression for the coefficient estimate at the \[t\]-th round of the online algorithm \[\hat{\b{x}}^{(t)}: =\b{x}^{(R)}\]; we use \[\hat{\b{x}}\] instead of \[\hat{\b{x}}^{(t)}\] for simplicity.
%Let  \[\hat{\b{x}}^{(t)}: =\b{x}^{(R)}\] denote the coefficient estimate, i.e., the \[R\]-th iterate of the iterative IHT update step at the \[t\]-th round of the online algorithm. 
\begin{itemize}
\item\textbf{Step II.A: Derive a bound on the error incurred by the coefficient estimate-- }\label{Step II.A}
Since Lemma~\ref{our:signed_supp} ensures that \[\hat{\b{x}}\] has the correct signed-support, we now focus on the error incurred by each coefficient element on the support by analyzing \[\hat{\b{x}}\]. To this end, we carefully analyze the effect of the recursive update \eqref{alg:coeff_iht}, to decompose the error incurred by each element on the support into two components -- one  that depends on the initial coefficient estimate \[\b{x}^{(0)}\] and other that depends on the error in the dictionary. 

We show that the effect of the component that depends on the initial coefficient estimate diminishes by a factor of \[(1 - \eta_{x} + \eta_x\tfrac{\mu_t}{\sqrt{n}})\] at each iteration \[r\]. Therefore, for a decay parameter \[\delta_R\], we can choose the number of IHT iterations \[R\], to make this component arbitrarily small. Therefore, the error in the coefficients only depends on the per column error in the dictionary, formalized by the following result.
%The error in each coefficient element \[\hat{\b{x}}_{i_1} : =\b{x}^{(R)}_{i_1}\] for \[i_1 \in S\] is given by Lemma~\ref{iht:x_R_error}.
 \begin{lemma}[\textbf{Upper-bound on the error in coefficient estimation}]\label{iht:x_R_error}
 	With probability at least \[(1 - \delta_{\beta}^{(t)} - \delta_{\HT}^{(t)})\] the error incurred by each element \[i_1 \in \supp(\b{x}^*)\] of the coefficient estimate is upper-bounded as
 	\begin{align*}
 	%|\b{x}_{i_1}^{(R)} - \b{x}_{i_1}^*| %C_{i_1}^{(R)} 
 	%&\leq   (c_x k + 1)(\tfrac{\epsilon_t^2}{2}|\b{x}_{\max}^*|+ t_\beta) + 2\delta_{R} k\eta_x\tfrac{\mu_t}{\sqrt{n}},\\
 	|\hat{\b{x}}_{i_1} - \b{x}_{i_1}^*|
 %	&\leq   (c_x k + 1)(\tfrac{\epsilon_t^2}{2}|\b{x}_{\max}^*|+ t_\beta) + {(R + 1)}k \eta_x\tfrac{\mu_t}{\sqrt{n}}\underset{i}{\max}|\b{x}_{i}^{(0)} - \b{x}_{i}^*|\delta_{R} + |\b{x}_{i_1}^{(0)} - \b{x}_{i_1}^*| \delta_{R}.
 	&\leq    \mathcal{O}(t_\beta) + \left({(R + 1)}k \eta_x\tfrac{\mu_t}{\sqrt{n}}~\underset{i}{\max}|\b{x}_{i}^{(0)} - \b{x}_{i}^*| + |\b{x}_{i_1}^{(0)} - \b{x}_{i_1}^*| \right)\delta_{R},
 	%&\leq \tilde{c_x}k\tfrac{\epsilon_t^2}{2}|\b{x}_{\max}^*|  + 2\delta_{R} k\eta_x\tfrac{\mu_t}{\sqrt{n}}.
 	=  \mathcal{O}(t_\beta)
 	\end{align*}
 	where \[t_\beta = \mathcal{O}(\sqrt{k\epsilon_t})\], \[\delta_{R} := (1 - \eta_{x} + \eta_x\tfrac{\mu_t}{\sqrt{n}})^{R}\], \[\delta_{\HT}^{(t)} = 2m~{\exp}({-\tfrac{C^2}{\mathcal{O}^*(\epsilon_t^2)}})\], \[\delta_{\beta}^{(t)} = 2k~{\exp}(-\tfrac{1}{\mathcal{O}(\epsilon_t)})\], and \[\mu_t\] is the incoherence between the columns of \[\b{A}^{(t)}\]; see Claim~\ref{claim:incoherence of B}.
 \end{lemma}
This result allows us to show that if the column-wise error in the dictionary decreases at each iteration \[t\], then the corresponding estimates of the coefficients also improve.

\item\textit{Step II.B: Developing an expression for the coefficient estimate--}\label{Step II.B}
Next, we derive the expression for the coefficient estimate in the following lemma. This expression is used to analyze the dictionary update.
%This serves as the our estimate of each non-zero coefficient element, and will be used to form as estimate for the coefficient vector.

 \begin{lemma}[\textbf{Expression for the coefficient estimate at the end of \[R\]-th IHT iteration}]\label{iht:R_th_term}
 	With probability at least \[(1 - \delta_\HT^{(t)} - \delta_{\beta}^{(t)})\] the \[i_1\]-th element of the coefficient estimate, for each \[i_1 \in \supp(\b{x}^*)\], is given by
 	\begin{align*}
 	\hat{\b{x}}_{i_1} := \b{x}_{i_1}^{(R)} =  \b{x}_{i_1}^* (1 - \lambda^{(t)}_{i_1}) + \vartheta^{(R)}_{i_1}.
 	\end{align*}
 	Here, \[\vartheta^{(R)}_{i_1}\] is \[|\vartheta^{(R)}_{i_1}| = \mathcal{O}(t_{\beta})\], 
 	where \[t_\beta = \mathcal{O}(\sqrt{k\epsilon_t})\]. Further, \[\lambda^{(t)}_{i_1} = |\langle \b{A}^{(t)}_{i_1}- \b{A}^{*}_{i_1}, \b{A}^*_{i_1}\rangle| \leq \tfrac{\epsilon_t^2}{2}\], \[\delta_{\HT}^{(t)} = 2m~{\exp}({-\tfrac{C^2}{\mathcal{O}^*(\epsilon_t^2)}})\] and \[\delta_{\beta}^{(t)} = 2k~{\exp}(-\tfrac{1}{\mathcal{O}(\epsilon_t)})\].
 \end{lemma}
 We again observe that the error in the coefficient estimate depends on the error in the dictionary via \[\lambda^{(t)}_{i_1}\] and \[\vartheta^{(R)}_{i_1}\]. 
\end{itemize}

\subsection*{Step III: Analyzing the gradient for dictionary update }
Given the coefficient estimate we now show that the choice of the gradient as shown in \eqref{eq:emp_grad_est} makes progress at each step. To this end, we analyze the gradient vector corresponding to each dictionary element to see if it satisfies the local descent condition of Def.~\ref{def:grad_alpha_beta}. Our analysis of the gradient is motivated from \cite{Arora15}. However, as opposed to the simple HT-based coefficient update step used by \cite{Arora15}, our IHT-based coefficient estimate adds to significant overhead in terms of analysis. Notwithstanding the complexity of the analysis, we show that this allows us to remove the bias in the gradient estimate.

 To this end, we first develop an expression for each expected gradient vector, show that the empirical gradient estimate concentrates around its mean, and finally show that the empirical gradient vector is \[(\Omega(k/m), \Omega(m/k),0 )\]-correlated with the descent direction, i.e. has no bias.

%In this section, we analyze the approximate gradient descent based dictionary update step. For this, we need that our each gradient vector is correlated with the descent direction. To this end, we first develop an expression for each expected gradient vector, show that the empirical gradient estimate concentrates around its mean, and finally show that the empirical gradient vector is correlated with the descent direction.

%First, the following lemma gives us  the \[R\]-th iterate of the coefficient update step.   
%\begin{lemma}[\hl{Probabilistic}]\label{iht:R_th_term}
%	With probability \[(1 - \delta_\HT^{(t)})\]
%	
%	\begin{align*}
%	%\b{x}_{i_1}^{(R)} 
%	%&\approxeq \b{x}_{i_1}^* (1 - \lambda^{(t)}_{i_1}) + \vartheta^{(R)}_{i_1}.
%	\b{x}_{i_1}^{(R)} 
%	&=  (1 - \eta_{x})^{R} \b{x}_{i_1}^{(0)}  + \b{x}_{i_1}^* (1 - \lambda^{(t)}_{i_1}) ( 1- (1 - \eta_{x})^{R}) + \vartheta^{(R)}_{i_1}.
%	\end{align*}
%	
%	Here, 
%	
%	\begin{align*}
%	|\vartheta^{(R)}_{i_1}| 
%	& \leq  \mathcal{O}(k^2\tfrac{\mu_t}{\sqrt{n}}\epsilon^2_t|\b{x}_{\max}^*| ),
%	\end{align*}
%	
%	with probability \[(1-\delta_{\beta}^{(t)})\]. Further, \[\delta_{\HT}^{(t)} = 2m{\rm exp}({-\tfrac{C^2}{O(1/\log(m))}})\] and \[\delta_{\beta}^{(t)} = 2k{\rm exp}(-\tfrac{t_{\beta}^2}{18k\epsilon_t^2})\].
%\end{lemma}
%Further, the upper-bound on \[\vartheta^{(R)}_{i_1}\] is established by the following lemma.
\begin{itemize}
\item\textbf{Step III.A: Develop an expression for the expected gradient vector corresponding to each dictionary element--}\label{Step III.A}
The expression for the expected gradient vector \[\b{g}^{(t)}_j \] corresponding to \[j\]-th dictionary element is given by the following lemma.
\begin{lemma}[\textbf{Expression for the expected gradient vector}]\label{grad_exp}
	Suppose that \[\b{A}^{(t)}\] is \[(\epsilon_t, 2)\]-near to \[\b{A}^*\]. Then, the dictionary update step in Algorithm~\ref{alg:main_alg} amounts to the following for the \[j\]-th dictionary element
	\vspace{-5pt}
		\begin{align*}
		\b{E}[\b{A}^{(t+1)}_j] = \b{A}^{(t)}_j + \eta_A \b{g}^{(t)}_j,
		\vspace{-5pt}
		\end{align*}	
		where \[\b{g}^{(t)}_j\] is given by
			\vspace{-5pt}
		\begin{align*}
		\b{g}^{(t)}_j = q_j p_j \big((1 - \lambda^{(t)}_j)\b{A}^{(t)}_j- \b{A}^*_j + \tfrac{1}{q_j p_j}\Delta^{(t)}_j \pm \gamma\big),
		\end{align*}
		%
%
%	Then, with probability \[(?)\] the expected value of the \[j\]-th column of the gradient vector for dictionary update step at the \[t\]-th iterate in Algorithm~\ref{alg:main_alg} , denoted by \[\b{g}^{(t)}_j \] is given by
%	
%	
%	\begin{align*}
%	\b{g}^{(t)}_j
%	&= q_j p_j \big((1 - \lambda^{(t)}_j)\b{A}^{(t)}_j- \b{A}^*_j + \tfrac{1}{q_j p_j}\Delta^{(t)}_j \pm \gamma\big),
%	\end{align*}
%	%
	\[\lambda^{(t)}_{j}=|\langle \b{A}^{(t)}_j - \b{A}^{*}_j, \b{A}^*_j\rangle|\], and \[\Delta^{(t)}_j:=\b{E}[\b{A}^{(t)}_S\vartheta^{(R)}_S\sgn(\b{x}^*_j)]\], where \[\|\Delta^{(t)}_{j}\|= \mathcal{O}(\sqrt{m}q_{i,j}p_{j}\epsilon_t\|\b{A}^{(t)}\|)\].
	%
%	\textcolor{red}{Here,\[\delta_{\HT}^{(t)} =?\], \[\delta_{\beta} =?\], \[\delta_{\gradvec} =?\].}
\end{lemma}

\item\textbf{Step III.B: Show that the empirical gradient vector concentrates around its expectation--} \label{Step III.B} Since we only have access to the empirical gradient vectors, we show that these concentrate around their expected value via the following lemma.

\begin{lemma}[\textbf{Concentration of the empirical gradient vector}]\label{lem:grad_vec_concentrates}
Given \[p = \tilde{\Omega}(mk^2)\] samples, the empirical gradient vector estimate corresponding to the \[i\]-th dictionary element, \[\hat{\b{g}}_i^{(t)}\] concentrates around its expectation, i.e.,
		\begin{align*}
		\|\hat{\b{g}}_i^{(t)} - \b{g}_i^{(t)}\| \leq o(\tfrac{k}{m}\epsilon_t).
		\end{align*}
		with probability at least \[(1 -\delta_{\gradvec}^{(t)}- \delta_{\beta}^{(t)} - \delta_{\HT}^{(t)} - \delta_{\rm HW}^{(t)})\], where \[\delta_{\gradvec}^{(t)} =  \exp(-\Omega(k))\].
\end{lemma}

\item\textit{Step III.C: Show that the empirical gradient vector is correlated with the descent direction--}\label{Step III.C}
%Further, we show that the empirical gradient vector \[\hat{\b{g}}^{(t)}_j \] concentrates around its expected value \[\b{g}^{(t)}_j \] in Lemma~\ref{lem:grad_vec_concentrates}.
Next, in the following lemma we show that the empirical gradient vector \[\hat{\b{g}}^{(t)}_j \] is correlated with the descent direction. This is the main result which enables the progress in the dictionary (and coefficients) at each iteration \[t\].
%\begin{lemma}[\textbf{Correlatedness of the gradient}]\label{grad_corr}
%	Suppose \[\b{A}^{(t)}\] is \[(\epsilon, 2)\]-near to \[\b{A}^*\] and \[\eta_A \leq \mathcal{O}(m/k)\]. Then, with probability \[(1 - \delta_{\HT}^{(t)} - \delta_{\beta} - \delta_{\gradvec}^{(t)})\] the expected gradient vector \[\b{g}^{(t)}_j\] given by
%	
%	
%	\begin{align*}
%	\b{g}^{(t)}_j = q_j p_j \big((1 - \lambda^{(t)}_j)\b{A}^{(t)}_j- \b{A}^*_j + \tfrac{1}{q_j p_j}\Delta^{(t)}_j \pm \gamma\big),
%	\end{align*}
%	%
%	 is \[(\Omega(k/m), \Omega(m/k),0 )\]-correlated with \[\b{A}^*_j\], and
%	 
%	
%	\begin{align*}
%	\b{E}[\|\b{A}^{(t+1)}_j- \b{A}^*_j\|^2] \leq (1 - \rho_{\_}\eta_A)\|\b{A}^{(t)}_j- \b{A}^*_j\|^2.
%	\end{align*}
%	
%	\textcolor{red}{Here,\[\delta_{\HT}^{(t)} =?\], \[\delta_{\beta} =?\], \[\delta_{\gradvec}^{(t)} =?\].}
%\end{lemma}

\begin{lemma}[\textbf{Empirical gradient vector is correlated with the descent direction}]\label{grad_corr}
Suppose \[\b{A}^{(t)}\] is \[(\epsilon_t, 2)\]-near to \[\b{A}^*\], \[k = \mathcal{O}(\sqrt{n})\] and \[\eta_A = \mathcal{O}(m/k)\]. Then, with probability at least \[(1 - \delta_{\HT}^{(t)} - \delta_{\beta}^{(t)} - \delta_{\rm HW}^{(t)} -\delta_{\gradvec}^{(t)})\] the empirical gradient vector \[\hat{\b{g}}^{(t)}_j\] %given by
	%\begin{align*}
	%\b{g}^{(t)}_j = q_j p_j \big((1 - \lambda^{(t)}_j)\b{A}^{(t)}_j- \b{A}^*_j + \tfrac{1}{q_j p_j}\Delta^{(t)}_j \pm \gamma\big),
	%\end{align*}
	%
	 is \[(\Omega(k/m), \Omega(m/k),0 )\]-correlated with \[(\b{A}^{(t)}_j - \b{A}^*_j)\], and for any \[t \in [T]\],
	\begin{align*}
	\|\b{A}^{(t+1)}_j- \b{A}^*_j\|^2 \leq (1 - \rho_{\_}\eta_A)\|\b{A}^{(t)}_j- \b{A}^*_j\|^2.
	\end{align*}
	%	\textcolor{red}{Here, \[\delta_{\HT}^{(t)} =?\], \[\delta_{\beta} =?\], \[\delta_{\gradvec}^{(t)} =?\].}
\end{lemma}
This result ensures for at any \[t \in [T]\], the gradient descent-based updates made via \eqref{eq:emp_grad_est} gets the columns of the dictionary estimate closer to the true dictionary, i.e., \[\epsilon_{t+1} \leq \epsilon_{t}\]. Moreover, this step requires closeness between the dictionary estimate \[\b{A}^{(t)}\] and \[\b{A}^*\], in the spectral norm-sense, as per Def~\ref{def:del_kappa}. 
\end{itemize}

\subsection*{Step IV: Show that the dictionary maintains the closeness property}
As discussed above, the closeness property (Def~\ref{def:del_kappa}) is crucial to show that the gradient vector is correlated with the descent direction. Therefore, we now ensure that the updated dictionary \[\b{A}^{(t+1)}\] maintains this closeness property. Lemma~\ref{grad_corr} already ensures that \[\epsilon_{t+1} \leq \epsilon_{t}\]. As a result, we show that \[\b{A}^{(t+1)}\] maintains closeness in the spectral norm-sense as required by our algorithm, i.e., that it is still \[(\epsilon_{t+1},2)\]-close to the true dictionary. Also, since we use the gradient matrix in this analysis, we show that the empirical gradient matrix concentrates around its mean.

%Therefore, in the next step we ensure that the closeness property Def~\ref{def:del_kappa} is maintained by the update dictionary, \[\b{A}^{(t+1)}\]. 
 %Therefore, in this step we derive the conditions under which this property is maintained. In order to show this, we will use the expression for the expected gradient matrix. Again, since we only have access to the empirical gradient, we will first show that it concentrates around its expectation. Then, using this bound we will show that the closeness property is maintained. 
\begin{itemize}
\item\textit{Step IV.A: The empirical gradient matrix concentrates around its expectation}\label{Step IV.A}:  We first show that the empirical gradient matrix concentrates as formalized by the following lemma.
\begin{lemma}[\textbf{Concentration of the empirical gradient matrix}]\label{lem:grad_mat}
With probability at least \[(1 - \delta_{\beta}^{(t)} - \delta_{\HT}^{(t)} - \delta_{\rm HW}^{(t)} - \delta_\gradmat^{(t)})\], 
	\[\|\hat{\b{g}}^{(t)} -  \b{g}^{(t)}\|\] is upper-bounded by \[ \mathcal{O}^*(\tfrac{k}{ m} \|\b{A}^*\| )\], 
		where \[\delta_{\gradmat}^{(t)} = (n+m)\exp(-\Omega(m\sqrt{\log(n)})\].
\end{lemma} 
\item\textit{Step IV.B: The ``closeness'' property is maintained after the updates made using the empirical gradient estimate}\label{Step IV.B}: Next, the following lemma shows that the updated dictionary \[\b{A}^{(t+1)}\] maintains the closeness property.
\begin{lemma}[\[\b{A}^{(t+1)}\] \textbf{maintains closeness}]\label{lem:closeness}
	Suppose \[\b{A}^{(t)}\] is \[(\epsilon_t, 2)\] near to \[\b{A}^*\] with \[\epsilon_t = \mathcal{O}^*(1/\log(n))\], and number of samples used in step \[t\] is \[p = \tilde{\Omega}(mk^2)\], then with probability at least \[(1 - \delta_{\HT}^{(t)} - \delta_{\beta}^{(t)} 
	- \delta_{\rm HW}^{(t)}- \delta_{\gradmat}^{(t)})\], \[\b{A}^{(t+1)}\] satisfies \[\|\b{A}^{(t+1)} - \b{A}^*\| \leq 2\|\b{A}^*\|\].
\end{lemma}
\end{itemize}

\subsection*{Step V: Combine results to show the main result}\label{Step V}
% % % % % % % % % % % % % % % % % % % % % % % % % % Main Theorem % % % % % % % % % % % % % % % % % % % % % % % % % % % % % % % % % %

\begin{proof}[Proof of Theorem ~\ref{main_result}]
	 From Lemma~\ref{grad_corr} we have that with probability at least \[(1 - \delta_{\HT}^{(t)} - \delta_{\beta}^{(t)} - \delta_{\rm HW}^{(t)} -\delta_{\gradvec}^{(t)})\], \[\b{g}^{(t)}_j\] is \[(\Omega(k/m), \Omega(m/k),0 )\]-correlated with \[\b{A}^*_j\]. Further, Lemma~\ref{lem:closeness} ensures that each iterate maintains the closeness property. Now, applying Lemma~\ref{arora:thm40} we have that, for \[\eta_A \leq \Theta(m/k)\], with probability at least \[(1 -\delta_{\text{alg}}^{(t)} )\] any \[t \in [T]\] satisfies
	 \vspace{-5pt}
	 \begin{align*}
	 \|\b{A}^{(t)}_j- \b{A}^*_j\|^2
	 &\leq (1 - \omega)^{t}\|\b{A}^{(0)}_j- \b{A}^*_j\|^2 \leq (1 - \omega)^{t} \epsilon_0^2.
	 \end{align*}
	 where for \[0 < \omega < 1/2\] with \[\omega = \Omega(k/m)\eta_A\]. That is, the updates converge geometrically to \[\b{A}^*\]. Further, from Lemma~\ref{iht:x_R_error}, we have that the  result on the error incurred by the coefficients. Here, \[\delta_{\text{alg}}^{(t)}  = \delta_{\HT}^{(t)} + \delta_{\beta}^{(t)} + \delta_{\rm HW}^{(t)} +\delta_{\gradvec}^{(t)} + \delta_{\gradmat}^{(t)})\]. That is, the updates converge geometrically to \[\b{A}^*\]. Further, from Lemma~\ref{iht:x_R_error}, we have that the error in the coefficients only depends on the error in the dictionary, which leads us to our result on the error incurred by the coefficients. This completes the proof of our main result. %In Table~	\ref{tab:res dependence} we provide a proof map to show the inter-dependence of results.
\end{proof}
\section{Appendix: Proof of Lemmas}\label{app:proof of main lemmas}
We present the proofs of the Lemmas used to establish our main result. Also, see Table~\ref{tab:res dependence} for a map of dependence between the results, and Appendix~\ref{app:proof of intermediate results} for proofs of intermediate results.
\begin{table}[!h]
	\caption{Proof map: dependence of results.}
	\label{tab:res dependence}
	\begin{minipage}{0.5\textwidth}
	\resizebox{\textwidth}{!}{
	\begin{tabular}{c|p{4.5cm}|p{2cm}}
		\hline
		%\multicolumn{2}{|l|}{\textbf{Lemmas}}\\ \hline
		\textbf{Lemmas} &\textbf{Result} &\textbf{Dependence}\\ \hline
		Lemma~\ref{lem:recover_sign} & Signed-support recovery by coefficient initialization step & --\\ \hline
		Lemma~\ref{our:signed_supp} & IHT update step preserves the correct signed-support & Claim~\ref{claim:incoherence of B}, Lemma~\ref{lem:recover_sign}, and Claim~\ref{lem:bound_beta}\\ \hline
		Lemma~\ref{iht:x_R_error}& Upper-bound on the error in coefficient estimation & Claim~\ref{claim:incoherence of B}, Claim~\ref{lem:bound_beta}, Claim~\ref{iht:gen_term_simple}, and Claim~\ref{iht:C_i1_inter_alpha_sum} \\ \hline
		Lemma~\ref{iht:R_th_term} & Expression for the coefficient estimate at the end of \[R\]-th IHT iteration &Claim~\ref{iht:var_theta_abs}\\ \hline 
		Lemma~\ref{grad_exp} & Expression for the expected gradient vector &Lemma~\ref{iht:R_th_term} and Claim~\ref{iht:bound_vareps} \\ \hline
		Lemma~\ref{lem:grad_vec_concentrates} &Concentration of the empirical gradient vector &Claim~\ref{norm_bound_y_Ax_s} and Claim~\ref{lem:var_w_vector_grad} \\ \hline 
		Lemma~\ref{grad_corr} &Empirical gradient vector is correlated with the descent direction &Lemma~\ref{grad_exp}, Claim~\ref{iht:bound_vareps} and Lemma~\ref{lem:grad_vec_concentrates} \\ \hline
		Lemma~\ref{lem:grad_mat} & Concentration of the empirical gradient matrix &Claim~\ref{norm_bound_y_Ax_s}  and Claim~\ref{norm_exp_yAs_yAst} \\ \hline 
		Lemma~\ref{lem:closeness} & \[\b{A}^{(t+1)}\] maintains closeness &Lemma~\ref{grad_exp}, Claim~\ref{iht:bound_vareps} and Lemma~\ref{lem:grad_mat} \\ \hline 
		\end{tabular}}
		\resizebox{\textwidth}{!}{
		\begin{tabular}{c|p{4.5cm}|p{2cm}}
		%\multicolumn{2}{|l|}{\textbf{Claims}}\\ 
		\hline
		\textbf{Claims} &\textbf{Result} &\textbf{Dependence}\\ \hline
		Claim~\ref{claim:incoherence of B} & Incoherence of \[\b{A}^{(t)}\] &-- \\\hline
		Claim~\ref{lem:bound_beta}& Bound on \[\beta_j^{(t)}\]: the noise component in coefficient estimate that depends on \[\epsilon_t\]&-- \\\hline
		Claim~\ref{iht:gen_term_simple}  &Error in coefficient estimation for a general iterate \[(r+1)\] &-- \\\hline
		Claim~\ref{iht:C_i1_inter_alpha_sum} &An intermediate result for bounding the error in coefficient calculations&  Claim~\ref{lem:bound_beta}\\ \hline
		Claim~\ref{iht:var_theta_abs} &Bound on the noise term in the estimation of a coefficient element in the support& Claim~\ref{iht:gen_x_C_term} \\ \hline 
		Claim~\ref{iht:gen_x_C_term} &An intermediate result for \[\vartheta^{(R)}_{i_1}\] calculations &Claim~\ref{iht:gen_term_simple} \\ \hline 
		Claim~\ref{iht:bound_vareps} &Bound on the noise term in expected gradient vector estimate& Claim~\ref{iht:gen_x_C_term} and Claim~\ref{lem:bound_beta} \\ \hline 		
		Claim~\ref{norm_bound_y_Ax_s} & An intermediate result for concentration results&Lemma~\ref{our:signed_supp} ,Lemma~\ref{iht:R_th_term} and Claim~\ref{iht:var_theta_abs} \\ \hline
	    Claim~\ref{lem:var_w_vector_grad} &Bound on variance parameter for concentration of gradient vector& Claim~\ref{iht:var_theta_abs} \\ \hline 
	    
		Claim~\ref{norm_exp_yAs_yAst} &Bound on variance parameter for concentration of gradient matrix& Lemma~\ref{our:signed_supp} , Lemma~\ref{iht:R_th_term} and Claim~\ref{iht:var_theta_abs}\\ \hline
	\end{tabular}}
	\end{minipage}
	\begin{minipage}{0.7\textwidth}
		\resizebox{0.7\textwidth}{!}{
		\begin{tikzpicture}[node distance=2cm, auto,]
		    \node[punkt,draw=blue, thick, fill=blue!20,opacity=.8,text opacity=1] (lem1) {Lemma~\ref{lem:recover_sign}};
		    \node[punkt, draw=red, thick,fill=red!20,opacity=.8,text opacity=1, above= of lem1] (cl1) {	Claim~\ref{claim:incoherence of B}};
		    \node[punkt, draw=red, thick,fill=red!20,opacity=.8,text opacity=1, above=of cl1] (cl2) {Claim~\ref{lem:bound_beta}};
		    \node[punkt,draw=blue, thick, fill=blue!20,opacity=.8,text opacity=1,below=of lem1] (lem2) {Lemma~\ref{our:signed_supp}};
		     %\path (cl1.west) edge[red, pil, bend right=20] (lem2.165);
		%     \path (cl2.west) edge[red, pil,bend right=20] (lem2.165);
		      \path [arrow,red,pil] (cl2.195) -| ([xshift=-1cm, yshift=0cm]lem2.165) |- (lem2.165);
		       \path [arrow,red,pil] (cl1.west) -- ([xshift=-1cm, yshift=0cm]cl1.west);
		      \path (lem1.south) edge[blue,pil] (lem2.north);

		      \node[punkt, draw=red, thick,fill=red!20,opacity=.8,text opacity=1, above=of cl2] (cl3) {Claim~\ref{iht:gen_term_simple}};
		      \node[punkt, draw=red, thick,fill=red!20,opacity=.8,text opacity=1, above=of cl3] (cl4) {Claim~\ref{iht:C_i1_inter_alpha_sum}};
		      \node[punkt, draw=blue, thick,fill=blue!20,opacity=.8,text opacity=1,right= 3cm of cl2] (lem3) {Lemma~\ref{iht:x_R_error}};
		  %    \path (cl1.east) edge[red,pil,bend right=20] (lem3.west);
		       \path [arrow,red,pil] (cl1.east) -| ([xshift=1.5cm, yshift=0cm]cl1.east) |-([xshift=1.5cm, yshift=-0.1cm]cl2.east);
		      \path (cl2.east) edge[red,pil] (lem3.west);
		    %  \path (cl3.east) edge[red,pil,bend left=20] (lem3.west);
		       \path [arrow,red,pil] (cl3.east) |- ([xshift=1.5cm, yshift=0cm]cl3.east) -|([xshift=1.5cm, yshift=0cm]cl2.east);
		      %\path (cl4.east) edge[red,pil,bend left=20] (lem3.west);
		       \path [arrow,red,pil] (cl4.east) |- ([xshift=1.5cm, yshift=0cm]cl4.east) -|([xshift=1.5cm, yshift=0cm]cl3.east);
		     %  \path (cl2.west) edge[red,pil,bend left=20] (cl4.west);
		         \path [arrow,red,pil] ([xshift=-1.7cm, yshift=0cm]cl4.west) -- (cl4.west);

		      \node[punkt,draw=red, thick,fill=red!20,opacity=.8,text opacity=1,above=of cl4] (cl5) {Claim~\ref{iht:var_theta_abs} };
		      \node[punkt, draw=blue, thick,fill=blue!20,opacity=.8,text opacity=1,above=of lem3] (lem4) {Lemma~\ref{iht:R_th_term}};
		      \node[punkt, draw=blue, thick,fill=blue!20,opacity=.8,text opacity=1,above=of lem4] (lem5) {Lemma~\ref{grad_exp} };
		      \node[punkt, draw=red, thick,fill=red!20,opacity=.8,text opacity=1,above=of cl5] (cl6) {Claim~\ref{iht:gen_x_C_term}  };
		      \node[punkt, draw=red, thick,fill=red!20,opacity=.8,text opacity=1,above=of cl6] (cl7) {Claim~\ref{iht:bound_vareps}};
		   %   \path (cl3.west) edge[red,pil,bend left=20] (cl6.west);
		      \path [arrow,red,pil] (cl3.west) -| ([xshift=-0.5cm, yshift=0cm]cl6.west) |- (cl6.west);
		      \path (cl6.south) edge[red,pil] (cl5.north);
		  %    \path (cl5.east) edge[red,pil,bend left=20] (lem4.west);
		      \path [arrow,red,pil] (cl5.350) -| ([xshift=1.8cm, yshift=0cm]cl5.350) |- ([xshift=-0.5cm, yshift=0cm]lem4.195) |- (lem4.195);
		      \path (cl6.north) edge[red,pil] (cl7.south);
		     % \path (cl2.west) edge[red,pil,bend left=20] (cl7.west);
		       \path [arrow,red,pil] (cl2.165) -| ([xshift=-1.7cm, yshift=0cm]cl7.165) |- (cl7.165);
		     % \path (cl7.east) edge[red,pil,bend right=20] (lem5.west);
		      \path [arrow,red,pil] (cl7.350) -- ([xshift=0.5cm, yshift=0cm]cl7.350) |- ([xshift=-0.5cm, yshift=0cm]lem5.195) |- (lem5.195);
		       \path (lem4.north) edge[blue,pil] (lem5.south);

		       \node[punkt,draw=blue, thick,fill=blue!20,opacity=.8,text opacity=1,above=of lem5] (lem6) {Lemma~\ref{lem:grad_vec_concentrates}};
		       \node[punkt, draw=red, thick,fill=red!20,opacity=.8,text opacity=1,above=of cl7] (cl8) {Claim~\ref{norm_bound_y_Ax_s}};
		       \node[punkt, draw=red, thick,fill=red!20,opacity=.8,text opacity=1,above=of cl8] (cl9) {Claim~\ref{lem:var_w_vector_grad} };
		  %     \path (cl8.east) edge[red,pil,bend left=10] (lem6.west);
		        \path [arrow,red,pil]  ([xshift=1.5cm, yshift=0cm]cl8.350) -- (cl8.350);
		    %   \path (cl9.east) edge[red,pil,bend left=10] (lem6.west);
		        \path [arrow,red,pil] (cl9.east) -| ([xshift=1.5cm, yshift=0cm]cl9.east) |- ([xshift=-0.5cm, yshift=0cm]lem6.west) |- (lem6.west);
		    %   \path (cl5.west) edge[red,pil,bend left=30] (cl9.west);
		        \path [arrow,red,pil] ([xshift=-1cm, yshift=0cm]cl9.west) -- (cl9.west);
		   %   \path (lem4.west) edge[blue,pil,bend right=10] (cl8.east);
		      \path [arrow,blue,pil] ([xshift=1.2cm, yshift=0cm]cl8.east) |- (cl8.east);
		    %  \path (cl5.west) edge[red,pil,bend left=20] (cl8.185);
		       \path [arrow,red,pil] ([xshift=-1cm, yshift=0cm]cl8.195) -- (cl8.195);
	%	      \path[->,every loop/.style={looseness=0.8}]  (lem2.200) edge[blue,pil,bend left=20] (cl8.165);
		       \path [arrow,blue,pil] ([xshift=-2cm, yshift=0cm]cl8.west) -- (cl8.west);
		      
		      \node[punkt, draw=blue, thick,fill=blue!20,opacity=.8,text opacity=1,above=of lem6] (lem7) {Lemma~\ref{grad_corr}};
		     % \path (lem5.east) edge[blue,pil,bend right=30] (lem7.east);
		        \path [arrow,blue,pil] ([xshift=0.5cm, yshift=0cm]lem7.20) |- (lem7.20);
		       \path (lem6) edge[blue,pil] (lem7);
		    %   \path (cl7.east) edge[red,pil,bend right=30] (lem7.west);
		       \path [arrow,red,pil]  ([xshift=-2.5cm, yshift=0cm]lem7.west) -- (lem7.west); 
		       
		       \node[punkt, draw=blue, thick,fill=blue!20,opacity=.8,text opacity=1,above=of lem7] (lem8) {	Lemma~\ref{lem:grad_mat}};
		       \node[punkt, draw=red, thick,fill=red!20,opacity=.8,text opacity=1,above=of cl9] (cl10) {Claim~\ref{norm_exp_yAs_yAst}};
		 %       \path (cl8.east) edge[red,pil,bend left=20] (lem8.west);
		          \path [arrow,red,pil] (cl8.10) -- ([xshift=1.8cm, yshift=0cm]cl8.10) ;
		    %    \path (cl10.10) edge[red,pil,bend left=10] (lem8.west);
		         \path [arrow,red,pil] (cl10.350) -| ([xshift=1.8cm, yshift=0cm]cl10.350) |- ([xshift=-0.5cm, yshift=0cm]lem8.west) |- (lem8.west);
		       %  \path (lem2.200) edge[blue,pil,bend left=20] (cl10.185);
		         \path [arrow,blue,pil] (lem2.195) -| ([xshift=-2cm, yshift=0cm]cl10.165) |- (cl10.165);
		      %   \path (lem4.west) edge[blue,pil,bend right=5] (cl10.east);
		           \path [arrow,blue,pil] (lem4.165) -| ([xshift=1.2cm, yshift=0cm]cl10.10) |- (cl10.10);
		      %   \path (cl5.west) edge[red,pil,bend left=20] (cl10.195);
		         \path [arrow,red,pil] (cl5.west) -| ([xshift=-1cm, yshift=0cm]cl10.west) |- (cl10.west);
		        
		         \node[punkt, draw=blue, thick,fill=blue!20,opacity=.8,text opacity=1,above=of lem8] (lem9) {Lemma~\ref{lem:closeness} };
		 %         \path (lem5.east) edge[blue,pil,bend right=30] (lem9.east);
		           \path [arrow,blue,pil] (lem5.10) -| ([xshift=0.5cm, yshift=0cm]lem9.350) |- (lem9.350);
		        %  \path (cl7.east) edge[red,pil,bend right=30] (lem9.west);
		            \path [arrow,red,pil] (cl7.20) -- ([xshift=2cm, yshift=0cm]cl7.20) |- (lem9.195);
		          \path (lem8) edge[blue,pil] (lem9);
		          
		           \node[punkt,draw=green, thick,fill=green!20,opacity=.8,text opacity=1,right=1cm of lem6] (th1) { Theorem~\ref{main_result}};
		           %\path (lem9.east) edge[green,pil,bend left=20] (th1.north);
		            \path [arrow,green,pil] (lem9.east) -- ([xshift=2cm, yshift=0cm]lem9.east) -- (th1.north);
		 %          \path (lem3.east) edge[green,pil,bend right=30] (th1.south);
		               \path [arrow,green,pil] (lem3.east) -- ([xshift=2cm, yshift=0cm]lem3.east) -- (th1.south);
		         %  \path (lem7.east) edge[green,pil,bend left=30] (th1.north);
		             \path [arrow,green,pil] (lem7.east) -- ([xshift=2cm, yshift=0cm]lem7.east);
		          
		\end{tikzpicture}}
	\end{minipage}
\end{table}

\begin{proof}[Proof of Lemma~\ref{lem:recover_sign}]
	Let \[\b{y} \in \mathbb{R}^{n}\] be general sample generated as \[\b{y} = \b{A}^*\b{x}^*\], where \[\b{x}^* \in \mathbb{R}^{m}\] is a sparse random vector with support \[S = \supp(\b{x}^*)\] distributed according to \textbf{D.\ref{dist_x}}.
	
	The initial decoding step at the \[t\]-th iteration (shown in Algorithm~\ref{alg:main_alg}) involves evaluating the inner-product between the estimate of the dictionary \[\b{A}^{(t)}\], and \[\b{y}\]. The \[i\]-th element of the resulting vector can be written as
	\begin{align*}
	\langle \b{A}_i^{(t)}, \b{y}\rangle &= 	%\langle A_i, D_{\omega}A^*x^*\rangle = 
	\langle \b{A}_i^{(t)}, \b{A}^*_i \rangle \b{x}^*_i + \b{w}_i,
	\end{align*}
	where \[\b{w}_i =  \langle \b{A}_i^{(t)}, \b{A}^*_{-i}\b{x}^*_{-i}\rangle \]. Now, since \[\|\b{A}_i^* - \b{A}^{(t)}_i\|_2 \leq \epsilon_t\] and
	\begin{align*}
	\|\b{A}_i^* - \b{A}^{(t)}_i\|_2^2 &= \|\b{A}_i^* \|^2 +  \|\b{A}^{(t)}_i\|^2 - 2\langle \b{A}^{(t)}_i, \b{A}^*_i \rangle = 2 - 2\langle \b{A}_i^{(t)}, \b{A}^*_i \rangle,
	\end{align*}
	we have
	\begin{align*}
	|\langle \b{A}^{(t)}_i, \b{A}^*_i \rangle| \geq 1 - \epsilon^2_t/2.
	\end{align*}
	Therefore, the term 
	\begin{align*}
	|\langle \b{A}_i^{(t)}, \b{A}^*_i\rangle \b{x}^*_i| 
	\begin{cases}
		\geq (1 - \tfrac{\epsilon_t^2}{2})C &,  \text{if} ~ i \in S,\\
		= 0&,  \text{otherwise}.
	\end{cases}
	\end{align*}
	%
%	if \[\b{x}_i^* \in S\], else is zero. 
	Now, we focus on the \[\b{w}_i\] and show that it is small. %If so, we can threshold at an appropriate level so as to eliminate the effect of this term.  
	By the definition of \[\b{w}_i\] we have
	\begin{align*}
	\b{w}_i = \langle \b{A}_i^{(t)}, \b{A}^*_{-i}\b{x}^*_{-i}\rangle  = \textstyle\sum\limits_{\ell\neq i}\langle \b{A}^{(t)}_i, \b{A}^*_\ell \rangle \b{x}^*_\ell = \textstyle\sum\limits_{\ell \in S\backslash \{i\}}\langle \b{A}^{(t)}_i, \b{A}^*_\ell \rangle \b{x}^*_\ell.
	\end{align*}
	Here, since \[var(\b{x}_\ell^*) = 1\], \[\b{w}_i\] is a zero-mean random variable with variance
	\begin{align*}
	var(\b{w}_i) = \textstyle\sum\limits_{\ell \in S\backslash \{i\}}\langle \b{A}^{(t)}_i, \b{A}^*_\ell \rangle^2.
	\end{align*}
	%
%	where we use the fact that \[var(\b{x}_\ell^*) = 1\].
	Now, each term in this sum can be bounded as,
	\begin{align*}
	\langle \b{A}_i^{(t)}, \b{A}^*_\ell \rangle^2 &= (\langle \b{A}_i^{(t)} - \b{A}_i^*, \b{A}^*_\ell\rangle + \langle \b{A}_i^*, \b{A}^*_\ell\rangle)^2\\
	&\leq 2(\langle \b{A}_i^{(t)} - \b{A}_i^*, \b{A}^*_\ell\rangle^2 + \langle \b{A}_i^*, \b{A}^*_\ell\rangle^2)\\
	&\leq 2(\langle \b{A}_i^{(t)} - \b{A}_i^*, \b{A}^*_\ell\rangle^2 + \tfrac{\mu^2}{n}).
	\end{align*}
	Next, \[\textstyle\sum\limits_{\ell\neq i} \langle \b{A}_i^{(t)} - \b{A}_i^*, \b{A}^*_\ell\rangle^2\] can be upper-bounded as	
	\begin{align*}
	&\textstyle\sum\limits_{\ell \in S\backslash \{i\}} \langle \b{A}_i^{(t)} - \b{A}_i^*, \b{A}^*_\ell\rangle^2 %&= \|D_{\omega}A_{S\backslash{\{i\}}}^{*}(A_i - A_i^*)\|^2
	\leq \|\b{A}_{S\backslash{\{i\}}}^{*}\|^2 \epsilon_t^2.
	\end{align*}%
	Therefore, we have the following as per our assumptions on \[\mu\] and \[k\], %for \[\tfrac{\mu}{\sqrt{n}} \leq \tfrac{1}{2k}\]  %and the result on the inner-product between dictionary samples under partial observations,
	\begin{align*}
	\|\b{A}_{S\backslash{\{i\}}}^{*}\|^2 \leq  (1 + k \tfrac{\mu}{\sqrt{n}}) \leq 2,
	\end{align*}
	using Gershgorin Circle Theorem \citep{Gershgorin1931}. Therefore, we have
	\begin{align*}
	\textstyle\sum\limits_{\ell \in S\backslash \{i\}} \langle \b{A}^{(t)}_i - \b{A}_i^*, \b{A}^*_\ell\rangle^2 \leq 2\epsilon_t^2.
	\end{align*}%
	Finally, we have that %since \[\epsilon_t = O^*(1/\sqrt{\log(m)})\]
	\begin{align*}
	\textstyle\sum\limits_{\ell \in S\backslash \{i\}}\langle \b{A}_i^{(t)}, \b{A}^*_\ell \rangle^2
	\leq 2( 2\epsilon^2_t + k \tfrac{\mu^2}{n}) = \mathcal{O}^*(\epsilon_t^2).
	\end{align*}
	Now, we apply the Chernoff bound for sub-Gaussian random variables \[\b{w}_i\] (shown in Lemma~\ref{theorem:subg_chern}) to conclude that
	\begin{align*}
	\mathbf{Pr}[|\b{w}_i|\geq C/4] \leq 2\exp(-\tfrac{C^2}{ \mathcal{O}^*(\epsilon_t^2)}).
	\end{align*}
	Further, \[\b{w}_i\] corresponding to each \[m\] should follow this bound, applying union bound we conclude that 
	\begin{align*}
	\mathbf{Pr}[ \underset{i}{\max}~|\b{w}_i| \geq C/4] \leq 2m\exp(-\tfrac{C^2}{\mathcal{O}^*(\epsilon_t^2)}) := \delta_{\HT}^{(t)} .
	\end{align*}
	%
%	\[\delta_{\HT}^{(t)} = 2m\exp(-\tfrac{C^2}{\mathcal{O}^*(\epsilon_t^2)})\].
\end{proof}

% % % % % % % % % % % % % % % % End of True signed support % % % % % % % % % % % % % % % % % % % % % %

% % % % % % % % % % % % % % % % % % % % signed support preserved IHT % % % % % % % % % % % % % % % % % % % % % % % % % % % % % %
%\noindent\textbf{Lemma~\ref{our:signed_supp}} \textbf{(IHT update step preservesthe correct signed-support).}
%\begin{em}
%	Suppose \[\b{A}^{(t)}\] is \[\epsilon_t\]-close to \[\b{A}^*\],  \[\mu = \c{O}(\log(n))\], \[k = \sqrt{n}/\mu\log(n)\], and \[\epsilon_t = \mathcal{O}^*(1/\log(m))\] Then, with probability at least \[(1 - \delta_{\beta}^{(t)} - \delta_{\HT}^{(t)}  )\], each iterate of the IHT-based coefficient update step shown in \eqref{supp:coeff_iht} has the correct signed-support, if for a constant \[c^{(r)}(\epsilon_t, \mu, k, n)\], the step size is chosen as \[\eta_x^{(r)}\leq c^{(r)}\] ,
%	%is chosen as
%%	
%%	\begin{align*}
%%	\eta_x^{(r)} \leq c^{(r)} \tfrac{\eta_x^{(r-1)}\|\b{x}^{(r-2)} - \b{x}^*\|_1}{\|\b{x}^{(r-1)} - \b{x}^*\|_1 + \eta_x^{(r-1)}\|\b{x}^{(r-2)} - \b{x}^*\|_1},
%%	\end{align*}
%%	
%	and the threshold \[\tau^{(r)}\] is chosen as
%	\begin{align*}
%	\tau^{(r)} = \eta_x^{(r)}(t_\beta + \tfrac{\mu_t}{\sqrt{n}} \|\b{x}^{(r-1)} - \b{x}^*\|_1).
%	\end{align*}
%	%
%	Here, \[t_\beta = \mathcal{O}(\sqrt{k\epsilon_t})\], \[\delta_{\HT}^{(t)} = 2m~{\exp}({-\tfrac{C^2}{\mathcal{O}^*(\epsilon_t^2)}})\] ,and \[\delta_{\beta}^{(t)} = 2k{\exp}(-\tfrac{1}{\mathcal{O}(\epsilon_t)})\].
%	\end{em}

\begin{proof}[Proof of Lemma~\ref{our:signed_supp}]
	%Consider a column \[\b{x}^{(0)}\] of the coefficient matrix \[\b{X}^{(0)}\]. Let \[\b{y}\] denote the data vector corresponding to this vector \[\b{x}^{(0)}\], i.e.,
%	The samples available to us are formed as 
	
%	
%	\begin{align*}
%	\b{y} = \b{A}^*\b{x}^*.
%	\end{align*}
%	% 
%	First, using Lemma~\ref{lem:recover_sign}, we know that \[\b{x}^{(0)}\] has the correct-signed support with probability at least \[(1-\delta_{\HT}^{(t)})\]. 
	Consider the \[(r+1)\]-th iterate \[\b{x}^{(r+1)}\] for the \[t\]-th dictionary iterate, where \[\|\b{A}^{(t)}_i - \b{A}^*_i\|\leq \epsilon_t\] for all \[i \in [1, m]\] evaluated as the following by the update step described in Algorithm~\ref{alg:main_alg},
	\begin{align}
	\label{proof:update_step}
	\b{x}^{(r+1)} &= \b{x}^{(r)} - \eta_{x}^{(r+1)}\b{A}^{(t)^\top}(\b{A}^{(t)}\b{x}^{(r)} - \b{y}) \notag\\
	&= (\b{I} - \eta_{x}^{(r+1)}\b{A}^{(t)^\top}\b{A}^{(t)})\b{x}^{(r)} - \eta_{x}^{(r+1)}\b{A}^{(t)^\top}\b{A}^*\b{x}^*,
	\end{align}
	where \[\eta_{x}^{(1)} < 1\] is the learning rate or the step-size parameter. Now, using 
	%\[\b{x}^{(0)} =  \HT_{C/2}(\b{A}^{(t)^\top}\b{y})\], and by
	 Lemma~\ref{lem:recover_sign} we know that \[\b{x}^{(0)}\] \eqref{alg:coeff_init} has the correct signed-support with probability at least \[(1-\delta_\HT^{(t)})\]. 
	Further, since \[\b{A}^{(t)^\top}\b{A}^*\] can be written as
	\begin{align*}
\b{A}^{(t)^\top}\b{A}^* = (\b{A}^{(t)}- \b{A}^{*})^{\top}\b{A}^* + \b{A}^{*\top}\b{A}^*,
	\end{align*}
	we can write the \[(r+1)\]-th iterate of the coefficient update step using \eqref{proof:update_step} as
	\begin{align*}
	\b{x}^{(r+1)}
	&= (\b{I} - \eta_{x}^{(r+1)}\b{A}^{(t)^\top}\b{A}^{(t)})\b{x}^{(r)} - \eta_{x}^{(r+1)} (\b{A}^{(t)} -\b{A}^{*})^{\top}\b{A}^*\b{x}^* + \eta_{x}^{(r+1)} \b{A}^{*\top}\b{A}^*\b{x}^*.
	\end{align*}
	Further, the \[j\]-th entry of this vector is given by
	\begin{align}
	\label{iht_eq:j_element}
	\b{x}^{(r+1)}_{j} \hspace{-3pt}= \hspace{-3pt}(\b{I} - \eta_{x}^{(r+1)} \b{A}^{(t)^\top}\hspace{-3pt}\b{A}^{(t)})_{(j,:)}\b{x}^{(r)} - \eta_{x}^{(r+1)} &((\b{A}^{(t)} - \b{A}^{*})^{\top}\hspace{-3pt}\b{A}^*)_{(j,:)}\b{x}^*  \hspace{-3pt}+\hspace{-3pt} \eta_{x}^{(r+1)}(\b{A}^{*\top}\hspace{-3pt}\b{A}^*)_{(j,:)}\b{x}^*.
	\end{align}%
	We now develop an expression for the \[j\]-th element of each of the term in \eqref{iht_eq:j_element} as follows. First, we can write the first term as
	\begin{align*}
	(\b{I} - \eta_{x}^{(r+1)}\b{A}^{(t)^\top}\b{A}^{(t)})_{(j,:)}\b{x}^{(r)} & = (1 - \eta_{x}^{(r+1)}) \b{x}_j^{(r)} - \eta_{x}^{(r+1)}\textstyle\sum\limits_{i \neq j}  \langle \b{A}^{(t)}_j, \b{A}^{(t)}_i \rangle\b{x}_i^{(r)}.
	\end{align*}%
	Next, the second term in \eqref{iht_eq:j_element} can be expressed as
	\begin{align*}
	\eta_{x}^{(r+1)} ((\b{A}^{(t)} - \b{A}^{*})^{\top}\b{A}^*)_{(j,:)}\b{x}^* &= \eta_{x}^{(r+1)}\textstyle\sum\limits_{i} \langle \b{A}^{(t)}_j - \b{A}^{*}_j, \b{A}^*_i\rangle \b{x}_i^*\\
	&= \eta_{x}^{(r+1)}\langle \b{A}^{(t)}_j - \b{A}^{*}_j, \b{A}^*_j\rangle \b{x}_j^* + \eta_{x}^{(r+1)}\textstyle\sum\limits_{i\neq j} \langle \b{A}^{(t)}_j - \b{A}^{*}_j, \b{A}^*_i\rangle \b{x}_i^*.
	\end{align*}%
	Finally, we have the following expression for the third term,
	\begin{align*}
	\eta_{x}^{(r+1)}(\b{A}^{*\top}\b{A}^*)_{(j,:)}\b{x}^* &= \eta_{x}^{(r+1)} \b{x}_j^* + \eta_{x}^{(r+1)}\textstyle\sum\limits_{i \neq j} \langle \b{A}^*_j, \b{A}^*_i \rangle \b{x}^*_i. %\\
	%&\leq  \eta_1(1 - \mu ~sign(x_j^*)) x_j^* + \eta_1\mu \|x^*\|_1 
	\end{align*}%
	Now using our definition of  \[\lambda^{(t)}_j = |\langle \b{A}^{(t)}_j - \b{A}^{*}_j, \b{A}^*_j\rangle| \leq \tfrac{\epsilon_t^2}{2}\], combining all the results for \eqref{iht_eq:j_element}, and using the fact that since \[\b{A}^{(t)}\] is close to \[\b{A}^*\], vectors \[\b{A}^{(t)}_j - \b{A}^{*}_j\] and \[\b{A}^*_j\] enclose an obtuse angle, we have the following for  the \[j\]-th entry of the \[(r+1)\]-th iterate, \[\b{x}^{(r+1)}\] is given by
%	
%	\begin{align}
%	&\b{x}_j^{(1)}  =  (1 - \eta_{x}^{(1)}) \b{x}_j^{(0)} + \eta_{x}^{(1)}(1 - |\langle \b{A}^{(t)}_j - \b{A}^{*}_j, \b{A}^*_j\rangle|) \b{x}_j^*  
%	 + \eta_{x}^{(1)}\sum_{i \neq j}( \langle \b{A}^{(t)}_j - \b{A}^{*}_j, \b{A}^*_i\rangle +  \langle \b{A}^{*}_j, \b{A}^*_i\rangle) \b{x}^*_i - \eta_{x}^{(1)}\displaystyle\sum_{i \neq j} \langle \b{A}^{(t)}_j, \b{A}^{(t)}_i\rangle \b{x}_i^{(0)},\notag
%	\end{align}
%	%
	%We are now ready to investigate what happens at the \[(r+1)\]-th iterate of coefficient update.  Let \[\lambda^{(t)}_j = |\langle \b{A}^{(t)}_j - \b{A}^{*}_j, \b{A}^*_j\rangle| \leq \tfrac{\epsilon_t^2}{2}\]. Then, the \[j\]-th entry of the \[(r)\]-th iterate, \[\b{x}^{(r+1)}\] is given by
%		
%		\begin{align}
%		\label{iht_eq:r_th_iterate}
%		\b{x}_j^{(r+1)} =  (1 - \eta_{x}^{(r+1)}) &\b{x}_j^{(r)} + \eta_{x}^{(r+1)}(1 - \lambda^{(t)}_j) \b{x}_j^* 
%		+ \eta_{x}^{(r+1)}\sum_{i \neq j}( \langle \b{A}^{(t)}_j - \b{A}^{*}_j, \b{A}^*_i\rangle +  \langle \b{A}^{*}_j, \b{A}^*_i\rangle) \b{x}^*_i  - \eta_{x}^{(r+1)}\displaystyle\sum_{i \neq j} \langle \b{A}^{(t)}_j, \b{A}^{(t)}_i\rangle \b{x}_i^{(r)}.
%		\end{align}
%		% 
%		where the first three terms are our relevant contributions of the \[j\]-th elements of \[\b{x}^{(0)}\] and \[\b{x}^*\], and the last three terms are components of the rest of the elements. Further, to ensure that \[\b{x}^1\] has the correct signed support, we need to ensure that the contribution by the last three terms is small as compared to the first three terms. 
		\begin{align}
		\label{iht_eq:r_th_iterate_xi}
		\b{x}_j^{(r+1)} =  (1 - \eta_{x}^{(r+1)}) &\b{x}_j^{(r)} + \eta_{x}^{(r+1)}(1 - \lambda^{(t)}_j) \b{x}_j^* + \eta_{x}^{(r+1)}\xi^{(r+1)}_{j} .
		\end{align}% 
		Here \[\xi^{(r+1)}_j\] is defined as 	 
		\begin{align*}
		\xi^{(r+1)}_{j}  &:= \textstyle\sum\limits_{i \neq j}( \langle \b{A}^{(t)}_j - \b{A}^{*}_j, \b{A}^*_i\rangle +  \langle \b{A}^{*}_j, \b{A}^*_i\rangle) \b{x}^*_i -\textstyle\sum\limits_{i \neq j} \langle \b{A}^{(t)}_j, \b{A}^{(t)}_i\rangle \b{x}_i^{(r)}.
		\end{align*}%
%		which can be written as
%		
%		\begin{align*}
%		\xi^{(r+1)}_{j}  %& = \displaystyle\sum_{i \neq j} \langle \b{A}^{*}_j, \b{A}^*_i\rangle \b{x}^*_i - \langle \b{A}^{(t)}_j, \b{A}^{(t)}_i\rangle \b{x}_i^r + \langle \b{A}^{(t)}_j - \b{A}^{*}_j, \b{A}^*_i\rangle\b{x}^*_i,\\
%		& = \displaystyle\sum_{i \neq j} (\langle \b{A}^{*}_j, \b{A}^*_i\rangle  - \langle \b{A}^{(t)}_j, \b{A}^{(t)}_i\rangle + \langle \b{A}^{(t)}_j - \b{A}^{*}_j, \b{A}^*_i\rangle) \b{x}_i^* + \langle \b{A}^{(t)}_j, \b{A}^{(t)}_i\rangle (\b{x}_i^* - \b{x}_i^{(r)} ) .
%		\end{align*}
%		%
		Since, \[\langle \b{A}^{*}_j, \b{A}^*_i\rangle  - \langle \b{A}^{(t)}_j, \b{A}^{(t)}_i\rangle = \langle \b{A}^{*}_j,  \b{A}^*_i - \b{A}^{(t)}_i\rangle  + \langle \b{A}^*_j -\b{A}^{(t)}_j, \b{A}^{(t)}_i\rangle\], we can write \[\xi^{(r+1)}_j\] as
%		
%		\begin{align*}
%		\langle \b{A}^{*}_j, \b{A}^*_i\rangle  - \langle \b{A}^{(t)}_j, \b{A}^{(t)}_i\rangle = \langle \b{A}^{*}_j,  \b{A}^{(t)}_i - \b{A}^*_i\rangle  + \langle \b{A}^{(t)}_j -\b{A}^*_j, \b{A}^{(t)}_i\rangle.
%		\end{align*}
%		%
%		Therefore,
%		
%		\begin{align*}
%		\xi^{(r+1)}_{j} 
%		%& = \displaystyle\sum_{i \neq j} (\langle \b{A}^{*}_j,  \b{A}^{(t)}_i - \b{A}^*_i\rangle  + \langle \b{A}^t_j -\b{A}^*_j, \b{A}^t_i\rangle + \langle \b{A}^t_j - \b{A}^{*}_j, \b{A}^*_i\rangle) \b{x}_i^* - \langle \b{A}^t_j, \b{A}^t_i\rangle (\b{x}_i^{(r)} - \b{x}_i^*),
%		& = \beta^{(t)}_j - \displaystyle\sum_{i \neq j}\langle \b{A}^t_j, \b{A}^t_i\rangle (\b{x}_i^{(r)} - \b{x}_i^*),
%		\end{align*}
%		%	
		\begin{align}\label{eq:def_xi_r}
		\xi^{(r+1)}_{j} 
		%& = \displaystyle\sum_{i \neq j} (\langle \b{A}^{*}_j,  \b{A}^t_i - \b{A}^*_i\rangle  + \langle \b{A}^t_j -\b{A}^*_j, \b{A}^t_i\rangle + \langle \b{A}^t_j - \b{A}^{*}_j, \b{A}^*_i\rangle) \b{x}_i^* - \langle \b{A}^t_j, \b{A}^t_i\rangle (\b{x}_i^{(r)} - \b{x}_i^*),
		& = \beta^{(t)}_j + \textstyle\sum\limits_{i \neq j}\langle \b{A}^{(t)}_j, \b{A}^{(t)}_i\rangle (\b{x}_i^*-\b{x}_i^{(r)}),
		\end{align}%
		where \[\beta^{(t)}_j\] is defined as 
		\begin{align}\label{eq:beta_t}
		\beta^{(t)}_j
		& := \textstyle\sum\limits_{i \neq j} (\langle \b{A}^{*}_j,  \b{A}^*_i - \b{A}^{(t)}_i\rangle  + \langle \b{A}^*_j -\b{A}^{(t)}_j, \b{A}^{(t)}_i\rangle + \langle \b{A}^{(t)}_j - \b{A}^{*}_j, \b{A}^*_i\rangle) \b{x}_i^*.
		\end{align}%
		Note that \[\beta^{(t)}_j\] does not change for each iteration \[r\] of the coefficient update step. Further, by Claim~\ref{lem:bound_beta} we show that \[|\beta^{(t)}_j| \leq t_\beta = \c{O}(\sqrt{k\epsilon_t})\] with probability at least \[(1-\delta_{\beta}^{(t)})\]. Next, we define \[\tilde{\xi}^{(r+1)}_{j}\] as
		\begin{align}\label{eq:xi_r_abs}
		\tilde{\xi}^{(r+1)}_{j} 
		& := \beta^{(t)}_j + \textstyle\sum\limits_{i \neq j}|\langle \b{A}^{(t)}_j, \b{A}^{(t)}_i\rangle| |\b{x}_i^{*} - \b{x}_i^{(r)}|.
	%	& = \beta^{(t)}_j + \displaystyle\sum_{i \neq j}\langle \b{A}^{(t)}_j, \b{A}^{(t)}_i\rangle (\b{x}_i^* -\b{x}_i^{(r)})\sgn(\langle \b{A}^{(t)}_j, \b{A}^{(t)}_i\rangle)\sgn ( \b{x}_i^* - \b{x}_i^{(r)}),
		\end{align}%	
%		Therefore
%		
%		\begin{align*}
%		\xi^{(r+1)}_{j} 
%		& = \beta^{(t)}_j + \displaystyle\sum_{i \neq j}|\langle \b{A}^{(t)}_j, \b{A}^{(t)}_i\rangle| |\b{x}_i^*-\b{x}_i^{(r)}|\sgn(\langle \b{A}^{(t)}_j, \b{A}^{(t)}_i\rangle)\sgn ( \b{x}_i^* - \b{x}_i^{(r)}),
%		\end{align*}
%		%
		where \[\xi^{(r+1)}_{j} \leq \tilde{\xi}^{(r+1)}_{j}\]. Further, using Claim~\ref{claim:incoherence of B},
		\begin{align}\label{eq:xi_til_max}
		\tilde{\xi}^{(r+1)}_{j}  &\leq t_\beta + \tfrac{\mu_t}{\sqrt{n}} \|\b{x}^*_j - \b{x}^{(r)}_j\|_1 := \tilde{\xi}^{(r+1)}_{\max} = \tilde{\mathcal{O}}(\tfrac{k}{\sqrt{n}}),
		\end{align}%
		since \[\|\b{x}^{(r-1)} - \b{x}^*\|_1 =\c{O}(k)\].
		% % %
		Therefore, for the \[(r+1)\]-th iteration, we choose the threshold to be 
		\begin{align}\label{eq:threshold}
		\tau^{(r+1)} := \eta_x^{(r+1)}\tilde{\xi}^{(r+1)}_{\max},
		\end{align}%
		and the step-size by setting the ``noise'' component of \eqref{iht_eq:r_th_iterate_xi} to be smaller than the ``signal'' part, specifically, half the signal component, i.e., 
		\begin{align*}
		\eta_x^{(r+1)}\tilde{\xi}^{(r+1)}_{\max}  
		&\leq \tfrac{(1 - \eta_{x}^{(r+1)})}{2}\b{x}^{(r)}_{\min} + \tfrac{\eta_{x}^{(r+1)}}{2}(1-\tfrac{\epsilon_t^2}{2})C,
		\end{align*}%
		 Also, since we choose the threshold as \[ \tau^{(r)} := \eta_x^{(r)}\tilde{\xi}^{(r)}_{\max}\], \[\b{x}^{(r)}_{\min} = \eta_x^{(r)}\tilde{\xi}^{(r)}_{\max}\], where \[\b{x}^{(0)}_{\min} = C/2\], we have the following for the \[(r+1)\]-th iteration,
		\begin{align*}
		\eta_x^{(r+1)}\tilde{\xi}^{(r+1)}_{\max}  
		&\leq \tfrac{(1 - \eta_{x}^{(r+1)})}{2}\eta_x^{(r)}\tilde{\xi}^{(r)}_{\max} + \tfrac{\eta_{x}^{(r+1)}}{2}(1-\tfrac{\epsilon_t^2}{2})C.
		\end{align*}%
		Therefore, for this step we choose \[\eta_x^{(r+1)}\] as
		\begin{align}
		\label{iht_eq:eta_x_r}
		\eta_x^{(r+1)} \leq  \tfrac{\tfrac{\eta_x^{(r)}}{2}\tilde{\xi}^{(r)}_{\max}}{\tilde{\xi}^{(r+1)}_{\max}+ \tfrac{\eta_x^{(r)}}{2}\tilde{\xi}^{(r)}_{\max} - \tfrac{1}{2}(1-\tfrac{\epsilon_t^2}{2})C},
		%	\eta_x^{r+1} \leq  \tfrac{\eta_x^r\xi^{r}_{\max}/2 }{\xi_{r+1}^{\max}+ (\eta_x^r/2)\xi_r^{\max} - (1/2)(1-\epsilon_t)C},
		\end{align}%
		%	where \[\xi_{r+1}^{\max} \] is defined as,
		%	
		%	\begin{align*}
		%	\xi_{r+1}^{\max} := 2k\epsilon_t\b{x}^*_{\max} + k\tfrac{\mu^{(t)}}{\sqrt{n}} |\b{x}_i^r - \b{x}_i^*|_{\max} + k\epsilon_t \b{x}^*_{\max}.
		%	%(\epsilon_t + \tfrac{\mu}{\sqrt{n}})\|\b{x}^*\|_1 +  \tfrac{\mu^{(t)}}{\sqrt{n}}\|\b{x}^r\|_1.
		%	\end{align*}
		%	%
		Therefore, 
		%since \[\tilde{\xi}^{(r)}_{\max}\]s are proportional to \[\|\b{x}^{(r-1)} - \b{x}^*\|_1 = \mathcal{O}(k)\], 
		\[\eta_x^{(r+1)}\] can be chosen as
	%	for a small constant \[c^{(r+1)}(\epsilon_t, \mu, \k, n)\], \[\eta_x^{(r+1)}\] can be chosen as
	%	for a constant \[c^{(r)}(\epsilon_t, \mu, k, n)\], \[\eta_x^{(r)}\leq c^{(r)}\] 
		\begin{align*}
		\eta_x^{(r+1)}\leq c^{(r+1)}(\epsilon_t, \mu, k, n),
		%\eta_x^{r+1} \leq c^{(r+1)} \eta_x^{(r)}\tfrac{\|\b{x}^{(r-1)} - \b{x}^*\|_1}{\|\b{x}^{(r)} - \b{x}^*\|_1 + \eta_x^{(r)}\|\b{x}^{(r-1)} - \b{x}^*\|_1}
		\end{align*}%
		for a small constant \[c^{(r+1)}(\epsilon_t, \mu, k, n)\], \[\eta_x^{(r+1)}\]. In addition, if we set all \[\eta_x^{(r)} = \eta_x\], we have that \[\eta_x = \tilde{\Omega}(\tfrac{k}{\sqrt{n}})\] and therefore \[\tau^{(r)} = \tau =  \tilde{\Omega}(\tfrac{k^2}{n})\]. Further, since we initialize with the hard-thresholding step, the entries in \[|\b{x}^{(0)}| \geq C/2\]. Here, we define \[\tilde{\xi}^{(0)}_{\max} = C\] and \[\eta_x^{(0)} = 1/2\], and set the threshold for initial step as \[\eta_x^{(0)}\tilde{\xi}^{(0)}_{\max}\]. 
\end{proof}
\begin{proof}[Proof of Lemma~\ref{iht:x_R_error}]
%\subsection{Characterizing \[\tilde{\xi}^{(\ell)}_{i_1}\]}
%Let 
%
%\begin{align*}
%\tilde{\xi}^{(\ell)}_j 
%& = \displaystyle\sum_{i \neq j} (\langle \b{A}^{*}_j,  \b{A}^t_i - \b{A}^*_i\rangle  + \langle \b{A}^t_j -\b{A}^*_j, \b{A}^t_i\rangle) \b{x}_i^* + \sum_{i \neq j}\langle \b{A}^t_j - \b{A}^{*}_j, \b{A}^*_i\rangle\b{x}^*_i +\displaystyle\sum_{i \neq j}|\langle \b{A}^t_j, \b{A}^t_i\rangle| |\b{x}_i^{(\ell-1)} - \b{x}_i^*|.
%\end{align*}
%
%\[ \xi^{(\ell)}_j  \leq \tilde{\xi}^{(\ell)}_j \]
%
%Let \[\beta^{(t)}_j\] be defined as 
%
%\begin{align*}
%\beta^{(t)}_j
%& = \displaystyle\sum_{i \neq j} (\langle \b{A}^{*}_j,  \b{A}^t_i - \b{A}^*_i\rangle  + \langle \b{A}^t_j -\b{A}^*_j, \b{A}^t_i\rangle) \b{x}_i^* + \sum_{i \neq j}\langle \b{A}^t_j - \b{A}^{*}_j, \b{A}^*_i\rangle\b{x}^*_i,
%\end{align*}
%
Using the definition of  \[\tilde{\xi}^{(\ell)}_{i_1}\] as in \eqref{eq:xi_r_abs}, we have
\begin{align*}
\tilde{\xi}^{(\ell)}_{i_1}
& = \beta^{(t)}_{i_1} +\textstyle\sum\limits_{i_2 \neq i_1}|\langle \b{A}^{(t)}_{i_1}, \b{A}^{(t)}_{i_2}\rangle| | \b{x}_{i_2}^* - \b{x}_{i_2}^{(\ell-1)}|.
\end{align*}
From Claim~\ref{lem:bound_beta}  we have that \[|\beta^{(t)}_{i_1}| \leq t_\beta\] with probability at least \[(1-\delta_{\beta}^{(t)})\]. Further, using Claim~\ref{claim:incoherence of B} , and letting \[C_{i}^{(\ell)} := |\b{x}_{i}^* - \b{x}_{i}^{(\ell)}| = |\b{x}_{i}^{(\ell)} - \b{x}_{i}^*|\], 
%and by \[|\langle \b{A}^{(t)}_{i_1}, \b{A}^{(t)}_{i_2}\rangle| \leq \tfrac{\mu_t}{\sqrt{n}}\],
 \[\tilde{\xi}^{(\ell)}_{i_1}\] can be upper-bounded as
\begin{align}\label{eq:xi_l_ub}
\tilde{\xi}^{(\ell)}_{i_1}
& \leq \beta^{(t)}_{i_1} + \tfrac{\mu_t}{\sqrt{n}}\textstyle\sum\limits_{i_2 \neq i_1} C_{i_2}^{(\ell-1)}.
\end{align}
Rearranging the expression for \[(r+1)\]-th update \eqref{iht_eq:r_th_iterate_xi}, and using \eqref{eq:xi_l_ub} we have the following upper-bound
%
%\begin{align}
%\b{x}_{i_1}^{(r+1)} =  (1 - \eta_{x}^{(r+1)}) \b{x}_{i_1}^{(r)} + \eta_{x}^{(r+1)}(1 - |\langle \b{A}^{(t)}_{i_1} - \b{A}^{*}_{i_1}, \b{A}^*_{i_1}\rangle|) \b{x}_{i_1}^*  \notag + \eta_{x}^{(r+1)} \xi^{(r+1)}_{i_1}.
%\end{align}
%% 
%Therefore, 
%
%\begin{align}
%\b{x}_{i_1}^{(r+1)} -  \b{x}_{i_1}^* =  (1 - \eta_{x}^{(r+1)}) (\b{x}_{i_1}^{(r)} -  \b{x}_{i_1}^*) - \eta_{x}^{(r+1)}\lambda_{i_1}^{(t)} \b{x}_{i_1}^*  \notag + \eta_{x}^{(r+1)} \xi^{(r+1)}_{i_1},
%\end{align}
%% 
%where \[\lambda^{(t)}_{i_1} = |\langle \b{A}^{(t)}_{i_1} - \b{A}^{*}_{i_1}, \b{A}^*_{i_1}\rangle|\]. Further, using the expression for \eqref{eq:xi_l_ub} we can upper bound this as 
\begin{align*}
C_{i_1}^{(r+1)} \leq  (1 - \eta_{x}^{(r+1)}) C_{i_1}^{(r)} + \eta_{x}^{(r+1)}\lambda_{i_1}^{(t)} |\b{x}_{i_1}^*|  + \eta_{x}^{(r+1)} \tilde{\xi}^{(r+1)}_{i_1}.
\end{align*}% 
Next, recursively substituting in for \[C_{i_1}^{(r)}\], where we define \[\textstyle\prod_{q = \ell}^{\ell} (1-\eta_x^{(q+1)})= 1\],
\begin{align*}
C_{i_1}^{(r+1)} \hspace{-3pt}\leq C_{i_1}^{(0)} \textstyle\prod\limits_{q = 0}^{r}\hspace{-3pt}(1 - \eta_x^{(q+1)})  + \lambda_{i_1}^{(t)} |\b{x}_{i_1}^*| \sum\limits_{\ell=1}^{r+1}\eta_x^{(\ell)}\prod\limits_{q = \ell}^{r+1} (1-\eta_x^{(q+1)})   + \sum\limits_{\ell=1}^{r+1} \eta_x^{(\ell)}\tilde{\xi}_{i_1}^{(\ell)}\prod\limits_{q = \ell}^{r+1} (1-\eta_x^{(q+1)}) .%+ \eta_x^{({r+1})}(\lambda_{i_1}^{(t)} |\b{x}_{i_1}^*| + \tilde{\xi}^{(r+1)}_{i_1}),
\end{align*}
Substituting for the upper-bound of \[\tilde{\xi}_{i_1}^{(\ell)}\] from \eqref{eq:xi_l_ub},
%
%\begin{align*}
%C_{i_1}^{(r+1)}  
%&\leq C_{i_1}^{(0)} \prod_{q = 0}^{r}(1 - \eta_x^{(q+1)})  + \lambda_{i_1}^{(t)} |\b{x}_{i_1}^*| \sum_{\ell=1}^{r+1}\eta_x^{(\ell)}\prod_{q = \ell}^{r+1} (1-\eta_x^{(q+1)})   + \sum_{\ell=1}^{r+1} \eta_x^{(\ell)}\tilde{\xi}_{i_1}^{(\ell)}\prod_{q = \ell}^{r+1} (1-\eta_x^{(q+1)}) \\
%&\leq C_{i_1}^{(0)} \prod_{q = 0}^{r}(1 - \eta_x^{(q+1)})  + \lambda_{i_1}^{(t)} |\b{x}_{i_1}^*| \sum_{\ell=1}^{r+1}\eta_x^{(\ell)}\prod_{q = \ell}^{r+1} (1-\eta_x^{(q+1)})   + \sum_{\ell=1}^{r+1} \eta_x^{(\ell)}\big( \beta^{(t)}_{i_1} + \tfrac{\mu_t}{\sqrt{n}}\displaystyle\sum_{i_2 \neq i_1} C_{i_2}^{(\ell-1)}\big)\prod_{q = \ell}^{r+1} (1-\eta_x^{(q+1)}) \\
%&\leq C_{i_1}^{(0)} \prod_{q = 0}^{r}(1 - \eta_x^{(q+1)})  + (\lambda_{i_1}^{(t)} |\b{x}_{i_1}^*| +\beta^{(t)}_{i_1}) \sum_{\ell=1}^{r+1}\eta_x^{(\ell)}\prod_{q = \ell}^{r+1} (1-\eta_x^{(q+1)})   + \sum_{\ell=1}^{r+1} \eta_x^{(\ell)}\big(\tfrac{\mu_t}{\sqrt{n}}\displaystyle\sum_{i_2 \neq i_1} C_{i_2}^{(\ell-1)}\big)\prod_{q = \ell}^{r+1} (1-\eta_x^{(q+1)}).
%\end{align*}
%
\begin{align}
\label{iht:def_C_j_rplus1}
C_{i_1}^{(r+1)} 
&\leq  \alpha^{(r+1)}_{i_1}  + \tfrac{\mu_t}{\sqrt{n}}\textstyle\sum\limits_{\ell=1}^{r+1} \eta_x^{(\ell)}\sum\limits_{i_2 \neq i_1} C_{i_2}^{(\ell-1)}\prod\limits_{q = \ell}^{r+1} (1-\eta_x^{(q+1)}).
\end{align}%
Here, \[\alpha^{(r+1)}_{i_1}\] is defined as
\begin{align}
\label{iht:def_alpha_j_rplus1}
\alpha^{(r+1)}_{i_1} = C_{i_1}^{(0)} \textstyle\prod\limits_{q = 0}^{r}(1 - \eta_x^{(q+1)})  + (\lambda_{i_1}^{(t)} |\b{x}_{i_1}^*| +\beta^{(t)}_{i_1}) \sum\limits_{\ell=1}^{r+1}\eta_x^{(\ell)}\prod\limits_{q = \ell}^{r+1} (1-\eta_x^{(q+1)}).
\end{align}
%
%then,
%
%\begin{align}
%\label{iht:def_C_j_rplus1}
%C_{i_1}^{(r+1)} 
%&\leq  \alpha^{(r+1)}_{i_1}  + \tfrac{\mu_t}{\sqrt{n}}\sum_{\ell=1}^{r+1} \eta_x^{(\ell)}\displaystyle\sum_{i_2 \neq i_1} C_{i_2}^{(\ell-1)}\prod_{q = \ell}^{r+1} (1-\eta_x^{(q+1)}).
%\end{align}
%%
Our aim now will be to express \[C_{i_1}^{(\ell)}\] for \[\ell>0\] in terms of \[C_{i_2}^{(0)}\]. Let each \[\alpha^{(\ell)}_{j} \leq \alpha^{(\ell)}_{i}\] where \[ j = i_1, i_2, \dots, i_k\]. Similarly, let \[C_{j}^{(0)} \leq C_i^{(0)}\] for \[j = i_1, i_2, \dots, i_k\], and all \[\eta_x^{(\ell)} = \eta_x\]. Then, using Claim~\ref{iht:gen_term_simple} we have the following expression for \[C_{i_1}^{({R}+1)}\], 
\begin{align*}
C_{i_1}^{({R}+1)} 
\leq \alpha^{({R}+1)}_{i_1} + (k-1)\eta_x\tfrac{\mu_t}{\sqrt{n}} \textstyle\sum\limits_{\ell = 1}^{{R}}  \alpha^{(\ell)}_{\max}\big(1-\eta_x + &\eta_x\tfrac{\mu_t}{\sqrt{n}} \big)^{{R} -\ell}
\\&+ (k-1)\eta_x\tfrac{\mu_t}{\sqrt{n}}C_{\max}^{(0)} \big(1-\eta_x + \eta_x\tfrac{\mu_t}{\sqrt{n}} \big)^{R}. 
\end{align*}%
Here,  \[(1 - \eta_{x})^{R} \leq (1 - \eta_{x} + \eta_x\tfrac{\mu_t}{\sqrt{n}})^{R} \leq \delta_{R}\]. Next from Claim~\ref{iht:C_i1_inter_alpha_sum} we have that with probability at least \[(1 - \delta_{\beta}^{(t)})\],
\begin{align*}
\textstyle\sum\limits_{\ell = 1}^{{R}}  \alpha^{(\ell)}_{\max}\big(1-\eta_x + \eta_x\tfrac{\mu_t}{\sqrt{n}} \big)^{{R} -\ell} 
&\leq C_{\max}^{(0)}{R}\delta_{R}  + \tfrac{1}{\eta_x( 1- \tfrac{\mu_t}{\sqrt{n}})}(\tfrac{\epsilon_t^2}{2} |\b{x}_{\max}^*| +t_\beta).
%\leq \delta_{R}  + \tfrac{1}{\eta_x( 1- \tfrac{\mu_t}{\sqrt{n}})}(\lambda_i^{(t)} |\b{x}_i^*| +\beta^{(t)}_i),
\end{align*}
%
%This choice of \[R\] ensures that \[(1 - \eta_x)^{{R}+1}\] is very small. %For the purposes of our analysis, we will assume \[(1 - \eta_x)^{{R}+1} \approx 0\].
Therefore, for \[c_x = \tfrac{\mu_t}{\sqrt{n}}/ (1- \tfrac{\mu_t}{\sqrt{n}})\]
%
%\begin{align*}
%C_{i_1}^{({R}+1)} 
%%&\leq \alpha^{({R}+1)}_{i_1} + (k-1)\eta_x\tfrac{\mu_t}{\sqrt{n}} \big( \delta_{R}  + \tfrac{1}{\eta_x( 1- \tfrac{\mu_t}{\sqrt{n}})}(\lambda^{(t)}_i |\b{x}_i^*| +\beta^{(t)}_i)\big) + (k-1)\eta_x\tfrac{\mu_t}{\sqrt{n}}C_{i}^{(0)} \big(1-\eta_x + \eta_x\tfrac{\mu_t}{\sqrt{n}} \big)^{R},\\
%&\leq \alpha^{({R}+1)}_{i_1} + \delta_{R} k\eta_x\tfrac{\mu_t}{\sqrt{n}}  + (k-1)\tfrac{\tfrac{\mu_t}{\sqrt{n}}}{ 1- \tfrac{\mu_t}{\sqrt{n}}}(\lambda^{(t)}_i |\b{x}_i^*| +\beta^{(t)}_i) + (k-1)\eta_x\tfrac{\mu_t}{\sqrt{n}}C_{i}^{(0)} \big(1-\eta_x + \eta_x\tfrac{\mu_t}{\sqrt{n}} \big)^{R}.
%\end{align*}
%%
%Note, that by our definition \[C_{i}^{(0)} \big(1-\eta_x + \eta_x\tfrac{\mu_t}{\sqrt{n}} \big)^{R} \leq \delta_{R}\], 
\begin{align*}
C_{i_1}^{({R}+1)} 
&\leq \alpha^{({R}+1)}_{i_1} + (k-1)c_x(\tfrac{\epsilon_t^2}{2} |\b{x}_{\max}^*| +t_\beta) + {(R + 1)}(k-1) \eta_x\tfrac{\mu_t}{\sqrt{n}}C_{\max}^{(0)}\delta_{R}. %+ (k-1)\eta_x\tfrac{\mu_t}{\sqrt{n}}C_{\max}^{(0)}\delta_{R}. 
\end{align*}%
Now, using the definition of \[\alpha^{({R}+1)}_{i_1}\], and using the result on sum of geometric series, we have
\begin{align*}
\alpha^{({R}+1)}_{i_1} &= C_{i_1}^{(0)} (1 - \eta_x)^{{R}+1}  + (\lambda^{(t)}_{i_1} |\b{x}_{i_1}^*| +\beta^{(t)}_{i_1}) \textstyle\sum\limits_{s=1}^{{R}+1}\eta_x (1-\eta_x)^{{R}-s +1},\\
%& \approxeq (\lambda^{(t)}_{i_1} |\b{x}_{i_1}^*| +\beta^{(t)}_{i_1}) (1- (1-\eta_x)^{{R}+1}),\\
& = C_{i_1}^{(0)} \delta_{R} + \lambda^{(t)}_{i_1} |\b{x}_{i_1}^*| +\beta^{(t)}_{i_1} \leq C_{i_1}^{(0)} \delta_{R + 1}  + \tfrac{\epsilon_t^2}{2} |\b{x}_{\max}^*| +t_\beta.
\end{align*}
Therefore,  \[C_{i_1}^{(R)} \] is upper-bounded as 
\begin{align*}
C_{i_1}^{(R)} 
&\leq   (c_x k + 1)(\tfrac{\epsilon_t^2}{2}|\b{x}_{\max}^*|+ t_\beta) + {(R + 1)}k \eta_x\tfrac{\mu_t}{\sqrt{n}}C_{\max}^{(0)}\delta_{R} + C_{i_1}^{(0)} \delta_{R}.
%&\leq \tilde{c_x}k\tfrac{\epsilon_t^2}{2}|\b{x}_{\max}^*|  + 2\delta_{R} k\eta_x\tfrac{\mu_t}{\sqrt{n}}.
\end{align*}%
Further, since \[k = \mathcal{O}(\sqrt{n}/\mu\log(n))\], \[kc_x < 1\], therefore, we have
\begin{align*}
C_{i_1}^{(R)} 
&\leq  \mathcal{O} (t_\beta) + {(R + 1)}k \eta_x\tfrac{\mu_t}{\sqrt{n}}C_{\max}^{(0)}\delta_{R} + C_{i_1}^{(0)} \delta_{R},
%&\leq \tilde{c_x}k\tfrac{\epsilon_t^2}{2}|\b{x}_{\max}^*|  + 2\delta_{R} k\eta_x\tfrac{\mu_t}{\sqrt{n}}.
\end{align*}%
with probability at least \[(1 - \delta_{\beta}^{(t)})\]. Here, \[{(R + 1)}k \eta_x\tfrac{\mu_t}{\sqrt{n}}C_{\max}^{(0)}\delta_{R} + C_{i_1}^{(0)} \delta_{R} \approxeq 0 \] for an appropriately large \[R\]. Therefore, the error in each non-zero coefficient is 
\begin{align*}
C_{i_1}^{(R)} &=  \mathcal{O} (t_\beta).
%&\leq \tilde{c_x}k\tfrac{\epsilon_t^2}{2}|\b{x}_{\max}^*|  + 2\delta_{R} k\eta_x\tfrac{\mu_t}{\sqrt{n}}.
\end{align*}
with probability at least \[(1 - \delta_{\beta}^{(t)})\]. 
\end{proof}
\begin{proof}[Proof of Lemma~\ref{iht:R_th_term}]
	Using the expression for \[\b{x}_{i_1}^{(R)} \] as defined in \eqref{iht_eq:r_th_iterate_xi}, and recursively substituting for \[\b{x}_{i_1}^{(r)}\] we have
	\begin{align*}
	\b{x}_{i_1}^{(R)} 
	&=  (1 - \eta_{x})^{R} \b{x}_j^{(0)} + \b{x}_{i_1}^*\textstyle\sum\limits_{r = 1}^{R}\eta_{x}(1 - \lambda^{(t)}_{i_1}) (1- \eta_x)^{{R} - r} + \sum\limits_{r = 1}^{R }\eta_{x}\xi^{(r)}_{i_1} (1- \eta_x)^{{R} - r},
	\end{align*}%
	where we set all \[\eta_x^{r}\] to be \[\eta_x\]. Further, on defining
	\begin{align}\label{eq:def_var_theta}
	\vartheta^{(R)}_{i_1} :=  \textstyle\sum\limits_{r = 1}^{R }\eta_{x}\xi^{(r)}_{i_1} (1- \eta_x)^{{R} - r} + \gamma^{(R)}_{i_1},
	\end{align}%
	where \[ \gamma^{(R)}_{i_1} :=(1 - \eta_{x})^{R} (\b{x}_{i_1}^{(0)}  - \b{x}_{i_1}^* (1 - \lambda^{(t)}_{i_1}))\], we have	
	\begin{align}
	\b{x}_{i_1}^{(R)} 
	& =  (1 - \eta_{x})^{R} \b{x}_{i_1}^{(0)}  + \b{x}_{i_1}^* (1 - \lambda^{(t)}_{i_1}) ( 1- (1 - \eta_{x})^{R}) + \textstyle\sum\limits_{r = 1}^{R }\eta_{x}\xi^{(r)}_{i_1} (1- \eta_x)^{{R} - r},\notag\\
	& = \b{x}_{i_1}^* (1 - \lambda^{(t)}_{i_1}) + \vartheta^{(R)}_{i_1}.
	\end{align}%
	Note that \[\gamma^{(R)}_{i_1}\] can be made appropriately small by choice of \[R\]. Further, by Claim~\ref{iht:var_theta_abs} we have
	%
%	
%	\begin{align*}
%	|\vartheta^{(R)}_{i_1}| 
%	& \leq  \mathcal{O}(k^2\tfrac{\mu_t}{\sqrt{n}}\epsilon^2_t|\b{x}_{\max}^*| ),
%	\end{align*}
%	%
   	\begin{align*}
   	|\vartheta^{(R)}_{i_1}| 
   	& \leq  \siri{\mathcal{O}(t_{\beta})}.
   	\end{align*}%
	with probability at least \[(1-\delta_{\beta}^{(t)})\], where \[t_{\beta} = \mathcal{O}(\sqrt{k\epsilon_t})\]. 
%	\hl{consider removing the following?}
%	Also, since \[(1 - \eta_{x})^{R} \leq \delta_{R}\], for a \[\delta_{R}\] which depends on our choice of \[{R}\]
%	
%	\begin{align*}
%	\b{x}_{i_1}^{(R)} 
	%&=  (1 - \eta_{x})^{R} \b{x}_{i_1}^{(0)} + \sum\limits_{r = 1}^{R}\eta_{x}\big((1 -  \lambda^{(t)}_{i_1})\b{x}_{i_1}^*\big) (1- \eta_x)^{{R} - r} + \vartheta^{(R)}_{i_1}, \\
	%&= (1 - \eta_{x})^{R} \b{x}_{i_1}^{(0)} + (1 -  \lambda^{(t)}_{i_1})\b{x}_{i_1}^*\sum\limits_{r = 1}^{R}\eta_{x} (1- \eta_x)^{{R} - r} + \vartheta^{(R)}_{i_1}, \\
%	&\approxeq \b{x}_{i_1}^* (1 - \lambda^{(t)}_{i_1}) + \vartheta^{(R)}_{i_1}.
%	\end{align*}
%	%
\end{proof}
% % % % % % % % % % % % % % % % % % % % R-th iterate % % % % % % % % % % % % % % % % % % % % % % % % % % % % % % % % %
% % % % % % % % % % % % % % % % % % % % % % % % % % end of expression for x -- characterizing xi % % % % % % % % % % % % % % % %

%\noindent\textbf{Lemma \ref{grad_exp}} \textbf{(Expression for the expected gradient vector).}
%\begin{em}
%	Suppose that \[\b{A}^{(t)}\] is \[(\epsilon, 2)\]-near to \[\b{A}^*\]. Then, the dictionary update step in Algorithm~\ref{alg:main_alg} amounts to the following for the \[j\]-th dictionary element
%		\begin{align*}
%		\b{E}[\b{A}^{(t+1)}_j] = \b{A}^{(t)}_j + \eta_A \b{g}^{(t)}_j,
%		\end{align*}%	
%		%where with probability at least \[(1 - \delta_{\HT}^{(t)} - \delta_{\beta} - \delta_{\gradvec}^{(t)})\], 
%		where \[\b{g}^{(t)}_j\] is given by
%		\begin{align*}
%		\b{g}^{(t)}_j = q_j p_j \big((1 - \lambda^{(t)}_j)\b{A}^{(t)}_j- \b{A}^*_j + \tfrac{1}{q_j p_j}\Delta^{(t)}_j \pm \gamma\big),
%		\end{align*}
%	\[\lambda^{(t)}_{j} = |\langle \b{A}^{(t)}_j - \b{A}^{*}_j, \b{A}^*_j\rangle|\], and \[\Delta^{(t)}_j = \b{E}[\b{A}^{(t)}_S\vartheta^{(R)}_S\sgn(\b{x}^*_j)]\] which is upper bounded as
%	\begin{align*}
%	\Delta^{(t)}_{j}&= \b{E}[\b{A}^{(t)}\vartheta^{(R)}\sgn(\b{x}^*_{j})] \leq \mathcal{O}(\sqrt{m}q_{i,j}p_{j}\epsilon_t\|\b{A}^{(t)}\|).
%	\end{align*}%
%\vspace{-14pt}
%\end{em}
\begin{proof}[Proof of Lemma~\ref{grad_exp}]
%	Note that, for \[j\in S:= \supp(x^*)\], \[\b{x}^{(0)}_j =  \b{A}^{(t)^\top}_j\b{y}\]. Further, we use \[\hat{\b{x}}\] to denote \[\b{x}^{(R)}\] for notational convenience, and we define \[\tilde{\xi}_S\] as a vector in \[\mathbb{R}^{k}\] with each element as \[\tilde{\xi}_i\] for \[i \in S\]
%	
%	
%	\begin{align*}
%	\Lambda^{(t)}_S (i,j) = 
%	\begin{cases}
%	\lambda^{(t)}_{j}, &\text{for}~ j = i, i \in S\\
%	0, &\text{otherwise}.
%	\end{cases}
%	\end{align*}
%	%
%	we use \[\Lambda^{(t)}_S\] for the case when \[S = [m]\]
%	
%	
%	\begin{align*}
%	\b{D}_{(1 - \lambda^{(t)})} = \b{I} - \Lambda^{(t)}
%	\end{align*}
%	%
	From Lemma \ref{iht:R_th_term} we have that for each \[j \in S\],
	\begin{align*}
	\hat{\b{x}}_S:= \b{x}^{(R)}_S
	&= (\b{I} - \Lambda^{(t)}_S) \b{x}^*_S  + \vartheta^{(R)}_S,
	\end{align*}%
	with probability at least \[(1 - \delta_\HT^{(t)} - \delta_{\beta}^{(t)})\]. Further, let \[\mathcal{F}_{\b{x}^*}\] be the event that \[\sgn(\b{x}^*) = \sgn(\hat{\b{x}})\], and let \[\mathbbm{1}_{\mathcal{F}_{\b{x}^*}}\] denote the indicator function corresponding to this event. As we show in Lemma~\ref{our:signed_supp}, this event occurs with probability at least \[(1 - \delta_{\beta}^{(t)} - \delta_{\HT}^{(t)})\]. Using this, we can write the expected gradient vector corresponding to the \[j\]-th sample as
	\[\mathbbm{1}_{\mathcal{F}_{\b{x}^*}}\]
	\begin{align*}
	\b{g}^{(t)}_j &= \mathbf{E}[(\b{A}^{(t)}\hat{\b{x}}- \b{y})\sgn(\b{x}^*_j)\mathbbm{1}_{\mathcal{F}_{\b{x}^*}}]  + \mathbf{E}[(\b{A}^{(t)}\hat{\b{x}}- \b{y})\sgn(\b{x}^*_j)\mathbbm{1}_{\overline{\mathcal{F}}_{\b{x}^*}}], \\
	&= \mathbf{E}[(\b{A}^{(t)}\hat{\b{x}}- \b{y})\sgn(\b{x}^*_j)\mathbbm{1}_{\mathcal{F}_{\b{x}^*}}]  \pm \gamma.
%	&=\mathbf{E}[ (1 - \eta_{x})^{R} \b{A}^{(t)}_S\b{x}_{S}^{(0)} + \b{A}^{(t)}_S(\b{I} - \Lambda^{(t)}_S)\b{x}_S^* + \b{A}^{(t)}_S\vartheta^{(R)}_S - \b{A}^*_Sx^*_S)\sgn(\b{x}^*_j)]  \pm \gamma
	%&=\mathbf{E}[ (1 - \lambda^{(t)})\b{A}^{(t)}_S\b{x}_S^* + \eta_x\b{A}^{(t)}_S\tilde{\xi}_S - \b{A}^*_Sx^*_S)\sgn(\b{x}^*_j)]  \pm \gamma.
	\end{align*}
	Here, \[\gamma:=\mathbf{E}[(\b{A}^{(t)}\hat{\b{x}} - \b{y})\sgn(\b{x}^*_j)\mathbbm{1}_{\overline{\mathcal{F}}_{\b{x}^*}}]\] is small and depends on \[ \delta_{\HT}^{(t)}\] and \[\delta_{\beta}^{(t)} \], which in turn drops with \[\epsilon_t\]. Therefore, \[\gamma\] diminishes with \[\epsilon_t\]. Further, since \[\mathbbm{1}_{\mathcal{F}_{\b{x}^*}} + \mathbbm{1}_{\overline{\mathcal{F}}_{\b{x}^*}} = 1\], and \[\b{Pr}[\mathcal{F}_{\b{x}^*}] = (1 - \delta_{\beta}^{(t)} - \delta_{\HT}^{(t)})\], is very large,
	\begin{align*}
	\b{g}^{(t)}_j %&= \mathbf{E}[(\b{A}^{(t)}\hat{\b{x}}- \b{y})\sgn(\b{x}^*_j)\mathbbm{1}_{\mathcal{F}_{\b{x}^*}}]  + \mathbf{E}[(\b{A}^{(t)}\hat{\b{x}}- \b{y})\sgn(\b{x}^*_j)\mathbbm{1}_{\overline{\mathcal{F}}_{\b{x}^*}}] \\
	&= \mathbf{E}[(\b{A}^{(t)}\hat{\b{x}}- \b{y})\sgn(\b{x}^*_j)(1 - \mathbbm{1}_{\overline{\mathcal{F}}_{\b{x}^*}})]  \pm \gamma,\\
	&= \mathbf{E}[(\b{A}^{(t)}\hat{\b{x}}- \b{y})\sgn(\b{x}^*_j)]  \pm \gamma.
%	&=\mathbf{E}[ (1 - \eta_{x})^{R} \b{A}^{(t)}_S\b{x}_{S}^{(0)} + \b{A}^{(t)}_S(\b{I} - \Lambda^{(t)}_S)\b{x}_S^* + \b{A}^{(t)}_S\vartheta^{(R)}_S - \b{A}^*_Sx^*_S)\sgn(\b{x}^*_j)]  \pm \gamma
	%&=\mathbf{E}[ (1 - \lambda^{(t)})\b{A}^{(t)}_S\b{x}_S^* + \eta_x\b{A}^{(t)}_S\tilde{\xi}_S - \b{A}^*_Sx^*_S)\sgn(\b{x}^*_j)]  \pm \gamma.
	\end{align*}%
	Therefore, we can write \[	\b{g}^{(t)}_j \] as
	\begin{align*}
	\b{g}^{(t)}_j 
	&= \mathbf{E}[(\b{A}^{(t)}\hat{\b{x}} - \b{y})\sgn(\b{x}^*_j)]  \pm \gamma,\\
	&=\mathbf{E}[ (1 - \eta_{x})^{R} \b{A}^{(t)}_S\b{x}_{S}^{(0)} + \b{A}^{(t)}_S(\b{I} - \Lambda^{(t)}_S)\b{x}_S^* + \b{A}^{(t)}_S\vartheta^{(R)}_S - \b{A}^*_Sx^*_S)\sgn(\b{x}^*_j)]  \pm \gamma.
	%&=\mathbf{E}[ (1 - \lambda^{(t)})\b{A}^{(t)}_S\b{x}_S^* + \eta_x\b{A}^{(t)}_S\tilde{\xi}_S - \b{A}^*_Sx^*_S)\sgn(\b{x}^*_j)]  \pm \gamma.
	\end{align*}%
	Since \[\mathbf{E}[ (1 - \eta_{x})^{R} \b{A}^{(t)}_S\b{x}_{S}^{(0)}]\] can be made very small by choice of \[R\], we absorb this term in \[\gamma\]. Therefore, 
	\begin{align*}
	\b{g}^{(t)}_j 
	&=\mathbf{E}[ \b{A}^{(t)}_S(\b{I} - \Lambda^{(t)}_S)\b{x}_S^* + \b{A}^{(t)}_S\vartheta^{(R)}_S - \b{A}^*_Sx^*_S)\sgn(\b{x}^*_j)]  \pm \gamma.
	%&=\mathbf{E}[ (1 - \lambda^{(t)})\b{A}^{(t)}_S\b{x}_S^* + \eta_x\b{A}^{(t)}_S\tilde{\xi}_S - \b{A}^*_Sx^*_S)\sgn(\b{x}^*_j)]  \pm \gamma.
	\end{align*}%
	Writing the expectation by sub-conditioning on the support,
	\begin{align*}
	&\b{g}^{(t)}_j
	=\mathbf{E}_S[ \mathbf{E}_{x_S^*}[\b{A}^{(t)}_S(\b{I} - \Lambda^{(t)}_S)\b{x}_S^* \sgn(\b{x}^*_j) - \b{A}^*_S\b{x}_S^*\sgn(\b{x}^*_j) + \b{A}^{(t)}_S\vartheta^{(R)}_S\sgn(\b{x}^*_j) | S]]  \pm \gamma,\\
	&=\mathbf{E}_S[\b{A}^{(t)}_S(\b{I} - \Lambda^{(t)}_S)\mathbf{E}_{x_S^*}[\b{x}_S^*\sgn(\b{x}^*_j)|S] - \b{A}^*_S\mathbf{E}_{x_S^*}[\b{x}_S^*\sgn(\b{x}^*_j)|S]] + \mathbf{E}[\b{A}^{(t)}_S\vartheta^{(R)}_S\sgn(\b{x}^*_j)] \pm \gamma,\\
	&= \mathbf{E}_S[p_j(1 - \lambda^{(t)}_j)\b{A}^{(t)}_j - p_j\b{A}^*_j] + \Delta^{(t)}_j\pm \gamma,
	\end{align*}%
	where we have used the fact that \[\mathbf{E}_{x_S^*}[\sgn(\b{x}^*_j)]=0\] and introduced 
	\begin{align*}
	\Delta^{(t)}_j &= \b{E}[\b{A}^{(t)}_S\vartheta^{(R)}_S\sgn(\b{x}^*_j)].% \leq k\|\b{A}^{(t)}_S\|\|\vartheta^{(R)}_j\|
	\end{align*}%
	Next, since \[ p_j = \mathbf{E}_{x_S^*}[\b{x}_j^*\sgn(\b{x}^*_j)|j \in S]\], therefore,
	\begin{align*}
	\b{g}^{(t)}_j
	&=\mathbf{E}_S[ p_j (1 - \lambda^{(t)}_j)\b{A}^{(t)}_j- p_j\b{A}^*_j] + \Delta^{(t)}_j\pm \gamma.
	\end{align*}%
	Further, since \[q_j = \b{Pr}[j \in S] = \mathcal{O}(k/m)\], 
	\begin{align*}
	\b{g}^{(t)}_j
	&= q_j p_j \big((1 - \lambda^{(t)}_j)\b{A}^{(t)}_j- \b{A}^*_j + \tfrac{1}{q_j p_j}\Delta^{(t)}_j \pm \gamma\big).
	\end{align*}%
	Further, by Claim~\ref{iht:bound_vareps} we have that 
	\begin{align*}
	\|\Delta^{(t)}_{j}\|&= \mathcal{O}(\sqrt{m}q_{i,j}p_{j}\epsilon_t\|\b{A}^{(t)}\|)].
	\end{align*}
    This completes the proof.
\end{proof}
% % % % % % % % % % % % % % % % % % %end of exp grad % % % % % % % % % % % % % % % % % % % % % % % % % %

% % % % % % % % % % % % % % % Each Empirical Gradient vector and Expected value of each Gradient vector % % % % % % % % % % % % % % % %
%\noindent\textbf{Lemma~\ref{lem:grad_vec_concentrates}} \textbf{(Concentration of the empirical gradient vector).}
%\begin{em}
%Given \[p = \tilde{\Omega}(mk^2)\] samples, the empirical gradient vector estimate corresponding to the \[i\]-th dictionary element, \[\hat{\b{g}}_i^{(t)}\] concentrates around its expectation, i.e.,
%		\begin{align*}
%		\|\hat{\b{g}}_i^{(t)} - \b{g}_i^{(t)}\| \leq \mathcal{O}(\tfrac{k}{m})o(\epsilon_t).
%		\end{align*}%
%		with probability at least \[(1 -\delta_{\gradvec}^{(t)}- \delta_{\beta} - \delta_{\HT}^{(t)} - \delta_{\rm HW}^{(t)})\], where \[\delta_{\gradvec}^{(t)} =  \exp(-\mathcal{O}(k))\].
%\end{em}
\begin{proof}[Proof of Lemma~\ref{lem:grad_vec_concentrates}]
	Let \[W = \{j: i \in \supp(\b{x}^{*}_{(j)})\}\] and then we have that 
	\begin{align*}
	\hat{\b{g}}_i^{(t)} = \tfrac{|W|}{p} \tfrac{1}{|W|} {\textstyle\sum_{j}} (\b{y}_{(j)}-\b{A}^{(t)}\hat{\b{x}}_{(j)}) \sgn(\hat{\b{x}}_{(j)}(i)),
	\end{align*}
	where \[\hat{\b{x}}_{(j)}(i)\] denotes the \[i\]-th element of the coefficient estimate corresponding to the \[(j)\]-th sample. Here, for \[\ell = |W|\] the summation 
	\begin{align*}
	 \textstyle \sum_{j}  \tfrac{1}{\ell}(\b{y}_{(j)}-\b{A}^{(t)}\hat{\b{x}}_{(j)}) \sgn(\hat{\b{x}}_{(j)}(i)),
	\end{align*}
	has the same distribution as \[\Sigma_{j=1}^{\ell}  \b{z}_j \], where each \[\b{z}_j\] belongs to a distribution as 
	\begin{align*}
	\b{z}:= \tfrac{1}{\ell}(\b{y}- \b{A}^{(t)}\hat{\b{x}})\sgn(\hat{\b{x}}_i)|i\in S.
	\end{align*}
	Also, \[\mathbf{E}[(\b{y} - \b{A}^{(t)}\hat{\b{x}})\sgn(\hat{\b{x}}_i)] = q_i \b{E}[\b{z}]\], where \[q_i = \b{Pr}[\b{x}_i^* \neq 0] = \Theta(\tfrac{k}{m})\]. Therefore, since \[p = \tilde{\Omega}(mk^2)\], we have \[\ell = pq_i = \tilde{\Omega}(k^3)\] non-zero vectors, 
	\begin{align}\label{eq:emp_grad_vec}
	\|\hat{\b{g}}_i^{(t)} - \b{g}_i^{(t)}\| &= \mathcal{O}(\tfrac{k}{m})\|\textstyle \Sigma_{j=1}^{\ell } (\b{z}_j  - \b{E}[\b{z}])\|.
%	&\leq \mathcal{O}(\tfrac{k}{m}\epsilon_t)
	\end{align}
	Let \[\b{w}_j = \b{z}_j  - \b{E}[\b{z}]\], we will now apply the vector Bernstein result shown in Lemma~\ref{vector_bernstein}. For this, we require bounds on two parameters for these -- \[L: = \|\b{w}_j\|\] and \[\sigma^2:=\|\Sigma_j \b{E}[\|\b{w}_j\|^2]\|\]. Note that, since the quantity of interest is a function of \[\b{x}_i^*\], which are sub-Gaussian, they are only bounded \textit{almost surely}.  To this end, we will employ Lemma~\ref{tech_lem} (Lemma 45 in \citep{Arora15}) to get a handle on the concentration.
	
	\noindent\textbf{Bound on the norm \[\|\b{w}\|\]: }
	This bound is evaluated in Claim~\ref{norm_bound_y_Ax_s}, which states that with probability at least \[(1 - \delta_{\beta}^{(t)} - \delta_{\HT}^{(t)} - \delta_{\rm HW}^{(t)})\],
	\begin{align*}
	L := \|\b{w}\| = \|\b{z}- \b{E}[\b{z}]\| =  \tfrac{2}{\ell}\| ( \b{y}- \b{A}^{(t)}\hat{\b{x}})\sgn(\hat{\b{x}}_i)|i\in S\| \leq \tfrac{2}{\ell}\| ( \b{y} - \b{A}^{(t)}\hat{\b{x}})\| = 
	%\textcolor{red}{\mathcal{O}( \tfrac{k^3}{\ell}\tfrac{\mu_t}{\sqrt{n}}\epsilon^2_t)} =
	 \siri{\tilde{\mathcal{O}}(\tfrac{kt_{\beta}}{\ell})}.
	\end{align*}

	\noindent\textbf{Bound on  variance parameter \[\b{E}[\|\b{w}\|^2]\]:}
	Using Claim~\ref{lem:var_w_vector_grad}, we have \[\b{E}[\|\b{z}\|^2] = \mathcal{O}(k\epsilon_t^2) +
	%  \textcolor{red}{\mathcal{O}( k^5\epsilon_t^4\tfrac{\mu_t^2}{n})}+ 
	\siri{ \mathcal{O}( kt_{\beta}^2)}\]. Therefore, the bound on the variance parameter \[\sigma^2\] is given by
	\begin{align*}
	\sigma^2:=\|\Sigma_j \b{E}[\|\b{w}_j\|^2]\| \leq \|\Sigma_j \b{E}[\|\b{z}_j\|^2]\|\leq\mathcal{O}(\tfrac{k}{\ell}\epsilon_t^2) +
	%  \textcolor{red}{\mathcal{O}( k^5\epsilon_t^4\tfrac{\mu_t^2}{n})}+ 
	\siri{ \mathcal{O}( \tfrac{kt_{\beta}^2}{\ell})}.
	\end{align*}
	From Claim~\ref{lem:bound_beta} we have that with probability at least \[(1-\delta_{\beta}^{(t)})\], \[t_\beta = \mathcal{O}(\sqrt{k\epsilon_t})\]. Applying vector Bernstein inequality shown in Lemma~\ref{vector_bernstein} and using Lemma~\ref{tech_lem} (Lemma 45 in \citep{Arora15}), choosing \[\ell = \tilde{\Omega}(k^3)\], we conclude
	\begin{align*}
	\|\textstyle \sum_{j=1}^{\ell}  \b{z}_j - \mathbf{E}[\b{z}]\| &= \mathcal{O}(L) + \mathcal{O}(\sigma) = o(\epsilon_t), %\\
	%&\leq \tilde{\mathcal{O}}(1/k\sqrt{k}) + \tilde{\mathcal{O}}(1/\sqrt{k})
	%& =  \mathcal{O}(1/\gamma k) + \mathcal{O}((1/\gamma)\sqrt{k/n}) +o(\epsilon/\gamma)
	%& =  \mathcal{O}((1/\gamma)\sqrt{k/n}) +o(\epsilon/\gamma)
	\end{align*}
	with probability at least \[(1 -\delta_{\gradvec}^{(t)})\], where \[\delta_{\gradvec}^{(t)} = \exp(-\Omega(k))\]. %where for \[\sigma^2 := \mathbf{E}[\|\b{w}\|^2]/\ell\] and \[R := \|\b{w}\|/\ell\],
%	
%	
%	\begin{align*}
%	\delta_{\gradvec}^{(t)} = 2n {\exp}\big( \tfrac{-t_{\gradvec}^2/2}{\sigma^2 + Lt_{\gradvec}/3}\big) 
%	= 2n ~{\rm exp}\big( \tfrac{-t_{\gradvec}^2/2}{o(\epsilon_t^2) (1 + t_{\gradvec}/3})\big).
%	\end{align*}
%	
	Finally, substituting in \eqref{eq:emp_grad_vec} we have
	\begin{align*}
	\|\hat{\b{g}}_i^{(t)} - \b{g}_i^{(t)}\| = \mathcal{O}(\tfrac{k}{m})o(\epsilon_t).
	\end{align*}
	with probability at least \[(1 -\delta_{\gradvec}^{(t)}- \delta_{\beta}^{(t)} - \delta_{\HT}^{(t)} - \delta_{\rm HW}^{(t)})\].
%Therefore, choosing \[t_{\gradvec} = \mathcal{O}(\sqrt{\epsilon_t})\] we can ensure that the empirical gradient estimate approaches its expectation with constant probability of \[(1 -\delta_{\gradvec}^{(t)})\].
\end{proof}
\begin{proof}[Proof of Lemma~\ref{grad_corr}]
	Since we only have access to the empirical estimate of the gradient \[\hat{\b{g}}_i^{(t)}\], we will show that this estimate is correlated with \[(\b{A}^{(t)}_j - \b{A}^*_j)\]. To this end, first from Lemma~\ref{lem:grad_vec_concentrates} we have that the empirical gradient vector concentrates around its mean, specifically,
	\begin{align*}
	\|\hat{\b{g}}_i^{(t)} - \b{g}_i^{(t)}\| \leq o(\tfrac{k}{m}\epsilon_t),
	\end{align*}
	with probability at least \[(1 -\delta_{\gradvec}^{(t)}- \delta_{\beta}^{(t)} - \delta_{\HT}^{(t)} - \delta_{\rm HW}^{(t)})\]. From Lemma~\ref{grad_exp}, we have the following expression for the expected gradient vector
		\begin{align*}
		\b{g}^{(t)}_j
		&=  p_jq_j (\b{A}^{(t)}_j- \b{A}^*_j)  +  p_jq_j(- \lambda^{(t)}_j\b{A}^{(t)}_j+ \tfrac{1}{ p_jq_j}\Delta^{(t)}_j \pm \gamma).
		\end{align*}
		Let \[\b{g}^{(t)}_j = 4\rho_{\_} (\b{A}^{(t)}_j - \b{A}^*_j) + v\], where \[4\rho_{\_} = p_jq_j\] and \[v\] is defined as
		\begin{align}\label{eq:def_v}
		v =  p_jq_j(- \lambda^{(t)}_{j}\b{A}^{(t)}_j+ \tfrac{1}{ p_jq_j}\Delta^{(t)}_j \pm \gamma).
		\end{align}
		Then, \[\hat{\b{g}}_i^{(t)} \] can be written as
	\begin{align}\label{eq:grad_hat_decomp}
	\hat{\b{g}}_i^{(t)}  &= \hat{\b{g}}_i^{(t)} - \b{g}_i^{(t)} + \b{g}_i^{(t)},\notag \\
	& = (\hat{\b{g}}_i^{(t)} - \b{g}_i^{(t)}) + 4\rho_{\_} (\b{A}^{(t)}_j - \b{A}^*_j) + v, \notag \\
	& = 4\rho_{\_} (\b{A}^{(t)}_j - \b{A}^*_j) + \tilde{v},
	\end{align}
	where \[\tilde{v} = v + (\hat{\b{g}}_i^{(t)} - \b{g}_i^{(t)})\]. Let  \[\|\tilde{v}\| \leq \rho_{\_}\|\b{A}^{(t)}_i - \b{A}^*_i\| \]. Using the definition of \[v\] as shown in \eqref{eq:def_v} we have%Now, we will show this is the case. Consider the expression for \[v\],
	\begin{align*}
	\|\tilde{v}\| \leq q_j p_j\lambda^{(t)}_{j}\|\b{A}^{(t)}_j\|+ \|\Delta^{(t)}_j\| + o(\tfrac{k}{m}\epsilon_t) \pm \gamma.
	\end{align*}
		Now for the first term, since \[\|\b{A}^{(t)}_j\| = 1\], we have \[\lambda^{(t)}_{j} = |\langle \b{A}^{(t)}_j - \b{A}^{*}_j, \b{A}^*_j\rangle| = \tfrac{1}{2}\| \b{A}^{(t)}_j - \b{A}^{*}_j\|^2\], therefore
		\begin{align*}
		q_j p_j\lambda^{(t)}_{j}\|\b{A}^{(t)}_j\| = q_j p_j\tfrac{1}{2}\| \b{A}^{(t)}_j - \b{A}^{*}_j\|^2,
		%|p_iq_i(\lambda_i-1)\b{A}^{(t)}_i| \leq  p_iq_i |(\lambda_i - 1)|.
		\end{align*}
		Further, using Claim~\ref{iht:bound_vareps}
		\begin{align*}
		\|\Delta^{(t)}_j\| %& \leq C\tfrac{k^4}{m} p_{i}\tfrac{\mu_t}{\sqrt{n}}\epsilon_t\\
		%& \leq C(m-1)q_{i,j}p_{i_1}\epsilon_t\|\b{A}^{(t)}\|
		& = \mathcal{O}(\sqrt{m}q_{i,j}p_{i_1}\epsilon_t\|\b{A}^{(t)}\|).
		\end{align*}
		%\hl{lemma closeness}
		Now, since \[ \|\b{A}^{(t)} - \b{A}^*\| \leq 2\|\b{A}^*\|\] (the closeness property (Def.\ref{def:del_kappa}) is maintained at every step using Lemma~\ref{lem:closeness}), and further since \[\|\b{A}^*\| = \mathcal{O}(\sqrt{m/n})\], we have that 
		\begin{align*}
		\|\b{A}^{(t)}\| \leq \|\b{A}^{(t)} - \b{A}^*\| + \|\b{A}^*\| = \mathcal{O}(\sqrt{\tfrac{m}{n}}).
		\end{align*}
		Therefore, we have
			\begin{align*}
					\|\Delta^{(t)}_j\| + o(\tfrac{k}{m}\epsilon_t) \pm \gamma = \mathcal{O}(\sqrt{m}q_{i,j}p_{i_1}\epsilon_t\|\b{A}^{(t)}\|). %\leq \tfrac{q_ip_i}{2}\| \b{A}^{(t)}_j - \b{A}^{*}_j\|.
			\end{align*}
			Here, we use the fact that \[\gamma\] drops with decreasing \[\epsilon_t\] as argued in Lemma~\ref{grad_exp}. Next, using \eqref{eq:grad_hat_decomp}, we have
						\begin{align*}
						\|\hat{\b{g}}^{(t)}_j\| \leq 4\rho_{\_} \|\b{A}^{(t)}_j - \b{A}^*_j\|+ \|\tilde{v}\|.
						\end{align*}
			Now, letting
				\begin{align}\label{var_eps_cond}
					\|\Delta^{(t)}_j\| + o(\tfrac{k}{m}\epsilon_t) \pm \gamma = \mathcal{O}(\sqrt{m}q_{i,j}p_{i_1}\epsilon_t\|\b{A}^{(t)}\|) \leq \tfrac{q_ip_i}{2}\| \b{A}^{(t)}_j - \b{A}^{*}_j\|,
					\end{align}
		   we have that, for \[k = \c{O} (\sqrt{n})\]
					%Therefore, for \[\|\Delta^{(t)}_j\| + o(\tfrac{k}{m}\epsilon_t) \pm \gamma \leq \tfrac{q_ip_i}{2}\| \b{A}^{(t)}_j - \b{A}^{*}_j\|\], 
					\begin{align*}
					\|\tilde{v}\| \leq q_ip_i\| \b{A}^{(t)}_j - \b{A}^{*}_j\|.
					\end{align*}

				Substituting for \[\|\tilde{v}\|\], this implies that \[\|\hat{\b{g}^{(t)}}_j\|^2 \leq 25\rho_{\_}^2 \|\b{A}^{(t)}_j - \b{A}^*_j\|^2\]. Further, we also have the following lower-bound
				\begin{align*}
				\langle \hat{\b{g}}^{(t)}_j, \b{A}^{(t)}_j - \b{A}^*_j\rangle \geq 4\rho_{\_} \|\b{A}^{(t)}_j - \b{A}^*_j\|^2 - \|\tilde{v}\| \|\b{A}^{(t)}_j - \b{A}^*_j\|.
				\end{align*}
				Here, we use the fact that R.H.S. can be minimized only if \[\tilde{v}\] is directed opposite to the direction of \[\b{A}^{(t)}_j - \b{A}^*_j\]. Now, we show that this gradient is \[(\rho_{\_}, 1/100\rho_{\_}, 0)\] correlated,
				\begin{align*}
				\langle\hat{\b{g}}^{(t)}_i, \b{A}^{(t)}_i - &\b{A}^*_i\rangle - \rho_{\_} \|\b{A}^{(t)}_i - \b{A}^*_i\|^2 - \tfrac{1}{100\rho_{\_}}\|\hat{\b{g}}^{(t)}_i\|^2, \\
				&\geq 4\rho_{\_} \|\b{A}^{(t)}_i - \b{A}^*_i\|^2 - \|\tilde{v}\| \|\b{A}^{(t)}_i - \b{A}^*_i\|  - \rho_{\_} \|\b{A}^{(t)}_i - \b{A}^*_i\|^2 - \tfrac{1}{100\rho_{\_}}\|\hat{\b{g}}^{(t)}_i\|^2, \\
				&\geq 4\rho_{\_} \|\b{A}^{(t)}_i - \b{A}^*_i\|^2 - 2\rho_{\_} \|\b{A}^{(t)}_i - \b{A}^*_i\|^2 -\tfrac{25\rho_{\_}^2 \|\b{A}^{(t)}_i - \b{A}^*_i\|^2}{100\rho_{\_}},\\
				&\geq \rho_{\_} \|\b{A}^{(t)}_i - \b{A}^*_i\|^2  \geq 0.
				\end{align*}
			Therefore, for this choice of \[k\], i.e.  \[k = \c{O} (\sqrt{n})\], there is no bias in dictionary estimation in comparison to \cite{Arora15}. This gain can be attributed to estimating the coefficients simultaneously with the dictionary.
			%From Lemma~\ref{lem:grad_alpha_beta} we have that \[\b{g}^{(t)}_j\] is \[(\rho_{\_}, 1/100\rho_{\_}, 0)\]-correlated with \[\b{A}^*_j\]. 
			Further, since we choose \[4\rho_{\_} = p_jq_j\], we have that \[\rho_{\_} = \Theta(k/m)\], as a result \[\rho_{_+} = 1/100\rho_{\_} = \Omega(m/k)\]. Applying Lemma~\ref{arora:thm40} we have 
				\begin{align*}
				\|\b{A}^{(t+1)}_j- \b{A}^*_j\|^2 \leq (1 - \rho_{\_}\eta_A)\|\b{A}^{(t)}_j- \b{A}^*_j\|^2,
				\end{align*}
				for \[\eta_A = \mathcal{O}(m/k)\] with probability at least \[(1 - \delta_{\HT}^{(t)} - \delta_{\beta}^{(t)} - \delta_{\gradvec}^{(t)})\].
\end{proof}
% % % % % % % % % % % % % % % % % % % end of gradient alpha beta correlated % % % % % % % % % % % % % % % % % % % % % % % % % %

%% % % % % % % % % % % % % % % % % % % % % % % % % % End of Apply Alt Min % % % % % % % % % % % % % % % % % % % % % % % % % % % % % % 

% % % % % % % % % % % % % % % % % % % % Each Empirical Gradient matrix and Expected value of each Gradient matrix % % % % % % %

%\noindent\textbf{Lemma~\ref{lem:grad_mat}} \textbf{(Concentration of the empirical gradient matrix).}
%\begin{em}
%With probability \[(1 - \delta_{\beta} - \delta_{\HT}^{(t)} - \delta_{\rm HW}^{(t)} - \delta_\gradmat)\], 
%	\[\|\hat{\b{g}}^{(t)} -  \b{g}^{(t)}\|\] is upper-bounded by \[ \mathcal{O}^*(\tfrac{k}{m}  \|\b{A}^*\| )\], 	where \[\delta_{\gradmat}^{(t)} = (n+m)\exp(-\Omega(m\sqrt{\log(n)})\].
%\end{em}
\begin{proof} [Proof of Lemma~\ref{lem:grad_mat}]
	Here, we will prove that \[\hat{\b{g}}^{(t)}\] defined as
	\begin{align*}
	\hat{\b{g}}^{(t)} =  \textstyle\sum_{j} (\b{y}_{(j)}-\b{A}^{(t)}\hat{\b{x}}_{(j)}) \sgn(\hat{\b{x}}_{(j)})^\top,
	\end{align*}
	concentrates around its mean. Notice that each summand \[(\b{y}_{(j)}-\b{A}^{(t)}\hat{\b{x}}_{(j)}) \sgn(\hat{\b{x}}_{(j)})^\top\] is a random matrix of the form \[(\b{y}-\b{A}^{(t)}\hat{\b{x}})\sgn(\hat{\b{x}})^\top\]. %, where \[\b{y} \in \mathbb{R}^{n}\] is formed as \[\b{y} = \b{A}^*\b{x}^*\], with \[\b{x}^* \in \mathbb{R}^{m}\] is a sparse random vector with support \[S = \supp(\b{x}^*)\] distributed according to A.\ref{Assumption:dist_X}. Therefore,
	Also, we have \[\b{g}^{(t)}\] defined as
	\begin{align*}
	\b{g}^{(t)} = \mathbf{E}[(\b{y}-\b{A}^{(t)}\hat{\b{x}})\sgn(\hat{\b{x}})^\top]. %= \mathbf{E}[(y-A^sx) sgn(x)^T] ,
	\end{align*}
	To bound \[\|\hat{\b{g}}^{(t)} - \b{g}^{(t)}\|\], we are interested in \[\|\sum_{j=1}^{p} \b{W}_j\|\], where each matrix \[\b{W}_j\] is given by
	\begin{align*}
	\b{W}_j = \tfrac{1}{p} (\b{y}_{(j)}-\b{A}^{(t)}\hat{\b{x}}_{(j)}) \sgn(\hat{\b{x}}_{(j)})^\top- \tfrac{1}{p} \mathbf{E}[(\b{y}-\b{A}^{(t)}\hat{\b{x}})\sgn(\hat{\b{x}})^\top].
	\end{align*}
	Noting that \[\mathbf{E}[\b{W}_j] = 0\], we will employ the matrix Bernstein result (Lemma~\ref{matrix_bernstein}) to bound \[\|\hat{\b{g}}^{(t)} - \b{g}^{(t)}\|\]. To this end, we will bound \[\|\b{W}_j\|\] and the variance proxy
	\begin{align*}
	v(\b{W}_j) = \max \{\|\textstyle\sum_{j=1}^{p} \mathbf{E}[\b{W}_j\b{W}_j^\top]\|, \|\textstyle\sum_{j=1}^{p} \mathbf{E}[\b{W}_j^\top \b{W}_j]\| \}.
	\end{align*}

	\noindent\textbf{Bound on \[\|\b{W}_j\|\]}-- First, we can bound both terms in the expression for \[\b{W}_j\]  by triangle inequality as
	\begin{align*}
	\|\b{W}_j\| %= \tfrac{1}{p}\|\tfrac{D_{\Omega_j}}{\gamma}  (Y_j-A^sX_j) sgn(X_j)^T- \mathbf{E}[(y-A^sx) sgn(x)^T]\|\\
	%& \leq \tfrac{1}{p} \|\tfrac{D_{\Omega_j}}{\gamma} (Y_j-A^sX_j) sgn(X_j)^T\| +  \tfrac{1}{p}\|\mathbf{E}[(y-A^sx) sgn(x)^T]\|\\
	&\leq \tfrac{1}{p} \|(\b{y}_{(j)}-\b{A}^{(t)}\hat{\b{x}}_{(j)}) \sgn(\hat{\b{x}}_{(j)})^\top\| + \tfrac{1}{p} \|\mathbf{E}[(\b{y}-\b{A}^{(t)}\hat{\b{x}})\sgn(\hat{\b{x}})^\top\|,\\
	& \leq \tfrac{2}{p}\|(\b{y}-\b{A}^{(t)}\hat{\b{x}})\sgn(\hat{\b{x}})^\top\|.
	\end{align*}
	Here, we use Jensen's inequality for the second term, followed by upper-bounding the expected value of the argument by \[\|(\b{y}-\b{A}^{(t)}\hat{\b{x}}) \sgn(\hat{\b{x}})^\top\|\].
	
	Next, using Claim~\ref{norm_bound_y_Ax_s} we have that with probability at least \[(1 - \delta_{\beta}^{(t)} - \delta_{\HT}^{(t)} - \delta_{\rm HW}^{(t)})\],  \[\|\b{y}-\b{A}^{(t)}\hat{\b{x}}\|\] is %\[\textcolor{red}{\mathcal{O}( k^3\tfrac{\mu_t}{\sqrt{n}}\epsilon^2_t)}\] 
	\[\siri{\tilde{\mathcal{O}}(kt_{\beta})}\], and the fact that \[\| \sgn(x)^T\| = \sqrt{k}\],
	\begin{align*}
	\|\b{W}_j\| &\leq \tfrac{2}{p}\sqrt{k}\|(\b{y}-\b{A}^{(t)}\hat{\b{x}})\|  = 
	%\textcolor{red}{\mathcal{O}(\tfrac{k^3\sqrt{k}}{p} \tfrac{\mu_t}{\sqrt{n}}\epsilon^2_t)} =
	 \siri{\mathcal{O}(\tfrac{k\sqrt{k}}{p}t_{\beta})}.
	\end{align*}
	\noindent\textbf{Bound on  the variance statistic \[v(\b{W}_j)\]}-- For the variance statistic, we first look at \[\|\textstyle\sum \mathbf{E}[\b{W}_j\b{W}_j^\top]\|\], 
	\begin{align*}
	\mathbf{E}[\b{W}_j\b{W}_j^\top] = %E[(\tfrac{D_{\mathcal{S}_j}}{\gamma p} (y^{(j)}-A^sx^{(j)}) sgn(x^{(j)})^T - \tfrac{1}{p}E[(y-A^sx) sgn(x)^T])(\tfrac{D_{\mathcal{S}_j}}{\gamma p} (y^{(j)}-A^sx^{(j)}) sgn(x^{(j)})^T - \tfrac{1}{p}E[(y-A^sx) sgn(x)^T])^T] \\
	%&=\tfrac{1}{p^2}(\tfrac{D_{\mathcal{S}_j}}{\gamma} (y^{(j)}-A^sx^{(j)}) sgn(x^{(j)})^T - E[(y-A^sx) sgn(x)^T])( sgn(x^{(j)}) (y^{(j)}-A^sx^{(j)})^T\tfrac{D_{\mathcal{S}_j}}{\gamma } - E[(y-A^sx) sgn(x)^T]^T)\\
	%& = \tfrac{1}{p^2}  (\tfrac{D_{\mathcal{S}_j}}{\gamma} (y^{(j)}-A^sx^{(j)}) sgn(x^{(j)})^Tsgn(x^{(j)}) (y^{(j)}-A^sx^{(j)})^T\tfrac{D_{\mathcal{S}_j}}{\gamma } - 2 \tfrac{D_{\mathcal{S}_j}}{\gamma} [(y^{(j)}-A^sx^{(j)}) sgn(x^{(j)})^T]E[(y-A^sx) sgn(x)^T]^T + E[(y-A^sx) sgn(x)^T]E[(y-A^sx) sgn(x)^T]^T)\\
	%& = \tfrac{1}{p^2}  E[(\tfrac{D_{\mathcal{S}_j}}{\gamma} (y^{(j)}-A^sx^{(j)}) sgn(x^{(j)})^Tsgn(x^{(j)}) (y^{(j)}-A^sx^{(j)})^T\tfrac{D_{\mathcal{S}_j}}{\gamma }] - 2 E[\tfrac{D_{\mathcal{S}_j}}{\gamma}] E[(y^{(j)}-A^sx^{(j)}) sgn(x^{(j)})^T]E[(y-A^sx) sgn(x)^T]^T + E[(y-A^sx) sgn(x)^T]E[(y-A^sx) sgn(x)^T]^T)\\
	\tfrac{1}{p^2}  \mathbf{E}[(\b{y}_{(j)}-\b{A}^{(t)}&\hat{\b{x}}_{(j)})\sgn(\hat{\b{x}}_{(j)})^\top  - \mathbf{E}[(\b{y}-\b{A}^{(t)}\hat{\b{x}}) \sgn(\hat{\b{x}})^\top]\\
	&\times [\sgn(\hat{\b{x}}_{(j)}) (\b{y}_{(j)}-\b{A}^{(t)}\hat{\b{x}}_{(j)}) ^\top -(\mathbf{E}[(\b{y}-\b{A}^{(t)}\hat{\b{x}}) \sgn(\hat{\b{x}})^\top)^\top].
	\end{align*}%
	Since \[\b{E}[(\b{y}-\b{A}^{(t)}\hat{\b{x}}) \sgn(\hat{\b{x}})^\top]\b{E}[(\b{y}-\b{A}^{(t)}\hat{\b{x}}) \sgn(\hat{\b{x}})^\top]^\top\] is positive semidefinite, % and \[\mathbf{E}[D_{\mathcal{S}_j}]  =\gamma I\],
	\begin{align*}
	\mathbf{E}[\b{W}_j\b{W}_j^\top] \preceq \tfrac{1}{p^2}  \mathbf{E}[(\b{y}_{(j)}-\b{A}^{(t)}\hat{\b{x}}_{(j)}) \sgn(\hat{\b{x}}_{(j)})^\top \sgn(\hat{\b{x}}_{(j)}) (\b{y}_{(j)}-\b{A}^{(t)}\hat{\b{x}}_{(j)})^\top].
	\end{align*}
	Now, since each \[\hat{\b{x}}_{(j)}\] has \[k\] non-zeros, \[\sgn(\hat{\b{x}}_{(j)})^\top\sgn(\hat{\b{x}}_{(j)}) = k\], and using Claim~\ref{norm_exp_yAs_yAst}, with probability at least \[(1-\delta_{\HT}^{(t)} - \delta_{\beta}^{(t)})\]
	\begin{align*}
	\|\textstyle\sum \mathbf{E}[\b{W}_j\b{W}_j^\top]\|
	&\leq  \tfrac{k}{p } \|\mathbf{E}[(\b{y}_{(j)}-\b{A}^{(t)}\hat{\b{x}}_{(j)}) (\b{y}_{(j)}-\b{A}^{(t)}\hat{\b{x}}_{(j)})^\top]\|, \\ 
	%&= \textcolor{green}{\mathcal{O}(\tfrac{k^2}{pm})\|\b{A}^*\|^2 ~or~ }
%	&=\mathcal{O}(\tfrac{k^2}{pm})\|\b{A}^*\|^2 
	&=\mathcal{O}(\tfrac{k^3t_{\beta}^2}{pm})\|\b{A}^*\|^2.
	\end{align*}
	Similarly, expanding \[\mathbf{E}[\b{W}_j^\top\b{W}_j] \], and using the fact that \[\mathbf{E}[(\b{y}-\b{A}^{(t)}\hat{\b{x}}) \sgn(\hat{\b{x}})^\top]^\top\mathbf{E}[(\b{y}-\b{A}^{(t)}\hat{\b{x}}) \sgn(\hat{\b{x}})^\top]\] is positive semi-definite.  Now, using Claim~\ref{norm_bound_y_Ax_s} and the fact that entries of \[\mathbf{E}[(\sgn(\hat{\b{x}}_{(j)}) \sgn(\hat{\b{x}}_{(j)})^\top]\] are \[q_i\] on the diagonal and zero elsewhere, where \[q_i = \mathcal{O}(k/m)\], 
	\begin{align*}
	\|\textstyle\sum \mathbf{E}[\b{W}_j^\top\b{W}_j]\| 
	&\preceq\tfrac{1}{p}\|\mathbf{E}[(\sgn(\hat{\b{x}}_{(j)})(\b{y}_{(j)}-\b{A}^{(t)}\hat{\b{x}}_{(j)})^\top (\b{y}_{(j)}-\b{A}^{(t)}\hat{\b{x}}_{(j)}) \sgn(x^{(R)}_{(j)})^\top]\|,\\
	&\leq \tfrac{1}{p}\|\mathbf{E}[(\sgn(\hat{\b{x}}_{(j)}) \sgn(\hat{\b{x}}_{(j)})^\top] \| \|\b{y}_{(j)}-\b{A}^{(t)}\hat{\b{x}}_{(j)}\|^2,\\
	% &\leq \tfrac{1-\gamma}{p^2\gamma}\mathbf{E}[(sgn(x) sgn(x)^T] \tilde{\mathcal{O}}(k^2)\\
	%&\leq  \textcolor{red}{\tilde{\mathcal{O}}(\tfrac{k^2}{mp}) \mathcal{O}( k^6\tfrac{\mu_t^2}{n}\epsilon^4_t)}
    &\leq  \siri{\mathcal{O}(\tfrac{k}{mp}) \tilde{\mathcal{O}}(k^2t_{\beta}^2)} = \siri{\tilde{\mathcal{O}}(\tfrac{k^3t_{\beta}^2}{mp}) }.
	\end{align*}
	%
	%      Finally,
	%       \begin{align*}
	%       \|\sum E[W_j^TW_j]\| \leq \mathcal{O}(\tfrac{k^4 (1-\gamma)}{pm\gamma}) 
	%       \end{align*}
	Now, we are ready to apply the matrix Bernstein result. Since, \[m = O(n)\] the variance statistic comes out to be \[\mathcal{O}(\tfrac{k^3t_{\beta}^2}{pm})\|\b{A}^*\|^2\], then as long as we choose \[p = \tilde{\Omega}(mk^2)\] (using the bound on \[t_\beta\]), with probability at least \[(1 - \delta_{\beta}^{(t)} - \delta_{\HT}^{(t)} - \delta_{\rm HW}^{(t)} - \delta_\gradmat^{(t)})\]
	\begin{align*}
	\|\hat{\b{g}}^{(t)} -  \b{g}^{(t)}\| &\leq 	\mathcal{O}(\tfrac{k\sqrt{k}}{p}t_{\beta}) + %\textcolor{red}{\|\b{A}^*\|\sqrt{ \tfrac{k^2}{mp}})} +  
\|\b{A}^*\|\sqrt{\mathcal{O}(\tfrac{k^3t_{\beta}^2}{pm})},\\ 
	&= \mathcal{O}^*(\tfrac{k}{m}  \|\b{A}^*\| ).
	\end{align*}
	where \[\delta_{\gradmat}^{(t)} = (n+m)\exp(-\Omega(m\sqrt{\log(n)})\].
%	
%	 
%	\begin{align*}
%	\b{Pr}\big[ \|\textstyle\sum_{j} \b{W}_j\| \geq t\big] &\leq (n+m) {\rm exp}\big( \tfrac{-t_\gradmat^2/2}{\sigma^2 + Rt_\gradmat/3}\big), \\
%	&\leq (n+m){\rm exp}\big( \tfrac{-t_\gradmat^2/2}{\tfrac{k^2}{pm}\|\b{A}^*\|^2 + \mathcal{O}(\tfrac{k^3\sqrt{k}}{p} \epsilon_t)t_\gradmat/3}\big)
%	\end{align*}
%	%
%	
%    \[\delta_\gradmat =  (n+m){\rm exp}\big( \tfrac{-t_\gradmat^2/2}{\tfrac{k^2}{pm}\|\b{A}^*\|^2 + \mathcal{O}(\tfrac{k^3\sqrt{k}}{p} \epsilon_t)t_\gradmat/3}\big)\]
\end{proof}
% % % % % % % % % % % % % % % % % % % % End of Each Empirical Gradient matrix and Expected value of each Gradient matrix % % % % % % %

% % % % % % % % % % % % % % % % % % % % % % At+1 is not too far % % % % % % % % % % % % % % % % % % % % % % % %
% % % % % % % % % % % % % % % % % % % % % % At+1 is not too far % % % % % % % % % % % % % % % % % % % % % % % %
%\noindent\textbf{Lemma~\ref{lem:closeness}} (\[\b{A}^{(t+1)}\] \textbf{maintains closeness}).
%\begin{em}
%Suppose \[\b{A}^{(t)}\] is \[(\epsilon_t, 2)\] near to \[\b{A}^*\] with \[\epsilon_t = \mathcal{O}^*(1/\log(n))\], and number of samples used in step \[t\] is \[p = \tilde{\Omega}(mk^2)\], then with probability at least \[(1 - \delta_{\HT}^{(t)} - \delta_{\beta} 
%	- \delta_{\rm HW}^{(t)}- \delta_{\gradmat}^{(t)})\], \[\b{A}^{(t+1)}\] satisfies \[\|\b{A}^{(t+1)} - \b{A}^*\| \leq 2\|\b{A}^*\|\].
%	\end{em}
	%Lemma 42 Assuming $\|A^s - A^*\| \leq 2\|A^*\|$
\begin{proof}[Proof of Lemma~\ref{lem:closeness}]
	This lemma ensures that the dictionary iterates maintain the closeness property (Def.\ref{def:del_kappa}) and satisfies the prerequisites for Lemma~\ref{grad_corr}. 
	
	The update step for the \[i\]-th dictionary element at the \[s+1\] iteration can be written as
	\begin{align*}
	\b{A}_i^{(t+1)} - \b{A} _i^* &= \b{A}_i^{(t)} - \b{A}_i^* - \eta_A \hat{\b{g}}_i^{(t)},\\
	&=  \b{A}_i^{(t)} - \b{A}_i^*  - \eta_A \b{g}_i^{(t)} - \eta_A (\hat{\b{g}}_i^{(t)} - \b{g}_i^{(t)}).
	\end{align*}
	Here, \[\b{g}_i^{(t)}\] is given by the following as per Lemma~\ref{grad_exp} with probability at least \[(1 - \delta_{\HT}^{(t)} - \delta_{\beta}^{(t)})\]
	\begin{align*}
	\b{g}^{(t)}_i
	&= q_i p_i (\b{A}^{(t)}_i- \b{A}^*_i)  + q_i p_i(- \lambda^{(t)}_i\b{A}^{(t)}_i+ \tfrac{1}{q_i p_i}\Delta^{(t)}_i \pm \gamma).
	\end{align*}
	Substituting the expression for \[\b{g}_i^{(t)}\] in the dictionary update step,
	\begin{align*}
	\b{A}_i^{(t+1)} - \b{A} _i^* &= (1 - \eta_A p_i q_i) (\b{A}_i^{(t)} - \b{A}_i^*)  - \eta_A p_i q_i\lambda^{(t)}_i \b{A}^{(t)}_i 
	- \eta_A \Delta^{(t)}_i   - \eta_A (\hat{\b{g}}_i^{(t)} -  \b{g}_i^{(t)}) \pm \gamma,
	\end{align*}
	where \[\Delta^{(t)}_j= \b{E}[\b{A}^{(t)}\vartheta^{(R)}\sgn(\b{x}^*_j)] _j\].
	%is defined 
%	
%	\begin{align*}
%	%	\Delta^{(t)}_{j}&= \b{E}[\b{A}^{(t)}_S\vartheta^{(R)}_S\sgn(\b{x}^*_{j})] \mathcal{O}(\sqrt{m}q_{i,j}p_{j}\epsilon_t\|\b{A}^{(t)}\|)] \\
%%= \mathcal{O}(\tfrac{k^2 }{\sqrt{mn}}\epsilon_t)
%	\end{align*}
%	%
	%	
	%	
	%	
	%	\begin{align*}
	%	\|\Delta^{(t)}_j\|&= \|\b{E}_S[\b{A}^t_S\b{E}_{\b{x}_S^*}[\vartheta^{(R)}_S\sgn(\b{x}^*_j)|S]]\|,\\
	%	&\leq  (q_{i} + (m-1)q_{i,j})3p_{i_1}k^2\tfrac{\mu_t}{\sqrt{n}}\epsilon_t\|\b{A}^t\|\\
	%	& \leq Cq_{i,j}p_{i}mk^2\tfrac{\mu_t}{\sqrt{n}}\epsilon_t\|\b{A}^t\|
	%	\end{align*}
	%	%
	%	
	%	\begin{align*}
	%	\|\Delta^{(t)}_j\| \leq C\tfrac{k^4}{m} p_{i}\tfrac{\mu_t}{\sqrt{n}}\epsilon_t.
	%	\end{align*}
	%	%
	%	where \[\Delta_t^i = \b{E}_S[\b{A}^t_S\b{E}_{\b{x}_S^*}[\vartheta^{(R)}_S\sgn(\b{x}^*_i)|S]]\]. 
	Therefore, the update step for the dictionary (matrix) can be written as
	\begin{align}\label{eq:closeness}
	\b{A}^{(t+1)} - \b{A}^* 
	&=  (\b{A}^{(t)} - \b{A}^*)\diag({(1 - \eta_A p_i q_i)})  + \eta_A \b{U}  - \eta_A \b{V} \pm \gamma - \eta_A (\hat{\b{g}}^{(t)} -  \b{g}^{(t)}),
	\end{align}
	where, \[\b{U} =  \b{A}^{(t)} \diag({p_i q_i\lambda^{(t)}_i}) \] and \[\b{V}  = \b{A}^{(t)} \b{Q}\], with the matrix \[\b{Q}\] given by,
	\begin{align*}
	\b{Q}_{i,j} = q_{i,j}\b{E}_{\b{x}_S^*}[\vartheta^{(R)}_i\sgn(\b{x}^*_{j})|S],
	%\begin{cases}
	%0,& for ~~i = j\\
	%q_{ij} \langle A_j^s, A^*_i \rangle,& for ~~i \neq j.
	%\end{cases}
	\end{align*}
	and using  the following intermediate result shown in Claim~\ref{iht:bound_vareps},
%		
%		\begin{align*}
%		\b{E}_{\b{x}_S^*}[\vartheta^{(R)}_{i}\sgn(\b{x}^*_{j})|S] 
%		\begin{cases}
%		\leq 3p_{j}\epsilon_t+ \tfrac{\mu}{\sqrt{n}}(\lambda^{(t)}_{\max} |\b{x}_{\max}^*| +t_\beta)( \eta_x + kc_x ) + \gamma^{(R)}_i, & ~\text{for}~ i = j, \\
%		= \gamma^{(R)}_i,& ~\text{for}~ i \neq j.
%		\end{cases}
%		\end{align*}
%		
%	
%	\begin{align*}
%	\b{E}_{\b{x}_S^*}[\vartheta^{(R)}_{i}\sgn(\b{x}^*_{j})|S] 
%	\begin{cases}
%	\leq \tfrac{\mu}{\sqrt{n}}(\lambda^{(t)}_{\max} |\b{x}_{\max}^*| +t_\beta)( 1 + kc_x ) + \gamma^{(R)}_i, & ~\text{for}~ i = j, \\
%	 	\leq 3p_{j}\epsilon_t + \gamma^{(R)}_i,& ~\text{for}~ i \neq j.
%	\end{cases}
%	\end{align*}
%	
		\begin{align*}
		\b{E}_{\b{x}_S^*}[\vartheta^{(R)}_{i}\sgn(\b{x}^*_{j})|S] 
		\begin{cases}
		\leq  \gamma^{(R)}_i, & ~\text{for}~ i = j, \\
		 = \mathcal{O}(p_{j}\epsilon_t ),& ~\text{for}~ i \neq j,
		\end{cases}
		\end{align*}
	  we have \[\|\b{Q}_i\| = \mathcal{O}(\sqrt{m}q_{i,j}p_{i}\epsilon_t)\]. Hence, we have
	%	
	%	\begin{align*}
	%	\|\b{Q}_i\| % &\leq Cq_{i,j}p_{i}mk^2\tfrac{\mu_t}{\sqrt{n}}\epsilon_t,\\
	%	& \leq \mathcal{O}(\sqrt{m}q_{i,j}p_{i_1}\epsilon_t)
	%	\end{align*}
	%	%
	\begin{align*}
	\|\b{Q}\|_F \leq   \mathcal{O}(mq_{i,j}p_{i}\epsilon_t).
	%\sqrt{m}Cq_{i,j}p_{i}mk^2\tfrac{\mu_t}{\sqrt{n}}\epsilon_t.
	\end{align*}
	Therefore,
	\begin{align*}
	\|\b{V}\| \leq \|\b{A}^{(t)} \b{Q}\| \leq 	\|\b{A}^{(t)}\| \|\b{Q}\|_F = \mathcal{O}(mq_{i,j}p_{i}\epsilon_t \|\b{A}^*\|) = \mathcal{O}(\tfrac{k^2}{m \log(n)}) \|\b{A}^*\|. %= o( \eta_A \tfrac{k}{m})\|\b{A}^*\| .
	\end{align*}
	%		
	%	
	%	\begin{align*}
	%	\|\b{V}\| \leq \|\b{A}^{(t)} \b{Q}\| \leq 	\|\b{A}^{(t)}\| \|\b{Q}\|_F \leq \sqrt{m}Cq_{i,j}p_{i}mk^2\tfrac{\mu_t}{\sqrt{n}}\epsilon_t \|\b{A}^*\| = \mathcal{O}(\tfrac{k^4}{\sqrt{m}})\tfrac{\mu_t}{\sqrt{n}}\epsilon_t \|\b{A}^*\| = o( \eta_A \tfrac{k}{m})\|\b{A}^*\| .
	%	\end{align*}
	%	%		
	We will now proceed to bound each term in \eqref{eq:closeness}. Starting with \[(\b{A}^{(t)} - \b{A}^*)\diag{(1 - \eta_A p_i q_i)}\], and using the fact that \[p_i = O(1)\], \[q_i = O(k/m)\], and \[\|\b{A}^{(t)} - \b{A}^*\| \leq 2\|\b{A}^*\|\], we have
	\begin{align*}
	\|(\b{A}^{(t)} - \b{A}^*)\diag{(1 - \eta_A p_i q_i)}\| &\leq (1 - \underset{i}{\min}~\eta_A p_i q_i)\|(\b{A}^{(t)} - \b{A}^*)\| 
	\leq 2(1 -  \Omega(\eta_A k/m))\|\b{A}^*\|.
	\end{align*}
	Next, since \[\|\b{A}^{(t)}_j\| = 1\], we have \[\lambda^{(t)}_{j} = |\langle \b{A}^{(t)}_j - \b{A}^{*}_j, \b{A}^*_j\rangle| = \tfrac{1}{2}\| \b{A}^{(t)}_j - \b{A}^{*}_j\|^2\], and \[\lambda^{(t)}_i \leq \epsilon^2_t/2\], therefore
	\begin{align*}
	\|\b{U}\| &= \|\b{A}^{(t)} \diag({p_i q_i\lambda^{(t)}_i}) \| \leq \underset{i}{\max}~p_i q_i\tfrac{\epsilon_t^2}{2}\|\b{A}^{(t)} - \b{A}^* + \b{A}^*\| \leq o(k/m) \| \b{A}^*\|.
	\end{align*}
	Using the results derived above, and the the result derived in Lemma~\ref{lem:grad_mat} which states that with probability at least \[(1 - \delta_{\beta}^{(t)} - \delta_{\HT}^{(t)} - \delta_{\rm HW}^{(t)} - \delta_\gradmat^{(t)})\], \[\|\hat{\b{g}}^{(t)} - \b{g}^{(t)}\| = \mathcal{O}^*(\tfrac{k}{m}  \|\b{A}^*\| ) )\] we have
	\begin{align*}
	\|\b{A}^{(t+1)} - &\b{A}^*\|= \|(\b{A}^{(t)} - \b{A}^*)\b{D}_{(1 - \eta_A p_i q_i)}\|  + \eta_A \|\b{U}\|  + \eta_A \|\b{V}\|  + \eta_A \|(\hat{\b{g}}^{(t)} -  \b{g}^{(t)})\| \pm \gamma,\\
	&\leq 2(1 -  \Omega(\eta_A \tfrac{k}{m}) \|\b{A}^*\|+    o(\eta_A \tfrac{k}{m})\|\b{A}^*\| +   \mathcal{O}( \eta_A \tfrac{k^2}{m\log(n)})\|\b{A}^*\|  + o(\eta_{A}\tfrac{k}{ m} \|\b{A}^*\| ) \pm \gamma,\\
	%	&\leq 2(1 -  \Omega(\eta_A \tfrac{k}{m})\|\b{A}^*\| +    o(\eta_A \tfrac{k}{m})\|\b{A}^*\|  + o(\eta_A\tfrac{\sqrt{k}}{ m}) )\|\b{A}^*\| \pm \gamma\\
	&\leq 2\|\b{A}^*\|.
	\end{align*}
\end{proof}
\clearpage

\section{Appendix: Proofs of intermediate results}\label{app:proof of intermediate results}
% % % % % % % % % % % % % % % % % % incoherence of B % % % % % % % % % % % % % %
\begin{claim}[\textbf{Incoherence of \[\b{A}^{(t)}\]}]\label{claim:incoherence of B}\textit{
		If \[\b{A^*} \in \mathbb{R}^{n \times m}\] is \[\mu\]-incoherent and \[\|\b{A}^*_i - \b{A}^{(t)}_i\| \leq \epsilon_t \]  holds for each \[i \in [1\dots m]\], then \[\b{A}^{(t)}\in \mathbb{R}^{n \times m}\] is \[\mu_t\]-incoherent, where \[\mu_t = \mu + 2\sqrt{n}\epsilon_t\].}%and columns of \[\b{A}^*\] and \[\b{A}^*\] are normalized to \[1\].}
\end{claim}
\begin{proof}[ Proof of Claim \ref{claim:incoherence of B}]
	We start by looking at the incoherence between the columns of \[\b{A}^*\], for \[j \neq i\],
	\begin{align*}
	\langle \b{A}^*_i, \b{A}^*_j \rangle &= \langle \b{A}^*_i - \b{A}^{(t)}_i, \b{A}^*_j \rangle + \langle \b{A}^{(t)}_i, \b{A}^*_j  \rangle,\\
	&= \langle \b{A}^*_i - \b{A}^{(t)}_i, \b{A}^*_j \rangle + \langle \b{A}^{(t)}_i,  \b{A}^*_j -  \b{A}^{(t)}_j \rangle +  \langle  \b{A}^{(t)}_i ,  \b{A}^{(t)}_j \rangle.
	\end{align*}
	Since \[\langle \b{A}^*_i, \b{A}^*_j \rangle \leq \tfrac{\mu}{\sqrt{n}}\], 
	\begin{align*}
	|\langle \b{A}^{(t)}_i, \b{A}^{(t)}_j\rangle| &\leq \langle \b{A}^*_i, \b{A}^*_j \rangle - \langle \b{A}^*_i - \b{A}^{(t)}_i, \b{A}^*_j \rangle - \langle \b{A}^{(t)}_i, \b{A}^*_j - \b{A}^{(t)}_j\rangle,\\
	&\leq \tfrac{\mu}{\sqrt{n}} + 2\epsilon_t.
	\end{align*}
	%
	%	choosing \[\epsilon_t \leq \mathcal{O}(1/k)\], we have that \[\mu' = \mu + \mathcal{O}(\sqrt{n}/k)\].
\end{proof}
% % % % % % % % % % % % % % % % % % % end of incoherence of B % % % % % % % % % %

% % % % % % % % % % % % % % % % % bound on beta % % % % % % % % % % % % % % % % % % % % % % % % % % % %
\begin{claim}[\textbf{Bound on \[\beta_j^{(t)}\]: the noise component in coefficient estimate that depends on \[\epsilon_t\]}]\label{lem:bound_beta}
\begin{em}
	With probability \[(1-\delta_{\beta}^{(t)})\], \[|\beta^{(t)}_j|\] is upper-bounded by \[t_\beta = \mathcal{O}(\sqrt{k\epsilon_t})\], where \[\delta_{\beta}^{(t)} = 2k~{\exp}(-\tfrac{1}{\mathcal{O}(\epsilon_t)})\].
\end{em}
\end{claim}
\begin{proof}[Proof of Claim~\ref{lem:bound_beta}]
	We have the following definition for \[\beta^{(t)}_j\] from \eqref{eq:beta_t},
	\begin{align*}
	\beta^{(t)}_j
	& = \textstyle\sum\limits_{i \neq j} (\langle \b{A}^{*}_j,  \b{A}^*_i - \b{A}^{(t)}_i\rangle  + \langle \b{A}^*_j -\b{A}^{(t)}_j, \b{A}^{(t)}_i\rangle + \langle \b{A}^{(t)}_j - \b{A}^{*}_j, \b{A}^*_i\rangle) \b{x}_i^*.
	\end{align*}%
	Here, since \[\b{x}_i^*\] are independent sub-Gaussian random variables, \[\beta^{(t)}_j\] is a sub-Gaussian random variable with the variance parameter evaluated as shown below
	\begin{align*}
	var[\beta^{(t)}_j] =  \textstyle\sum_{i \neq j} (\langle \b{A}^{*}_j,  \b{A}^{(t)}_i - \b{A}^*_i\rangle  + \langle \b{A}^{(t)}_j -\b{A}^*_j, \b{A}^{(t)}_i\rangle + \langle \b{A}^{(t)}_j - \b{A}^{*}_j, \b{A}^*_i\rangle)^2 \leq 9k\epsilon_t^2.
	\end{align*}
	Therefore, by Lemma~\ref{theorem:subg_chern} 
	\begin{align*}
	\b{Pr}[|\beta^{(t)}_j|>t_{\beta}] 
	%	&\leq 2{\rm exp}(-\tfrac{t_{\beta}^2}{2\sigma^2}),\\
	&\leq 2{\rm exp}(-\tfrac{t_{\beta}^2}{18k\epsilon_t^2}).
	\end{align*}
	Now, we need this for each \[\beta^{(t)}_j\] for \[j \in \supp(\b{x}^*)\],  union bounding over \[k\] coefficients
	\begin{align*}
	\b{Pr}[\max|\beta^{(t)}_j|>t_{\beta}] 
	&\leq \delta_{\beta}^{(t)},
	\end{align*}
	where \[\delta_{\beta}^{(t)} = 2k~{\exp}(-\tfrac{t_{\beta}^2}{18k\epsilon_t^2})\]. Choosing \[t_\beta = \mathcal{O}(\sqrt{k\epsilon_t})\], we have that \[\delta_{\beta}^{(t)} = 2k~{\exp}(-\tfrac{1}{\mathcal{O}(\epsilon_t)})\]. %We choose \[t_\beta = \mathcal{O}(\sqrt{k}\epsilon_t)\], this ensures that \[\delta_{\beta}^{(t)} = 2k\exp(-\mathcal{O}(1))\]
\end{proof}
% % % % % % % % % % % % % % % % % End of bound on beta % % % % % % % % % % % % % % % % % % % % % % % % 

%----
% % % % % % % % % % % % % % % % % % % Begin of general term simple % % % % % % % % % % % % % % % % % % % % % %
\begin{claim}[\textbf{Error in coefficient estimation for a general iterate \[(r+1)\]}]\label{iht:gen_term_simple}
\begin{em}
The error in a general iterate \[r\]  of the coefficient estimation is upper-bounded as
	\begin{align*}
	C_{i_1}^{(r+1)} 
	&\leq \alpha^{(r+1)}_{i_1} + (k-1)\eta_x\tfrac{\mu_t}{\sqrt{n}} \textstyle\sum\limits_{\ell = 1}^{r}  \alpha^{(\ell)}_{\max}\big(1-\eta_x + \eta_x\tfrac{\mu_t}{\sqrt{n}} \big)^{r -\ell}
	\\&\hspace{7cm}+ (k-1)\eta_x\tfrac{\mu_t}{\sqrt{n}}C_{\max}^{(0)} \big(1-\eta_x + \eta_x\tfrac{\mu_t}{\sqrt{n}} \big)^r.
	\end{align*}
	\end{em}
\end{claim}
\begin{proof}[Proof of Claim~\ref{iht:gen_term_simple} ]
	From \eqref{iht:def_C_j_rplus1} we have the following expression for \[C_{i_1}^{(r+1)}\]
	\begin{align*}
	C_{i_1}^{(r+1)} 
	&\leq  \alpha^{(r+1)}_{i_1}  + \tfrac{\mu_t}{\sqrt{n}}\textstyle\sum\limits_{\ell=1}^{r+1} \eta_x^{(\ell)}\sum\limits_{i_2 \neq i_1} C_{i_2}^{(\ell-1)}\prod\limits_{q = \ell}^{r+1} (1-\eta_x^{(q+1)}).
	\end{align*}
	Our aim will be to recursively substitute for \[C_{i_1}^{(\ell-1)} \] to develop an expression for \[C_{i_1}^{(r+1)}\] as a function of \[C_{\max}^{0}\]. To this end, we start by analyzing the iterates \[C_{i_1}^{(1)}\], \[C_{i_1}^{(2)}\], and so on to develop an expression for \[C_{i_1}^{(r+1)}\] as follows. 
	
	\noindent\textbf{Expression for \[C_{i_1}^{(1)}\]} -- Consider \[C_{i_1}^{(1)}\]
	\begin{align}\label{eq:C_i_1}
	C_{i_1}^{(1)} 
	&\leq \textstyle \alpha^{(1)}_{i_1}  + \tfrac{\mu_t}{\sqrt{n}}\sum\limits_{\ell=1}^{1} \eta_x\sum\limits_{i_2 \neq i_1} C_{i_2}^{(\ell-1)}\prod\limits_{q = \ell}^{1} (1-\eta_x),\notag \\
	&=  \textstyle\alpha^{(1)}_{i_1}  +  \eta_x\big(\tfrac{\mu_t}{\sqrt{n}}\sum\limits_{i_1 \neq i_2} C_{i_2}^{(0)}\big).
	\end{align}

	\noindent\textbf{Expression for \[C_{i_1}^{(2)}\]}-- Next, \[C_{i_1}^{(2)}\] is given by
	\begin{align*}
	C_{i_1}^{(2)} 
	&\leq  \alpha^{(2)}_{i_1}  + \eta_x\tfrac{\mu_t}{\sqrt{n}}\textstyle\sum\limits_{\ell=1}^{2}\sum\limits_{i_2 \neq i_1} C_{i_2}^{(\ell-1)} \prod\limits_{q = \ell}^{2} (1-\eta_x),\\
	&\leq\textstyle \alpha^{(2)}_{i_1}  + \eta_x\tfrac{\mu_t}{\sqrt{n}}\big(\sum\limits_{i_2 \neq i_1} C_{i_2}^{(1)} + \sum\limits_{i_2 \neq i_1} C_{i_2}^{(0)}(1-\eta_x)  \big).
	\end{align*}
	Further, we know from \eqref{eq:C_i_1} we have
	\begin{align*}
	C_{i_2}^{(1)} 
	%&\leq  \alpha^{(1)}  + \sum_{\ell=1}^{1} \eta_x\big(\tfrac{\mu_t}{\sqrt{n}}\displaystyle\sum_{i_3 \neq i_2} C_{i_3}^{(0)}\big)\prod_{q = \ell}^{1} (1-\eta_x),\\
	&=  \alpha^{(1)}_{i_2}  +  \eta_x\tfrac{\mu_t}{\sqrt{n}}\textstyle\sum\limits_{i_3 \neq i_2} C_{i_3}^{(0)}.
	\end{align*}
	Therefore, since \[\sum\limits_{i_2 \neq i_1}\sum\limits_{i_3 \neq i_2} = \sum\limits_{i_3 \neq i_2, i_1}\],
	\begin{align}\label{eq:C_i_2}
	C_{i_1}^{(2)} 
	&\leq\textstyle \alpha^{(2)}_{i_1}  + \eta_x\tfrac{\mu_t}{\sqrt{n}}\big(\sum\limits_{i_2 \neq i_1} \big(\alpha^{(1)}_{i_2}  +  \eta_x\tfrac{\mu_t}{\sqrt{n}}\textstyle\sum\limits_{i_3 \neq i_2} C_{i_3}^{(0)}\big)+ \textstyle\sum\limits_{i_2 \neq i_1} C_{i_2}^{(0)}(1-\eta_x)\big),\notag\\
	&= \alpha^{(2)}_{i_1}  + \eta_x\tfrac{\mu_t}{\sqrt{n}}\textstyle\sum\limits_{i_2 \neq i_1} \alpha^{(1)}_{i_2}  +  \eta_x\tfrac{\mu_t}{\sqrt{n}}\big(\eta_x\tfrac{\mu_t}{\sqrt{n}}\textstyle\sum\limits_{i_3 \neq i_2, i_1} C_{i_3}^{(0)}+ \textstyle\sum\limits_{i_2 \neq i_1} C_{i_2}^{(0)}(1-\eta_x)\big).
	\end{align}
	%
	%Therefore, 
	%
	%
	%\begin{align}
	%C_{i_1}^{(2)} 
	%%&\leq  \alpha^{(2)}  + \eta_x\big(\tfrac{\mu_t}{\sqrt{n}}\displaystyle\sum_{i_2 \neq i_1} C_{i_2}^{(1)}\big)\sum_{\ell=1}^{2} \prod_{q = \ell}^{2} (1-\eta_x),\\
	%&\leq \alpha^{(2)}_{i_1}  + \eta_x\big(\tfrac{\mu_t}{\sqrt{n}}\displaystyle\sum_{i_2 \neq i_1} \big(    \alpha^{(1)}_{i_2}  +  \eta_x\tfrac{\mu_t}{\sqrt{n}}\displaystyle\sum_{i_3 \neq i_2} C_{i_3}^{(0)}  \big)\big)\sum_{\ell=1}^{2}(1-\eta_x)^{\ell - 1},\\
	%&= \alpha^{(2)}_{i_1}  + \eta_x\tfrac{\mu_t}{\sqrt{n}}\sum_{i_2 \neq i_1}\alpha^{(1)}_{i_2}\sum_{\ell=1}^{2}(1-\eta_x)^{\ell - 1}  +  \eta_x^2(\tfrac{\mu_t}{\sqrt{n}})^2\sum_{i_2 \neq i_1}\sum_{i_3 \neq i_2} C_{i_3}^{(0)}  \sum_{\ell=1}^{2}(1-\eta_x)^{\ell - 1}.
	%\end{align}
	%
	%since \[\sum_{i_2 \neq i_1}\sum_{i_3 \neq i_2} = \sum_{i_3 \neq i_2, i_1}\],
	%
	%
	%\begin{align}
	%C_{i_1}^{(2)} 
	%&= \alpha^{(2)}_{i_1}  + \eta_x\tfrac{\mu_t}{\sqrt{n}}\sum_{i_2 \neq i_1}\alpha^{(1)}_{i_2}\sum_{\ell=1}^{2}(1-\eta_x)^{\ell - 1}  +  \eta_x^2(\tfrac{\mu_t}{\sqrt{n}})^2\sum_{i_3 \neq i_2, i_1} C_{i_3}^{(0)}  \sum_{\ell=1}^{2}(1-\eta_x)^{\ell - 1}.
	%\end{align}
	%
	\noindent\textbf{Expression for \[C_{i_1}^{(3)}\]}-- Next, we writing \[C_{i_1}^{(3)}\],
	\begin{align*}
	C_{i_1}^{(3)} 
	%&\leq  \alpha^{(3)}  + \eta_x\big(\tfrac{\mu_t}{\sqrt{n}}\displaystyle\sum_{i_2 \neq i_1} C_{i_2}^{(2)}\big)\sum_{\ell=1}^{3} \prod_{q = \ell}^{3} (1-\eta_x),\\
	%&\leq  \alpha^{(r+1)}_j  + \sum_{\ell=1}^{r+1} \eta_x^{(\ell)}\big(\tfrac{\mu_t}{\sqrt{n}}\displaystyle\sum_{i \neq j} C_i^{(\ell-1)}\big)\prod_{q = \ell}^{r+1} (1-\eta_x^{(q+1)}).\\
	&\leq \alpha^{(3)}_{i_1}  + \eta_x\tfrac{\mu_t}{\sqrt{n}}\textstyle\sum_{\ell=1}^{3}\textstyle\sum\limits_{i_2 \neq i_1} C_{i_2}^{(\ell -1)}(1-\eta_x)^{3 - \ell},\\
	&= \alpha^{(3)}_{i_1}  + \eta_x\tfrac{\mu_t}{\sqrt{n}} \textstyle\sum\limits_{i_2 \neq i_1}\big(  C_{i_2}^{(0)}(1-\eta_x)^{2} + C_{i_2}^{(1)}(1-\eta_x) +  C_{i_2}^{(2)}\big),\\
	&\leq \alpha^{(3)}_{i_1}  + \eta_x\tfrac{\mu_t}{\sqrt{n}} \textstyle\sum\limits_{i_2 \neq i_1}\big(  C_{i_2}^{(0)}(1-\eta_x)^{2} + \big(\alpha^{(1)}_{i_2}  +  \eta_x\tfrac{\mu_t}{\sqrt{n}}\textstyle\sum\limits_{i_3 \neq i_2} C_{i_3}^{(0)}\big)(1-\eta_x) +  C_{i_2}^{(2)}\big).
	\end{align*}
	Here, using \eqref{eq:C_i_2} we have the following expression for \[C_{i_2}^{(2)}\]
	\begin{align*}
	C_{i_2}^{(2)} 
	&\leq \alpha^{(2)}_{i_2}  + \eta_x\tfrac{\mu_t}{\sqrt{n}}\textstyle\sum\limits_{i_3 \neq i_2} \alpha^{(1)}_{i_3}  +  \eta_x\tfrac{\mu_t}{\sqrt{n}}\big(\eta_x\tfrac{\mu_t}{\sqrt{n}}\textstyle\sum\limits_{i_4 \neq i_3, i_2} C_{i_4}^{(0)}+ \textstyle\sum\limits_{i_3 \neq i_2} C_{i_3}^{(0)}(1-\eta_x)\big).
	\end{align*}
	Substituting for \[C_{i_2}^{(2)}\] in the expression for \[C_{i_1}^{(3)}\],
%	
%	\begin{align*}
%	C_{i_1}^{(3)} 
%	&\leq \alpha^{(3)}_{i_1}  + \eta_x\tfrac{\mu_t}{\sqrt{n}} \displaystyle\sum_{i_2 \neq i_1}\bigg(  C_{i_2}^{(0)}(1-\eta_x)^{2} + \big(\alpha^{(1)}_{i_2}  +  \eta_x\tfrac{\mu_t}{\sqrt{n}}\displaystyle\sum_{i_3 \neq i_2} C_{i_3}^{(0)}\big)(1-\eta_x) +  \alpha^{(2)}_{i_2} \notag \\&~~~~~~~~~~~~~~~~~~~~~~~~~~~~~~~~~~~~~~~~~~~~
%	+ \eta_x\tfrac{\mu_t}{\sqrt{n}}\bigg(\displaystyle\sum_{i_3 \neq i_2} \alpha^{(1)}_{i_3}  +  \eta_x\tfrac{\mu_t}{\sqrt{n}}\displaystyle\sum_{i_4 \neq i_3, i_2} C_{i_4}^{(0)}+ \displaystyle\sum_{i_3 \neq i_2} C_{i_3}^{(0)}(1-\eta_x)\bigg)\bigg),\notag\\
%	&= \alpha^{(3)}_{i_1} + \eta_x\tfrac{\mu_t}{\sqrt{n}} \displaystyle\sum_{i_2 \neq i_1}\alpha^{(2)}_{i_2} +\eta_x\tfrac{\mu_t}{\sqrt{n}} \displaystyle\sum_{i_2 \neq i_1}\alpha^{(1)}_{i_2}(1-\eta_x)  + (\eta_x\tfrac{\mu_t}{\sqrt{n}})^2 \displaystyle\sum_{i_2 \neq i_1}\displaystyle\sum_{i_3 \neq i_2} \alpha^{(1)}_{i_3} \notag \\
%	&~~~~~~
%	+   \displaystyle\sum_{i_2 \neq i_1}\eta_x\tfrac{\mu_t}{\sqrt{n}}C_{i_2}^{(0)}(1-\eta_x)^{2} +  2(1-\eta_x)(\eta_x\tfrac{\mu_t}{\sqrt{n}})^2\displaystyle\sum_{i_3 \neq i_2, i_1} C_{i_3}^{(0)} +  (\eta_x\tfrac{\mu_t}{\sqrt{n}})^3\displaystyle\sum_{i_4 \neq i_3, i_2, i_1} C_{i_4}^{(0)}.
%	\end{align*}
%	%
	and rearranging the terms in the expression for \[C_{i_1}^{(3)} \], we have
	\begin{align}\label{eq:C_i_3}
	C_{i_1}^{(3)} 
	&\leq \alpha^{(3)}_{i_1} + \eta_x\tfrac{\mu_t}{\sqrt{n}} \textstyle\sum\limits_{i_2 \neq i_1}\alpha^{(2)}_{i_2} +\eta_x\tfrac{\mu_t}{\sqrt{n}}\big( (1-\eta_x) \textstyle\sum\limits_{i_2 \neq i_1}\alpha^{(1)}_{i_2} + \eta_x\tfrac{\mu_t}{\sqrt{n}} \textstyle\sum\limits_{i_3 \neq i_2, i_1} \alpha^{(1)}_{i_3}\big) \notag \\
	&
	+  \eta_x\tfrac{\mu_t}{\sqrt{n}}\big((1-\eta_x)^{2} \textstyle\sum\limits_{i_2 \neq i_1}C_{i_2}^{(0)} +  2(1-\eta_x)(\eta_x\tfrac{\mu_t}{\sqrt{n}})\textstyle\sum\limits_{i_3 \neq i_2, i_1} C_{i_3}^{(0)} +  (\eta_x\tfrac{\mu_t}{\sqrt{n}})^2\textstyle\sum\limits_{i_4 \neq i_3, i_2, i_1} C_{i_4}^{(0)}\big).
	\end{align}
	\noindent\textbf{Expression for \[C_{i_1}^{(4)}\]}--	Now, consider \[C_{i_1}^{(4)}\]
	\begin{align*}
	&C_{i_1}^{(4)} 
	\leq \alpha^{(4)}_{i_1}  + \eta_x\tfrac{\mu_t}{\sqrt{n}}\textstyle\sum_{\ell=1}^{4}\textstyle\sum_{i_2 \neq i_1} C_{i_2}^{(\ell -1)}(1-\eta_x)^{4 - \ell},\\
	%&\leq \alpha^{(4)}_{i_1}  + \eta_x\tfrac{\mu_t}{\sqrt{n}}\bigg(\sum_{i_2 \neq i_1} C_{i_2}^{(0)}(1-\eta_x)^{3} + \sum_{i_2 \neq i_1} C_{i_2}^{(1)}(1-\eta_x)^{2} + \sum_{i_2 \neq i_1} C_{i_2}^{(2)}(1-\eta_x)^{1}+ \sum_{i_2 \neq i_1} C_{i_2}^{(3)}(1-\eta_x)^{0}\bigg),\\
	&\leq \alpha^{(4)}_{i_1}  + \eta_x\tfrac{\mu_t}{\sqrt{n}}\big(\textstyle\sum_{i_2 \neq i_1} C_{i_2}^{(0)}(1-\eta_x)^{3} + \textstyle\sum_{i_2 \neq i_1} C_{i_2}^{(1)}(1-\eta_x)^{2} + \textstyle\sum_{i_2 \neq i_1} C_{i_2}^{(2)}(1-\eta_x)^{1} \\&\hspace{10cm}+ \textstyle\sum_{i_2 \neq i_1} C_{i_2}^{(3)}(1-\eta_x)^{0}\big).
	\end{align*}
	Substituting for \[C_{i_2}^{(3)}\] from \eqref{eq:C_i_3}, \[C_{i_2}^{(2)}\] from \eqref{eq:C_i_2}, \[C_{i_2}^{(1)}\] using \eqref{eq:C_i_1}, and rearranging,
	\begin{align*}
	&C_{i_1}^{(4)} 
	\leq \alpha^{(4)}_{i_1} + \eta_x\tfrac{\mu_t}{\sqrt{n}}\bigg[  \textstyle\sum\limits_{i_2 \neq i_1} \alpha^{(3)}_{i_2} + \bigg((1-\eta_x)^{1}\sum_{i_2 \neq i_1} \alpha^{(2)}_{i_2} + \eta_x\tfrac{\mu_t}{\sqrt{n}} \sum\limits_{i_3 \neq i_2, i_1}\alpha^{(2)}_{i_3}\bigg) \notag \\
	&\hspace{2cm}  + \big(\textstyle\sum\limits_{i_2 \neq i_1} \alpha^{(1)}_{i_2}(1-\eta_x)^{2} +2\eta_x\tfrac{\mu_t}{\sqrt{n}}(1-\eta_x) \sum\limits_{i_3 \neq i_2, i_1}\alpha^{(1)}_{i_3} + (\eta_x\tfrac{\mu_t}{\sqrt{n}})^2 \sum\limits_{i_4 \neq i_3, i_2, i_1} \alpha^{(1)}_{i_4} \big) \bigg] \notag\\
	&~~~+ \eta_x\tfrac{\mu_t}{\sqrt{n}}\bigg[\textstyle\sum\limits_{i_2 \neq i_1} C_{i_2}^{(0)}(1-\eta_x)^{3}
	+ 3\eta_x\tfrac{\mu_t}{\sqrt{n}}(1-\eta_x)^{2}\textstyle\sum\limits_{i_3 \neq i_2, i_1} C_{i_3}^{(0)}
	\\& \hspace{4.2cm}+3(\eta_x\tfrac{\mu_t}{\sqrt{n}})^2(1-\eta_x)^{1}\textstyle\sum\limits_{i_4 \neq i_3, i_2, i_1} C_{i_4}^{(0)} 
	 +  (\eta_x\tfrac{\mu_t}{\sqrt{n}})^3\textstyle\sum\limits_{i_5 \neq i_4, i_3, i_2, i_1} C_{i_5}^{(0)}\bigg].
	\end{align*}

	Notice that the terms have a binomial series like form.  To reveal this structure, let each \[\alpha^{(\ell)}_{j} \leq \alpha^{(\ell)}_{\max}\] where \[ j = i_1, i_2, \dots, i_k\]. Similarly, let \[C_{j}^{(0)} \leq C_{\max}^{(0)}\] for \[j = i_1, i_2, \dots, i_k\]. Therefore, we have
	\begin{align*}
	C_{i_1}^{(4)} 
	&\leq \alpha^{(4)}_{i_1} + \eta_x\tfrac{\mu_t}{\sqrt{n}}\bigg[ (k-1) \alpha^{(3)}_{i} + \alpha^{(2)}_{i}\bigg((1-\eta_x)^{1}(k-1) + \eta_x\tfrac{\mu_t}{\sqrt{n}} (k-2)\bigg) \notag \\
	&  + \alpha^{(1)}_{i}\bigg( (k-1) (1-\eta_x)^{2} +2(k-2)\eta_x\tfrac{\mu_t}{\sqrt{n}}(1-\eta_x) + (k-3)(\eta_x\tfrac{\mu_t}{\sqrt{n}})^2 \bigg) \bigg] \notag\\
	&+ \eta_x\tfrac{\mu_t}{\sqrt{n}}C_{i}^{(0)}\bigg[(k-1) (1-\eta_x)^{3}
	+ 3(k-2)\eta_x\tfrac{\mu_t}{\sqrt{n}}(1-\eta_x)^{2} 
	\\&\hspace{5.8cm}+3(k-3)(\eta_x\tfrac{\mu_t}{\sqrt{n}})^2(1-\eta_x)^{1}
	+  (k-4)(\eta_x\tfrac{\mu_t}{\sqrt{n}})^3 \bigg].
	\end{align*}
	%
	%Note that there are just \[k\] terms in each binomial sum, but we still take \[(k-1)\] common
	Further upper-bounding the expression, we have
	\begin{align*}
	C_{i_1}^{(4)} 
	&\leq \alpha^{(4)}_{i_1} + (k-1)\eta_x\tfrac{\mu_t}{\sqrt{n}}\bigg[ \alpha^{(3)}_{i} + \alpha^{(2)}_{i}\bigg((1-\eta_x) + \eta_x\tfrac{\mu_t}{\sqrt{n}} \bigg) 
	\\&\hspace{5.5cm}+ \alpha^{(1)}_{i}\bigg(  (1-\eta_x)^{2} +2\eta_x\tfrac{\mu_t}{\sqrt{n}}(1-\eta_x) + (\eta_x\tfrac{\mu_t}{\sqrt{n}})^2 \bigg) \bigg] \notag \\
	&~~~~~~~+ (k-1)\eta_x\tfrac{\mu_t}{\sqrt{n}}C_{i}^{(0)}\bigg[ (1-\eta_x)^{3}
	+ 3\eta_x\tfrac{\mu_t}{\sqrt{n}}(1-\eta_x)^{2} 
	+3(\eta_x\tfrac{\mu_t}{\sqrt{n}})^2(1-\eta_x)
	+  (\eta_x\tfrac{\mu_t}{\sqrt{n}})^3 \bigg].
	\end{align*}
	Therefore, 
	\begin{align}\label{eq:C_i_4}
	C_{i_1}^{(4)} 
	\leq \alpha^{(4)}_{i_1} + (k-1)\eta_x\tfrac{\mu_t}{\sqrt{n}}\bigg[ \alpha^{(3)}_{i} + &\alpha^{(2)}_{i}\big(1-\eta_x + \eta_x\tfrac{\mu_t}{\sqrt{n}} \big)^1 
	+ \alpha^{(1)}_{i}\big(1-\eta_x + \eta_x\tfrac{\mu_t}{\sqrt{n}} \big)^2  \bigg] \notag\\
	&~~~~~~~~~~~~~~~~~~+ (k-1)\eta_x\tfrac{\mu_t}{\sqrt{n}}C_{i}^{(0)} \big(1-\eta_x + \eta_x\tfrac{\mu_t}{\sqrt{n}} \big)^3.
	\end{align}
	\noindent\textbf{Expression for \[C_{i_1}^{(r+1)}\]}-- With this, we are ready to write the general term, 
	\begin{align*}
	C_{i_1}^{(r+1)} 
	\leq \alpha^{(r+1)}_{i_1} + (k-1)\eta_x\tfrac{\mu_t}{\sqrt{n}} \textstyle\sum\limits_{\ell = 1}^{r}  \alpha^{(\ell)}_{\max}\big(1-\eta_x &+ \eta_x\tfrac{\mu_t}{\sqrt{n}} \big)^{r -\ell}
	\\&+ (k-1)\eta_x\tfrac{\mu_t}{\sqrt{n}}C_{\max}^{(0)} \big(1-\eta_x + \eta_x\tfrac{\mu_t}{\sqrt{n}} \big)^r. %\\
	%&\leq \alpha^{(r+1)}_{i_1} + k\eta_x\tfrac{\mu_t}{\sqrt{n}} \sum_{\ell = 1}^{r}  \alpha^{(\ell)}_{i}\big(1-\eta_x + \eta_x\tfrac{\mu_t}{\sqrt{n}} \big)^{r -\ell}
	%+ k\eta_x\tfrac{\mu_t}{\sqrt{n}}C_{i}^{(0)} \big(1-\eta_x + \eta_x\tfrac{\mu_t}{\sqrt{n}} \big)^r.
	\end{align*}
\end{proof}
% % % % % % % % % % % % % % % % % % % End of general term simple % % % % % % % % % % % % % % % % % % % % % %

% % % % % % % % % % % % % % % % % % % begin intermediate alpha sum % % % % % % % % % % % % % % % % % % % % % % % % % % % %
\begin{claim}[\textbf{An intermediate result for bounding the error in coefficient calculations}]
	\label{iht:C_i1_inter_alpha_sum}
	\begin{em}
	With probability \[(1 -\delta_{\HT}^{(t)}-\delta_{\beta}^{(t)})\],
	\begin{align*}
	\textstyle\sum\limits_{\ell = 1}^{{R}}  \alpha^{(\ell)}_{\max}\big(1-\eta_x + \eta_x\tfrac{\mu_t}{\sqrt{n}} \big)^{{R} -\ell} \leq C_i^{(0)}{R}\delta_{R}  + \tfrac{1}{\eta_x( 1- \tfrac{\mu_t}{\sqrt{n}})}(\tfrac{\epsilon_t^2}{2} |\b{x}_{\max}^*| +t_\beta).
	\end{align*}%
	\end{em}
\end{claim}
% % % % % %
\begin{proof}[Proof of Claim~\ref{iht:C_i1_inter_alpha_sum} ]
	%Consider the following sum
	%
	%\begin{align}
	%\sum_{\ell = 1}^{{R}}  \alpha^{(\ell)}_{i}\big(1-\eta_x + \eta_x\tfrac{\mu_t}{\sqrt{n}} \big)^{{R} -\ell}
	%\end{align}
	%
	Using \eqref{iht:def_alpha_j_rplus1}, the quantity \[\alpha^{(\ell)}_i\] is defined as
	\begin{align*}
	\alpha^{(\ell)}_i = C_i^{(0)} (1 - \eta_x)^{\ell}  + (\lambda^{(t)}_{i} |\b{x}_i^*| +\beta^{(t)}_i) \textstyle\sum\limits_{s=1}^{\ell}\eta_x (1-\eta_x)^{\ell-s + 1}.
	\end{align*}
	Therefore, we are interested in 
	\begin{align*}
	\textstyle\sum\limits_{\ell = 1}^{{R}}  C_i^{(0)} (1 - \eta_x)^{\ell}\big(1-\eta_x + &\eta_x\tfrac{\mu_t}{\sqrt{n}} \big)^{{R} -\ell}  \\&\textstyle+ (\lambda^{(t)}_{i} |\b{x}_i^*| +\beta^{(t)}_i)\sum\limits_{\ell = 1}^{{R}} \big(1-\eta_x + \eta_x\tfrac{\mu_t}{\sqrt{n}} \big)^{{R} -\ell} \sum\limits_{s=1}^{\ell}\eta_x (1-\eta_x)^{\ell-s + 1}.
	\end{align*}
	Consider the first term which depends on \[C_i^{(0)}\].
	%
	%\begin{align}
	%C_i^{(0)}\sum_{\ell = 1}^{{R}}  (1 - \eta_x)^{\ell}\big(1-\eta_x + \eta_x\tfrac{\mu_t}{\sqrt{n}} \big)^{{R} -\ell}.
	%\end{align}
	%
	Since \[(1 - \eta_x) \leq (1-\eta_x + \eta_x\tfrac{\mu_t}{\sqrt{n}} \big)\] %(equality holds when the incoherence is zero)
	, we have 
	\begin{align*}
	C_i^{(0)}\textstyle\sum\limits_{\ell = 1}^{{R}}  (1 - \eta_x)^{\ell}\big(1-\eta_x + \eta_x\tfrac{\mu_t}{\sqrt{n}} \big)^{{R} -\ell}
	\leq C_i^{(0)}{R}\big(1-\eta_x + \eta_x\tfrac{\mu_t}{\sqrt{n}} \big)^{{R}} \leq C_i^{(0)}{R}\delta_{R},
	\end{align*}
	where \[\delta_{R}\] is a small constant, and a parameter which determines the number of iterations \[R\] required for the coefficient update step. Now, coming back to the quantity of interest %-- \[C_{i_1}^{({R}+1)}\], we have
	%
	%
	%\begin{align}
	%C_{i_1}^{(r+1)} 
	%&\leq \alpha^{(r+1)}_{i_1} + k\eta_x\tfrac{\mu_t}{\sqrt{n}} \sum_{\ell = 1}^{r}  \alpha^{(\ell)}_{i}\big(1-\eta_x + \eta_x\tfrac{\mu_t}{\sqrt{n}} \big)^{r -\ell}
	%+ k\eta_x\tfrac{\mu_t}{\sqrt{n}}C_{i}^{(0)} \big(1-\eta_x + \eta_x\tfrac{\mu_t}{\sqrt{n}} \big)^r.
	%\end{align}
	%
	\begin{align*}
	\textstyle\sum\limits_{\ell = 1}^{{R}}  \alpha^{(\ell)}_{i}\big(1-\eta_x + \eta_x\tfrac{\mu_t}{\sqrt{n}} \big)^{{R} -\ell}
	%&\leq \sum_{\ell = 1}^{{R}}  C_i^{(0)} (1 - \eta_x)^{\ell}\big(1-\eta_x + \eta_x\tfrac{\mu_t}{\sqrt{n}} \big)^{{R} -\ell}  + (\lambda^{(t)}_{i} |\b{x}_i^*| +\beta^{(t)}_i)\sum_{\ell = 1}^{{R}} \big(1-\eta_x + \eta_x\tfrac{\mu_t}{\sqrt{n}} \big)^{{R} -\ell} \sum_{s=1}^{\ell}\eta_x (1-\eta_x)^{\ell-s + 1},\\
	&\leq C_i^{(0)}{R}\delta_{R}  \\&\textstyle+ (\lambda^{(t)}_{i} |\b{x}_i^*| +\beta^{(t)}_i)\sum\limits_{\ell = 1}^{{R}} \big(1-\eta_x + \eta_x\tfrac{\mu_t}{\sqrt{n}} \big)^{{R} -\ell} \sum\limits_{s=1}^{\ell}\eta_x (1-\eta_x)^{\ell-s + 1}.
	\end{align*}
	Now, using sum of geometric series result, we have that \[	\textstyle\sum\limits_{s=1}^{\ell}\eta_x (1-\eta_x)^{\ell-s + 1} \], and
%	
%	
%	\begin{align*}
%
%	%= (1-\eta_x)( 1-(1-\eta_x)^{\ell + 1} ) 
%	\leq 1,
%	\end{align*}
%	%
%	and 
	\begin{align*}
	%\sum_{\ell = 1}^{{R}} \big(1-\eta_x + \eta_x\tfrac{\mu_t}{\sqrt{n}} \big)^{{R} -\ell} ( 1-(1-\eta_x)^{\ell+1} ) 
	%&\leq 
\textstyle\sum\limits_{\ell = 1}^{{R}} \big(1-\eta_x + \eta_x\tfrac{\mu_t}{\sqrt{n}} \big)^{{R} -\ell}& = \tfrac{1 - \big(1-\eta_x + \eta_x\tfrac{\mu_t}{\sqrt{n}} \big)^{{R}}}{\eta_x - \eta_x\tfrac{\mu_t}{\sqrt{n}}} \leq \tfrac{1}{\eta_x( 1- \tfrac{\mu_t}{\sqrt{n}})}.
	\end{align*}
	%
	%
	%\begin{align}
	%&\sum_{\ell = 1}^{{R}}  \alpha^{(\ell)}_{i}\big(1-\eta_x + \eta_x\tfrac{\mu_t}{\sqrt{n}} \big)^{r -\ell} \leq \delta_{R}  + (\lambda^{(t)}_{i} |\b{x}_i^*| +\beta^{(t)}_i)\sum_{\ell = 1}^{{R}} \big(1-\eta_x + \eta_x\tfrac{\mu_t}{\sqrt{n}} \big)^{{R} -\ell} ( 1-(1-\eta_x)^{\ell+ 1} ).
	%\end{align}
	%
	%
	%Consider the following sum, since  \[( 1-(1-\eta_x)^{\ell + 1} ) \leq 1\],
	Therefore, with probability at least \[(1 - \delta_{\beta}^{(t)})\],
	\begin{align*}
	\textstyle\sum\limits_{\ell = 1}^{{R}}  \alpha^{(\ell)}_{\max}\big(1-\eta_x + \eta_x\tfrac{\mu_t}{\sqrt{n}} \big)^{{R} -\ell} 
	%&\leq C_i^{(0)}{R}\delta_{R}  + \tfrac{1}{\eta_x( 1- \tfrac{\mu_t}{\sqrt{n}})}(\lambda^{(t)}_{i} |\b{x}_i^*| +\beta^{(t)}_i),\\
	&\leq C_i^{(0)}{R}\delta_{R}  + \tfrac{1}{\eta_x( 1- \tfrac{\mu_t}{\sqrt{n}})}(\tfrac{\epsilon_t^2}{2} |\b{x}_{\max}^*| +t_\beta),
	\end{align*}
	where \[\lambda^{(t)}_{i} \leq \tfrac{\epsilon_t^2}{2}\] and \[|\beta^{(t)}_i| = t_\beta\] with probability at least \[(1 - \delta_{\beta}^{(t)})\] using Claim~\ref{lem:bound_beta}.
	% % % % %
\end{proof}
% % % % % % % % % % % % % % % % % % % end intermediate alpha sum % % % % % % % % % % % % % % % % % % % % % % % % % % % %
% % % %-------

% % % % % % % % % % % % % % % % % % % % % %Begin var theta abs % % % % % % % % % % % % % % % % % % % % % % %
\begin{claim}[\textbf{Bound on the noise term in the estimation of a coefficient element in the support}]\label{iht:var_theta_abs}
\begin{em}
	With probability \[(1-\delta_{\beta}^{(t)})\], each entry \[\vartheta^{(R)}_{i_1}\] of \[\vartheta^{(R)}\] is upper-bounded as
	\begin{align*}
	|\vartheta^{(R)}_{i_1}| 
	& \leq  \mathcal{O}(t_{\beta}).
	\end{align*}
	\end{em}
\end{claim}
\begin{proof}[Proof of Claim~\ref{iht:var_theta_abs}]
	From \eqref{eq:def_var_theta} \[\vartheta^{(R)}_{i_1}\] is defined as 
	\begin{align*}
	\vartheta^{(R)}_{i_1} :=  \textstyle\sum\limits_{r = 1}^{R }\eta_{x}\xi^{(r)}_{i_1} (1- \eta_x)^{{R} - r} + \gamma^{(R)}_{i_1},
	\end{align*}
	where \[ \gamma^{(R)}_{i_1}:= (1 - \eta_{x})^{R} (\b{x}_{i_1}^{(0)}  - \b{x}_{i_1}^* (1 - \lambda^{(t)}_{i_1}))\].  
	%\[\vartheta^{(R)}_{i_1} = \sum\limits_{r = 1}^{R }\eta_{x}\xi^{(r)}_{i_1} (1- \eta_x)^{{R} - r}\]. 
	Further, \[\xi^{(r)}_{i_1}\] is as defined in \eqref{eq:def_xi_r},
	\begin{align*}
	\xi^{(r)}_{i_1} = \beta^{(t)}_{i_1} + \textstyle\sum\limits_{i_2 \neq i_1} |\langle \b{A}^{(t)}_{i_1}, \b{A}^{(t)}_{i_2}\rangle|\sgn(\langle \b{A}^{(t)}_{i_1}, \b{A}^{(t)}_{i_2}\rangle)C_{i_2}^{(r-1)} \sgn(\b{x}_{i_2}^* - \b{x}_{i_2}^{(r)}).
	\end{align*}
	Therefore, we have the following expression for \[\vartheta^{(R)}_{i_1}\]
	\begin{align}\label{eq:vartheta_j_delta}
	\vartheta^{(R)}_{i_1} = &\beta^{(t)}_{i_1} \textstyle\sum\limits_{r = 1}^{R}\eta_{x}(1- \eta_x)^{{R} - r} \notag\\
	&+ \textstyle\sum\limits_{r = 1}^{R }\eta_{x}\sum\limits_{i_2 \neq i_1} |\langle \b{A}^{(t)}_{i_1}, \b{A}^{(t)}_{i_2}\rangle|\sgn(\langle \b{A}^{(t)}_{i_1}, \b{A}^{(t)}_{i_2}\rangle)C_{i_2}^{(r-1)} \sgn(\b{x}_{i_2}^* - \b{x}_{i_2}^{(r)})  (1- \eta_x)^{{R} - r} +  \gamma^{(R)}_{i_1}.
	\end{align}
	Now \[\vartheta^{(R)}_{i_1}\] can be upper-bounded as
	\begin{align*}
	\vartheta^{(R)}_{i_1} &\leq \textstyle\beta^{(t)}_{i_1} \sum\limits_{r = 1}^{R}\eta_{x}(1- \eta_x)^{{R} - r} + \eta_{x}\tfrac{\mu_t}{\sqrt{n}}\sum\limits_{r = 1}^{R }\sum_{i_2 \neq i_1} C_{i_2}^{(r-1)} (1- \eta_x)^{{R} - r} + \gamma^{(R)}_{i_1},\\
	&\leq  \textstyle\beta^{(t)}_{i_1} + (k-1)\eta_{x}\tfrac{\mu_t}{\sqrt{n}}\sum\limits_{r = 1}^{R } C_{i_2}^{(r-1)} (1- \eta_x)^{{R} - r} + \gamma^{(R)}_{i_1}.
	%&\leq \beta^{(t)}_{i_1} + (k-1)\eta_{x}\tfrac{\mu_t}{\sqrt{n}}(\lambda^{(t)}_{i} |\b{x}_i^*| +\beta^{(t)}_i) \bigg[\sum\limits_{r = 1}^{R } \sum_{s=1}^{r-1}\eta_x(1- \eta_x)^{{R} - s} + 
	%k\tfrac{\tfrac{\mu_t}{\sqrt{n}}}{1 - \tfrac{\mu_t}{\sqrt{n}}}\sum\limits_{r = 1}^{R }(1- \eta_x)^{{R} - r}\bigg].
	\end{align*}
	Since from Claim~\ref{iht:gen_x_C_term} we have 
	\begin{align*}
	C_{i_2}^{(r-1)} (1- \eta_x)^{{R} - r}
	\leq (\lambda^{(t)}_{\max} |\b{x}_{\max}^*| +\beta^{(t)}_{\max}) \big[ \textstyle\sum\limits_{s=1}^{r-1}\eta_x(1- \eta_x)^{{R} - s} &+ 
	kc_x(1- \eta_x)^{{R} - r}\big] \\ &+ k\eta_x\tfrac{\mu_t}{\sqrt{n}}C_{\max}^{(0)} \delta_{{R}-2}.
	\end{align*}
	%
	%we have the following bound on \[\vartheta^{(R)}_{i_1}\]
	%
	%\begin{align*}
	%\vartheta^{(R)}_{i_1}
	%&\leq \beta^{(t)}_{i_1} + k\eta_{x}\tfrac{\mu_t}{\sqrt{n}}(\lambda^{(t)}_{\max} |\b{x}_{\max}^*| +\beta^{(t)}_{\max}) \bigg[\sum\limits_{r = 1}^{R } \sum_{s=1}^{r-1}\eta_x(1- \eta_x)^{{R} - s} + 
	%kc_x\sum\limits_{r = 1}^{R }(1- \eta_x)^{{R} - r}\bigg] +  \big(k\eta_x\tfrac{\mu_t}{\sqrt{n}}\big)^2{R}C_{\max}^{(0)} \delta_{{R}-2} + \gamma_{R}
	%\end{align*}
	%
	Further, since \[1 - (1- \eta_x)^{r-1} \leq 1\], we have that
	\begin{align*}
	\textstyle\sum\limits_{r = 1}^{R } 	\textstyle\sum\limits_{s=1}^{r-1}\eta_x(1- \eta_x)^{{R} - s } &= \textstyle\sum\limits_{r = 1}^{R } \eta_x(1- \eta_x)^{{R} - r + 1} \frac{1 - (1- \eta_x)^{r-1} }{\eta_x}
	\leq  \textstyle\sum\limits_{r = 1}^{R }(1- \eta_x)^{{R} - r + 1}
%	\leq \tfrac{1 - (1- \eta_x)^{{R}}}{\eta_x}
	\leq \tfrac{1}{\eta_x}.
	\end{align*}
	Therefore,
	\begin{align*}
	|\vartheta^{(R)}_{i_1}| 
	&\leq |\beta^{(t)}_{i_1}| + (k-1)\tfrac{\mu_t}{\sqrt{n}}(\lambda^{(t)}_{\max} |\b{x}_{\max}^*| +|\beta^{(t)}_{\max}|) (1 + 
	kc_x)  + \big(k\eta_x\tfrac{\mu_t}{\sqrt{n}}\big)^2{R}C_{\max}^{(0)} \delta_{{R}-2} + \gamma^{(R)}_{i_1}.
%	& \leq |\beta^{(t)}_{i_1}| + \tfrac{\mu_t}{\sqrt{n}}k^2(\lambda^{(t)}_{{\max}} |\b{x}_{\max}^*| +|\beta^{(t)}_{\max}|) + \big(k\eta_x\tfrac{\mu_t}{\sqrt{n}}\big)^2{R}C_{\max}^{(0)} \delta_{{R}-2} + \gamma^{(R)}_{i_1}.
	\end{align*}
	Now, since each \[|\beta_i^{(t)}| = t_\beta\] with probability at least \[( 1- \delta_{\beta}^{(t)})\] for the \[t\]-th iterate, and \[ k = \mathcal{O}^*(\tfrac{\sqrt{n}}{\mu \log(n)})\], therefore \[kc_x < 1\], we have that 
	%	From \eqref{beta_rho} we have with probability \[(1-\delta_{\beta}^{(t)})\]
	%	
	%	\begin{align*}
	%	\lambda^{(t)}_{i} |\b{x}_i^*| +\beta^{(t)}_i \leq (\tfrac{\epsilon_t^2}{2} + t_\beta)|\b{x}_{\max}^*|
	%	\end{align*}
	%	
	%	and 
	%	
	%	\begin{align*}
	%	\beta = t_\beta
	%	\end{align*}
	%	
	%	with probability \[(1-\delta_{\beta}^{(t)})\].
	%	
	\begin{align*}
	|\vartheta^{(R)}_{i_1}| 
%	& \leq  \mathcal{O}(k^2\tfrac{\mu_t}{\sqrt{n}}\epsilon^2_t|\b{x}_{\max}^*|),\\
	& \leq  \siri{\mathcal{O}(t_{\beta})}.
	\end{align*}
	with probability at least \[( 1- \delta_{\beta}^{(t)})\].
\end{proof}
% % % % % % % % % % % % % % % End of abs var theta % % % % % % % % % % % % % % % % % % % % % %

% % % % % % % % % % % % % Begin gen C x  term  bound % % % % % % % % % % % % % % % % % % % % % % % % % % % % % % % % % % %
\begin{claim}[\textbf{An intermediate result for \[\vartheta^{(R)}_{i_1}\] calculations}]\label{iht:gen_x_C_term}
	For \[c_x = {\tfrac{\mu_t}{\sqrt{n}}}/{(1 - \tfrac{\mu_t}{\sqrt{n}})}\], we have 
	\begin{align*}
	C_{i_2}^{(r-1)} (1- \eta_x)^{{R} - r}\hspace{-5pt}
	\leq (\lambda^{(t)}_{\max} |\b{x}_{\max}^*| +\beta^{(t)}_{\max}) \bigg[ \textstyle\sum\limits_{s=1}^{r-1}\eta_x(1- \eta_x)^{{R} - s} &+ 
	kc_x(1- \eta_x)^{{R} - r}\bigg] \\&+ k\eta_x\tfrac{\mu_t}{\sqrt{n}}C_{\max}^{(0)} \delta_{{R}-2}.
	\end{align*}
\end{claim}
% % % % % % % % % % % % % Begin gen C x  term  bound % % % % % % % % % % % % % % % % % % % % % % % % % % % % % % % % % % %
\begin{proof}[Proof of Claim~\ref{iht:gen_x_C_term}]
	Here, from Claim~\ref{iht:gen_term_simple} we have that for any \[i_1\],
	\begin{align*}
	C_{i_1}^{(r+1)} 
	&\leq \alpha^{(r+1)}_{i_1} + k\eta_x\tfrac{\mu_t}{\sqrt{n}} \textstyle\sum\limits_{\ell = 1}^{r}  \alpha^{(\ell)}_{\max}\big(1-\eta_x + \eta_x\tfrac{\mu_t}{\sqrt{n}} \big)^{r -\ell}
	+ k\eta_x\tfrac{\mu_t}{\sqrt{n}}C_{\max}^{(0)} \big(1-\eta_x + \eta_x\tfrac{\mu_t}{\sqrt{n}} \big)^r.
	\end{align*}
	therefore \[C_{i_2}^{(r-1)} \] is given by
	\begin{align*}
	C_{i_2}^{(r-1)} 
	&\leq \alpha^{(r-1)}_{i_2} + k\eta_x\tfrac{\mu_t}{\sqrt{n}} \textstyle\sum\limits_{\ell = 1}^{r-2}  \alpha^{(\ell)}_{\max}\big(1-\eta_x + \eta_x\tfrac{\mu_t}{\sqrt{n}} \big)^{r -\ell-2} \hspace{-3pt}
	+ \hspace{-3pt}k\eta_x\tfrac{\mu_t}{\sqrt{n}}C_{\max}^{(0)} \big(1-\eta_x + \eta_x\tfrac{\mu_t}{\sqrt{n}} \big)^{r-2}.
	\end{align*}
	Further, the term of interest \[C_{i_2}^{(r-1)} (1- \eta_x)^{{R} - r}\] can be upper-bounded by
	\begin{align*}
	C_{i_2}^{(r-1)} (1- \eta_x)^{{R} - r}\hspace{-3pt}
	&\leq \alpha^{(r-1)}_{i_2}(1- \eta_x)^{{R} - r} \hspace{-3pt}+ \hspace{-3pt}
	(1- \eta_x)^{{R} - r}k\eta_x\tfrac{\mu_t}{\sqrt{n}} \textstyle\sum\limits_{\ell = 1}^{r-2}  \alpha^{(\ell)}_{\max}\big(1-\eta_x + \eta_x\tfrac{\mu_t}{\sqrt{n}} \big)^{r -\ell-2}
	\\&~~~~~~~~~~~~~~~~~~~~~~~~~~~~~~~~~~~~~~+ k\eta_x\tfrac{\mu_t}{\sqrt{n}}C_{\max}^{(0)} \big(1-\eta_x + \eta_x\tfrac{\mu_t}{\sqrt{n}} \big)^{r-2}(1- \eta_x)^{{R} - r}.
	\end{align*}
	From the definition of \[ \alpha^{(\ell)}_{i}\] from \eqref{iht:def_alpha_j_rplus1}, \[ \alpha^{(r-1)}_{i_2}\] can be written as
	\begin{align*}
	\alpha^{(r-1)}_{i_2} &= C_{\max}^{(0)} (1 - \eta_x)^{r-1}  + (\lambda^{(t)}_{\max} |\b{x}_{\max}^*| +\beta^{(t)}_{\max}) \textstyle\sum\limits_{s=1}^{r-1}\eta_x (1-\eta_x)^{r-s}.
	%& \approxeq (\lambda^t_{i} |\b{x}_i^*| +\beta^{(t)}_i) (1- (1-\eta_x)^{r+1}),\\
	%&\approxeq \lambda^t_{i} |\b{x}_i^*| +\beta^{(t)}_i,
	\end{align*}
	Therefore, we have
	\begin{align*}
	\alpha^{(r-1)}_{i_2}(1- \eta_x)^{{R} - r}
	%&= C_i^{(0)} (1 - \eta_x)^{r-1}(1- \eta_x)^{{R} - r}  + (\lambda^{(t)}_{i} |\b{x}_i^*| +\beta^{(t)}_i) \sum_{s=1}^{r-1}\eta_x (1-\eta_x)^{r-s}(1- \eta_x)^{{R} - r}, \\
	&= C_{\max}^{(0)} (1 - \eta_x)^{{R} - 1}  + (\lambda^{(t)}_{\max} |\b{x}_{\max}^*| +\beta^{(t)}_{\max}) \textstyle\sum\limits_{s=1}^{r-1}\eta_x(1- \eta_x)^{{R} - s}. %\\
	%&\approxeq  C_{\max}^{(0)} (1 - \eta_x)^{{R} - 1} + (\lambda^{(t)}_{\max} |\b{x}_{\max}^*| +\beta^{(t)}_{\max}) \sum_{s=1}^{r-1}\eta_x(1- \eta_x)^{{R} - s}.
	%& \approxeq (\lambda^t_{i} |\b{x}_i^*| +\beta^{(t)}_i) (1- (1-\eta_x)^{r+1}),\\
	%&\approxeq \lambda^t_{i} |\b{x}_i^*| +\beta^{(t)}_i,
	\end{align*}
	Next, to get a handle on \[ \alpha^{(\ell)}_{\max}\big(1-\eta_x + \eta_x\tfrac{\mu_t}{\sqrt{n}} \big)^{r -\ell-2}\], consider the following using the definition of \[ \alpha^{(\ell)}_{i}\] from \eqref{iht:def_alpha_j_rplus1}, where \[\eta_x^{(i)} = \eta_x\] for all \[i\],
	\begin{align*}
	\textstyle\sum\limits_{\ell = 1}^{r}  \alpha^{(\ell)}_{\max}\big(1-&\eta_x + \eta_x\tfrac{\mu_t}{\sqrt{n}} \big)^{r -\ell}
    = \textstyle\sum\limits_{\ell = 1}^{r}  C_{\max}^{(0)} (1 - \eta_x)^{\ell}\big(1-\eta_x + \eta_x\tfrac{\mu_t}{\sqrt{n}} \big)^{r -\ell}  \\& \hspace{1.5cm}+ (\lambda^{(t)}_{\max} |\b{x}_{\max}^*| +\beta^{(t)}_{\max})\textstyle\sum\limits_{\ell = 1}^{r} \big(1-\eta_x + \eta_x\tfrac{\mu_t}{\sqrt{n}} \big)^{r -\ell} \textstyle\sum\limits_{s=1}^{\ell}\eta_x (1-\eta_x)^{\ell-s + 1},\\
	%&\leq \sum_{\ell = 1}^{r}  C_{\max}^{(0)} \big(1-\eta_x + \eta_x\tfrac{\mu_t}{\sqrt{n}} \big)^{r}  + (\lambda^{(t)}_{\max} |\b{x}_{\max}^*| +\beta^{(t)}_{\max})\sum_{\ell = 1}^{r} \big(1-\eta_x + \eta_x\tfrac{\mu_t}{\sqrt{n}} \big)^{r -\ell} ( 1-(1-\eta_x)^{\ell + 1} ),\\
	&\leq \textstyle\sum\limits_{\ell = 1}^{r}  C_{\max}^{(0)} \big(1-\eta_x + \eta_x\tfrac{\mu_t}{\sqrt{n}} \big)^{r}  + (\lambda^{(t)}_{\max} |\b{x}_{\max}^*| +\beta^{(t)}_{\max})\textstyle\sum\limits_{\ell = 1}^{r} \big(1-\eta_x + \eta_x\tfrac{\mu_t}{\sqrt{n}} \big)^{r -\ell}.
	%&\leq \sum_{\ell = 1}^{r}  C_i^{(0)} \big(1-\eta_x + \eta_x\tfrac{\mu_t}{\sqrt{n}} \big)^{r -1}  + (\lambda^t_{i} |\b{x}_i^*| +\beta^{(t)}_i) \tfrac{1}{\eta_x( 1- \tfrac{\mu_t}{\sqrt{n}})}.
	\end{align*}
	Therefore, 
	\begin{align*}
	(1- \eta_x)^{{R} - r } \textstyle\sum\limits_{\ell = 1}^{r-2}  &\alpha^{(\ell)}_{\max}\big(1-\eta_x + \eta_x\tfrac{\mu_t}{\sqrt{n}} \big)^{r -\ell-2}
	\leq \textstyle\sum\limits_{\ell = 1}^{r-2}  C_{\max}^{(0)} \big(1-\eta_x + \eta_x\tfrac{\mu_t}{\sqrt{n}} \big)^{r -2}(1- \eta_x)^{{R} - r}  \\& \hspace{2cm}+ (\lambda^{(t)}_{\max} |\b{x}_{\max}^*| +\beta^{(t)}_{\max})(1- \eta_x)^{{R} - r}\textstyle\sum\limits_{\ell = 1}^{r-2} \big(1-\eta_x + \eta_x\tfrac{\mu_t}{\sqrt{n}} \big)^{r -\ell - 2},\\
%	&\leq (r-2)  C_{\max}^{(0)} \big(1-\eta_x + \eta_x\tfrac{\mu_t}{\sqrt{n}} \big)^{{R} - 2}  + (\lambda^{(t)}_{\max} |\b{x}_{\max}^*| +\beta^{(t)}_{\max})(1- \eta_x)^{{R} - r}\tfrac{(1 -\big(1-\eta_x + \eta_x\tfrac{\mu_t}{\sqrt{n}} \big)^{r -3} )}{\eta_x (1 - \tfrac{\mu_t}{\sqrt{n}})},\\
	&\leq ({R}-2)  C_{\max}^{(0)} \big(1-\eta_x + \eta_x\tfrac{\mu_t}{\sqrt{n}} \big)^{{R} - 2}  + (\lambda^{(t)}_{\max} |\b{x}_{\max}^*| +\beta^{(t)}_{\max})\tfrac{(1- \eta_x)^{{R} - r}}{\eta_x (1 - \tfrac{\mu_t}{\sqrt{n}})}.
	\end{align*}
	Therefore, % for a large enough \[R\] we have
	\begin{align*}
	(1- \eta_x)^{{R} - r } \textstyle\sum\limits_{\ell = 1}^{r-2}  \alpha^{(\ell)}_{\max}&\big(1-\eta_x + \eta_x\tfrac{\mu_t}{\sqrt{n}} \big)^{r -\ell-2}\\
	&\leq (r-2)  C_{\max}^{(0)} \big(1-\eta_x + \eta_x\tfrac{\mu_t}{\sqrt{n}} \big)^{{R} - 2} + (\lambda^{(t)}_{\max} |\b{x}_{\max}^*| +\beta^{(t)}_{\max})\tfrac{(1- \eta_x)^{{R} - r}}{\eta_x (1 - \tfrac{\mu_t}{\sqrt{n}})}.
	\end{align*}
	Further, since \[(1- \eta_x) \leq (1-\eta_x + \eta_x\tfrac{\mu_t}{\sqrt{n}})\],
	\begin{align*}
	k\eta_x\tfrac{\mu_t}{\sqrt{n}}C_{\max}^{(0)} \big(1-\eta_x + \eta_x\tfrac{\mu_t}{\sqrt{n}} \big)^{r-2}(1- \eta_x)^{{R} - r} 
	%\leq k\eta_x\tfrac{\mu_t}{\sqrt{n}}C_{\max}^{(0)} \big(1-\eta_x + \eta_x\tfrac{\mu_t}{\sqrt{n}} \big)^{{R}-2} 
	\leq k\eta_x\tfrac{\mu_t}{\sqrt{n}}C_{\max}^{(0)} \delta_{{R}-2}.
	\end{align*}
	%
	%	Therefore, 
	%		
	%		\begin{align*}
	%		C_{i_2}^{(r-1)} (1- \eta_x)^{{R} - r}
	%		%&\leq \alpha^{(r-1)}_{i_2}(1- \eta_x)^{{R} - r } + (1- \eta_x)^{{R} - r}k\eta_x\tfrac{\mu_t}{\sqrt{n}} \sum_{\ell = 1}^{r-2}  \alpha^{(\ell)}_{i}\big(1-\eta_x + \eta_x\tfrac{\mu_t}{\sqrt{n}} \big)^{r -\ell-2} + k\eta_x\tfrac{\mu_t}{\sqrt{n}}C_{i}^{(0)} \big(1-\eta_x + \eta_x\tfrac{\mu_t}{\sqrt{n}} \big)^{r-2}(1- \eta_x)^{{R} - r},\\
	%		\leq \alpha^{(r-1)}_{i_2}(1- \eta_x)^{{R} - r } + 
	%		k\tfrac{\mu_t}{\sqrt{n}} (\lambda^{(t)}_{i} |\b{x}_i^*| +\beta^{(t)}_i)\tfrac{(1- \eta_x)^{{R} - r}}{1 - \tfrac{\mu_t}{\sqrt{n}}}
	%		+ k\eta_x\tfrac{\mu_t}{\sqrt{n}}C_{i}^{(0)} \big(1-\eta_x + \eta_x\tfrac{\mu_t}{\sqrt{n}} \big)^{{R}- 2},
	%		\end{align*}
	%		
	Therefore, combining all the results we have that, for a constant 	\[c_x = {\tfrac{\mu_t}{\sqrt{n}}}/{(1 - \tfrac{\mu_t}{\sqrt{n}})}\],
	\begin{align*}
	C_{i_2}^{(r-1)} &(1- \eta_x)^{{R} - r}\\
	%	&\leq \alpha^{(r-1)}_{i_2}(1- \eta_x)^{{R} - r } + k\tfrac{\mu_t}{\sqrt{n}} (\lambda^{(t)}_{i} |\b{x}_i^*| +\beta^{(t)}_i)\tfrac{(1- \eta_x)^{{R} - r}}{1 - \tfrac{\mu_t}{\sqrt{n}}} + k\eta_x\tfrac{\mu_t}{\sqrt{n}}C_{i}^{(0)} \big(1-\eta_x + \eta_x\tfrac{\mu_t}{\sqrt{n}} \big)^{{R}- 2},\\
	%	&\leq (\lambda^{(t)}_{\max} |\b{x}_i^*| +\beta^{(t)}_{\max}) \sum_{s=1}^{r-1}\eta_x(1- \eta_x)^{{R} - s} + 	k\tfrac{\mu_t}{\sqrt{n}} (\lambda^{(t)}_{\max} |\b{x}_{\max}^*| +\beta^{(t)}_{\max})\tfrac{(1- \eta_x)^{{R} - r}}{1 - \tfrac{\mu_t}{\sqrt{n}}} + k\eta_x\tfrac{\mu_t}{\sqrt{n}}C_{\max}^{(0)} \delta_{{R}-2}, \\
	&\leq (\lambda^{(t)}_{\max} |\b{x}_{\max}^*| +\beta^{(t)}_{\max}) \bigg[ \textstyle\sum\limits_{s=1}^{r-1}\eta_x(1- \eta_x)^{{R} - s} + 
	kc_x(1- \eta_x)^{{R} - r}\bigg] + k\eta_x\tfrac{\mu_t}{\sqrt{n}}C_{\max}^{(0)} \delta_{{R}-2}.
	\end{align*}
\end{proof}
% % % % % % % % % % % % % % % % % % % % gen r C term % % % % % % % % % % % % % % % % % % % % % % % %

% % % % % % % % % % % % % % % % % % % % % % % % % Bound var eps % % % % % % % % % % % % % % % % % % % % % % % % % % % % % % % % % %
\begin{claim} [\textbf{Bound on the noise term in expected gradient vector estimate}]\label{iht:bound_vareps}  
\begin{em}
   \[\|\Delta^{(t)}_{j}\|\] where \[\Delta^{(t)}_{j}:= \b{E}[\b{A}^{(t)}\vartheta^{(R)}\sgn(\b{x}^*_{j})] \] is upper-bounded as,
	\begin{align*}
	\|\Delta^{(t)}_{j}\| = \mathcal{O}(\sqrt{m}q_{i,j}p_{j}\epsilon_t\|\b{A}^{(t)}\|)].
	\end{align*}
\end{em}
\end{claim}
\begin{proof}[Proof of Claim~\ref{iht:bound_vareps}]
	\begin{align*}
	\Delta^{(t)}_{j}&= \b{E}[\b{A}^{(t)}\vartheta^{(R)}\sgn(\b{x}^*_{j})] = \b{E}_S[\b{A}^{(t)}_S\b{E}_{\b{x}^*_S}[\vartheta^{(R)}_S\sgn(\b{x}^*_{j})|S]] 
	\end{align*}
	%
	%	\hl{incorrect}
	%	
	%	\begin{align*}
	%	\vartheta^{(R)}_{i_1} := \beta^{(t)}_{i_1} \sum\limits_{r = 1}^{R}\eta_{x}\sgn(\b{x}_{i_1}^r - \b{x}_{i_1}^*)(1- \eta_x)^{{R} - r}+ \sum\limits_{r = 1}^{R }\eta_{x}\sgn(\b{x}_{i_1}^r - \b{x}_{i_1}^*)\tfrac{\mu_t}{\sqrt{n}}\sum_{i_2 \neq i_1} C_{i_2}^{(r-1)} (1- \eta_x)^{{R} - r }
	%	\end{align*}
	%	
	From \eqref{eq:vartheta_j_delta} we have the following definition for \[ \vartheta^{(R)}_{j}\]
	\begin{align*}
	\vartheta^{(R)}_{j} \hspace{-3pt}=  \beta^{(t)}_{j} \hspace{-3pt}+ \textstyle\sum\limits_{r = 1}^{R }\eta_{x}\textstyle\sum\limits_{i \neq j} |\langle \b{A}^{(t)}_{j}, \b{A}^{(t)}_{i}\rangle|\sgn(\langle \b{A}^{(t)}_{j}, \b{A}^{(t)}_{i}\rangle)C_{i}^{(r-1)} \sgn(\b{x}_{i}^* - \b{x}_{i}^{(r)})  (1- \eta_x)^{{R} - r} \hspace{-3pt}+ \gamma^{(R)}_j,
	\end{align*}
	where \[\beta_{j}^{(t)}\] is defined as the following  \eqref{eq:beta_t}	
	\begin{align*}
	\beta_{j}^{(t)}
	& = \textstyle\sum\limits_{i \neq j } (\langle \b{A}^{*}_{j},  \b{A}^*_i - \b{A}^{(t)}_i\rangle  + \langle \b{A}^*_{j} -\b{A}^{(t)}_{j}, \b{A}^{(t)}_i\rangle) \b{x}_i^* +\textstyle\sum_{i\neq j}\langle \b{A}^{(t)}_{j} - \b{A}^{*}_{j}, \b{A}^*_i\rangle\b{x}^*_i.
	\end{align*}
	Consider \[\b{E}_{\b{x}^*_S}[\vartheta^{(R)}_S\sgn(\b{x}^*_{j})|S]\], where \[\vartheta^{(R)}_S\] is a vector with each element as defined in \eqref{eq:vartheta_j_delta}. Therefore, the elements of the vector \[\b{E}_{\b{x}^*_S}[\vartheta^{(R)}_S\sgn(\b{x}^*_{j})|S]]\] are given by 
	\begin{align*}
	\b{E}_{\b{x}^*_S}[\vartheta^{(R)}_{i}\sgn(\b{x}^*_{j})|S] =
	\begin{cases}
	\b{E}_{\b{x}^*_S}[\vartheta^{(R)}_{i}\sgn(\b{x}^*_{j})|S], & \text{for} ~i \neq j,\\
	\b{E}_{\b{x}^*_S}[\vartheta^{(R)}_{j}\sgn(\b{x}^*_{j})|S], & \text{for} ~i = j.
	\end{cases}
	\end{align*}
	Consider the general term of interest
	\begin{align*}
	&\b{E}_{\b{x}_S^*}[\vartheta^{(R)}_{i}\sgn(\b{x}^*_{j})|S]\\
	%\\&= \b{E}_{\b{x}_S^*}[ \sum\limits_{r = 1}^{R}\eta_{x}\beta_{i_2}\sgn(\b{x}^*_{i_1})(1- \eta_x)^{{R} - r}|S]+ \b{E}_{\b{x}_S^*}[\sum\limits_{r = 1}^{R }\eta_{x}\sum_{i \neq i_2} C_{i}^{(r-1)}|\langle \b{A}^{(t)}_{i_2}, \b{A}^{(t)}_{i}\rangle|\sgn(\langle \b{A}^{(t)}_{i_2}, \b{A}^{(t)}_{i}\rangle) \sgn({\b{x}_{i}^* - \b{x}_{i}^r}) (1- \eta_x)^{{R} - r }\sgn(\b{x}^*_{i_1})|S],\\
	%&= \sum\limits_{r = 1}^{R}\eta_{x}\b{E}_{\b{x}_S^*}[\beta_{i_2} \sgn(\b{x}^*_{i_1})|S](1- \eta_x)^{{R} - r}+ \sum\limits_{r = 1}^{R }\eta_{x}\sum_{i \neq i_2} |\langle \b{A}^{(t)}_{i_2}, \b{A}^{(t)}_{i}\rangle|\sgn(\langle \b{A}^{(t)}_{i_2}, \b{A}^{(t)}_{i}\rangle)\b{E}_{\b{x}_S^*}[C_{i}^{(r-1)} \sgn({\b{x}_{i}^* - \b{x}_{i}^r})\sgn(\b{x}^*_{i_1})|S] (1- \eta_x)^{{R} - r },\\
	&\leq \textstyle\sum\limits_{r = 1}^{R}\eta_{x}(1- \eta_x)^{{R} - r} \underbrace{\b{E}_{\b{x}_S^*}[\beta_{i} \sgn(\b{x}^*_{j})|S]}_{\clubsuit} \\ &\hspace{2cm}+ \tfrac{\mu_t}{\sqrt{n}}\sum\limits_{r = 1}^{R }\eta_{x}(1- \eta_x)^{{R} - r }\textstyle\sum_{s \neq i} \underbrace{\b{E}_{\b{x}_S^*}[C_{s}^{(r-1)} \sgn(\b{x}_{s}^* - \b{x}_{s}^{(r)})\sgn(\b{x}^*_{j})|S]}_{\spadesuit} + \gamma^{(R)}_{i} .
	%&\leq 3p_{i_1}\epsilon_t\sum\limits_{r = 1}^{R}\eta_{x}(1- \eta_x)^{{R} - r}+  |\langle \b{A}^{(t)}_{i_2}, \b{A}^{(t)}_{i_1}\rangle|\sgn(\langle \b{A}^{(t)}_{i_2}, \b{A}^{(t)}_{i_1}\rangle)\sum\limits_{r = 1}^{R }	h^{r-1}_{i_1}\eta_{x} (1- \eta_x)^{{R} - r },\\
	%&\leq 3p_{i_1}\epsilon_t\sum\limits_{r = 1}^{R}\eta_{x}(1- \eta_x)^{{R} - r}+ \tfrac{\mu}{\sqrt{n}}\sum\limits_{r = 1}^{R }\sum_{i \neq i_2}	\b{E}_{\b{x}_S^*}[C_{i}^{(r-1)} \sgn({\b{x}_{i}^* - \b{x}_{i}^r})\sgn(\b{x}^*_{i_1})|S]\eta_{x} (1- \eta_x)^{{R} - r },\\
	%&\approx 3p_{i_1}\epsilon_t+ \tfrac{\mu}{\sqrt{n}}\sum\limits_{r = 1}^{R }	h^{r-1}_{i_1}\eta_{x} (1- \eta_x)^{{R} - r }.
	\end{align*}
     Further, since 
	\begin{align*}
	\b{E}_{\b{x}_S^*}[\b{x}_i^*\sgn(\b{x}^*_{j})|S] = 
	\begin{cases}
	0, & \text{for}~ i \neq j,\\
	p_{j}, & \text{for}~ i = j,
	\end{cases}
	\end{align*}
	we have that
	\begin{align}\label{eq:beta_cases}
	\clubsuit:= \b{E}_{\b{x}_S^*}[ \beta^{(t)}_{i}\sgn(\b{x}^*_{j})| S] \leq
	\begin{cases}
		3p_{j}\epsilon_t &,\text{for}~ i \neq j,\\
	0 &,\text{for}~ i = j.
	\end{cases}
	\end{align}
	Further, for  \[\spadesuit_s : = \b{E}_{\b{x}_S^*}[C_{s}^{(r-1)} \sgn(\b{x}_{s}^* - \b{x}_{s}^{(r)})\sgn(\b{x}^*_{j})|S]\] we have that
	%Consider \[\b{E}_{\b{x}_S^*}[C_{i}^{(r-1)} \sgn({\b{x}_{i}^* - \b{x}_{i}^{r-1}})\sgn(\b{x}^*_{i_1})|S]\], 
	\begin{align*}
	\spadesuit_s %: =\b{E}_{\b{x}_S^*}[C_{s}^{(r-1)} \sgn({\b{x}_{s}^* - \b{x}_{s}^{(r-1)}})\sgn(\b{x}^*_{j})|S]=
	=
	\begin{cases}
	\b{E}_{\b{x}_S^*}[C_{j}^{(r-1)}(\b{x}_{j}^* - \b{x}_{j}^{(r-1)})\sgn(\b{x}^*_{j})|S] \leq C^{(r-1)}_{j}, &\text{for}~ s = j\\
	0, &\text{for}~ s \neq j.
	\end{cases}
	\end{align*}
	In addition, for \[\textstyle\sum_{s \neq i} \spadesuit_s\] we have that
	\begin{align}\label{eq:C_cases}
	\textstyle\sum_{s \neq i} \spadesuit_s 
	%: =\textstyle\sum_{s \neq i} \b{E}_{\b{x}_S^*}[C_{s}^{(r-1)} \sgn({\b{x}_{s}^* - \b{x}_{s}^{(r-1)}})\sgn(\b{x}^*_{j})|S]
	=
	\begin{cases}
	C^{(r-1)}_{j}, &\text{for}~  i \neq j\\
	0, &\text{for}~ i = j.
	\end{cases}
	\end{align}
	Therefore, using the results for \[\clubsuit\] and \[\textstyle\sum\limits_{s \neq i} \spadesuit_s\], we have that \[\b{E}_{\b{x}_S^*}[\vartheta^{(R)}_{j}\sgn(\b{x}^*_{j})|S] =  \gamma^{(R)}_i\] for \[i=j\], and for \[i \neq j\] we have
	\begin{align}\label{eq:exp_var_theta_int}
	\b{E}_{\b{x}_S^*}&[\vartheta^{(R)}_{i}\sgn(\b{x}^*_{j})|S]\notag\\
	%	&\leq \sum\limits_{r = 1}^{R}\eta_{x}(1- \eta_x)^{{R} - r} \b{E}_{\b{x}_S^*}[\beta_{i_2} \sgn(\b{x}^*_{i_1})|S]+ \tfrac{\mu_t}{\sqrt{n}}\sum\limits_{r = 1}^{R }\eta_{x}(1- \eta_x)^{{R} - r }\sum_{i \neq i_2} \b{E}_{\b{x}_S^*}[C_{i}^{(r-1)} \sgn({\b{x}_{i}^* - \b{x}_{i}^r})\sgn(\b{x}^*_{i_1})|S] ,\\
	%&\leq 3p_{i_1}\epsilon_t\sum\limits_{r = 1}^{R}\eta_{x}(1- \eta_x)^{{R} - r}+  |\langle \b{A}^{(t)}_{i_2}, \b{A}^{(t)}_{i_1}\rangle|\sgn(\langle \b{A}^{(t)}_{i_2}, \b{A}^{(t)}_{i_1}\rangle)\sum\limits_{r = 1}^{R }	h^{r-1}_{i_1}\eta_{x} (1- \eta_x)^{{R} - r },\\
%old	&\leq 3p_{j}\epsilon_t\sum\limits_{r = 1}^{R}\eta_{x}(1- \eta_x)^{{R} - r}+ \tfrac{\mu_t}{\sqrt{n}}\sum\limits_{r = 1}^{R }	\b{E}_{\b{x}_S^*}[C_{j}^{(r-1)} \sgn({\b{x}_{j}^* - \b{x}_{j}^r})\sgn(\b{x}^*_{j})|S]\eta_{x} (1- \eta_x)^{{R} - r } + \gamma^{(R)}_i,\notag\\
%old	&\leq 3p_{j}\epsilon_t+ \tfrac{\mu_t}{\sqrt{n}}\sum\limits_{r = 1}^{R }C_{j}^{(r-1)}\eta_{x} (1- \eta_x)^{{R} - r } + \gamma^{(R)}_i.
&\leq 3p_{j}\epsilon_t + \tfrac{\mu_t}{\sqrt{n}}\textstyle\sum\limits_{r = 1}^{R }	\b{E}_{\b{x}_S^*}[C_{j}^{(r-1)} \sgn(\b{x}_{j}^* - \b{x}_{j}^{(r)})\sgn(\b{x}^*_{j})|S]\eta_{x} (1- \eta_x)^{{R} - r } + \gamma^{(R)}_i,\notag\\
	&\leq 3p_{j}\epsilon_t + \tfrac{\mu_t}{\sqrt{n}}\textstyle\sum\limits_{r = 1}^{R }C_{j}^{(r-1)}\eta_{x} (1- \eta_x)^{{R} - r } + \gamma^{(R)}_i.
	\end{align}
	%
	%	Consider \[\sum\limits_{r = 1}^{R }	h^{r-1}_{i_1}\eta_{x} (1- \eta_x)^{{R} - r }\]
	%	
	%	\begin{align*}
	%	\sum\limits_{r = 1}^{R }	h^{r-1}_{i_1}\eta_{x} (1- \eta_x)^{{R} - r } \leq \sum\limits_{r = 1}^{R }	C^{r-1}_{i_1}\eta_{x} (1- \eta_x)^{{R} - r }.
	%	\end{align*}
	%	
	Here, from Claim~\ref{iht:gen_x_C_term}, for \[c_x = {\tfrac{\mu_t}{\sqrt{n}}}/{(1 - \tfrac{\mu_t}{\sqrt{n}})}\] we have
	\begin{align*}
	C_{j}^{(r-1)} &(1- \eta_x)^{{R} - r}\\
	&\leq (\lambda^{(t)}_{\max} |\b{x}_{\max}^*| +\beta^{(t)}_{\max}) \bigg[ \textstyle\sum\limits_{s=1}^{r-1}\eta_x(1- \eta_x)^{{R} - s} + 
	kc_x(1- \eta_x)^{{R} - r}\bigg] + k\eta_x\tfrac{\mu_t}{\sqrt{n}}C_{\max}^{(0)} \delta_{{R}-2}.
	\end{align*}
	%
	%	
	%	
	%	
	%	\begin{align*}
	%	C_{i_1}^{(r-1)} (1- \eta_x)^{{R} - r}
	%	&\leq (\lambda^{(t)}_{i_1} |\b{x}_{i_1}^*| +\beta^{(t)}_{i_1}) \bigg[ \sum_{s=1}^{r-1}\eta_x(1- \eta_x)^{{R} - s} + 
	%	k\tfrac{\tfrac{\mu_t}{\sqrt{n}}}{1 - \tfrac{\mu_t}{\sqrt{n}}}(1- \eta_x)^{{R} - r}\bigg]
	%	\end{align*}
	%	
	Further, due to our assumptions on sparsity, \[kc_x \leq 1\]; in addition by Claim~\ref{lem:bound_beta}, and with probability at least \[(1 - \delta_{\beta}^{(t)})\] we have \[|\beta^{(t)}_{\max}| \leq t_\beta\], substituting,
	\begin{align*}
	\textstyle\sum\limits_{r = 1}^{R }	C^{(r-1)}_{j}&\eta_{x} (1- \eta_x)^{{R} - r }\\
	&\leq  (\lambda^{(t)}_{\max} |\b{x}_{\max}^*| +\beta^{(t)}_{\max}) \bigg[ \textstyle\sum\limits_{r = 1}^{R }	\eta_{x}\textstyle\sum\limits_{s=1}^{r-1}\eta_x(1- \eta_x)^{{R} - s} + kc_x\sum\limits_{r = 1}^{R }	\eta_{x}(1- \eta_x)^{{R} - r}\bigg],\\
	&\leq(\lambda^{(t)}_{\max} |\b{x}_{\max}^*| +t_\beta)( 1 + kc_x ),\\
	&= \mathcal{O}(t_\beta),
	%	&\leq(\lambda^{(t)}_{\max} |\b{x}_{\max}^*| +t_\beta)( \eta_x + k\tfrac{\tfrac{\mu_t}{\sqrt{n}}}{1 - \tfrac{\mu_t}{\sqrt{n}}} ).
	\end{align*}
	with probability at least \[(1 - \delta_{\beta}^{(t)})\].
	Combining results from \eqref{eq:beta_cases}, \eqref{eq:C_cases} and substituting for the terms in \eqref{eq:exp_var_theta_int} using the analysis above,
	\begin{align*}
	\b{E}_{\b{x}_S^*}[\vartheta^{(R)}_{i}\sgn(\b{x}^*_{j})|S] 
	\begin{cases}
	\leq  \gamma^{(R)}_i, & ~\text{for}~ i = j, \\
%	 	\leq 3p_{j}\epsilon_t + \tfrac{\mu}{\sqrt{n}}(\lambda^{(t)}_{\max} |\b{x}_{\max}^*| +t_\beta)( 1+ kc_x ) + \gamma^{(R)}_i,& ~\text{for}~ i \neq j.
 	\leq 3p_{j}\epsilon_t + \tfrac{\mu}{\sqrt{n}}t_\beta + \gamma^{(R)}_i = \mathcal{O}(p_j\epsilon_t), & ~\text{for}~ i \neq j.
	\end{cases}
	\end{align*}
	Note that since \[ \gamma^{(R)}_{i} :=(1 - \eta_{x})^{R} (\b{x}_{i}^{(0)}  - \b{x}_{i}^* (1 - \lambda^{(t)}_{i}))\] can be made small by choice of \[R\]. Also, since \[\b{Pr}[i,j \in S] = q_{i,j}\], we have
	\begin{align*}
	\|\Delta^{(t)}_{j}\|&= \|\b{E}_S[\b{A}^{(t)}_S\b{E}_{\b{x}_S^*}[\vartheta^{(R)}_S\sgn(\b{x}^*_{j})|S]]\|,\\
	&\leq  \mathcal{O}(\sqrt{m}q_{i,j}p_{j}\epsilon_t\|\b{A}^{(t)}\|).
	\end{align*}
\end{proof}
% % % % % % % % % % % % % % % % % % % % % % % % % End of Bound var eps % % % % % % % % % % % % % % % % % % % % % % % % % % % % % % % %

% % % % % % % % % % % % Norm of y - Ax % % % % % % % % % % % % % % % %
\begin{claim} [\textbf{An intermediate result for concentration results}]\label{norm_bound_y_Ax_s} 
\begin{em}
	With probability \[(1 - \delta_{\beta}^{(t)} - \delta_{\HT}^{(t)} - \delta_{\rm HW}^{(t)})\] \[\|\b{y}-\b{A}^{(t)}\hat{\b{x}}\|\] is upper-bounded by  %\[\textcolor{red}{\mathcal{O}( k^3\tfrac{\mu_t}{\sqrt{n}}\epsilon^2_t)}\] 
	\[\siri{\tilde{\mathcal{O}}(kt_{\beta})}\] .
	\end{em}
	%\[\|\b{y}-\b{A}^{(t)}\hat{\b{x}}\|\] is \[\mathcal{O}(k^2\epsilon_t)\].
\end{claim}
\begin{proof}[Proof of Claim~\ref{norm_bound_y_Ax_s}]
	First, using Lemma~\ref{iht:R_th_term} we have
%	
%	\begin{align*}
%	\b{x}_{i_1}^{(R)} 
%	&=  (1 - \eta_{x})^{R} \b{x}_{i_1}^{(0)}  + \b{x}_{i_1}^* (1 - \lambda^{(t)}_{i_1}) ( 1- (1 - \eta_{x})^{R}) + \vartheta^{(R)}_{i_1},
%	\end{align*}
%	%
%	For the purposes of our analysis we assume that \[R\] is large such that we have
	\begin{align*}
	\hat{\b{x}}_{i_1} :=	\b{x}_{i_1}^{(R)} =  \b{x}_{i_1}^* (1 - \lambda^{(t)}_{i_1}) + \vartheta^{(R)}_{i_1}.
	\end{align*}
	%
%	where, as per \eqref{eq:vartheta_j_delta}
%	
%		
%		\begin{align*}
%		\vartheta^{(R)}_{i_1} = \beta^{(t)}_{i_1} + \sum\limits_{r = 1}^{R }\eta_{x}\sum_{i_2 \neq i_1} |\langle \b{A}^{(t)}_{i_1}, \b{A}^{(t)}_{i_2}\rangle|\sgn(\langle \b{A}^{(t)}_{i_1}, \b{A}^{(t)}_{i_2}\rangle)C_{i_2}^{(r-1)} \sgn(\b{x}_{i_2}^* - \b{x}_{i_2}^{(r)})  (1- \eta_x)^{{R} - r} +  \gamma^{(R)}_{i_1}.
%		\end{align*}
%		%
%		
%	Now,  since \[\beta^{(t)}_{i_1}\] from \eqref{eq:beta_t} is defined as
%			
%			
%			\begin{align}
%			\beta^{(t)}_{i_1}
%			& = \displaystyle\sum_{i \neq i_1} (\langle \b{A}^{*}_{i_1},  \b{A}^*_i - \b{A}^{(t)}_i\rangle  + \langle \b{A}^*_{i_1} -\b{A}^{(t)}_{i_1}, \b{A}^{(t)}_i\rangle + \langle \b{A}^{(t)}_{i_1} - \b{A}^{*}_{i_1}, \b{A}^*_i\rangle) \b{x}_i^*.
%			\end{align}
%			%
%	
%	\begin{align*}
%	\textcolor{red}{\b{x}_{i_1}^{(R)}} &=  \b{x}_{i_1}^* (1 - \lambda^{(t)}_{i_1}) + \vartheta^{(R)}_{i_1}.
%	\end{align*}
%	%
	Therefore, the vector \[\hat{\b{x}}_S \], for \[S \in \supp(\b{x}^*)\] can be written as
	\begin{align}\label{eq:x_R_S}
	\hat{\b{x}}_S := \b{x}^{(R)}_S &= (\b{I} - \Lambda^{(t)}_S) \b{x}^*_S  + \vartheta^{(R)}_S,
	\end{align}
	%
%	First, by the definition of \[\hat{\b{x}}\] we have 
%	
%	\begin{align*}
%	\b{x}^{(R)} 
%	&\approxeq  (\b{I} - \Lambda^{(t)}) \b{x}^* + \vartheta^{(R)}, 
%	\end{align*}
%	%
	where \[\hat{\b{x}}\] has the correct signed-support with probability at least \[(1 - \delta_{\c{T}})\] using Lemma~\ref{our:signed_supp}. %Notice that \[\b{y}-\b{A}^t\hat{\b{x}}= \b{A}^*\b{x}^* - \b{A}^t\b{x}^t\], here due to the thresholding step and the fact that $x$ has the correct signed support, we have,
	Using this result, we can write \[\|\b{y} - \b{A}^{(t)}\hat{\b{x}}\| \] as
	\begin{align*}
	\|\b{y} - \b{A}^{(t)}\hat{\b{x}}\| 
	&=\|\b{A}^*_S\b{x}^*_S  - \b{A}^{(t)}_S(\b{I} - \Lambda^{(t)}_S)\b{x}^*_S  - \b{A}^{(t)}_S\vartheta^{(R)}_S\|. %\\
%	&\leq \|(\b{A}^*\b{x}^*  - (1 - \tfrac{\epsilon_t^2}{2}) \b{A}^{(t)}\b{x}^*  - \b{A}^{(t)}\vartheta^{(R)})_S\|\\
%	&=\|((1 - \tfrac{\epsilon_t^2}{2})(\b{A}^*_S  - \b{A}^{(t)}_S) +\tfrac{\epsilon_t^2}{2}\b{A}^*_S)\b{x}^*_S - \b{A}^{(t)}_S\vartheta^{(R)}_S\|
	%     &=\|[\b{A}^*_S - (1 - \eta_{x})^{t} \b{A}^t_S\b{A}^{0^\top}_S\b{A}^*_S - (1- (1 - \eta_{x})^{t})(1 - \lambda)\b{A}^t_S]\b{x}^*_S - (1- (1 - \eta_{x})^{t})\xi\b{A}^t_S\mathbbm{1}_S\|\\
	%     &\leq\|[\b{A}^*_S - (1 - \eta_{x})^{t} \b{A}^t_S\b{A}^{0^\top}_S\b{A}^*_S - (1- (1 - \eta_{x})^{t})(1 - \lambda)\b{A}^t_S]\b{x}^*_S\|+ \|(1- (1 - \eta_{x})^{t})\xi\b{A}^t_S\mathbbm{1}_S\|\\
	%     &\leq\|[\b{A}^*_S - (1 - \eta_{x})^{t} \b{A}^t_S\b{A}^{0^\top}_S\b{A}^*_S - (1- (1 - \eta_{x})^{t})(1 - \lambda)\b{A}^t_S]\b{x}^*_S\|+ \|\xi\|\|\b{A}^t_S\|\|\mathbbm{1}_S\|
	\end{align*}
	Now, since \[\Lambda^{(t)}_{ii} \leq \tfrac{\epsilon_t^2}{2}\] we have
    \begin{align*}
    \|\b{y} - \b{A}^{(t)}\hat{\b{x}}\| 
   % &=\|\b{A}^*_S\b{x}^*_S  - \b{A}^{(t)}_S(\b{I} - \Lambda^{(t)}_S)\b{x}^*_S  - \b{A}^{(t)}_S\vartheta^{(R)}_S\|\\
    &\leq \|\b{A}^*_S\b{x}^*_S  - (1 - \tfrac{\epsilon_t^2}{2}) \b{A}^{(t)}_S\b{x}^*_S  - \b{A}^{(t)}_S\vartheta^{(R)}_S\|,\\
    &=\|\underbrace{((1 - \tfrac{\epsilon_t^2}{2})(\b{A}^*_S  - \b{A}^{(t)}_S) +\tfrac{\epsilon_t^2}{2}\b{A}^*_S)}_{\clubsuit}\b{x}^*_S - \underbrace{\b{A}^{(t)}_S\vartheta^{(R)}_S}_{\spadesuit}\|.
    \end{align*}
	With \[\b{x}^*_S\] being independent and sub-Gaussian, using Lemma~\ref{thm:Hanson}, which is a result based on the Hanson-Wright result \citep{Hanson1971} for sub-Gaussian random variables, and since \[\|\b{A}^{(t)}_S - \b{A}^*_S\| \leq \|\b{A}^{(t)}_S - \b{A}^*_S\|_F \leq \sqrt{k}\epsilon_t\], we have that with probability at least \[(1 - \delta_{\rm HW}^{(t)})\]
	\begin{align*}
	\|\clubsuit\b{x}^*_S\| = \|((1 - \tfrac{\epsilon_t^2}{2})(\b{A}^*_S  - \b{A}^{(t)}_S) +\tfrac{\epsilon_t^2}{2}\b{A}^*_S)\b{x}^*_S\| \leq \tilde{\mathcal{O}}(\|(1 - \tfrac{\epsilon_t^2}{2})(\b{A}^*_S  - \b{A}^{(t)}_S) +\tfrac{\epsilon_t^2}{2}\b{A}^*_S\|_F),
	\end{align*}
    where \[\delta_{\rm HW}^{(t)} = \exp(-\tfrac{1}{\mathcal{O}(\epsilon_t)})\].
    
  % 	We select \[ t_{\rm HW} = \sqrt[4]{k}\sqrt{\epsilon_t}\] to obtain 
	Now, consider the \[\|\clubsuit\|_F\], since \[\|\b{A}^{(t)}_S - \b{A}^*_S\|_F \leq \sqrt{k}\epsilon_t\]
	\begin{align*}
	\|\clubsuit\|_F := \|(1 - \tfrac{\epsilon_t^2}{2})(\b{A}^*_S  - \b{A}^{(t)}_S) +\tfrac{\epsilon_t^2}{2}\b{A}^*_S\|_F
	&\leq  (1 - \tfrac{\epsilon_t^2}{2})\| (\b{A}^*_S -\b{A}^{(t)}_S)\|_F + \tfrac{\epsilon_t^2}{2}\|\b{A}^*_S\|_F,\\
	& \leq  \sqrt{k}(1 - \tfrac{\epsilon_t^2}{2})\epsilon_t + \tfrac{\epsilon_t^2}{2}\|\b{A}^*_S\|_F.
	\end{align*}
	Consider the \[\|\spadesuit\|\] term. Using Claim~\ref{iht:var_theta_abs}, each \[\vartheta^{(R)}_j\] is bounded by \[ \siri{\mathcal{O}(t_{\beta})}.\] with probability at least \[(1 - \delta_{\beta}^{(t)})\]
%	
%	
%	\begin{align*}
%	|\vartheta^{(R)}_{i_1}| 
%	& \leq  \mathcal{O}(k^2\tfrac{\mu_t}{\sqrt{n}}\epsilon^2_t|\b{x}_{\max}^*|).
%	\end{align*}
%	%
%	
%	\begin{align*}
%	|\vartheta^{(R)}_{i_1}| & \leq c_m k^3\epsilon_t|\b{x}_{\max}^*|.
%	\end{align*}
%	%
Therefore,
	\begin{align*}
    \|\spadesuit\| = \|\b{A}^{(t)}_S\vartheta^{(R)}_S \| &\leq \|\b{A}^{(t)}_S\|\|\vartheta^{(R)}_S\|
%	& = \|\b{A}^{(t)}_S\|\sqrt{k}\mathcal{O}( k^2\tfrac{\mu_t}{\sqrt{n}}\epsilon^2_t|\b{x}_{\max}^*|).\\
	 = \|\b{A}^{(t)}_S\|\sqrt{k}\siri{\mathcal{O}(t_{\beta})}.
	\end{align*}
	Again, since \[\|\b{A}^{(t)}_S - \b{A}^*_S\| \leq \|\b{A}^{(t)}_S - \b{A}^*_S\|_F \leq \sqrt{k}\epsilon_t\],
	\begin{align*}
	\|\b{A}^{(t)}_S\| &\leq \|\b{A}^{(t)}_S - \b{A}^*_S + \b{A}^*_S\| \leq \|\b{A}^{(t)}_S - \b{A}^*_S\| + \|\b{A}^*_S\| \leq \sqrt{k}\epsilon_t + 2.
	\end{align*}
	Finally, combining all the results and using the fact that \[\|\b{A}^*_S\|_F \leq \sqrt{k}\|\b{A}^*_S\| \leq 2\sqrt{k}\], , 
	\begin{align*}
	\|\b{y} - \b{A}^{(t)}\hat{\b{x}}\| 
	%&= \mathcal{O}(\sqrt{k}(1 - \tfrac{\epsilon_t^2}{2})\epsilon_t + \epsilon_t^2\sqrt{k} + \mathcal{O}( k^2\sqrt{k}\tfrac{\mu_t}{\sqrt{n}}\epsilon^2_t|\b{x}_{\max}^*|\|\b{A}^{(t)}_S\|)),\\
	%&=  \mathcal{O}( k^3\tfrac{\mu_t}{\sqrt{n}}\epsilon^2_t|\b{x}_{\max}^*|),\\
	%	&=  \mathcal{O}( k^3\tfrac{\mu_t}{\sqrt{n}}\epsilon^2_t).\\
	&= \tilde{\mathcal{O}}(\sqrt{k}(1 - \tfrac{\epsilon_t^2}{2})\epsilon_t + \epsilon_t^2\sqrt{k}) +\|\b{A}^{(t)}_S\|\sqrt{k}\siri{\mathcal{O}(t_{\beta})},\\
	&=  \siri{\tilde{\mathcal{O}}(kt_{\beta})}.
%    &=  \mathcal{O}( k^3\sqrt{k}\tfrac{\mu_t}{\sqrt{n}}\epsilon^2_t).\\	
	\end{align*}
	%	
%	\hl{needs} \[k \leq \sqrt{\log(n)}\]?
%	\hl{Needs the technical lemma}
\end{proof}
% % % % % % % % % % % % End of Norm of y - Ax % % % % % % % % % % % % % % % %

% % % Lemma variance of w (the gradient vector) % % % % % % % % % % % % % % % %
\begin{claim}[\textbf{Bound on variance parameter for concentration of gradient vector}]
	\label{lem:var_w_vector_grad}
	\begin{em}
	For  \[\b{z}:= (\b{y}- \b{A}^{(t)}\hat{\b{x}})\sgn(\hat{\b{x}}_i)|i\in S\] the variance parameter  \[\b{E}[\|\b{z}\|^2]\] is bounded as
%	\[\b{E}[\|\b{w}\|^2] = \mathcal{O}(k^3\epsilon_t^2)\] where \[\b{w}:= (\b{y}- \b{A}^{(t)}\hat{\b{x}})\sgn(\hat{\b{x}}_i)|i\in S\].
    \[\b{E}[\|\b{z}\|^2] = \mathcal{O}(k\epsilon_t^2) +
    %  \textcolor{red}{\mathcal{O}( k^5\epsilon_t^4\tfrac{\mu_t^2}{n})}+ 
    \siri{ \mathcal{O}( kt_{\beta}^2)}\] with probability at least \[(1 - \delta_{\beta}^{(t)} - \delta_{\HT}^{(t)})\].
    \end{em}
\end{claim}

\begin{proof}[Proof of Claim~\ref{lem:var_w_vector_grad}]
	For the variance \[\b{E}[\|\b{z}\|^2]\], we focus on the following,
	\begin{align*}
	\b{E}[\|\b{z}\|^2] &= \b{E}[ \|( \b{y}- \b{A}^{(t)}\hat{\b{x}})\sgn(\hat{\b{x}}_i)\|^2|i\in S].
	\end{align*}
	Here, \[\hat{\b{x}}_S\] is given by
%	Therefore, substituting for \[\hat{\b{x}}_S\] from \eqref{eq:x_R_S}, we have
%		Therefore, the vector \[\hat{\b{x}}_S \], for \[S \in \supp(\b{x}^*)\] can be written as
		\begin{align*}
		\hat{\b{x}}_S &= (\b{I} - \Lambda^{(t)}_S) \b{x}^*_S  + \vartheta^{(R)}_S.
		\end{align*}
		Therefore, \[\b{E}[\|\b{z}\|^2]\] can we written as
	\begin{align}
	\label{iht:varZ}
	\b{E}[\|( \b{y}- \b{A}^{(t)}&\hat{\b{x}})\sgn(\hat{\b{x}}_i)\|^2|i\in S] \notag\\
	&= \b{E}[\|( \b{y}- \b{A}^{(t)}_S(\b{I} - \Lambda^{(t)}_S)\b{x}^*_S  - \b{A}^{(t)}_S\vartheta^{(R)}_S)\sgn(\hat{\b{x}}_i)\|^2|i\in S],\notag\\
%	&= \b{E}[\|( \b{A}^*_S- \b{A}^{(t)}(\b{I} - \Lambda^{(t)}_S))\b{x}^*_S\sgn(\b{x}_i^{R})  - \b{A}^{(t)}\vartheta^{(R)}\sgn(\b{x}_i^{R})\|^2|i\in S], \notag\\
%	&=  \b{E}[\|( \b{A}^*_S- \b{A}^{(t)}_S(\b{I} - \Lambda^{(t)}_S))\b{x}^*_S \|^2|i\in S]+ \b{E}[\|\b{A}^{(t)}_S\vartheta^{(R)}_S\sgn(\b{x}_i^{R})\|^2|i\in S] \notag\\& ~~~~~~~-2\b{E}[\| (\b{A}^*_S- \b{A}^{(t)}_S(\b{I} - \Lambda^{(t)}_S))\b{x}^*_S \|\|\b{A}^{(t)}_S\vartheta^{(R)}_S\sgn(\b{x}_i^{R})\||i\in S],\notag\\
	&\leq  \underbrace{\b{E}[\|( \b{A}^*_S- \b{A}^{(t)}_S(\b{I} - \Lambda^{(t)}_S))\b{x}^*_S \|^2|i\in S]}_{\heartsuit}+ \underbrace{\b{E}[\|\b{A}^{(t)}_S\vartheta^{(R)}_S\sgn(\hat{\b{x}}_i)\|^2|i\in S]}_{\diamondsuit}.
	\end{align}
	We will now consider each term in \eqref{iht:varZ} separately. We start with \[\heartsuit\]. Since \[\b{x}^*_S\]s are conditionally independent of \[S\], \[\b{E}[\b{x}_S^*\b{x}_S^{*^\top}] = \b{I}\]. Therefore, we can simplify this expression as
	\begin{align*}
	\heartsuit &:= \b{E}[\|( \b{A}^*_S- \b{A}^{(t)}_S(\b{I} - \Lambda^{(t)}_S))\b{x}^*_S \|^2|i\in S] 
	= \b{E}[\| \b{A}^*_S- \b{A}^{(t)}_S(\b{I} - \Lambda^{(t)}_S)\|^2_F|i\in S].
	\end{align*}
	Rearranging the terms we have the following for \[\heartsuit\],
	\begin{align*}
	\heartsuit & %\b{E}[\|( \b{A}^*_S- \b{A}^{(t)}_S(\b{I} - \Lambda^{(t)}_S))\b{x}^*_S \|^2|i\in S] 
	= \b{E}[\| \b{A}^*_S- \b{A}^{(t)}_S(\b{I} - \Lambda^{(t)}_S)\|^2_F|i\in S]
%	&= \b{E}[\| \b{A}^*_S- \b{A}^*_S(\b{I} - \Lambda^{(t)}_S) + \b{A}^*_S(\b{I} - \Lambda^{(t)}_S)- \b{A}^{(t)}_S(\b{I} - \Lambda^{(t)}_S)\|^2_F|i\in S], \notag \\
	= \b{E}[\| \b{A}^*_S\Lambda^{(t)}_S + (\b{A}^*_S - \b{A}^{(t)}_S)(\b{I} - \Lambda^{(t)}_S)\|^2_F|i\in S].
%	&= \b{E}[\| \b{A}^*_S\Lambda^{(t)}_S\|^2_F|i\in S] + \b{E}[\|(\b{A}^*_S - \b{A}^{(t)}_S)(\b{I} - \Lambda^{(t)}_S)\|^2_F|i\in S] + 2\b{E}[\| \b{A}^*_S\Lambda^{(t)}_S\|_F\|(\b{A}^*_S - \b{A}^{(t)}_S)(\b{I} - \Lambda^{(t)}_S)\|_F|i\in S].
	\end{align*}
	Therefore, \[\heartsuit\] can be upper-bounded as
	\begin{align}
	\label{iht:varZ_term1}
	\heartsuit %&:= \b{E}[\|( \b{A}^*_S- \b{A}^{(t)}_S(\b{I} - \Lambda^{(t)}_S))\b{x}^*_S \|^2|i\in S] 
	%= \b{E}[\| \b{A}^*_S- \b{A}^{(t)}_S(\b{I} - \Lambda^{(t)}_S)\|^2_F|i\in S],\notag \\
	%&= \b{E}[\| \b{A}^*_S- \b{A}^*_S(\b{I} - \Lambda^{(t)}_S) + \b{A}^*_S(\b{I} - \Lambda^{(t)}_S)- \b{A}^{(t)}_S(\b{I} - \Lambda^{(t)}_S)\|^2_F|i\in S], \notag \\
	%&= \b{E}[\| \b{A}^*_S\Lambda^{(t)}_S + (\b{A}^*_S - \b{A}^{(t)}_S)(\b{I} - \Lambda^{(t)}_S)\|^2_F|i\in S], \notag\\
	\leq \underbrace{\b{E}[\| \b{A}^*_S\Lambda^{(t)}_S\|^2_F|i\in S]}_{\heartsuit_1} + &\underbrace{\b{E}[\|(\b{A}^*_S - \b{A}^{(t)}_S)(\b{I} - \Lambda^{(t)}_S)\|^2_F|i\in S]}_{\heartsuit_2} \notag\\
	&~~~~~~~~~~~~~~~~+ \underbrace{2\b{E}[\| \b{A}^*_S\Lambda^{(t)}_S\|_F\|(\b{A}^*_S - \b{A}^{(t)}_S)(\b{I} - \Lambda^{(t)}_S)\|_F|i\in S]}_{\heartsuit_3}.
	\end{align}
	
	For \[\heartsuit_1\], since \[ \|\b{A}^{(t)}_S\| \leq \sqrt{k}\epsilon_t + 2\], we have
	\begin{align*}
	\heartsuit_1:=\b{E}[\| \b{A}^*_S\Lambda^{(t)}_S\|^2_F|i\in S] &\leq \b{E}[\|\b{A}^*_S\|\|\Lambda^{(t)}_S\|^2_F|i\in S]
	%\b{E}[\|\Lambda^{(t)}_S\|^2_F|i\in S]
%	 &\leq \|\b{A}^*_S\|\sum_{j\in S} \b{E}[(\lambda^{(t)}_j)^2|i\in S] \leq \sum_{j\in S}(\lambda^{(t)}_j)^2
	 \leq \|\b{A}^*_S\|\textstyle\sum_{j\in S} (\lambda^{(t)}_j)^2 %\leq \sum_{j\in S}(\lambda^{(t)}_j)^2
	\leq  k(\sqrt{k}\epsilon_t + 2)\tfrac{\epsilon_t^4}{4}.
	\end{align*}
	%
%	where we have used the fact that \[(\lambda^{(t)}_j)^2 \leq \tfrac{\epsilon_t^4}{4}\], therefore, 
%	
%	\begin{align}
%	\b{E}[\|\Lambda^{(t)}_S\|^2_F|i\in S] 
%	\leq k\tfrac{\epsilon_t^4}{4}.
%	\end{align}
%	%
	Next, since \[(1 - \lambda^{(t)}_j) \leq 1\], we have the following bound for \[\heartsuit_2\]
	\begin{align*}
	\heartsuit_2 := \b{E}[\|(\b{A}^*_S - \b{A}^{(t)}_S)(\b{I} - \Lambda^{(t)}_S)\|^2_F|i\in S] \leq \b{E}[\|\b{A}^*_S - \b{A}^{(t)}_S\|^2_F|i\in S] \leq \|\b{A}^*_S - \b{A}^{(t)}_S\|^2_F \leq k\epsilon_t^2.
	\end{align*}
	%
%	 Now, 
%	
%	\begin{align}
%	\label{iht:varZ_term1_2}
%	\b{E}[\|\b{A}^*_S - \b{A}^{(t)}_S\|^2_F|i\in S] \leq \|\b{A}^*_S - \b{A}^{(t)}_S\|^2_F \leq k\epsilon_t^2
%	\end{align}
%	%
	%	For the next term, 
	%	
	%	\begin{align}
	%	\label{iht:varZ_term1_3}
	%	2\b{E}[\| \b{A}^*_S\b{D}_{\lambda^{t}_S}\|_F\|(\b{A}^*_S - \b{A}^{(t)}_S)\b{D}_{(1 - \lambda^{t}_S)}\|_F|i\in S] \leq k\epsilon_t^3
	%	\end{align}
	%	%
	Further, \[\heartsuit_3\] can be upper-bounded by using bounds for \[\heartsuit_1\] and \[\heartsuit_2\]. Combining the results of upper-bounding \[\heartsuit_1\], \[\heartsuit_2\], and \[\heartsuit_3\] we have the following for \eqref{iht:varZ_term1}
	\begin{align*}
	\heartsuit \leq \b{E}[\|( \b{A}^*_S- \b{A}^{(t)}_S(\b{I} - \Lambda^{(t)}_S))\b{x}^*_S \|^2|i\in S] = \mathcal{O}(k\epsilon_t^2).
	\end{align*}

    Next, by Claim~\ref{iht:var_theta_abs}, \[\vartheta^{(R)}_j\] is upper-bounded as \[	|\vartheta^{(R)}_{j}|  \leq  \siri{\mathcal{O}(t_{\beta})}.\]
%    
%    
%    \begin{align*}
%    |\vartheta^{(R)}_{i_1}| 
%    & \leq  \mathcal{O}(k^2\tfrac{\mu_t}{\sqrt{n}}\epsilon^2_t).
%    \end{align*}
%    %
%    
%    	
%    	\begin{align*}
%    
%    	\end{align*}
%    	%
	with probability \[(1 - \delta_{\beta}^{(t)})\]. Therefore, the term \[\diamondsuit\], the second term of \eqref{iht:varZ}, can be bounded as
	\begin{align*}
	%\b{E}[\|\b{A}^{(t)}_S\vartheta^{(R)}_S\sgn(\b{x}_i^{R})\|^2|i\in S]
	\diamondsuit \leq \|\b{A}^{(t)}_S\vartheta^{(R)}_S\sgn(\hat{\b{x}}_i)\|^2 \leq (\sqrt{k}\epsilon_t + 2)^2 k\siri{\mathcal{O}(t_{\beta})}^2 
	%= \textcolor{red}{\mathcal{O}( k^5\epsilon_t^4\tfrac{\mu_t^2}{n})} 
	= \siri{ \mathcal{O}( kt_{\beta}^2)}.
	\end{align*}
	%
   
%  
%	
%	\begin{align*}
%	\|\b{A}^{(t)}_S\| &\leq (\sqrt{k}\epsilon_t + 2).
%	\end{align*}
%	
%	Finally, consider the third term of \eqref{iht:varZ},
%	
%	\begin{align}
%	\label{iht:varZ_term3}
%	2\b{E}[\| (\b{A}^*_S- \b{A}^{(t)}_S\b{D}_{(1 - \lambda^{t}_S)})\b{x}^*_S \|\|\b{A}^{(t)}_S\vartheta^{(R)}_S\sgn(\b{x}_i^{R})\||i\in S]
%	&\leq  2\| (\b{A}^*_S- \b{A}^{(t)}_S\b{D}_{(1 - \lambda^{t}_S)})\b{x}^*_S \|\|\b{A}^{(t)}_S\vartheta^{(R)}_S\sgn(\b{x}_i^{R})\|,\notag\\
%	&\leq \mathcal{O}(k^4\epsilon_t^3\tfrac{\mu_t}{\sqrt{n}}))
%	\end{align}
%	%
	Finally, combining all the results, the term of interest in \eqref{iht:varZ} has the following form
	\begin{align*}
	&\b{E}[\|( \b{y}- \b{A}^{(t)}\hat{\b{x}})\sgn(\hat{\b{x}}_i)\|^2|i\in S] 
	= \mathcal{O}(k\epsilon_t^2) +  
	%\textcolor{red}{\mathcal{O}( k^5\epsilon_t^4\tfrac{\mu_t^2}{n})}+ 
	\siri{ \mathcal{O}( kt_{\beta}^2)}.
	\end{align*}
\end{proof}
% % % end of Lemma variance of w (the gradient vector) % % % % % % % % % % % % % % % % % % % % % % % % % % % % % % % % % % % % % % % % %

% % % % % % % % % % % % % % % % % % % % variance statistic of gradient matrix % % % % % % % % % % % % % % % % % % % % % % % % % % % % % 
\begin{claim}[\textbf{Bound on variance parameter for concentration of gradient matrix}] \label{norm_exp_yAs_yAst} 
\begin{em}
With probability \[(1-\delta_{\HT}^{(t)} - \delta_{\beta}^{(t)})\], the variance parameter
	\[\|\mathbf{E}[(\b{y}-\b{A}^{(t)}\hat{\b{x}}) (\b{y}-\b{A}\hat{\b{x}})^\top]\|\] is upper-bounded by %\[\textcolor{green}{\mathcal{O}(\tfrac{k}{m})\|\b{A}^*\|^2 ~or~ }\]
	\[\mathcal{O}(\tfrac{k^2t_{\beta}^2}{m})\|\b{A}^*\|^2\].
	\end{em}
\end{claim}
% % % % % % % % % % % % % % % % % % % % Concentration of gradient matrix % % % % % % % % % % % % % % % % % % % % % % % % % % % % % % %
\begin{proof}[Proof of Claim~\ref{norm_exp_yAs_yAst}]
	%We begin by considering \[(\b{y}-\b{A}^{(t)}\hat{\b{x}}) (\b{y}-\b{A}^{(t)}\hat{\b{x}})^\top\]. 
	Let \[\mathcal{F}_{\b{x}^*}\] be the event that \[\sgn(\b{x}^*) = \sgn(\hat{\b{x}})\], and let \[\mathbbm{1}_{\mathcal{F}_{\b{x}^*}}\] denote the indicator function corresponding to this event. As we show in Lemma~\ref{our:signed_supp}, this event occurs with probability at least \[(1 - \delta_{\beta}^{(t)} - \delta_{\HT}^{(t)})\], therefore,
	\begin{align*}
	\mathbf{E}[(\b{y}-\b{A}^{(t)}&\hat{\b{x}})(\b{y}-\b{A}^{(t)}\hat{\b{x}})^\top] \\
	&= \mathbf{E}[(\b{y}-\b{A}^{(t)}\hat{\b{x}}) (\b{y}-\b{A}^{(t)}\hat{\b{x}})^\top\mathbbm{1}_{\mathcal{F}_{\b{x}^*}}] +\mathbf{E}[(\b{y}-\b{A}^{(t)}\hat{\b{x}}) (\b{y}-\b{A}^{(t)}\hat{\b{x}})^\top\mathbbm{1}_{\bar{\mathcal{F}}_{x^*}}], \\
	&= \mathbf{E}[(\b{y}-\b{A}^{(t)}\hat{\b{x}}) (\b{y}-\b{A}^{(t)}\hat{\b{x}})^\top\mathbbm{1}_{\mathcal{F}_{x^*}}] \pm \gamma.
	\end{align*}
	Here, \[\gamma\] is small. Under the event \[\mathcal{F}_{x^*}\], \[\hat{\b{x}}\] has the correct signed-support. Again, since \[\mathbbm{1}_{\mathcal{F}_{\b{x}^*}} = \mathbf{1} - \mathbbm{1}_{\bar{\mathcal{F}}_{\b{x}^*}}\],
	\begin{align*}
	\mathbf{E}[(\b{y}-\b{A}^{(t)}\hat{\b{x}}) (\b{y}-\b{A}^{(t)}\hat{\b{x}})^\top] &= \mathbf{E}[(\b{y}-\b{A}^{(t)}\hat{\b{x}}) (\b{y}-\b{A}^{(t)}\hat{\b{x}})^\top(\mathbf{1} - \mathbbm{1}_{\bar{\mathcal{F}}_{x^*}})] \pm \gamma,\\
	&= \mathbf{E}[(\b{y}-\b{A}^{(t)}\hat{\b{x}}) (\b{y}-\b{A}^{(t)}\hat{\b{x}})^\top] \pm \gamma.
	\end{align*}
	Now, using Lemma~\ref{iht:R_th_term} with probability at least \[(1-\delta_{\HT}^{(t)} - \delta_{\beta}^{(t)})\], \[\hat{\b{x}}_S\] admits the following expression
	\begin{align*}
	\hat{\b{x}}_S := \b{x}^{(R)}_S = (\b{I} - \Lambda^{(t)}_S) \b{x}^*_S  + \vartheta^{(R)}_S.
	\end{align*}
	Therefore we have
	\begin{align*}
	\b{y} - \b{A}^{(t)}\hat{\b{x}} &=(\b{A}^*_S  - \b{A}^{(t)}_S(\b{I} - \Lambda^{(t)}_S))\b{x}^*_S  - \b{A}^{(t)}_S\vartheta^{(R)}_S.
	\end{align*}
     Using the expression above \[\b{E}[(\b{y} - \b{A}^{(t)}\hat{\b{x}})(\b{y} - \b{A}^{(t)}\hat{\b{x}})^\top] \] can be written as
	\begin{align*}
	\b{E}[(\b{y} &- \b{A}^{(t)}\hat{\b{x}})(\b{y} - \b{A}^{(t)}\hat{\b{x}})^\top] \\
	&= \b{E}[((\b{A}^*_S  - \b{A}^{(t)}_S(\b{I} - \Lambda^{(t)}_S))\b{x}^*_S  - \b{A}^{(t)}_S\vartheta^{(R)}_S)((\b{A}^*_S  - \b{A}^{(t)}_S(\b{I} - \Lambda^{(t)}_S))\b{x}^*_S  - \b{A}^{(t)}_S\vartheta^{(R)}_S)^\top].
%	&= \b{E}[(\b{A}^*_S  - \b{A}^{(t)}_S(\b{I} - \Lambda^{(t)}_S))\b{x}^*_S\b{x}^{*^\top}_S(\b{A}^{*^\top}_S  - (\b{I} - \Lambda^{(t)}_S)\b{A}^{(t)^\top}_S)] - \b{E}[\b{A}^{(t)}_S\vartheta^{(R)}_S\b{x}^{*^\top}_S(\b{A}^{*^\top}_S  - (\b{I} - \Lambda^{(t)}_S)\b{A}^{(t)^\top}_S)] \\&~~~~~~~~~~~~~~- \b{E}[(\b{A}^*_S  - \b{A}^{(t)}_S(\b{I} - \Lambda^{(t)}_S))\b{x}^*_S(\vartheta^{(R)}_S)^\top\b{A}^{(t)^\top}_S] + \b{E}[\b{A}^{(t)}_S\vartheta^{(R)}_S(\vartheta^{(R)}_S)^\top\b{A}^{(t)^\top}_S],
	\end{align*}
	Sub-conditioning, we have
	\begin{align*}
	\b{E}[(\b{y} - \b{A}^{(t)}\hat{\b{x}})(\b{y} - &\b{A}^{(t)}\hat{\b{x}})^\top]\\
= &\b{E}_S[(\b{A}^*_S  - \b{A}^{(t)}_S(\b{I} - \Lambda^{(t)}_S))\b{E}_{\b{x}^*_S}[\b{x}^*_S\b{x}^{*^\top}_S|S](\b{A}^{*^\top}_S  - (\b{I} - \Lambda^{(t)}_S)\b{A}^{(t)^\top}_S)] \\
&- \b{E}_S[\b{A}^{(t)}_S\b{E}_{\b{x}^*_S}[\vartheta^{(R)}_S\b{x}^{*^\top}_S|S](\b{A}^{*^\top}_S  - (\b{I} - \Lambda^{(t)}_S)\b{A}^{(t)^\top}_S)] \\
	&- \b{E}_S[(\b{A}^*_S  - \b{A}^{(t)}_S(\b{I} - \Lambda^{(t)}_S))\b{E}_{\b{x}^*_S}[\b{x}^*_S(\vartheta^{(R)}_S)^\top|S]\b{A}^{(t)^\top}_S] \\
	&+ \b{E}_S[\b{A}^{(t)}_S\b{E}_{\b{x}^*_S}[\vartheta^{(R)}_S(\vartheta^{(R)}_S)^\top|S]\b{A}^{(t)^\top}_S].
	\end{align*}
			
	Now, since \[\b{E}_{\b{x}^*_S}[\b{x}^*_S\b{x}^{*^\top}_S|S] = \b{I}\],
%	
%	\begin{align*}
%	&\b{E}[(\b{y} - \b{A}^{(t)}\hat{\b{x}})(\b{y} - \b{A}^{(t)}\hat{\b{x}})^\top]\\
%	&= \b{E}_S[(\b{A}^*_S  - \b{A}^{(t)}_S(\b{I} - \Lambda^{(t)}_S))(\b{A}^{*^\top}_S  - (\b{I} - \Lambda^{(t)}_S)\b{A}^{(t)^\top}_S)]
%	- \b{E}_S[\b{A}^{(t)}_S\b{E}_{\b{x}^*_S}[\vartheta^{(R)}_S\b{x}^{*^\top}_S|S](\b{A}^{*^\top}_S  - (\b{I} - \Lambda^{(t)}_S)\b{A}^{(t)^\top}_S)] \notag \\
%	&~~~~~~~~~~~~~~ - \b{E}_S[(\b{A}^*_S  - \b{A}^{(t)}_S(\b{I} - \Lambda^{(t)}_S))\b{E}_{\b{x}^*_S}[\b{x}^*_S(\vartheta^{(R)}_S)^\top|S]\b{A}^{(t)^\top}_S] + \b{E}_S[\b{A}^{(t)}_S\b{E}_{\b{x}^*_S}[\vartheta^{(R)}_S(\vartheta^{(R)}_S)^\top|S]\b{A}^{(t)^\top}_S].
%	\end{align*}
%	%	 
	\begin{align}\label{iht:mat_W_var}
	\|\b{E}[(\b{y} - \b{A}^{(t)}\hat{\b{x}})(\b{y} - \b{A}^{(t)}\hat{\b{x}})^\top]\|\notag
&\leq \underbrace{\|\b{E}_S[(\b{A}^*_S  - \b{A}^{(t)}_S(\b{I} - \Lambda^{(t)}_S))(\b{A}^{*^\top}_S  - (\b{I} - \Lambda^{(t)}_S)\b{A}^{(t)^\top}_S)]\|}_{\clubsuit} \\
&~~~~+\underbrace{\| \b{E}_S[\b{A}^{(t)}_S\b{E}_{\b{x}^*_S}[\vartheta^{(R)}_S\b{x}^{*^\top}_S|S](\b{A}^{*^\top}_S  - (\b{I} - \Lambda^{(t)}_S)\b{A}^{(t)^\top}_S)]\|}_{\spadesuit} \notag \\
	&~~~~+\underbrace{\| \b{E}_S[(\b{A}^*_S  - \b{A}^{(t)}_S(\b{I} - \Lambda^{(t)}_S))\b{E}_{\b{x}^*_S}[\b{x}^*_S(\vartheta^{(R)}_S)^\top|S]\b{A}^{(t)^\top}_S]\|}_{\heartsuit} \notag\\
	 &~~~~+ \underbrace{\|\b{E}_S[\b{A}^{(t)}_S\b{E}_{\b{x}^*_S}[\vartheta^{(R)}_S(\vartheta^{(R)}_S)^\top|S]\b{A}^{(t)^\top}_S]\|}_{\diamondsuit} .
	\end{align}
	
%	
%	\begin{align}
%	\label{iht:mat_W_var}
%	&\|\b{E}[(\b{y} - \b{A}^{(t)}\b{x}^{(R)})(\b{y} - \b{A}^{(t)}\b{x}^{(R)})^\top]\|\\
%	&\leq \|\b{E}_S[(\b{A}^*_S  - \b{A}^{(t)}_S(\b{I} - \Lambda^{(t)}_S))(\b{A}^{*^\top}_S  - (\b{I} - \Lambda^{(t)}_S)\b{A}^{(t)^\top}_S)]\|
%	+ \|\b{E}_S[\b{A}^{(t)}_S\b{E}_{\b{x}^*_S}[\vartheta^{(R)}_S\b{x}^{*^\top}_S|S](\b{A}^{*^\top}_S  - (\b{I} - \Lambda^{(t)}_S)\b{A}^{(t)^\top}_S)]\| \notag \\
%	&~~~~~~~~~~~~~~ + \|\b{E}_S[(\b{A}^*_S  - \b{A}^{(t)}_S(\b{I} - \Lambda^{(t)}_S))\b{E}_{\b{x}^*_S}[\b{x}^*_S(\vartheta^{(R)}_S)^\top|S]\b{A}^{(t)^\top}_S]\| + \|\b{E}_S[\b{A}^{(t)}_S\b{E}_{\b{x}^*_S}[\vartheta^{(R)}_S(\vartheta^{(R)}_S)^\top|S]\b{A}^{(t)^\top}_S]\|.
%	\end{align}
%	%		
	Let's start with the first term \[(\clubsuit)\] of \eqref{iht:mat_W_var}, which can be written as%we have \[\b{E}_S[(\b{A}^*_S  - \b{A}^{(t)}_S(\b{I} - \Lambda^{(t)}_S))(\b{A}^{*^\top}_S  - (\b{I} - \Lambda^{(t)}_S)\b{A}^{(t)^\top}_S)]\],
	\begin{align}
	\label{iht:mat_W_var_term_1}
	\clubsuit : %\b{E}_S[(\b{A}^*_S  - \b{A}^{(t)}_S(\b{I} - \Lambda^{(t)}_S))(\b{A}^{*^\top}_S  - (\b{I} - \Lambda^{(t)}_S)\b{A}^{(t)^\top}_S)]\\
\leq \underbrace{\|\b{E}_S[\b{A}^*_S \b{A}^{*^\top}_S]\|}_{\clubsuit_1}  +  \underbrace{\|\b{E}_S[\b{A}^*_S (\b{I} - \Lambda^{(t)}_S)\b{A}^{(t)^\top}_S]\|}_{\clubsuit_2} &+  \underbrace{\|\b{E}_S[\b{A}^{(t)}_S(\b{I} - \Lambda^{(t)}_S)\b{A}^{*^\top}_S]\|}_{\clubsuit_3} \notag \\&+  \underbrace{\|\b{E}_S[\b{A}^{(t)}_S(\b{I} - \Lambda^{(t)}_S)^2\b{A}^{(t)^\top}_S]\|}_{\clubsuit_4}.
	\end{align}
	
	Now consider each term of equation \eqref{iht:mat_W_var_term_1}. First, since
	\begin{align*}
	\b{E}_S[\b{A}^*_S \b{A}^{*^\top}_S]  &= \b{E}_S[\textstyle\sum\limits_{i ,j\in S} \b{A}^*_i \b{A}^{*^\top}_j \mathbbm{1}_{i,j\in S}] = \textstyle\sum\limits_{i,j= 1}^{m} \b{A}^*_i \b{A}^{*^\top}_i \b{E}_S[\mathbbm{1}_{i,j\in S}],
	\end{align*}
	and \[\b{E}_S[\mathbbm{1}_{i,j\in S}] = \mathcal{O}(\tfrac{k^2}{m^2})\], we can upper-bound \[\clubsuit_1:=\|\b{E}_S[\b{A}^*_S \b{A}^{*^\top}_S] \| \] as 
	\begin{align*}
	\clubsuit_1:= \|\b{E}_S[\b{A}^*_S \b{A}^{*^\top}_S] \| 
	= \mathcal{O}(\tfrac{k^2}{m^2})\|\b{A}^* \b{1}^{m \times m} \b{A}^{*^\top}\| =  \mathcal{O}(\tfrac{k^2}{m})\|\b{A}^*\|^2,
	\end{align*}
	where \[\b{1}^{m \times m} \] denotes an \[m \times m\] matrix of ones.
	Now, we turn to \[\clubsuit_2:=\|\b{E}_S[\b{A}^*_S (\b{I} - \Lambda^{(t)}_S)\b{A}^{(t)^\top}_S] \|\] in  \eqref{iht:mat_W_var_term_1}, which can be simplified as
	\begin{align*}
	\clubsuit_2 %:=\|\b{E}_S[\b{A}^*_S (\b{I} - \Lambda^{(t)}_S)\b{A}^{(t)^\top}_S] \| 
	%&= \|\b{E}_S[\textstyle\sum_{i \in S} (1 - \lambda^{(t)}_i)\b{A}^*_i \b{A}^{(t)^\top}_i \mathbbm{1}_{i\in S}]\|,\\
%	\leq \|\textstyle\sum_{i,j,\ell =1}^{m} (1 - \lambda^{(t)})\b{A}^*_i \b{A}^{(t)^\top}_j \b{E}_S[\mathbbm{1}_{i\in S}]\|
	\leq \|\textstyle\sum\limits_{i,j =1}^{m} \b{A}^*_i \b{A}^{(t)^\top}_j \b{E}_S[\mathbbm{1}_{i,j\in S}]\|
%	\leq \|\textstyle\sum_{i =1}^{m}  \b{A}^*_i \b{A}^{(t)^\top}_i \b{E}_S[\mathbbm{1}_{i\in S}]\|
	 \leq\mathcal{O}(\tfrac{k^2}{m})\|\b{A}^*\|\|\b{A}^{(t)}\|.
	\end{align*}
	%
	% old Further, since \[\|\b{A}^{(t)}\| \leq \|\b{A}^{(t)} - \b{A}^*\|  + \|\b{A}^*\| \leq \sqrt{m} \epsilon_t + \|\b{A}^*\|\]
	Further, since \[\b{A}^{(t)}\] is \[(\epsilon_t,2)\]-close to \[\b{A}^*\], we have that \[\|\b{A}^{(t)}\| \leq \|\b{A}^{(t)} - \b{A}^*\|  + \|\b{A}^*\| \leq 3\|\b{A}^*\|\], therefore
	\begin{align*}
	\clubsuit_2 := \|\b{E}_S[\b{A}^*_S (\b{I} - \Lambda^{(t)}_S)\b{A}^{(t)^\top}_S] \| 
	& =\mathcal{O}(\tfrac{k^2}{m})\|\b{A}^*\|^2.
	\end{align*}
	Similarly, \[\clubsuit_3 \] \eqref{iht:mat_W_var_term_1} is also \[\mathcal{O}(\tfrac{k^2}{m})\|\b{A}^*\|^2\]. Next, we consider \[\clubsuit_4:= \|\b{E}_S[\b{A}^{(t)}_S(\b{I} - \Lambda^{(t)}_S)^2\b{A}^{(t)^\top}_S]\|  \] in \eqref{iht:mat_W_var_term_1} which can also be bounded similarly as
	\begin{align*}
	\clubsuit_4 %:= \|\b{E}_S[\b{A}^{(t)}_S(\b{I} - \Lambda^{(t)}_S)^2\b{A}^{(t)^\top}_S]\| &
%	= \b{E}_S[\textstyle\sum_{i \in S} (1 - \lambda^{(t)}_i)^2\b{A}^*_i \b{A}^{*^\top}_i \mathbbm{1}_{i\in S}]
%	=\textstyle\sum_{i \in [m]} \b{A}^*_i \b{A}^{*^\top}_i  \b{E}_S[\mathbbm{1}_{i\in S}]
	=  \mathcal{O}(\tfrac{k^2}{m})\|\b{A}^*\|^2.
	\end{align*}

	Therefore, we have the following for \[\clubsuit\] in \eqref{iht:mat_W_var} 
	\begin{align}
	\label{iht:mat_W_var_term_1_sim}
	\clubsuit:=\b{E}_S[(\b{A}^*_S  - \b{A}^{(t)}_S(\b{I} - \Lambda^{(t)}_S))(\b{A}^{*^\top}_S  - (\b{I} - \Lambda^{(t)}_S)\b{A}^{(t)^\top}_S)]
	&= \mathcal{O}(\tfrac{k^2}{m})\|\b{A}^*\|^2.
	\end{align}
	Consider \[\spadesuit\]  in \eqref{iht:mat_W_var}. Letting \[\b{M} = \b{E}_{\b{x}^*_S}[\vartheta^{(R)}_S\b{x}^{*^\top}_S|S]\], and using the analysis similar to that shown in \ref{iht:bound_vareps}, we have that elements of  \[\b{M} \in \mathbb{R}^{k \times k}\] are given by
		\begin{align*}
	\b{M}_{i,j} =	\b{E}_{\b{x}_S^*}[\vartheta^{(R)}_{i}\b{x}^*_{j}|S] 
		\begin{cases}
		\leq  \mathcal{O}(\gamma^{(R)}_i), & ~\text{for}~ i = j, \\
		 	= \mathcal{O}(\epsilon_t),& ~\text{for}~ i \neq j.
		\end{cases}
		\end{align*}

    We have the following,
	\begin{align*}
	\spadesuit:= \b{E}_S[\b{A}^{(t)}_S\b{E}_{\b{x}^*_S}[\vartheta^{(R)}_S\b{x}^{*^\top}_S|S](\b{A}^{*^\top}_S  \hspace{-3pt} - (\b{I} - \Lambda^{(t)}_S)\b{A}^{(t)^\top}_S)]  = \b{E}_S[\b{A}^{(t)}_S\b{M}(\b{A}^{*^\top}_S \hspace{-3pt} - (\b{I} - \Lambda^{(t)}_S)\b{A}^{(t)^\top}_S)].
	\end{align*}
	Therefore, since \[\b{E}_S[\mathbbm{1}_{i,j\in S}|S] = \mathcal{O}(\tfrac{k^2}{m^2})\], and \[\|\b{1}^{m \times m}\| = m\],
	\begin{align*}
	\spadesuit 
	:=\|\b{E}_S[\b{A}^{(t)}_S\b{M}(\b{A}^{*^\top}_S  - (\b{I} - \Lambda^{(t)}_S)&\b{A}^{(t)^\top}_S)]\|\\
	&= \|\textstyle\sum\limits_{i,j = 1}^{m} \b{M}_{i,j}\b{A}^{(t)}_i(\b{A}^{*^\top}_j  - (1 - \lambda^{(t)}_j)\b{A}^{(t)^\top}_j) \b{E}_S[\mathbbm{1}_{i,j\in S}|S]\|,\\
	&= \mathcal{O}(\epsilon_t)\|\textstyle\sum\limits_{i,j = 1}^{m} \b{A}^{(t)}_i(\b{A}^{*^\top}_j  - (1 - \lambda^{(t)}_j)\b{A}^{(t)^\top}_j) \b{E}_S[\mathbbm{1}_{i,j\in S}|S]\|,\\
		&= \mathcal{O}(\epsilon_t)\mathcal{O}(\tfrac{k^2}{m^2})  (\|\b{A}^{(t)}\b{1}^{m \times m} \b{A}^{*^\top}\|  + \|\b{A}^{(t)}\b{1}^{m \times m} \b{A}^{(t)^\top}\|),\\
		&= \mathcal{O}(\epsilon_t)\mathcal{O}(\tfrac{k^2}{m}) \|\b{A}^{*}\|^2.
	%&\leq m^2 |\vartheta^{(R)}_j||\b{x}^{*}_{\max}|\mathcal{O}(\tfrac{k^2}{m^2})  (\|\b{A}^{(t)}\b{A}^{*^\top}\|  + \|\b{A}^{(t)}\b{A}^{(t)^\top}\|),\\
%	&\leq m^2|\vartheta^{(R)}_j||\b{x}^{*}_{\max}|\mathcal{O}(\tfrac{k^2}{m^2}) \|\b{A}^{*}\|^2.
	\end{align*}
	
	Therefore,
%	Further, since by Lemma~
%		
%		\begin{align*}
%		|\vartheta^{(R)}_{i_1}| 
%		& \leq  \mathcal{O}(k^2\tfrac{\mu_t}{\sqrt{n}}\epsilon^2_t|\b{x}_{\max}^*|),
%		\end{align*}
%		%
%	using Lemma~\ref{iht:var_theta_abs}, we have 
	\begin{align*}
	\spadesuit:= \|\b{E}_S[\b{A}^{(t)}_S\b{M}(\b{A}^{*^\top}_S  - (\b{I} - \Lambda^{(t)}_S)\b{A}^{(t)^\top}_S)]\|
	%\textcolor{red}{\leq \mathcal{O}(k^4\tfrac{\mu_t}{\sqrt{n}}\epsilon^2_t|\b{x}_{\max}^*|^2) \|\b{A}^{*}\|^2
	%\leq \mathcal{O}(k^4\tfrac{\mu_t}{\sqrt{n}}\epsilon^2_t) \|\b{A}^{*}\|^2.} 
	\siri{ = \mathcal{O}(\tfrac{k^2}{m}) \epsilon_t\|\b{A}^{*}\|^2.}
	\end{align*}
	Similarly, \[\heartsuit\] in \eqref{iht:mat_W_var} is also bounded as \[\spadesuit\]. Next, we consider \[\diamondsuit\] in \eqref{iht:mat_W_var}. In this case, letting \[\b{E}_{\b{x}^*_S}[\vartheta^{(R)}_S(\vartheta^{(R)}_S)^\top|S] = \b{N}\], where \[\b{N} \in \mathbb{R}^{k \times k}\] is a matrix whose each entry \[\b{N}_{i,j} \leq |\vartheta^{(R)}_i||\vartheta^{(R)}_j|\]. Further, by Claim~\ref{iht:var_theta_abs}, each element \[\vartheta^{(R)}_{j}\] is upper-bounded as 
	%	    
	%	    \begin{align*}
	%	    |\vartheta^{(R)}_{i_1}| 
	%	    & \leq  \mathcal{O}(k^2\tfrac{\mu_t}{\sqrt{n}}\epsilon^2_t).
	%	    \end{align*}
	%	    
		   	\begin{align*}
		   	|\vartheta^{(R)}_{j}| 
		   	& \leq  \siri{\mathcal{O}(t_{\beta})}.
		   	\end{align*}
			with probability at least \[(1 - \delta_{\beta}^{(t)})\]. Therefore,
	\begin{align*}
	\diamondsuit %:=\|\b{E}_S[\b{A}^{(t)}_S\b{N}\b{A}^{(t)^\top}_S]\|  
	&= \|\textstyle\sum\limits_{i,j = 1}^{m} \b{N}_{i,j}\b{A}^{(t)}_i\b{A}^{(t)^\top}_j \b{E}_S[\mathbbm{1}_{i,j\in S}|S]\|=  \underset{i,j}{\max} |\vartheta^{(R)}_i||\vartheta^{(R)}_j|\mathcal{O}(\tfrac{k^2}{m^2}) \|\textstyle\sum\limits_{i,j = 1}^{m} \b{A}^{(t)}_i\b{A}^{(t)^\top}_j\|.
	%\\
%	&=  m^2\|\vartheta^{(R)}\|^2\mathcal{O}(\tfrac{k^2}{m^2}) \|\b{A}^{(t)}\|\|\b{A}^{(t)}\|,\\
%	&=\mathcal{O}(k^2\|\vartheta^{(R)}\|^2)\|\b{A}^*\|^2,\\
	%&= \mathcal{O}(mk^2(k^2\tfrac{\mu_t}{\sqrt{n}}\epsilon^2_t|\b{x}_{\max}^*|)^2)\|\b{A}^*\|^2,\\
%	&= \mathcal{O}(mk^6\tfrac{\mu_t^2}{n}\epsilon^4_t)\|\b{A}^*\|^2.
	\end{align*}
	Again, using the result on \[|\vartheta^{(R)}_{i_1}| \], we have
	\begin{align*}
	\diamondsuit:=\|\b{E}_S[\b{A}^{(t)}_S\b{N}\b{A}^{(t)^\top}_S]\|  %&= \|\textstyle\sum_{i,j = 1}^{m} \b{N}_{i,j}\b{A}^{(t)}_i\b{A}^{(t)^\top}_j \b{E}_S[\mathbbm{1}_{i,j\in S}|S]\|=  \|\vartheta^{(R)}\|^2\mathcal{O}(\tfrac{k^2}{m^2}) \|\textstyle\sum_{i,j = 1}^{m} \b{A}^{(t)}_i\b{A}^{(t)^\top}_j\|,\\
	=  m~\underset{i,j}{\max} |\vartheta^{(R)}_i||\vartheta^{(R)}_j|\mathcal{O}(\tfrac{k^2}{m^2}) \|\b{A}^{(t)}\|\|\b{A}^{(t)}\|
%	\leq\mathcal{O}(k^2\|\vartheta^{(R)}\|^2)\|\b{A}^*\|^2
	%&= \mathcal{O}(mk^2(k^2\tfrac{\mu_t}{\sqrt{n}}\epsilon^2_t|\b{x}_{\max}^*|)^2)\|\b{A}^*\|^2,\\
	%\textcolor{red}{\leq \mathcal{O}(k^6\tfrac{\mu_t^2}{n}\epsilon^4_t)\|\b{A}^*\|^2.} 
	\siri{= \mathcal{O}(\tfrac{k^2t_{\beta}^2}{m})\|\b{A}^*\|^2.}
	\end{align*}

   Combining all the results for \[\clubsuit, ~\spadesuit, ~\heartsuit\] and \[\diamondsuit\], we have, 
  	\begin{align*}
  	\|\b{E}[(\b{y} - \b{A}^{(t)}\hat{\b{x}})(\b{y} - &\b{A}^{(t)}\hat{\b{x}})^\top]\|\\
  %	&\leq \|\b{E}_S[(\b{A}^*_S  - \b{A}^{(t)}_S(\b{I} - \Lambda^{(t)}_S))(\b{A}^{*^\top}_S  - (\b{I} - \Lambda^{(t)}_S)\b{A}^{(t)^\top}_S)]\|
%  	+ \|\b{E}_S[\b{A}^{(t)}_S\b{E}_{\b{x}^*_S}[\vartheta^{(R)}_S\b{x}^{*^\top}_S|S](\b{A}^{*^\top}_S  - (\b{I} - \Lambda^{(t)}_S)\b{A}^{(t)^\top}_S)]\| \notag \\
 % 	&~~~~~~~~~~~~~~ + \|\b{E}_S[(\b{A}^*_S  - \b{A}^{(t)}_S(\b{I} - \Lambda^{(t)}_S))\b{E}_{\b{x}^*_S}[\b{x}^*_S(\vartheta^{(R)}_S)^\top|S]\b{A}^{(t)^\top}_S]\| + \|\b{E}_S[\b{A}^{(t)}_S\b{E}_{\b{x}^*_S}[\vartheta^{(R)}_S(\vartheta^{(R)}_S)^\top|S]\b{A}^{(t)^\top}_S]\|,\\
  	%& =  \textcolor{red}{\mathcal{O}(\tfrac{k}{m})\|\b{A}^*\|^2 +  \mathcal{O}(k^4\tfrac{\mu_t}{\sqrt{n}}\epsilon^2_t) \|\b{A}^{*}\|^2 +  \mathcal{O}(k^4\tfrac{\mu_t}{\sqrt{n}}\epsilon^2_t) \|\b{A}^{*}\|^2 + \mathcal{O}(k^6\tfrac{\mu_t^2}{n}\epsilon^4_t)\|\b{A}^*\|^2,}\\
  	%& = \mathcal{O}(\tfrac{k}{m})\|\b{A}^*\|^2.\\
  	& =  \siri{\mathcal{O}(\tfrac{k^2}{m})\|\b{A}^*\|^2 +  \mathcal{O}(\tfrac{k^2}{m}) \epsilon_t\|\b{A}^{*}\|^2 +  \mathcal{O}(\tfrac{k^2}{m}) \epsilon_t\|\b{A}^{*}\|^2 + \mathcal{O}(\tfrac{k^2t_{\beta}^2}{m})\|\b{A}^*\|^2,}\\
%  	& = \mathcal{O}(\tfrac{k}{m})\|\b{A}^*\|^2 or 
  	&=\mathcal{O}(\tfrac{k^2t_{\beta}^2}{m})\|\b{A}^*\|^2.
  	\end{align*}
\end{proof}
% % % % % % % % % % % % % % % % % % % % end of variance statistic of gradient matrix % % % % % % % % % % % % % % % % % % % % % % % % % 
\section{Additional Experimental Results}\label{app:additional_res}

We now present some additional results to highlight the features of NOODL. Specifically, we compare the performance of NOODL (for both dictionary and coefficient recovery) with the state-of-the-art provable techniques for DL presented in \cite{Arora15} (when the coefficients are recovered via a sparse approximation step after DL)\footnote{ The associated code is made available at \texttt{https://github.com/srambhatla/NOODL}.}. We also compare the performance of NOODL with the popular online DL algorithm in \cite{Mairal09}, denoted by {\texttt{Mairal} `\texttt{09}}. Here, the authors show that alternating between a \[\ell_1\]-based sparse approximation and dictionary update based on block co-ordinate descent converges to a stationary point, as compared to the true factors in case of NOODL.

\noindent\textbf{Data Generation:} We generate a \[(n = 1000) \times (m = 1500)\] matrix, with entries drawn from \[\c{N}(0,1)\], and normalize its columns to form the ground-truth dictionary \[\b{A}^*\]. Next, we perturb \[\b{A}^*\] with random Gaussian noise, such that the unit-norm columns of the resulting matrix, \[\b{A}^{(0)}\] are \[2/\log(n)\] away from \[\b{A}^*\], in \[\ell_2\]-norm sense, i.e., \[\epsilon_0 = 2/\log(n)\]; this satisfies the initialization assumptions in \ref{assumption:close}. At each iteration, we generate \[p = 5000\] samples \[\b{Y} \in \RR^{1000 \times 5000}\] as \[\b{Y} = \b{A}^*\b{X}^*\], where \[\b{X}^* \in \RR^{m \times p}\] has at most \[k = 10,~20, ~50,~\text{and} ~100\], entries per column, drawn from the Radamacher distribution.  We report the results in terms of relative Frobenius error for all the experiments, i.e., for a recovered matrix \[\hat{\b{M}}\], we report \[{\|\hat{\b{M}} - \b{M}^*\|_{\Fr}}/{\|\b{M}^*\|_{\Fr}}\]. To form the coefficient estimate for {\texttt{Mairal} `\texttt{09}} via Lasso \citep{Tibshirani1996} we use the FISTA \citep{Beck09} algorithm by searching across $10$ values of the regularization parameter at each iteration. Note that, although our phase transition analysis for NOODL shows that \[p=m\] suffices, we use \[p=5000\] in our convergence analysis for a fair comparison with related techniques. %to see if the bias of other techniques drops when provided with a large number of samples.  

\subsection{Coefficient Recovery}
Table~\ref{table_res} summarizes the results of the convergence analysis shown in Fig.~\ref{fig:recovery}. Here, we compare the dictionary and coefficient recovery performance of NOODL with other techniques. For \texttt{Arora15(``biased'')} and \texttt{Arora15(``unbiased'')}, we report the error in recovered coefficients after the HT step (\[\Xb_{\rm HT}\]) and the best error via sparse approximation using Lasso\footnote{We use the  Fast Iterative Shrinkage-Thresholding Algorithm (FISTA) \citep{Beck09}, which is among the most efficient algorithms for solving the \[\ell_1\]-regularized problems. Note that, in our experiments we fix the step-size for FISTA as \[1/L\], where \[L\] is the estimate of the Lipschitz constant (since \[\b{A}\] is not known exactly).} \cite{Tibshirani1996}, denoted as \[\Xb_{\rm Lasso}\], by scanning over \[50\] values of regularization parameter. For {\texttt{Mairal} `\texttt{09}} at each iteration of the algorithm we scan across $10$ values\footnote{ Note that, although scanning across $50$ values of the regularization parameter for this case would have led to better coefficient estimates and dictionary recovery, we choose $10$ values for this case since it is very expensive to scan across $50$  of regularization parameter at each step. This also highlights why {\texttt{Mairal} `\texttt{09}} may be prohibitive for large scale applications.} of the regularization parameter, to recover the best coefficient estimate using Lasso ( via FISTA), denoted as \[\b{X}_{\rm Lasso}\].

We observe that NOODL exhibits significantly superior performance across the board. Also, we observe that using sparse approximation after dictionary recovery, when the dictionary suffers from a bias, leads to poor coefficient recovery\footnote{When the dictionary is not known exactly, the guarantees may exist on coefficient recovery only in terms of closeness in \[\ell_2\]-norm sense, due to the error-in-variables (EIV) model for the dictionary \citep{Fuller2009, Wainwright2009}.}, as is the case with  \texttt{Arora15(``biased'')}, \texttt{Arora15(``unbiased'')}, and {\texttt{Mairal} `\texttt{09}}. This highlights the applicability of our approach in real-world machine learning tasks where coefficient recovery is of interest. In fact, it is a testament to the fact that, even in cases where dictionary recovery is the primary goal, making progress on the coefficients is also important for dictionary recovery.

In addition, the coefficient estimation step is also online in case of NOODL, while for the state-of-the-art provable techniques (which only recover the dictionary and incur bias in estimation) need additional sparse approximation step for coefficient recovery. Moreover, these sparse approximation techniques (such as Lasso) are expensive to use in practice, and need significant tuning. % while NOODL does not require such an expensive tuning procedure. 

%that sparse approximation using biased dictionaries (i.e., under the EIV models) leads to very poor coefficient recovery. This again highlights the advantages of NOODL for exact recovery of both the dictionary and the coefficients simultaneously.
{\small\begin{table}[!t]
\centering
\caption{\footnotesize Final error in recovery of the factors by various techniques and the computation time taken per iteration (in seconds) corresponding to Fig.~\ref{fig:recovery} across techniques. We report the coefficient estimate after the HT step (in \cite{Arora15}) as $\Xb_{\rm HT}$. For the techniques presented in \cite{Arora15}, we scan across $50$ values of the regularization parameter for coefficient estimation using Lasso after learning the dictionary ($\mathbf{A}$), and report the optimal estimation error for the coefficients ($\mathbf{X}_{\rm Lasso}$), while for {\texttt{Mairal} `\texttt{09}}, at each step the coefficients estimate is chosen by scanning across $10$ values of the regularization parameters. For $k=100$, the algorithms of \cite{Arora15} do not converge (shown as N/A). }
\resizebox{\textwidth}{!}{
\begin{tabular}{c|L{3.6cm}|c|c|c|c}
\hline
\textbf{Technique}& \textbf{Recovered Factor ~~~~~~and Timing} & $\b{k = 10}$& $\b{k = 20}$& $\b{k = 50}$ & $\b{k = 100}$\\
\hline
%\multirow{2}{*}{NOODL} & $\Ab$ & $<10^{-5}$ & $<10^{-5}$ & $<10^{-5}$ & $<10^{-5}$\\ 
%& $\Xb$ & $<10^{-5}$ & $<10^{-5}$ & $<10^{-5}$ & $<10^{-5}$ \\
 \multirow{3}{*}{\texttt{NOODL}} & $\b{A}$ & $9.44\times10^{-11}$ & $8.82\times10^{-11}$ & $9.70\times10^{-11}$ & $7.33\times10^{-11}$\\ 
							                     & $\Xb$ & $1.14\times10^{-11}$ & $1.76\times10^{-11}$  & $3.58\times10^{-11}$& $4.74\times10^{-11}$ \\\cline{2-6}
							                     & \textbf{Avg. Time/iteration}& $46.500$ sec & $53.303$ sec & $64.800$ sec & $96.195$ sec\\
\hline
\multirow{7}{2.8cm}{\texttt{~~Arora15 (``biased'')}} & $\Ab$ & $0.013$ & $0.031$ &$0.137$ & N/A \\ 
                             & $\Xb_{\rm HT}$ & $0.077$ & $0.120$ & $0.308$ & N/A \\ 
                                 & $\Xb_{\rm Lasso}$ & $0.006$ & $0.018$ & $0.097$ & N/A \\ \cline{2-6}
                                 & \textbf{Avg. Time/iteration} ~~~~~~(\textit{Accounting for one Lasso update})& $39.390$ sec& $39.371$ sec&$39.434$ sec& $40.063$ sec \\ 
                                 & \textbf{Avg. Time/iteration} ~~~~~~(\textit{Overall Lasso search})& $389.368$ sec& $388.886$ sec&$389.566$ sec& $395.137$ sec \\ 
\hline
\multirow{7}{2.8cm}{\texttt{~~~Arora15 (``unbiased'')}} & $\Ab$ & $0.011$ & $0.027$ & $0.148$ & N/A\\ 
								 & $\Xb_{\rm HT}$ & $0.078$ & $0.122$ & $0.371$ & N/A \\ 
								     & $\Xb_{\rm Lasso}$ & $0.005$ & $0.015$ & $0.0921$ & N/A \\ \cline{2-6}
								     & \textbf{Avg. Time/iteration} ~~~~~~(\textit{Accounting for one Lasso update})& $567.830$ sec & $597.543$ sec& $592.081$ sec & $686.694$ sec\\ 
								     & \textbf{Avg. Time/iteration} ~~~~~~(\textit{Overall Lasso search})& $917.809$ sec & $947.059$ sec& $942.212$ sec & $1041.767$ sec\\ 
\hline
 \multirow{6}{*}{{\texttt{Mairal} `\texttt{09}}} & $\b{A}$ & $0.009$ & $0.015$ & $0.021$ & $0.037$\\ 
							              & $\Xb_{\rm Lasso}$ & $0.183$    & $0.209$  & $0.275$& $0.353$ \\\cline{2-6}
							              & \textbf{Avg. Time/iteration} ~~~~~~(\textit{Accounting for one Lasso update})&  $39.110$ sec& $39.077$ sec&$39.163$ sec& $39.672$ sec\\ 
							              & \textbf{Avg. Time/iteration} ~~~~~~(\textit{Overall Lasso search})& $388.978$ sec& $388.614$ sec&$389.512$ sec& $394.566$ sec\\ 
\hline
%\multicolumn{6}{l}{*N/A denotes that the algorithm does not converge.}
\end{tabular}}
\label{table_res}
\end{table}}

\vspace{-6pt}
\subsection{Computational Time}

In addition to these convergence results, we also report the computational time taken by each of these algorithms in Table \ref{table_res}. The results shown here were compiled using $\rm{5}$ cores and  $\rm{200}$GB RAM of Intel Xeon E$\rm{5-2670}$ Sandy Bridge and Haswell E5-2680v3 processors. 

The primary takeaway is that although NOODL takes marginally more time per iteration as compared to other methods when accounting for just one Lasso update step for the coefficients, it (a) is in fact faster per iteration since it does not involve any computationally expensive tuning procedure to scan across regularization parameters; owing to its geometric convergence property (b) achieves orders of magnitude superior error at convergence, and as a result, (c) overall takes significantly less time to reach such a solution. Further, NOODL's computation time can be further reduced via implementations using the neural architecture illustrated in Section~\ref{app:neural}.

Note that since the coefficient estimates using just the HT step at every step may not yield a usable result for \texttt{Arora15(``unbiased'')} and \texttt{Arora15(``biased'')} as shown in Table~\ref{table_res}, in practice, one has to employ an additional \[\ell_1\]-based sparse recovery step. Therefore, for a fair comparison, we account for running sparse recovery step(s) using Lasso (via the  Fast Iterative Shrinkage-Thresholding Algorithm (FISTA) \citep{Beck09} ) at every iteration of the algorithms \texttt{Arora15(``biased'')} and \texttt{Arora15(``unbiased'')}.

For our technique, we report the average computation time taken per iteration. However, for the rest of the techniques, the coefficient recovery using Lasso (via FISTA) involves a search over various values of the regularization parameters ($10$ values for this current exposition). As a result, we analyze the computation time per iteration via two metrics.  First of these is the average computation time taken per iteration by accounting for the average time take per Lasso update (denoted as ``\textit{Accounting for one Lasso update}''), and the second is the average time taken per iteration to scan over all ($10$) values of the regularization parameter (denoted as ``\textit{Overall Lasso search}'') .

 As shown in Table \ref{table_res}, in comparison to NOODL the techniques described in \cite{Arora15} still incur a large error at convergence, while the popular online DL algorithm of \cite{Mairal09} exhibits very slow convergence rate. Combined with the convergence results shown in Fig.~\ref{fig:recovery}, we observe that due to NOODL's superior convergence properties, it is overall faster and also geometrically converges to the true factors. This again highlights the applicability of NOODL in practical applications, while guaranteeing convergence to the true factors.

\section{Appendix: Standard Results}\label{useful results}

\begin{definition}[sub-Gaussian Random variable]
	Let \[x \sim {\rm subGaussian}(\sigma^2)\]. Then, for any \[t>0\], it holds that
	\begin{align*}
	\b{Pr}[|x| > t] \leq 2\exp\big(\tfrac{t^2}{2\sigma^2}\big).
	\end{align*}
\end{definition}
\subsection{Concentration results}

\begin{lemma}[Matrix Bernstein \citep{Tropp2015}] \label{matrix_bernstein}
Consider a finite sequence {\[\b{W}_k \in \mathbb{R}^{n \times m} \]} of independent, random, centered matrices with dimension \[n\]. Assume that each random matrix satisfies \[\b{E}[\b{W}_k] = 0\] and \[\|\b{W}_k\| \leq R\] almost surely. Then, for all \[t\geq 0\],
\begin{align*}
\b{Pr}\big\{ \|\textstyle\sum\limits_{k} \b{W}_k\| \geq t\big\} \leq (n+m) {\rm exp}\big( \tfrac{-t^2/2}{\sigma^2 + Rt/3}\big),
\end{align*}
where \[\sigma^2 := \max\{\|\textstyle\sum\limits_{k} \b{E}[\b{W}_k\b{W}_k^\top]\|, \|\textstyle\sum\limits_{k} \b{E}[\b{W}_k^\top \b{W}_k]\|\}\]. 

Furthermore,
\begin{align*}
\b{E}[\|\textstyle\sum\limits_{k} \b{W}_k\|] \leq \sqrt{2 \sigma^2 \log(n+m)} + \tfrac{1}{3}R\log(n+m).
\end{align*}
\end{lemma}

\begin{lemma}[Vector Bernstein \citep{Tropp2015}] \label{vector_bernstein}
	Consider a finite sequence {\[\b{w}_k \in \mathbb{R}^{n} \]} of independent, random, zero mean vectors with dimension \[n\]. Assume that each random vector satisfies \[\b{E}[\b{w}_k] = 0\] and \[\|\b{w}_k\| \leq R\] almost surely. Then, for all \[t\geq 0\],
	\begin{align*}
	\b{Pr}\big\{ \|\textstyle\sum\limits_{k} \b{w}_k\| \geq t\big\} \leq 2n {\rm exp}\big( \tfrac{-t^2/2}{\sigma^2 + Rt/3}\big),
	\end{align*}
	where \[\sigma^2 := \|\textstyle\sum\limits_{k} \b{E}[\|\b{w}_k\|^2]\|\]. %= \|\sum_{k} \b{E}[\b{w}_k\b{w}_k^\top]\|\]. 
	Furthermore,
	\begin{align*}
	\b{E}[\|\textstyle\sum\limits_{k} \b{w}_k\|] \leq \sqrt{2 \sigma^2 \log(2n)} + \tfrac{1}{3}R\log(2n).
	\end{align*}
	
\end{lemma}

\begin{lemma}\textbf{Chernoff Bound for sub-Gaussian Random Variables} \label{theorem:subg_chern} 
	Let \[w\] be an independent sub-Gaussian random variables with variance parameter \[\sigma^2\], then for any \[t>0\] it holds that
	\begin{align*}
	\b{Pr}[|w|>t] \leq 2{\rm exp}(-\tfrac{t^2}{2\sigma^2}).
	\end{align*}
	
\end{lemma}

% Hanson Wright
\begin{lemma}[\textbf{Sub-Gaussian concentration \citep{Rudelson13}}]\label{thm:Hanson}
	Let \[\b{M} \in \mathbb{R}^{n \times m}\] be a fixed matrix. Let \[\b{w}\] be a vector of independent, sub-Gaussian random variables with mean zero and variance one. Then, for an absolute constant \[c\],
	\begin{align*}
	\b{Pr}[\|\b{M}\b{x}\|_2 - \|\b{M}\|_F> t] \leq {\rm exp}(-\tfrac{ct^2}{\|\b{M}\|^2}).
	\end{align*}

\end{lemma}

\subsection{Results from \citep{Arora15}}

\begin{lemma}[\citep{Arora15} Lemma 45] \label{tech_lem}
Suppose that the distribution of \[\b{Z}\] satisfies \[\b{Pr}[\|\b{Z}\|\geq L(\log(1/\rho))^C] \leq \rho]\] for some constant \[C>0\], then 
\begin{enumerate}
	\item If \[p = n^{\mathcal{O}(1)}\] then \[\|\b{Z}^{(j)}\| \leq \tilde{\mathcal{O}}(L)\] holds for each \[j\]  with probability at least\[(1 - \rho)\] and,
	\item \[\|\b{E}[\b{Z}\mathbbm{1}_{\|\b{Z}\|\geq \tilde{\Omega}(L)}]\| = n^{- \omega(1)}\].
\end{enumerate}
In particular, if \[\tfrac{1}{p}\sum_{j =1}^{p} \b{Z}^{(j)} (1 - \mathbbm{1}_{\|\b{Z}\|\geq \tilde{\Omega}(L)})\] is concentrated with probability \[( 1- \rho)\], then so is \[\tfrac{1}{p}\sum_{j = 1}^{p}\b{Z}^{(j)}\].
\end{lemma}
\begin{lemma}[Theorem 40 \citep{Arora15}]\label{arora:thm40}
	Suppose random vector \[\b{g}^{(t)}\] is a \[(\rho_{\_}, \rho_{_+}, \zeta_t)\]-correlated with high probability with \[\b{z}^*\] for \[t \in [T] \] where \[T \leq poly(n)\], and \[\eta_A\] satisfies \[0 < \eta_A \leq 2\rho_{_+}\], then for any \[t \in [T]\], 
	\begin{align*}
	\b{E}[\|\b{z}^{(t+1)}- \b{z}^*\|^2] \leq (1 - 2\rho_{\_}\eta_A)\|\b{z}^{(t)} - \b{z}^*\| + 2\eta_A \zeta_t.
	\end{align*}
	In particular, if \[\|\b{z}^{(0)} - \b{z}^*\| \leq  \epsilon_0 \] and \[ \zeta_t \leq (\rho_{\_})o((1 - 2\rho_{\_}\eta)^t) \epsilon_0^2 +  \zeta\], where \[ \zeta = {\max}_{t \in [T]}  \zeta_t\], then the updates converge to \[\b{z}^*\] geometrically with systematic error \[ \zeta/\rho_{\_}\] in the sense that
	\begin{align*}
	\b{E}[\|\b{z}^{(t+1)} - \b{z}^*\|^2] \leq (1 - 2\rho_{\_}\eta_A)^t\epsilon_0^2 +  \zeta/\rho_{\_}. 
	\end{align*}
\end{lemma}
\clearpage

\end{document}